\documentclass{article}

\usepackage[preprint]{neurips_2021}

\usepackage{comment}
\usepackage[utf8]{inputenc} 
\usepackage[T1]{fontenc}    
\usepackage{hyperref}       
\hypersetup{
	colorlinks=true,
	linkcolor=[rgb]{0,0.5,0.5},
	filecolor=cyan,      
	urlcolor=red,
	citecolor=blue,
}
\usepackage{url}            
\usepackage{booktabs}       
\usepackage{amsfonts}       
\usepackage{nicefrac}       
\usepackage{microtype}      
\usepackage{appendix}
\usepackage{microtype}      
\usepackage{xcolor}         

\usepackage{svg}
\usepackage{graphicx}
\usepackage{subfigure}
\usepackage{amsmath}
\usepackage{amsthm}
\usepackage{mathrsfs}
\usepackage{algorithm}
\usepackage{algorithmic}
\usepackage{amssymb}

\DeclareMathOperator*{\argmax}{arg\,max}

\newtheorem{Theorem}{\textbf{Theorem}}
\newtheorem{Lemma}{\textbf{Lemma}}

\newtheorem*{Remark}{\textbf{Remark}}
\newtheorem*{Proof}{\textbf{Proof}}
\newtheorem{Assumption}{\textbf{Assumption}}

\def\GDIHmeanhns{9620.98}
\def\GDIHmeanhnsle{4.81E-07}

\def\GDIHmedianhns{1146.39}
\def\GDIHmedianhnsle{5.73E-08}

\def\GDIHHWRB{22}

\def\GDIHmeanHWRNS{154.27}
\def\GDIHmeanHWRNSle{7.71E-09}

\def\GDIHmedianHWRNS{50.63}
\def\GDIHmedianHWRNSle{2.53E-09}

\def\GDIHmeanSABER{71.26}
\def\GDIHmeanSABERle{3.56E-09}

\def\GDIHmedianSABER{50.63}
\def\GDIHmedianSABERle{2.53E-09}

\def\GDIHgametime{0.114}
\def\GDIImeanhns{7810.6}
\def\GDIImedianhns{832.5}
\def\GDIIHWRB{17}
\def\GDIImeanHWRNS{117.99}
\def\GDIImedianHWRNS{35.78}
\def\GDIImeanSABER{61.66}
\def\GDIImedianSABER{35.78}

\def\GDIIgametime{0.114}
\newcommand{\best}[1]{\textbf{#1}}

\title{GDI: Rethinking What Makes Reinforcement Learning Different from Supervised Learning}

\author{%
  Jiajun Fan \\
  Tsinghua University\\
  \texttt{fanjj21@mails.tsinghua.edu.cn} \\
   \And
  Changnan Xiao \\
  ByteDance\\
\texttt{xiaochangnan@bytedance.com} \\
   \And
  Yue Huang \\
  ByteDance \\
  \texttt{yuehuanghit@gmail.com}\\
}

\begin{document}

\maketitle
 
\begin{abstract}

Deep Q Network (DQN) firstly kicked the door of deep reinforcement learning (DRL) via combining deep learning (DL) with reinforcement learning (RL), which has noticed that the distribution of the acquired data would change during the training process. DQN found this property might cause instability for training, so it proposed effective methods to handle the downside of the property. Instead of focusing on the unfavorable aspects, we find it critical for RL to ease the gap between the estimated data distribution and the ground truth data distribution while supervised learning (SL) fails to do so. From this new perspective, we extend the basic paradigm of RL called the Generalized Policy Iteration (GPI) into a more generalized version, which is called the Generalized Data Distribution Iteration (GDI). We see massive RL algorithms and techniques can be unified into the GDI paradigm, which can be considered as one of the special cases of GDI. We provide theoretical proof of why GDI is better than GPI and how it works. Several practical algorithms based on GDI have been proposed to verify its effectiveness and extensiveness.  Empirical experiments prove our state-of-the-art (SOTA) performance on Arcade Learning Environment (ALE), wherein our algorithm has achieved \textbf{\GDIHmeanhns\%} mean human
normalized score (HNS), \textbf{\GDIHmedianhns\%} median HNS and \textbf{\GDIHHWRB} human world record breakthroughs (HWRB) using only \textbf{200M} training frames. Our work aims to lead the RL research to step into the journey of conquering the human world records and seek real superhuman agents on both performance and efficiency.
\end{abstract}

\section{Introduction}
\label{sec: introduction}

Machine learning (ML) can be defined as improving some measure performance P at some task T according to the acquired data or experience E \citep{mitchell1997machine}. As one of the three main components of ML \citep{mitchell1997machine}, the training experiences matter in ML, which can be reflected from many aspects. For example, three major ML paradigms can be distinguished from the  perspective of the different training experiences. Supervised learning (SL) is learning from a training set of \textbf{labeled experiences} provided by a knowledgable external supervisor \citep{sutton}. Unsupervised learning (UL) is typically about seeking structure hidden in collections of \textbf{unlabeled experiences} \citep{sutton}. 
Unlike UL or SL, reinforcement learning (RL) focuses on the problem that agents learn from \textbf{experiences gained through trial-and-error interactions with a dynamic environment} \citep{kaelbling1996reinforcement}. As \citep{mitchell1997machine} said, there is no free lunch in the ML problem - no way to generalize beyond the specific training examples. The performance can only be improved through learning from the acquired experiences in ML problems \citep{mitchell1997machine}. All of them have revealed the importance of the training experiences and thus the selection of the training distribution appears to be a fundamental problem in ML.

Recalling these three paradigms, SL and RL receive explicit learning signals from data. 
In SL, there is no way to make up the gap between the distribution estimated by the collected data and the ground truth without any domain knowledge unless collecting more data.
Researchers have found RL explicitly and naturally transforming the training distribution \citep{dqn}, which makes RL distinguished from SL. 
In the recent RL advances, many researchers \citep{dqn} have realized that RL agents hold the property of changing the data distribution and massive works have revealed the unfavorable aspect of the property. Among those algorithms, DQN \citep{dqn} firstly noticed the unique property of RL and considered it as one of the reasons for the training instability of DRL. After that, massive methods like replay buffer \citep{dqn}, periodically updated target \citep{dqn} and importance sampling \citep{impala} have been proposed to mitigate the impact of the data distribution shift. However, after rethinking this property, we wonder whether changing the data distribution always brings unfavorable nature. 
What if we can control it? More precisely, what if we can control the ability to  select superior data distribution for training automatically? Prior works in ML have revealed the great potential of this property. As \citep{cohn1996active} put it, when training examples are appropriately selected, the data requirements for some problems decrease drastically, and some NP-complete learning problems become polynomial in computation time \citep{angluin1988queries,baum1991neural}, which means that carefully selecting good training data benefits learning efficiency. 
Inspired by this perspective, instead of discussing how to ease the disadvantages caused by the change of data distribution like other prior works of RL, in this paper, we rethink the property distinguishing RL from SL and explore more effective aspects of it. 
One of the fundamental reasons RL holds the ability to change the data distribution is the change of behavior policies, which directly interact with the dynamic environments to obtain training data \citep{dqn}. Therefore, the training experiences can be controlled by adjusting the behavior policies, which makes behavior selection the bridge between RL agents and training examples.

In the RL problem, the agent has to exploit what it already knows to obtain the reward, but it also has to explore to make better action selections in
the future, which is called the exploration and exploitation dilemma \citep{sutton}. 
Therefore, diversity is one of the main factors that should be considered while selecting the training examples. 
In the recent advances of RL, some works have also noticed the importance of the diversity of training experiences \citep{agent57,ngu,DvD,niu2011hogwild,li2019generalized,DIAYN}, most of which have obtained diverse data via enriching the policy diversity. 
Among those algorithms, DIAYN \citep{DIAYN} focused entirely on the diversity of policy via learning skills without a reward function, which has revealed the effect of policy diversity but ignored its relationship with the RL objective.
DvD \citep{DvD} introduced a diversity-based regularizer into the RL objective to obtain more diverse data, which changed the optimal solution of the environment \citep{sutton}. 
Besides, training a population of agents to gather more diverse experiences seems to be a promising approach. 
Agent57 \citep{agent57} and NGU \citep{ngu} trained a family of policies with different degrees of exploratory behaviors using a shared network architecture. Both of them have obtained SOTA performance at the cost of increasing the uncertainty of environmental transition, which leads to extremely low learning efficiency. 
Through those successes, it is evident that the diversity  of the training data benefit the RL training. 
However, why does it perform better and whether more diverse data always benefit RL training? 
In other words, we have to explore the following question:

\begin{center}
    \textit{Does diverse data always benefit effective learning?}
\end{center}

To investigate this problem, we seek inspiration from the natural biological processes. 
In nature, the population evolves typically faster than individuals because the diversity of the populations boosts more \textbf{beneficial mutations} which provide more possibility for acquiring more adaptive direction of evolution \citep{mutation}. 
Furthermore, beneficial mutations rapidly spread among the population, thus enhancing population adaptability  \citep{mutation}. 
Therefore, an appropriate diversity brings high-value individuals, and  active learning among the population promotes its prosperity.\footnote{According to \citep{evolution}, most mutations are deleterious and cause a reduction in population fitness known as the mutational load. 
Therefore, excessive and redundant diversity may be harmful.} 
From this perspective, the RL agents have to pay more attention to \textbf{experiences worthy of learning from}.  
DisCor \citep{discor}, which re-weighted the \textbf{existing data buffer} by the distribution that explicitly optimizes for corrective feedback, has also noticed the fact that \textbf{the choice of the sampling distribution is of crucial importance for the stability and efficiency
of approximation dynamic
programming algorithms}. 
Unfortunately, DisCor only changes the existing data distribution instead of directly controlling the source of the training experiences, which may be more important and also more complex.
In conclusion, it seems that both \textbf{expanding the capacity of policy space for behaviors} and \textbf{selecting suitable behavior policies from a diverse behavior population} matter for efficient learning. This new perspective motivates us to investigate another critical problem:

\begin{center}
    \textit{How to select superior behaviors from the behavior policy space?}
\end{center}

To address those problems, we proposed a novel RL paradigm called \textbf{G}eneralized \textbf{D}ata Distribution \textbf{I}teration (\textbf{GDI}), which consists of two major process, the policy iteration operator $\mathcal{T}$ and the data distribution iteration operator $\mathcal{E}$. 
Specifically, behaviors will be sampled from a policy space according to a selective distribution, which will be iteratively optimized through the operator $\mathcal{E}$. 
Simultaneously, elite training data will be used for policy iteration via the operator $\mathcal{T}$. 
More details about our methodology can see Sec. \ref{Sec: Methodology}.
\begin{figure*}[!t]
    \centering
	\subfigure{
		\includegraphics[width=0.46\textwidth]{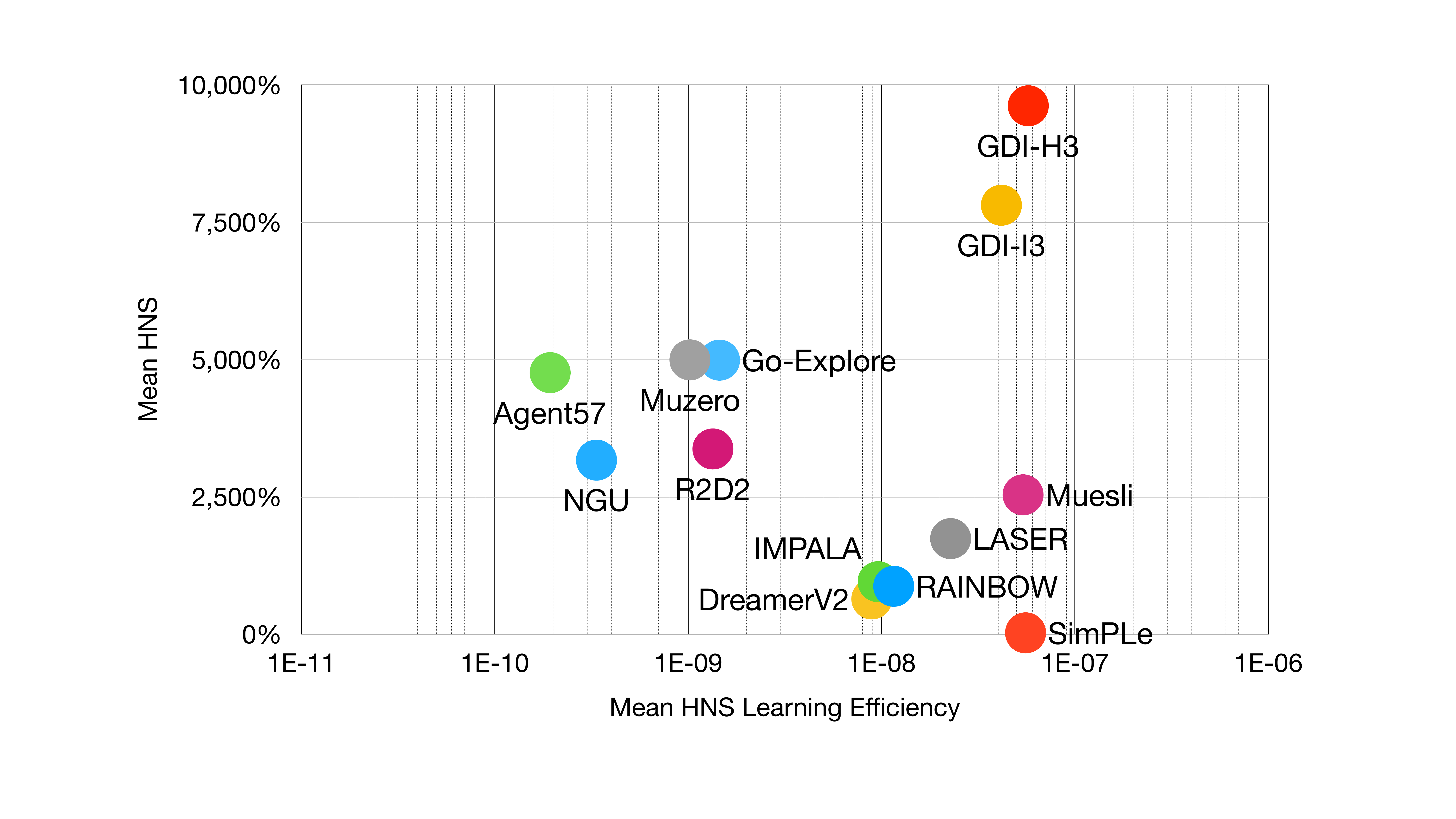}
	}
	\subfigure{
		\includegraphics[width=0.46\textwidth]{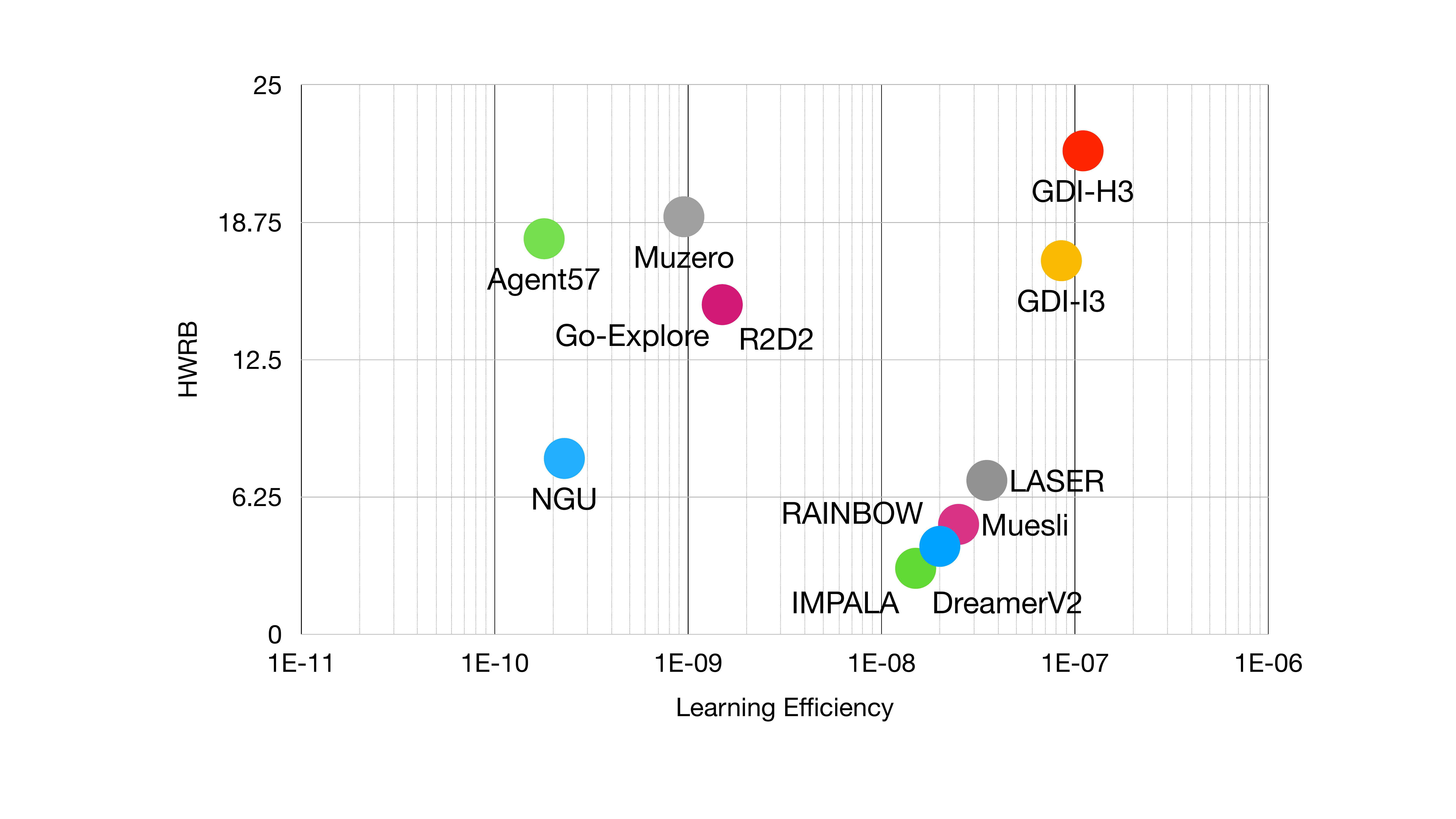}
	}
	\centering
	\caption{Performance of SOTA algorithms of Atari 57 games on mean HNS(\%) with corresponding learning efficiency and human world record breakthrough with corresponding game time. Details on those evaluation criteria can see App. \ref{sec:app Experiment Details}.}
	\label{fig: mean med hns and learning efficiency}
\end{figure*}

In conclusion, the main contributions of our work are:
\begin{itemize}
    \item \textbf{A Novel RL Paradigm:} Rethinking the difference between RL and SL, we discover RL can ease the gap between the sampled data distribution and the ground truth data distribution via adjusting the behavior policies. Based on the perspective, we extend GPI into GDI, a more general version containing a data optimization process. This novel perspective allows us to unify massive RL algorithms, and various improvements can be considered a special case of data distribution optimization, detailed in Sec. \ref{Sec: Methodology}.
   
   \item \textbf{Theoretical Proof of GDI:} We provide sufficient theoretical proof of GDI. 
    The effectiveness of the data distribution optimization of GDI has been proved on both first-order optimization and second-order optimization, and the guarantee of monotonic improvement induced by the data distribution optimization operator $\mathcal{E}$ has also been proved. 
    More details can see Sec. \ref{Sec: Methodology} and App. \ref{App: proof}.

    \item \textbf{A General Practical Framework of GDI:} Based on GDI, we propose a general practical framework, wherein behavior policy belongs to a soft $\epsilon$-greedy space which unifies $\epsilon$-greedy policies \citep{epsilongreedy} and Boltzmann policies \citep{bozeman}. 
    As a practical framework of GDI, a self-adaptable meta-controller is proposed to optimize the distribution of the behavior policies. 
    More implementation details can see App. \ref{App: Algorithm Pseudocode} and App. \ref{Sec: appendix MAB}.
    
    \item \textbf{The State-Of-The-Art Performance:} 
    From Figs. \ref{fig: mean med hns and learning efficiency}, our approach has achieved \GDIHmeanhns\% mean HNS and \GDIHmedianhns\% median HNS, which achieves new SOTA. More importantly, our learning efficiency has approached the human level as achieving the SOTA performance within less than 1.5 months of game time. We have also illustrated the RL Benchmark on HNS in App. \ref{app: RL Benchmarks on HNS} and recorded their scores in App. \ref{app: Atari Games Table of Scores Based on Human Average Records}.

    \item \textbf{Human World Records Breakthrough:} As our algorithms have achieved SOTA on mean HNS, median HNS and learning efficiency, we aim to lead RL research on ALE to step into a new era of conquering human world records and seeking  the real superhuman agents. Therefore, we propose several novel evaluation criteria and an open challenge on the  Atari benchmark based on the human world records.  From 
   Figs. \ref{fig: mean med hns and learning efficiency}, our method has surpassed \GDIHHWRB \ human world records, which  has also surpassed all previous algorithms. We have also illustrated the RL Benchmark on human world records normalized scores (HWRNS), SABER \citep{atarihuman}, HWRB in App. \ref{app: RL Benchmarks on HWRNS}, \ref{app: RL Benchmarks on SABER} and  \ref{app: RL Benchmarks on HWRB}, respectively. Relevant scores are recorded in App.  \ref{app: Atari Games Table of Scores Based on Human World Records} and App.  \ref{app: Atari Games Table of Scores Based on SABER}.
\end{itemize}

\section{Preliminaries}

 The RL problem can be formulated as a Markov Decision Process \citep[MDP]{howard1960dynamic} defined by $\left(\mathcal{S}, \mathcal{A}, p, r, \gamma, \rho_{0}\right)$. 
 Considering a discounted episodic MDP, the initial state $s_0$ is sampled from the initial distribution $\rho_0(s): \mathcal{S} \rightarrow \Delta(\mathcal{S})$, where we use $\Delta$ to represent the probability simplex.
 At each time $t$, the agent chooses an action $a_t \in \mathcal{A}$ according to the policy $\pi(a_t|s_t): \mathcal{S} \rightarrow \Delta(\mathcal{A})$ at state $s_t \in \mathcal{S}$. 
 The environment receives $a_t$, produces the reward $r_t \sim r(s,a): \mathcal{S} \times \mathcal{A} \rightarrow \mathbf{R}$ and transfers to the next state $s_{t+1}$  according to the transition distribution $p\left(s^{\prime} \mid s, a\right): \mathcal{S} \times \mathcal{A} \rightarrow \Delta(\mathcal{S})$. 
 The process continues until the agent reaches a terminal state or a maximum time step. 
 Define the discounted state visitation distribution as 
 $d_{\rho_0}^{\pi} (s) = (1 - \gamma) \textbf{E}_{s_0 \sim \rho_0} 
 \left[ \sum_{t=0}^{\infty} \gamma^t \textbf{P} (s_t = s | s_0) \right]$.
 The goal of reinforcement learning is to find the optimal policy $\pi^*$ that maximizes the expected sum of discounted rewards, denoted by $\mathcal{J}$ \citep{sutton}:
\begin{equation}
\label{eq_accmulate_reward}
\pi^{*}
=\underset{\pi}{\operatorname{argmax}} \mathcal{J}_{\pi} 
= \underset{\pi}{\operatorname{argmax}} \textbf{E}_{s_t \sim d_{\rho_0}^{\pi}} \textbf{E}_{\pi}\left[G_t | s_t \right]
= \underset{\pi}{\operatorname{argmax}} \textbf{E}_{s_t \sim d_{\rho_0}^{\pi}} \textbf{E}_{\pi} \left[\sum_{k=0}^{\infty} \gamma^{k} r_{t+k} | s_t \right]
\end{equation}
where $\gamma \in(0,1)$ is the discount factor.

RL algorithms can be divided into off-policy manners \citep{dqn,a3c,sac,impala} and on-policy manners \citep{ppo}. 
Off-policy algorithms select actions according to a behavior policy $\mu$ that can be different from the learning policy $\pi$.
On-policy algorithms evaluate and improve the learning policy through data sampled from the same policy.
RL algorithms can also be divided into value-based methods \citep{dqn,doubledqn,dueling_q,rainbow,apex} and policy-based methods \citep{ppo,a3c,impala,laser}. 
In the value-based methods, agents learn the policy indirectly, where the policy is defined by consulting the learned value function, like $\epsilon$-greedy, and the value function is learned by a typical GPI.
In the policy-based methods, agents learn the policy directly, where the correctness of the gradient direction is guaranteed by the policy gradient theorem \citep{sutton}, and the convergence of the policy gradient methods is also guaranteed \citep{pgtheory}. 
More background on RL can see App. \ref{app: background on RL}.

\section{Methodology}
\label{Sec: Methodology}

\subsection{Generalized Data Distribution Iteration}

Let's abstract our notations first, which is also summarized in App. \ref{app: Abbreviation and Notation}.

Define $\Lambda$ to be an index set, $\Lambda \subseteq \textbf{R}^k$.
$\lambda \in \Lambda$ is an index in $\Lambda$.
$(\Lambda, \mathcal{B}|_{\Lambda}, \mathcal{P}_{\Lambda})$ is a probability space, where $\mathcal{B}|_{\Lambda}$ is a Borel $\sigma$-algebra restricted to $\Lambda$.
Under the setting of meta-RL, $\Lambda$ can be regarded as the set of all possible meta information.
Under the setting of population-based training (PBT) \citep{PBT}, $\Lambda$ can be regarded as the set of the whole population.

Define $\Theta$ to be a set of all possible values of parameters.
$\theta \in \Theta$ is some specific value of parameters.
For each index $\lambda$, there exists a specific mapping between each parameter of $\theta$ and $\lambda$, denoted as $\theta_\lambda$, to indicate the parameters in $\theta$ corresponding to $\lambda$.
Under the setting of linear regression $y = w \cdot x$, $\Theta = \{w \in R^n\}$ and $\theta = w$.
If $\lambda$ represents using only the first half features to make regression, assume $w = (w_1, w_2)$, then $\theta_\lambda = w_1$.  
Under the setting of RL, $\theta_{\lambda}$ defines a parameterized policy indexed by $\lambda$, denoted as $\pi_{\theta_{\lambda}}$.

Define $\mathcal{D} \overset{def}{=} \{d^\pi_{\rho_{0}} |\ \pi \in {\Delta (\mathcal{A})}^\mathcal{S}, \rho_{0} \in \Delta(\mathcal{S}) \}$ to be the set of all states visitation distributions.
For the parameterized policies, denote 
$\mathcal{D}_{\Lambda, \Theta, \rho_{0}} \overset{def}{=} \{d^{\pi_{\theta_{\lambda}}}_{\rho_{0}} |\ \theta \in \Theta, \lambda \in \Lambda \}$.
Note that $(\Lambda, \mathcal{B}|_{\Lambda}, \mathcal{P}_{\Lambda})$ is a probability space on $\Lambda$, 
which induces a probability space on $\mathcal{D}_{\Theta, \Lambda, \rho_{0}}$,
with the probability measure given by 
$\mathcal{P}_{\mathcal{D}} (\mathcal{D}_{\Lambda_0, \Theta, \rho_{0}}) 
= \mathcal{P}_{\Lambda} (\Lambda_0),\ \forall \Lambda_0 \in \mathcal{B}|_\Lambda$.

We use $x$ to represent one sample, which contains all necessary information for learning. 
For DQN, $x = (s_t, a_t, r_t, s_{t+1})$.
For R2D2, $x = (s_t, a_t, r_t, \dots, s_{t+N}, a_{t+N}, r_{t+N}, s_{t+N+1})$.
For IMPALA, $x$ also contains the distribution of the behavior policy.
The content of $x$ depends on the algorithm, but it's sufficient for learning.
We use $\mathcal{X}$ to represent the set of samples.
At training stage $t$, 
given the parameter $\theta = \theta^{(t)}$, 
the distribution of the index set $\mathcal{P}_{\Lambda} = \mathcal{P}^{(t)}_{\Lambda}$
and the distribution of the initial state $\rho_0$, 
we denote the set of samples as
\begin{equation*}
    \mathcal{X}_{\rho_{0}}^{(t)}
    \overset{def}{=} \bigcup_{d_{\rho_{0}}^\pi \sim \mathcal{P}_\mathcal{D}^{(t)}} \{ x | x \sim d_{\rho_{0}}^\pi \} 
    = \bigcup_{\lambda \sim \mathcal{P}_{\Lambda}^{(t)}} 
    \{ x | x \sim d_{\rho_{0}}^{\pi_\theta}, 
    \theta = {\theta^{(t)}_{\lambda}} \}
    \triangleq \bigcup_{\lambda \sim \mathcal{P}_{\Lambda}^{(t)}} \mathcal{X}^{(t)}_{\rho_{0}, \lambda}.
\end{equation*}

Now we introduce our main algorithm:
\begin{figure}[ht]
  \centering
  \begin{minipage}{.7\linewidth}
    \begin{algorithm}[H]
      \caption{Generalized Data Distribution Iteration (GDI).}  
          \begin{algorithmic}
            \STATE Initialize $\Lambda$, $\Theta$, $\mathcal{P}_{\Lambda}^{(0)}$, $\theta^{(0)}$.
            \FOR{$t=0,1,2,\dots$}
                \STATE Sample $\{\mathcal{X}^{(t)}_{\rho_0, \lambda}\}_{\lambda \sim \mathcal{P}^{(t)}_{\Lambda}}$.
                \STATE $\theta^{(t+1)} = \mathcal{T}( \theta^{(t)}, \{\mathcal{X}^{(t)}_{\rho_0, \lambda}\}_{\lambda \sim \mathcal{P}^{(t)}_{\Lambda}} )$.
                \STATE $\mathcal{P}_{\Lambda}^{(t+1)}  = \mathcal{E}(\mathcal{P}_{\Lambda}^{(t)}, \{\mathcal{X}^{(t)}_{\rho_0, \lambda}\}_{\lambda \sim \mathcal{P}^{(t)}_{\Lambda}} )$.
            \ENDFOR
          \end{algorithmic}
        \label{alg:GDI}
    \end{algorithm}
  \end{minipage}
\end{figure}

$\mathcal{T}$ defined as $\theta^{(t+1)} 
= \mathcal{T}( \theta^{(t)}, \{\mathcal{X}^{(t)}_{\rho_0, \lambda}\}_{\lambda \sim \mathcal{P}^{(t)}_{\Lambda}} )$
is a typical optimization operator of RL algorithms, 
which utilizes the collected samples to update the parameters for maximizing some function $L_{\mathcal{T}}$.
For instance, $L_{\mathcal{T}}$ may contain the policy gradient and the state value evaluation for the policy-based methods, 
may contain generalized policy iteration for the value-based methods, 
may also contain some auxiliary tasks or intrinsic rewards for special designed methods.

$\mathcal{E}$ defined as $\mathcal{P}_{\Lambda}^{(t+1)}  
= \mathcal{E}(\mathcal{P}_{\Lambda}^{(t)}, \{\mathcal{X}^{(t)}_{\rho_0, \lambda}\}_{\lambda \sim \mathcal{P}^{(t)}_{\Lambda}} )$ 
is a data distribution optimization operator.
It uses the samples $\{\mathcal{X}^{(t)}_{\rho_0, \lambda}\}_{\lambda \sim \mathcal{P}^{(t)}_{\Lambda}}$ to maximize some function $L_{\mathcal{E}}$, namely,
\begin{equation*}
    \mathcal{P}_{\Lambda}^{(t+1)} = \argmax_{\mathcal{P}_{\Lambda}} L_{\mathcal{E}} (\{\mathcal{X}^{(t)}_{\rho_0, \lambda}\}_{\lambda \sim \mathcal{P}_{\Lambda}}).
\end{equation*}
When $\mathcal{P}_{\Lambda}$ is parameterized, we abuse the notation and use $\mathcal{P}_{\Lambda}$ to represent the parameter of $\mathcal{P}_{\Lambda}$.
If $\mathcal{E}$ is a first order optimization operator, then we can write $\mathcal{E}$ explicitly as
\begin{equation*}
    \mathcal{P}_{\Lambda}^{(t+1)} = \mathcal{P}_{\Lambda}^{(t)} + \eta \nabla_{\mathcal{P}_{\Lambda}^{(t)}} L_{\mathcal{E}} (\{\mathcal{X}^{(t)}_{\rho_0, \lambda}\}_{\lambda \sim \mathcal{P}^{(t)}_{\Lambda}}).
\end{equation*}
If $\mathcal{E}$ is a second order optimization operator, like natural gradient, we can write $\mathcal{E}$ formally as
    \begin{gather*}
        \mathcal{P}_{\Lambda}^{(t+1)} = \mathcal{P}_{\Lambda}^{(t)} + \eta
        \textbf{F}(\mathcal{P}_{\Lambda}^{(t)})^\dagger
        \nabla_{\mathcal{P}_{\Lambda}^{(t)}} L_{\mathcal{E}} (\{\mathcal{X}^{(t)}_{\rho_0, \lambda}\}_{\lambda \sim \mathcal{P}^{(t)}_{\Lambda}}), \\
        \textbf{F}(\mathcal{P}_{\Lambda}^{(t)}) = \left[\nabla_{\mathcal{P}_{\Lambda}^{(t)}} \log \mathcal{P}_{\Lambda}^{(t)} \right]
        \cdot
        \left[\nabla_{\mathcal{P}_{\Lambda}^{(t)}} \log \mathcal{P}_{\Lambda}^{(t)} \right]^\top, \\
    \end{gather*}
where $\dagger$ denotes the Moore-Penrose pseudoinverse of the matrix.

\subsection{Systematization of GDI}

We can further divide all algorithms into two categories, GDI-I$^n$ and GDI-H$^n$.
$n$ represents the degree of freedom of $\Lambda$.
I represents Isomorphism. 
We say one algorithm belongs to GDI-I$^n$, if $\theta = \theta_{\lambda}, \, \forall \lambda \in \Lambda$.
H represents Heterogeneous.
We say one algorithm belongs to GDI-H$^n$, if $\theta_{\lambda_1} \neq \theta_{\lambda_2}, \, \exists \lambda_1, \lambda_2 \in \Lambda$.
We say one algorithm is "w/o $\mathcal{E}$" if it doesn't have the operator $\mathcal{E}$, in another word, its $\mathcal{E}$ is an identical mapping.

Now we discuss the connections between GDI and some algorithms.

For DQN, RAINBOW, PPO and IMPALA, they are in GDI-I$^0$ w/o $\mathcal{E}$. Let $|\Lambda| = 1$, WLOG, assume $\Lambda = \{\lambda_0\}$.
The probability measure $\mathcal{P}_{\Lambda}$ collapses to $\mathcal{P}_{\Lambda} (\lambda_0) = 1$.
$\Theta = \{\theta_{\lambda_0}\}$.
$\mathcal{E}$ is an identical mapping of $\mathcal{P}_{\Lambda}^{(t)}$.
$\mathcal{T}$ is the first order operator that optimizes the loss functions, respectively.

For Ape-X and R2D2, they are in GDI-I$^1$ w/o $\mathcal{E}$. 
Let $\Lambda = \{\epsilon_l |\ l = 1, \dots, 256\}$.
$\mathcal{P}_{\Lambda}$ is uniform, $\mathcal{P}_{\Lambda} (\epsilon_l) = |\Lambda|^{-1}$.
Since all actors and the learner share parameters, we have $\theta_{\epsilon_1} = \theta_{\epsilon_2}$ for $\forall \epsilon_1, \epsilon_2 \in \Lambda$, hence $\Theta = \bigcup_{\epsilon \in \Lambda} \{\theta_{\epsilon}\} = \{\theta_{\epsilon_l}\}, \ \forall\ l = 1,\dots, 256$.
$\mathcal{E}$ is an identical mapping, because $\mathcal{P}_{\Lambda}^{(t)}$ is always a uniform distribution.
$\mathcal{T}$ is the first order operator that optimizes the loss functions.

For LASER, it's in GDI-H$^1$ w/o $\mathcal{E}$. 
Let $\Lambda = \{i |\ i = 1, \dots, K\}$ to be the number of learners.
$\mathcal{P}_{\Lambda}$ is uniform, $\mathcal{P}_{\Lambda} (i) = |\Lambda|^{-1}$.
Since different learners don't share parameters, $\theta_{i_1} \cap \theta_{i_2} = \emptyset$ for $\forall i_1, i_2 \in \Lambda$, hence $\Theta = \bigcup_{i \in \Lambda} \{\theta_i\}$.
$\mathcal{E}$ is an identical mapping.
$\mathcal{T}$ can be formulated as a union of $\theta^{(t+1)}_i 
= \mathcal{T}_{i}( \theta^{(t)}_i, \{\mathcal{X}^{(t)}_{\rho_0, \lambda}\}_{\lambda \sim \mathcal{P}^{(t)}_{\Lambda}} )$, 
which represents optimizing $\theta_i$ of $i$th learner with shared samples from other learners.

For PBT, it's in GDI-H$^{n+1}$, where $n$ is the number of  searched hyperparameters.
Let $\Lambda = \{h\} \times \{i| i=1,\dots, K\}$, where $h$ represents the hyperparameters being searched and $K$ is the population size.
$\Theta = \bigcup_{i=1,\dots, K} \{\theta_{i, h}\}$, where $\theta_{i, h_1} = \theta_{i, h_2}$ for $\forall (h_1, i), (h_2, i) \in \Lambda$.
$\mathcal{E}$ is the meta-controller that adjusts $h$ for each $i$, which can be formally written as 
$\mathcal{P}_{\Lambda}^{(t+1)}(\cdot, i)
= \mathcal{E}_{i}(\mathcal{P}_{\Lambda}^{(t)}(\cdot, i), \{\mathcal{X}^{(t)}_{\rho_0, (h, i)}\}_{h \sim \mathcal{P}^{(t)}_{\Lambda}(\cdot, i)} )$,
which optimizes $\mathcal{P}_{\Lambda}$ according to the performance of all agents in the population.
$\mathcal{T}$ can also be formulated as a union of $\mathcal{T}_i$, but is 
$\theta^{(t+1)}_i 
= \mathcal{T}_{i}( \theta^{(t)}_i, \{\mathcal{X}^{(t)}_{\rho_0, (h, i)}\}_{h \sim \mathcal{P}^{(t)}_{\Lambda}(\cdot, i)})$,
which represents optimizing the $i$th agent with only samples from the $i$th agent.

For NGU and Agent57, it's in GDI-I$^2$. 
Let $\Lambda = \{\beta_i | i=1,\dots,m\} \times \{\gamma_j | j=1,\dots,n\}$, where $\beta$ is the weight of the intrinsic value function and $\gamma$ is the discount factor.
Since all actors and the learner share variables, $\Theta = \bigcup_{(\beta, \gamma) \in \Lambda} \{\theta_{(\beta, \gamma)}\} = \{\theta_{(\beta, \gamma)}\}$ for $\forall (\beta, \gamma) \in \Lambda$.
$\mathcal{E}$ is an optimization operator of a multi-arm bandit controller with UCB, which aims to maximize the expected cumulative rewards by adjusting $\mathcal{P}_{\Lambda}$.
Different from above, $\mathcal{T}$ is identical to our general definition $\theta^{(t+1)} 
= \mathcal{T}( \theta^{(t)}, \{\mathcal{X}^{(t)}_{\rho_0, \lambda}\}_{\lambda \sim \mathcal{P}^{(t)}_{\Lambda}} )$,
which utilizes samples from different $\lambda$s to update the shared $\theta$.

For Go-Explore, it's in GDI-H$^1$. 
Let $\Lambda = \{\tau\}$, where $\tau$ represents the stopping time of switching between robustification and exploration.
$\Theta = \{\theta_r\} \cup \{\theta_e\}$, where $\theta_r$ is the robustification model and $\theta_e$ is the exploration model.
$\mathcal{E}$ is a search-based controller, which defines the next $\mathcal{P}_{\Lambda}$ for a better exploration.
$\mathcal{T}$ can be decomposed into $(\mathcal{T}_r, \mathcal{T}_e)$.

\subsection{Monotonic Data Distribution Optimization}

We see massive algorithms can be formulated as a special case of GDI.
For the algorithms without a meta-controller, whose data distribution optimization operator $\mathcal{E}$ is trivially an identical mapping, the guarantee that the learned policy could converge to the optimal policy has been wildly studied, for instance, GPI in \citep{sutton} and policy gradient in \citep{pgtheory}.
But for the algorithms with a meta-controller, whose data distribution optimization operator $\mathcal{E}$ is non-identical, though most algorithms in this class show superior performance, it still lacks a general study on why the data distribution optimization operator $\mathcal{E}$ helps.
In this section, with a few assumptions, we show that given the same optimization operator $\mathcal{T}$, a GDI with a non-identical data distribution optimization operator $\mathcal{E}$ is always superior to a GDI w/o $\mathcal{E}$.

For brevity, we denote the expectation of $L_\mathcal{E}, L_\mathcal{T}$ for each $\lambda \in \Lambda$ as
    $\mathcal{L}_{\mathcal{E}} (\lambda, \theta_\lambda) 
    = \textbf{E}_{x \sim \pi_{\theta_\lambda}} [L_{\mathcal{E}} (\{\mathcal{X}_{\rho_0, \lambda}\})], \ 
    \mathcal{L}_{\mathcal{T}} (\lambda, \theta_\lambda) 
    = \textbf{E}_{x \sim \pi_{\theta_\lambda}} [L_{\mathcal{T}} (\{\mathcal{X}_{\rho_0, \lambda}\})], $
and denote the expectation of $L_\mathcal{E}, L_\mathcal{T}$ for any $\mathcal{P}_{\Lambda}$ as 
    $\mathcal{L}_{\mathcal{E}} (\mathcal{P}_{\Lambda}, \theta) 
    = \textbf{E}_{\lambda \sim \mathcal{P}_{\Lambda}} [\mathcal{L}_{\mathcal{E}} (\lambda, \theta_\lambda) ], \ 
    \mathcal{L}_{\mathcal{T}} (\mathcal{P}_{\Lambda}, \theta) 
    = \textbf{E}_{\lambda \sim \mathcal{P}_{\Lambda}}  [\mathcal{L}_{\mathcal{T}} (\lambda, \theta_\lambda)].$

\begin{Assumption}[Uniform Continuous Assumption]
    For $\forall \epsilon > 0,\  \forall s \in \mathcal{S},\ \exists\, \delta > 0,\ s.t. |V^{\pi_1} (s) - V^{\pi_2} (s)| < \epsilon,\ \forall\, d_{\pi} (\pi_1, \pi_2) < \delta$,
    where $d_\pi$ is a metric on ${\Delta(\mathcal{A})}^\mathcal{S}$.
    If $\pi$ is parameterized by $\theta$, then for $\forall \epsilon > 0,\  \forall s \in \mathcal{S},\ \exists\, \delta > 0,\ s.t. |V^{\pi_{\theta_1}} (s) - V^{\pi_{\theta_2}} (s)| < \epsilon,\ \forall\, || \theta_1 - \theta_2 || < \delta$.
\label{asp:1}
\end{Assumption}
\begin{Remark}
    \citep{polytope} shows $V^\pi$ is infinitely differentiable everywhere on $\Delta (\mathcal{A})^{\mathcal{S}}$ if $|\mathcal{S}| < \infty, |\mathcal{A}| < \infty$.
    \citep{pgtheory} shows $V^\pi$ is $\beta$-smooth, namely bounded second order derivative, for direct parameterization.
    If $\Delta (\mathcal{A})^{\mathcal{S}}$ is compact, continuity implies uniform continuity.
\end{Remark}

\begin{Assumption}[Formulation of $\mathcal{E}$ Assumption]
    Assume
    $\mathcal{P}_{\Lambda}^{(t+1)} 
    = \mathcal{E}(\mathcal{P}_{\Lambda}^{(t)}, \{\mathcal{X}^{(t)}_{\rho_0, \lambda}\}_{\lambda \sim \mathcal{P}^{(t)}_{\Lambda}}) $
    can be written as 
    $\mathcal{P}_{\Lambda}^{(t+1)} (\lambda)= \mathcal{P}_{\Lambda}^{(t)}(\lambda) \frac{\exp (\eta \mathcal{L}_{\mathcal{E}} (\lambda, \theta_{\lambda}^{(t)})  )}{Z^{(t+1)}}$, 
    $Z^{(t+1)} = \textbf{E}_{\lambda \sim \mathcal{P}_{\Lambda}^{(t)}}[\exp (\eta \mathcal{L}_{\mathcal{E}} (\lambda, \theta_{\lambda}^{(t)})  )]$.
\label{asp:2}
\end{Assumption}
\begin{Remark}
    The assumption is actually general.
    Regarding $\Lambda$ as an action space and 
    $r_\lambda 
    = \mathcal{L}_{\mathcal{E}} (\lambda, \theta_{\lambda}^{(t)})$, when solving $\argmax_{\mathcal{P}_{\Lambda}} \textbf{E}_{\lambda \sim \mathcal{P}_{\Lambda}} [\mathcal{L}_{\mathcal{E}} (\lambda, \theta_{\lambda}^{(t)})] 
    = \argmax_{\mathcal{P}_{\Lambda}} \textbf{E}_{\lambda \sim \mathcal{P}_{\Lambda}} [r_\lambda]$, the data distribution optimization operator $\mathcal{E}$ is equivalent to solving a multi-arm bandit (MAB) problem.
    For the first order optimization, \citep{eq_pg_q} shows that the solution of a KL-regularized version, $\argmax_{\mathcal{P}_{\Lambda}} \textbf{E}_{\lambda \sim \mathcal{P}_{\Lambda}} [r_\lambda] - \eta KL(\mathcal{P}_{\Lambda} || \mathcal{P}_{\Lambda}^{(t)})$, is exactly the assumption.
    For the second order optimization, let $\mathcal{P}_{\Lambda} = softmax (\{r_\lambda\})$, \citep{pgtheory} shows that the natural policy gradient of a softmax parameterization also induces exactly the assumption.
\end{Remark}

\begin{Assumption}[First Order Optimization Co-Monotonic Assumption]
    For $\forall\, \lambda_1, \lambda_2 \in \Lambda$, we have
    $[ \mathcal{L}_{\mathcal{E}} (\lambda_1, \theta_{\lambda_1})  -  \mathcal{L}_{\mathcal{E}} (\lambda_2, \theta_{\lambda_2}) ] \cdot
    [ \mathcal{L}_{\mathcal{T}} (\lambda_1, \theta_{\lambda_1})  -  \mathcal{L}_{\mathcal{T}} (\lambda_2, \theta_{\lambda_2}) ] \geq 0$.
\label{asp:3}
\end{Assumption}

\begin{Assumption}[Second Order Optimization Co-Monotonic Assumption]
    For $\forall\, \lambda_1, \lambda_2 \in \Lambda$,
    $\exists\, \eta_0 > 0$, s.t. $\forall\, 0 < \eta < \eta_0$, we have
    $[ \mathcal{L}_{\mathcal{E}} (\lambda_1, \theta_{\lambda_1})  -  \mathcal{L}_{\mathcal{E}} (\lambda_2, \theta_{\lambda_2}) ] \cdot
    [ G^{\eta} \mathcal{L}_{\mathcal{T}} (\lambda_1, \theta_{\lambda_1}) 
    - G^{\eta} \mathcal{L}_{\mathcal{T}} (\lambda_2, \theta_{\lambda_2}) ] \geq 0$,
    where $\theta_{\lambda}^{\eta} = \theta_{\lambda} + \eta \nabla_{\theta_{\lambda}} \mathcal{L}_{\mathcal{T}} (\lambda, \theta_{\lambda})$ and
    $G^{\eta} \mathcal{L}_{\mathcal{T}} (\lambda, \theta_{\lambda})
    = \frac{1}{\eta} \left[\mathcal{L}_{\mathcal{T}} (\lambda, \theta_{\lambda}^{\eta}) - \mathcal{L}_{\mathcal{T}} (\lambda, \theta_{\lambda}) \right]$.
\label{asp:4}
\end{Assumption}

Under Assumption \eqref{asp:1} \eqref{asp:2} \eqref{asp:3}, if $\mathcal{T}$ is a first order operator, namely a gradient accent operator, to maximize $\mathcal{L}_{\mathcal{T}}$, GDI can be guaranteed to be superior to that w/o $\mathcal{E}$.
Under Assumption \eqref{asp:1} \eqref{asp:2} \eqref{asp:4}, if $\mathcal{T}$ is a second order operator, namely a natural gradient operator, to maximize $\mathcal{L}_{\mathcal{T}}$, GDI can also be guaranteed to be superior to that w/o $\mathcal{E}$.

\begin{Theorem}[First Order Optimization with Superior Target]
    Under Assumption \eqref{asp:1} \eqref{asp:2} \eqref{asp:3}, we have
    $\mathcal{L}_{\mathcal{T}} (\mathcal{P}_{\Lambda}^{(t+1)}, \theta^{(t+1)}) 
     = \textbf{E}_{\lambda \sim \mathcal{P}_{\Lambda}^{(t+1)}}  [\mathcal{L}_{\mathcal{T}} (\lambda, \theta^{(t+1)}_{\lambda})]
     \geq \textbf{E}_{\lambda \sim \mathcal{P}_{\Lambda}^{(t)}}  [\mathcal{L}_{\mathcal{T}} (\lambda, \theta^{(t+1)}_{\lambda})]
     = \mathcal{L}_{\mathcal{T}} (\mathcal{P}_{\Lambda}^{(t)}, \theta^{(t+1)})$.
\label{thm:1st_gdi}
\end{Theorem}

\begin{Proof}
    By \textbf{Theorem} \ref{thm:cts_Rp} (see App. \ref{App: proof}), the upper triangular transport inequality, 
    let $f(\lambda) = \mathcal{L}_{\mathcal{T}} (\lambda, \theta_{\lambda})$ and 
    $g(\lambda) = \mathcal{L}_{\mathcal{E}} (\lambda, \theta_{\lambda})$,
    the proof is done.
\end{Proof}

\begin{Remark}[Why Superior Target]
     In Algorithm \ref{alg:GDI}, if $\mathcal{E}$ updates $\mathcal{P}_{\Lambda}^{(t)}$ at time $t$, then the operator $\mathcal{T}$ at time $t+1$ can be written as $\theta^{(t+2)} = \theta^{(t+1)} + \eta \nabla_{\theta^{(t+1)}} \mathcal{L}_{\mathcal{T}} (\mathcal{P}_{\Lambda}^{(t+1)}, \theta^{(t+1)})$.
     If $\mathcal{P}_{\Lambda}^{(t)}$ hasn't been updated at time $t$, then the operator $\mathcal{T}$ at time $t+1$ can be written as $\theta^{(t+2)} = \theta^{(t+1)} + \eta \nabla_{\theta^{(t+1)}} \mathcal{L}_{\mathcal{T}} (\mathcal{P}_{\Lambda}^{(t)}, \theta^{(t+1)})$.
     \textbf{Theorem} \ref{thm:1st_gdi} shows that the target of $\mathcal{T}$ at time $t+1$ becomes higher if $\mathcal{P}_{\Lambda}^{(t)}$ is updated by $\mathcal{E}$ at time $t$.
\end{Remark}

\begin{Remark}[Practical Implementation]
    We provide one possible practical setting of GDI. 
    Let $\mathcal{L}_{\mathcal{E}} (\lambda, \theta_{\lambda}) = \mathcal{J}_{\pi_{\theta_{\lambda}}}$ and $\mathcal{L}_{\mathcal{T}} (\lambda, \theta_{\lambda}) = \mathcal{J}_{\pi_{\theta_{\lambda}}}$.
    $\mathcal{E}$ can update $\mathcal{P}_{\Lambda}$ by the Monte-Carlo estimation of $\mathcal{J}_{\pi_{\theta_{\lambda}}}$.
    $\mathcal{T}$ is to maximize $\mathcal{J}_{\pi_{\theta_{\lambda}}}$, which can be any RL algorithms.
\end{Remark}

\begin{Theorem}[Second Order Optimization with Superior Improvement]
    Under Assumption \eqref{asp:1} \eqref{asp:2} \eqref{asp:4}, we have
    $\textbf{E}_{\lambda \sim \mathcal{P}_{\Lambda}^{(t+1)}}  [G^{\eta} \mathcal{L}_{\mathcal{T}} (\lambda, \theta_{\lambda}^{(t+1)})] 
    \geq \textbf{E}_{\lambda \sim \mathcal{P}_{\Lambda}^{(t)}}  [G^{\eta} \mathcal{L}_{\mathcal{T}} (\lambda, \theta_{\lambda}^{(t+1)})] $, more specifically,
    $\textbf{E}_{\lambda \sim \mathcal{P}_{\Lambda}^{(t+1)}}
    [\mathcal{L}_{\mathcal{T}} (\lambda, \theta_{\lambda}^{(t+1),\eta}) - \mathcal{L}_{\mathcal{T}} (\lambda, \theta_{\lambda}^{(t+1)}) ]
    \geq \textbf{E}_{\lambda \sim \mathcal{P}_{\Lambda}^{(t)}}
    [\mathcal{L}_{\mathcal{T}} (\lambda, \theta_{\lambda}^{(t+1),\eta}) - \mathcal{L}_{\mathcal{T}} (\lambda, \theta_{\lambda}^{(t+1)}) ]$.
\label{thm:2nd_gdi}
\end{Theorem}

\begin{Proof}
    By \textbf{Theorem} \ref{thm:cts_Rp} (see App. \ref{App: proof}), the upper triangular transport inequality,
    let $f(\lambda) = G^{\eta} \mathcal{L}_{\mathcal{T}} (\lambda, \theta_{\lambda})$ and 
    $g(\lambda) = \mathcal{L}_{\mathcal{E}} (\lambda, \theta_{\lambda})$,
    the proof is done.
\end{Proof}

\begin{Remark}[Why Superior Improvement]
    \textbf{Theorem} \ref{thm:2nd_gdi} shows that, if $\mathcal{P}_{\Lambda}$ is updated by $\mathcal{E}$, the expected improvement of $\mathcal{T}$ is higher.
\end{Remark}

\begin{Remark}[Practical Implementation]
    Let $\mathcal{L}_{\mathcal{E}} (\lambda, \theta_{\lambda}) = \textbf{E}_{s \sim d_{\rho_0}^{\pi}} \textbf{E}_{a \sim \pi(\cdot | s) \exp(\epsilon A^{\pi}(s, \cdot))/Z} [A^{\pi}(s, a)]$, where $\pi = \pi_{\theta_{\lambda}}$.
    Let $\mathcal{L}_{\mathcal{T}} (\lambda, \theta_{\lambda}) = \mathcal{J}_{\pi_{\theta_{\lambda}}}$.
    If we optimize $\mathcal{L}_{\mathcal{T}} (\lambda, \theta_{\lambda})$ by natural gradient, 
    \citep{pgtheory} shows that, for direct parameterization, the natural policy gradient gives $\pi^{(t+1)} \propto \pi^{(t)} \exp (\epsilon A^{\pi^{(t)}})$, by \textbf{Lemma} \ref{lemma:perfdiff} (see App. \ref{App: proof}), the performance difference lemma,
    $V^{\pi} (s_0) - V^{\pi'} (s_0) = \frac{1}{1 - \gamma} \textbf{E}_{s \sim d_{s_0}^\pi} \textbf{E}_{a \sim \pi (\cdot | s)} [ A^{\pi'} (s, a) ]$, 
    hence if we ignore the gap between the states visitation distributions of $\pi^{(t)}$ and $\pi^{(t+1)}$, 
    $\mathcal{L}_{\mathcal{E}} (\lambda, \theta_{\lambda}^{(t)}) \approx \frac{1}{1 - \gamma} \textbf{E}_{s \sim d_{\rho_0}^{\pi}} [V^{\pi^{(t+1)}}(s) - V^{\pi^{(t)}} (s)]$, 
    where $\pi^{(t)} = \pi_{\theta_{\lambda}^{(t)}}$.
    Hence, $\mathcal{E}$ is actually putting more measure on $\lambda$ that can achieve more improvement.
\end{Remark}

\section{Experiment}
\label{sec: experiment}
We begin this section by describing our experimental setup. Then we report and analyze our SOTA results on ALE, specifically, 57 games, which are summarized and illustrated in App. \ref{appendix: experiment results}. To further investigate the mechanism of our algorithm, we study the effect of several major components.

\subsection{Experimental Setup}

The overall training architecture is on the top of the Learner-Actor  framework \citep{impala}, which supports large-scale training. Additionally, the recurrent encoder with LSTM \citep{lstm} is used to handle the partially observable MDP problem \citep{ale}. 
\textit{burn-in} technique is adopted to deal with the representational drift as \citep{r2d2}, and we train each sample twice.
A complete description of the hyperparameters can be found in App. \ref{Sec: appendix hyperparameters}. 
We employ additional environments to evaluate the scores during training, and the undiscounted episode returns averaged over 32 environments with different seeds have been recorded. 
Details on ALE and relevant evaluation criteria can be found in App. \ref{sec:app Experiment Details}.

To illustrate the generality and efficiency of GDI, we propose one implementation of GDI-I$^3$ and GDI-H$^3$, respectively. 
Let $\Lambda = \{\lambda | \lambda = (\tau_1, \tau_2, \epsilon)\}$.
The behavior policy belongs to a soft $\epsilon$-greedy policy space, which contains $\epsilon$-greedy policy and Boltzmann policy.
We define the behavior policy $\pi_{\theta_{\lambda}}$ as
\begin{equation}
\label{equ: soft epsilon policy space}
    \lambda = (\tau_1, \tau_2, \epsilon), \ 
    \pi_{\theta_{\lambda}}=\varepsilon \cdot \operatorname{Softmax}\left(\frac{A_1}{\tau_{1}}\right)+(1-\varepsilon) \cdot \operatorname{Softmax}\left(\frac{A_2}{\tau_{2}}\right)
\end{equation}
For GDI-I$^3$, $A_1$ and $A_2$ are identical, so it is estimated by an isomorphic family of trainable variables.
The learning policy is also $\pi_{\theta_{\lambda}}$.
For GDI-H$^3$, $A_1$ and $A_2$ are different, and they are estimated by two different families of trainable variables.
Since GDI needn't assume $A_1$ and $A_2$ are learned from the same MDP, so we use two kinds of reward shaping to learn $A_1$ and $A_2$ respectively, which can be found in App. \ref{app: Hyperparameters Used}.
Full algorithm can be found in App. \ref{App: Algorithm Pseudocode}.

The operator $\mathcal{T}$ is achieved by policy gradient, V-Trace and ReTrace \citep{impala, retrace} (see App. \ref{app: background on RL}), which meets Theorem \ref{thm:1st_gdi} by first order optimization. 

The operator $\mathcal{E}$, which optimizes $\mathcal{P}_{\Lambda}$, is achieved by a variant of Multi-Arm Bandits \citep[MAB]{sutton}, 
where Assumption \ref{asp:2} holds naturally.
More details can be found in App. \ref{Sec: appendix MAB}. 

\subsection{Summary of Results}
\label{sec: Summary of Results}

\begin{table}[H]
\small
\setlength{\tabcolsep}{1.0pt}
    \centering
    \begin{tabular}{c c c c c c c c c}
    \toprule
                      & GDI-H$^3$                   & GDI-I$^3$               & Muesli & RAINBOW & LASER & R2D2 & NGU & Agent57\\
    \midrule
    Num. Frames       &\textbf{2E+8}        &\textbf{2E+8}       &\textbf{2E+8}  & \textbf{2E+8} & \textbf{2E+8}  & 1E+10   & 3.5E+10  &1E+11 \\
    Game Time (year)  & \textbf{\GDIIgametime}       & \textbf{\GDIHgametime}      & \textbf{0.114} & \textbf{0.114} & \textbf{0.114}  & 5.7    & 19.9    & 57 \\
    HWRB              &\textbf{\GDIHHWRB}          &\GDIIHWRB         & 5             & 4             & 7              & 15      & 8                  & 18 \\
    Mean HNS(\%)      &\textbf{\GDIHmeanhns}     &\GDIImeanhns    & 2538.66        & 873.97        & 1741.36        & 3374.31 &3169.90   &4763.69 \\
    Median HNS(\%)    &\GDIHmedianhns               &\GDIImedianhns               & 1077.47        & 230.99        & 454.91         & 1342.27 &1208.11   &\textbf{1933.49}\\
    Mean HWRNS(\%)    &\textbf{\GDIHmeanHWRNS}      &\GDIImeanHWRNS              & 75.52         & 28.39         &45.39           & 98.78   & 76.00     &125.92\\
    Median HWRNS(\%)  &\textbf{\GDIHmedianHWRNS}                &\GDIImedianHWRNS               & 24.86          & 4.92          & 8.08           &33.62    & 21.19    &43.62\\
    Mean SABER(\%)    &\GDIHmeanSABER                &\GDIImeanSABER               & 48.74          & 28.39         & 36.78          &60.43    & 50.47    &\textbf{76.26}\\
    Median SABER(\%)  &\textbf{\GDIHmedianSABER}                & \GDIImedianSABER              & 24.68          & 4.92         & 8.08           &33.62    & 21.19     &43.62\\
    \bottomrule
    \end{tabular}
    \caption{Experiment results of Atari.
    Muesli's scores are from \citep{muesli}.
    RAINBOW's scores are from \citep{impala}.
    LASER's scores are from \citep{laser}, no sweep at 200M.
    R2D2's scores are from \citep{r2d2}.
    NGU's scores are from \citep{ngu}.
    Agent57's scores are from \citep{agent57}. More details on abbreviations and notations can see App. \ref{app: Abbreviation and Notation} and \ref{sec:app Experiment Details}. Full comparison among all algorithms can see App. \ref{appendix: experiment results}.}
    \label{tab:atari_results}
\end{table}
\normalsize

We construct a multivariate evaluation system to emphasize the superiority of our algorithm in all aspects, and more discussions on those evaluation criteria are in App. \ref{App: ALE} and details are in App. \ref{sec:app Experiment Details}. Furthermore, to avoid any issues that aggregated metrics may have, App. \ref{appendix: experiment results} provides full learning curves for all games, as well as detailed comparison tables of raw and normalized scores. 

The aggregated results across games are reported in Tab. \ref{tab:atari_results}. Our agents obtain the highest mean HNS with an extraordinary learning efficiency from this table. 
Furthermore, our agents have achieved \GDIHHWRB \  human world record breakthroughs and more than 90 times the average human score of Atari games via playing from scratch for less than 1.5 months. 
Although Agent57 obtains the highest median HNS, it costs each of the agents more than 57 years to obtain such performance, revealing its low learning efficiency.  
It is obvious that there is no such world record achieved by a human who played for over 57 years. 
This is due to the fact that Agent57 fails to handle the balance between exploration and exploitation, thus collecting a large number of inferior samples, which further hinders the efficient-learning and makes it harder for policy improvement. 
Other algorithms gain higher learning efficiency than Agent57 but relatively lower final performance, such as NGU and R2D2, which acquire over 10B frames. 
Except for median HNS, our performance is better on all criteria than NGU and R2D2.
In addition, other algorithms with 200M training frames are struggling to match our performance.

These results come from the following aspects:
\begin{enumerate}
    \item Several games have been solved completely, achieving the historically highest score, such as RoadRunner, Seaquest, Jamesbond.
    \item Massive games show enormous potentialities for improvement but fail to converge for lack of training, such as BeamRider, BattleZone, SpaceInvaders.
    \item This paper aims to illustrate that GDI is general for seeking a suitable balance between exploration and exploitation, so we refuse to adopt any handcrafted and domain-specific tricks such as the intrinsic reward. Therefore, we suffer from the hard exploration problem, such as PrivateEye, Surround, Amidar. 
\end{enumerate}
Therefore, there are several aspects of potential improvement. For example, a more extensive training scale may benefit higher performance. More exploration techniques can be incorporated into GDI to handle those hard-exploration problems through guiding the direction of the acquired samples.

\subsection{Ablation Study}
In the ablation study, we further investigate the effects of several properties of GDI. 
We set GDI-I$^3$ and GDI-H$^3$ as our baseline control group. 
To prove the effects of the data distribution optimization operator $\mathcal{E}$, we set two ablation groups, which are Fixed Selection from GDI-I$^0$ w/o $\mathcal{E}$ and Random Selection from GDI-I$^3$ w/o $\mathcal{E}$.
To prove the capacity of the behavior policy space matters in GDI, we set two ablation groups, which are $\epsilon$-greedy Selection $\Lambda = \{\lambda | \lambda = (\epsilon)\}$ and Boltzmann Selection $\Lambda = \{\lambda | \lambda = (\tau)\}$. 
Both $\epsilon$-greedy Selection and Boltzmann Selection implement $\mathcal{E}$ by the same MAB as our baselines'. 
More details on ablation study can see App. \ref{app: appendix Ablation Study Details}.

From results in App. \ref{Sec: Appendix Ablation Study Results}, it is evident that both the data distribution optimization operator $\mathcal{E}$ and  the capacity of the behavior policy space are critical. 
This is since if they lack the cognition to identify suitable experiences from various data, high variance and massive poor experiences will hinder the policy improvement, and if the RL agents lack the vision to find more examples to learn, they may ignore some shortcuts. 
To further prove the capacity of the policy space does bring more diverse data, we draw the t-SNE of GDI-I$^3$, GDI-H$^3$ and Boltzmann Selection in App. \ref{app: tsne}, from which we see GDI-I$^3$ and GDI-H$^3$ can explore more high-value states that Boltzmann selection has less chance to find. We also evaluate Fixed Selection and Boltzmann Selection in all 57 Atari games, and recorded the comparison tables of raw and normalized scores in App. \ref{app: ablation score}.

\section{Conclusion}
This paper proposes a novel RL paradigm to effectively and adaptively trade-off the exploration and exploitation, integrating the data distribution optimization into the generalized policy iteration paradigm. 
Under this paradigm, we propose feasible implementations, which both have achieved new SOTA among all 200M scale algorithms on all evaluation criteria and obtained the best mean final performance and learning efficiency compared with all 10B+ scale algorithms. 
Furthermore, we have achieved 22 human world record breakthroughs within less than 1.5 months of game time. 
It implies that our algorithm obtains both superhuman learning performance and human-level learning efficiency. 
In the experiment, we discuss the potential improvement of our method in future work.

\bibliographystyle{rusnat}
\bibliography{ref}

\begin{appendices}

\clearpage

\section{Abbreviation and Notation}
\label{app: Abbreviation and Notation}
In this Section, we briefly summarize some common notations and abbreviations in our paper for the convenience of readers, which is illustraed in Tab. \ref{tab: abbreviation} and Tab. \ref{tab: notation}.

\begin{table}[!hb]
	\centering
	\caption{Abbreviation}
	\label{tab: abbreviation}
	\begin{tabular}{c c}
		\toprule
		\textbf{Abbreviation} &\textbf{Description}\\
		\midrule
		SOTA & State-of-The-Art \citep{agent57}    \\
		RL  & Reinforcement Learning \citep{sutton} \\
		DRL & Deep Reinforcement Learning \citep{sutton} \\
        GPI & Generalized Policy Iteration \citep{sutton} \\
        PG  & Policy Gradient \citep{sutton} \\
        AC  & Actor Critic \citep{sutton} \\
        ALE & Atari Learning Environment \citep{ale} \\
        HNS   & Human Normalized Score \citep{ale} \\
        HWRB & Human World Records Breakthrough \\
        HWRNS & Human World Records Normalized Score \\
        SABER & Standardized Atari BEnchmark for RL \citep{atarihuman}\\
        CHWRNS & Capped Human World Records Normalized Score \\      
        WLOG   & without loss of generality \\
        w/o    & without \\
		\bottomrule
	\end{tabular} 
\end{table}

\begin{table}[!hb]
	\centering
	\caption{Notation}
	\label{tab: notation}
	\begin{tabular}{c c}
		\toprule
		\textbf{Symbol} &\textbf{Description}\\
		\midrule
	    $s$ & state \\
	    $a$ & action \\
	    $\mathcal{S} $ & set of all states \\
	    $\mathcal{A} $ & set of all actions \\
	    $\Delta$  & probability simplex \\
	    $\mu $ & behavior policy \\
	    $\pi $ & target policy \\
	    $G_t $ & cumulative discounted reward or return at $t$ \\
	    $d_{\rho_0}^{\pi}$  & the states visitation distribution of $\pi$ with the initial state distribution $\rho_0$ \\
	    $J_{\pi}$     & the expectation of the returns with the states visitation distribution of $\pi$ \\
	    $V^{\pi}$ & the state value function of $\pi$\\
	    $Q^{\pi}$ &  the state-action value function of $\pi$\\
	    $\gamma$ & discount-rate parameter \\
	    $\delta_{t}$ & temporal-difference error at $t$\\
	    $\Lambda$ & set of indexes  \\
	    $\lambda$ & one index in $\Lambda$ \\
	    $\mathcal{P}_{\Lambda}$ & one probability measure on $\Lambda$ \\
	    $\Theta$  & set of all possible parameter values \\
	    $\theta$  & one parameter value in $\Theta$ \\
	    $\theta_\lambda$ & a subset of $\theta$, indicates the parameter in $\theta$ being used by the index $\lambda$ \\
	    $\mathcal{X}$ & set of samples \\
	    $x$ & one sample in $\mathcal{X}$ \\
	    $\mathcal{D}$ & set of all possible states visitation distributions \\
	    $\mathcal{E}$ & the data distribution optimization operator \\
	    $\mathcal{T}$ & the RL algorithm optimization operator \\
	    $L_{\mathcal{E}}$ & the loss function of $\mathcal{E}$ to be maximized, calculated by the samples set $\mathcal{X}$ \\
	    $\mathcal{L}_\mathcal{E}$ & expectation of $L_{\mathcal{E}}$, with respect to each sample $x \in \mathcal{X}$ \\
	    $L_{\mathcal{T}}$ & the loss function of $\mathcal{T}$ to be maximized, calculated by the samples set $\mathcal{X}$ \\
	    $\mathcal{L}_\mathcal{T}$ & expectation of $L_{\mathcal{T}}$, with respect to each sample $x \in \mathcal{X}$ \\
		\bottomrule
	\end{tabular} 
\end{table}

\clearpage

\section{Background on RL}
\label{app: background on RL}

 The RL problem can be formulated by a Markov decision process \citep[MDP]{howard1960dynamic} defined by the tuple  $\left(\mathcal{S}, \mathcal{A}, p, r, \gamma, \rho_{0}\right)$. 
 Considering a discounted episodic MDP, the initial state $s_0$ will be sampled from the distribution denoted by $\rho_0(s): \mathcal{S} \rightarrow \Delta(\mathcal{S})$. 
 At each time t, the agent choose an action $a_t \in \mathcal{A}$ according to the policy $\pi(a_t|s_t): \mathcal{S} \rightarrow \Delta(\mathcal{A})$ at state $s_t \in \mathcal{S}$. 
 The environment receives the action, produces a reward $r_t \sim r(s,a): \mathcal{S} \times \mathcal{A} \rightarrow \mathbf{R}$ and transfers to the next state $s_{t+1}$  submitted to the transition distribution $p\left(s^{\prime} \mid s, a\right): \mathcal{S} \times \mathcal{A} \rightarrow \Delta(\mathcal{S})$. 
 The process continues until the agent reaches a terminal state or a maximum time step. 
 Define return $G_t = \sum_{k=0}^\infty \gamma^k r_{t+k}$, state value function $V^{\pi}(s_t) = \textbf{E}\left[ \sum_{k=0}^\infty \gamma^k r_{t+k} | s_t \right]$, state-action value function $Q^{\pi}(s_t, a_t) = \textbf{E}\left[ \sum_{k=0}^\infty \gamma^k r_{t+k} | s_t, a_t \right]$, and advantage function $A^{\pi}(s_t,a_t) = Q^{\pi}(s_t, a_t) - V^{\pi}(s_t)$, wherein $\gamma \in(0,1)$ is the discount factor.
The connections between $V^\pi$ and $Q^\pi$ is given by the Bellman equation,
\begin{equation*}
    \mathcal{T}Q^{\pi} (s_t, a_t) = \textbf{E}_{'\pi} [r_t + \gamma V^{\pi}(s_{t+1})],
\end{equation*}
where
\begin{equation*}
    V^{\pi} (s_t)  = \textbf{E}_{\pi} [Q^{\pi} (s_t, a_t)].
\end{equation*}
The goal of reinforcement learning is to find the optimal policy $\pi^*$ that maximizes the expected sum of discounted rewards, denoted by $\mathcal{J}$ \citep{sutton}:
\begin{equation*}
\pi^{*}=\underset{\pi}{\operatorname{argmax}} \mathcal{J}_{\pi}(\tau) = \underset{\pi}{\operatorname{argmax}} \textbf{E}_{\pi}\left[G_{t}\right]= \underset{\pi}{\operatorname{argmax}} \textbf{E}_{\pi}[\sum_{k=0}^{\infty} \gamma^{k} r_{t+k}]
\end{equation*}


Model-free reinforcement learning (MFRL) has made many impressive breakthroughs in a wide range of Markov decision processes  \citep[MDP]{alpha_star,ftw,agent57}.
MFRL mainly consists of two categories, valued-based methods \citep{dqn,rainbow} and policy-based methods \citep{trpo,ppo,impala}.

Value-based methods learn state-action values and select actions according to these values. 
One merit of value-based methods is to accurately control the exploration rate of the behavior policies by some trivial mechanism, such like $\epsilon$-greedy.
The drawback is also apparent. 
The policy improvement of valued-based methods totally depends on the policy evaluation. 
Unless the selected action is changed by a more accurate policy evaluation, the policy won't be improved. 
So the policy improvement of each policy iteration is limited, which leads to a low learning efficiency.
Previous works equip valued-based methods with many appropriated designed structures, achieving a more promising learning efficiency \citep{dueling_q,priority_q,r2d2}.

In practice, value-based methods maximize $\mathcal{J}$ by policy iteration \citep{sutton}. 
The policy evaluation is fulfilled by minimizing $\textbf{E}_{\pi} [(G - Q^\pi) ^ 2]$, which gives the gradient ascent direction 
$\textbf{E}_{\pi} [(G - Q^\pi) \nabla Q^\pi]$. 
The policy improvement is usually achieved by $\epsilon$-greedy.

Q-learning is a typical value-based methods, which updates the state-action value function $Q(s,a)$ with Bellman Optimality Equation \citep{qlearning}: 
\begin{equation*}
    \begin{array}{c}
    \delta_{t}=r_{t+1}+\gamma \arg \max _{a} Q\left(s_{t+1}, a\right)-Q\left(s_{t}, a_{t}\right) \\
    Q\left(s_{t}, a_{t}\right) \leftarrow Q\left(s_{t}, a_{t}\right)+\alpha \delta_{t}
    \end{array}
\end{equation*}
wherein $\delta_t$ is the temporal difference error \citep{TDerror}, and $\alpha$ is the learning rate.

A refined structure design of $Q^\pi$ is achieved by \citep{dueling_q}. It estimates $Q^\pi$ by a summation of two separated networks, $Q^\pi = A^\pi + V^\pi$.

Policy gradient \citep[PG]{williams1992simple} methods is an outstanding representative of policy-based RL algorithms, which directly parameterizes the policy and  updates through optimizing the following objective: 
\begin{equation*}
    \mathcal{J} (\theta)=\mathbb{E}_{\pi}\left[\sum_{t=0}^{\infty} \log \pi_{\theta}\left(a_{t} \mid s_{t}\right) R(\tau)\right]
\end{equation*}
wherein $R(\tau)$ is the cumulative return on trajectory $\tau$. In PG method, policy improves via ascending  along the gradient of the above equation, denoted as policy gradient:
\begin{equation*}
\nabla_{\theta} \mathcal{J} \left(\pi_{\theta}\right) =\underset{\tau \sim \pi_{\theta}}{\mathrm{E}}\left[\sum_{t=0}^{\infty} \nabla_{\theta} \log \pi_{\theta}\left(a_{t} \mid s_{t}\right) R(\tau)\right]
\end{equation*}

One merit of policy-based methods is that they incorporate a policy improvement phase every training step, suggesting a higher learning efficiency than value-based methods.
Nevertheless, policy-based methods easily fall into a suboptimal solution, where the entropy drops to $0$ \citep{sac}.
The actor-critic methods introduce a value function as the baseline to reduce the variance of the policy gradient \citep{a3c}, but maintain the other characteristics unchanged.

Actor-Critic \citep[AC]{sutton} reinforcement learning updates the policy gradient with an value-based critic, which can reduce variance of estimates and thus ensure  more stable and rapid optimization.
\begin{equation*}
    \nabla_{\theta} \mathcal{J}(\theta)=\mathbb{E}_{\pi}\left[\sum_{t=0}^{\infty} \psi_{t} \nabla_{\theta} \log \pi_{\theta}\left(a_{t} \mid s_{t}\right)\right]
\end{equation*}
wherein $\psi_{t}$ is the critic to guide the improvement directions of policy improvement, which can be the state-action value function $Q^{\pi}\left(s_{t}, a_{t}\right)$, the advantage function $A^{\pi}\left(s_{t}, a_{t}\right)=Q^{\pi}\left(s_{t}, a_{t}\right)-V^{\pi}(s_t)$.

\subsection{Retrace}

When large scale training is involved, the off-policy problem is inevitable.
Denote $\mu$ to be the behavior policy, $\pi$ to be the target policy, and $c_t = \min\{\frac{\pi_t}{\mu_t}, \Bar{c}\}$ to be the clipped importance sampling. 
For brevity, denote $c_{[t: t+k]} = \prod_{i=0}^{k} c_{t+i}$. 
ReTrace \citep{retrace} estimates $Q(s_t, a_t)$ by clipped per-step importance sampling
\begin{equation*}
\label{Equ: retrace}
    Q^{\Tilde{\pi}} (s_t, a_t) 
= \textbf{E}_{\mu} [ Q(s_t, a_t) + \sum_{k \geq 0} \gamma^k 
c_{[t+1:t+k]} \delta^{Q}_{t+k} Q ],
\end{equation*}
where $\delta^{Q}_t Q \overset{def}{=} r_t + \gamma Q(s_{t+1}, a_{t+1}) - Q(s_t, a_t)$. 
The above operator is a contraction mapping, 
and $Q$ converges to $Q^{\Tilde{\pi}_{ReTrace}}$ that corresponds to some $\Tilde{\pi}_{ReTrace}$.

\subsection{Vtrace}
Policy-based methods maximize $\mathcal{J}$ by policy gradient. 
It's shown \citep{sutton} that $\nabla \mathcal{J} = \textbf{E}_\pi [G \nabla \log \pi]$. 
When involved with a baseline, it becomes an actor-critic algorithm such as $\nabla \mathcal{J} = \textbf{E}_\pi [(G - V^\pi) \nabla \log \pi]$, where $V^\pi$ is optimized by minimizing $\textbf{E}_\pi [(G - V^\pi)^2]$, i.e. gradient ascent direction $\textbf{E}_\pi [(G - V^\pi)\nabla V^\pi]$.

IMPALA \citep{impala} introduces V-Trace off-policy actor-critic algorithm to correct for the discrepancy between target policy and behavior policy. Denote $\rho_t = \min\{\frac{\pi_t}{\mu_t}, \Bar{\rho} \}$. V-Trace estimates $V(s_t)$ by
\begin{equation*}
\label{Equ: vtrace}
    V^{\Tilde{\pi}} (s_t) 
        = \textbf{E}_{\mu} [ 
        V(s_t) + \sum_{k \geq 0} \gamma^k 
     c_{[t:t+k-1]} \rho_{t+k}  \delta^{V}_{t+k} V ],
\end{equation*}
where $\delta^{V}_t V \overset{def}{=} r_t + \gamma V(s_{t+1}) - V(s_t)$. 
If $\Bar{c} \leq \Bar{\rho}$, the above operator is a contraction mapping, and $V$ converges to $V^{\Tilde{\pi}}$ that corresponds to 
$$
        \Tilde{\pi}(a|s) = \frac
        {\min \left\{\Bar{\rho} \mu (a|s), \pi(a|s)\right\}}
        {\sum_{b \in \mathcal{A}}\min \left\{\Bar{\rho} \mu (b|s), \pi(b|s)\right\}}.
$$
The policy gradient is given by
$$
\textbf{E}_\mu \left[\rho_t (r_t + \gamma V^{\Tilde{\pi}}(s_{t+1}) - V(s_t)) \nabla \log \pi \right].
$$

\clearpage
\section{Background on ALE}
\label{App: ALE}
Human intelligence is able to solve many tasks of different natures. In pursuit of generality in artificial intelligence, video games have become an important testing ground: they require a wide set of skills such as perception, exploration and control. Reinforcement Learning  is at the forefront of this development, especially when combined with deep neural networks in DRL.

The Arcade Learning Environment \citep[ALE]{ale} was proposed as a platform for empirically assessing agents designed for general competency across a wide range of games. It provides many different tasks ranging from simple paddle control in the ball game Pong to complex labyrinth exploration in Montezuma’s Revenge which remains unsolved by general algorithms up to today. ALE offers an interface to a diverse set of Atari 2600 game environments designed to be engaging and challenging for human players. As \citep{ale} put it, the Atari 2600 games are well suited for evaluating
general competency in AI agents for three main reasons:
\begin{enumerate}
    \item Varied enough to claim generality.
    \item Each interesting enough to be representative of settings that might be faced in practice.
    \item Each created by an independent party to be free of experimenter’s bias.
\end{enumerate}

\subsection{Human Normalized Score}

Agents are expected to perform well in as many games as possible without the use of game-specific information. Deep Q-Networks  \citep[DQN]{dqn} was the first algorithm to achieve human-level control in a large number of the Atari 2600 games, measured by human normalized scores \citep[HNS]{ale}. Subsequently, using HNS to assess performance on Atari games has become one of the most widely used benchmarks in deep reinforcement learning , despite the human baseline scores potentially underestimating human performance relative to what is possible \citep{atarihuman}.

\subsection{Human World Records Baseline}

Except for comparing with the average human performance, a more common way to evaluate AI for games is to let agents compete against human world champions. Recent examples for DRL include the victory of OpenAI Five on Dota 2 \citep{dota} or AlphaStar versus Mana for StarCraft 2 \citep{alpha_star}. In the same spirit, one
of the most used metric for evaluating RL agents on Atari is to compare them to the human baseline
introduced by \citep{ale}. Previous works use the normalized human score, i.e. 0\% is the score of a random player and 100\% is the score of the human baseline, which allows to summarize the
performance on the whole Atari set in one number, instead of individually comparing raw scores for
each of the 57 games. However, it's obvious that this human baseline is far from being representative of
the best human player, which means that using it to claim superhuman performance is misleading.

\subsection{Human World Records Normalized Score}
As \citep{atarihuman} said, previous claims of superhuman performance of RL might not be accurate owing to comparing with the averaged performance of  normal human  instead of the human world records, which means there are still massive games of Atari where human champions outperform the RL agents. 
Thus, we believe the \emph{human world records normalized score} (HWRNS) can serve as a more suitable evaluation criterion than the origin human normalized score, which directly compare the RL agents with the best human performance. HWRNS of a Atari game  surpass 100\% proves the fact that the DRL agents surpass the human world records and actually surpass the human on that game. When the mean HWRNS surpass 100\% we can say the RL agents can reach and even surpass the highest level of humanity, and then we can say our algorithms really achieve the superhuman level control. Recommended by \citep{atarihuman}, we also adopt the capped HWRNS that each HWRNS will be capped below 200\% as a evaluation criterion to avoid argument. 

\subsection{Learning Efficiency}

The goal of reinforcement learning is to achieve human level control. It is reflected in two aspects. On the one hand, the RL agents can reach and even surpass the human world records, which is the central focus of massive studies. On the other hand, we should not ignore the essential pursuit of reinforcement learning is to master human learning ability, which acquire the RL agents to not only learn how to do but also learn how to learn efficiently. For example, human can achieve one world records of Atari within only few years or even few months, however present SOTA RL algorithms like Agent57 acquires tens of years to achieve similar results, which implies the fact that there is still much room to improve the learning efficiency of reinforcement learning algorithm. 
\clearpage

\section{Atari Benchmark}
\label{app: atari benchmark}

Artificial intelligence (AI) in video games is a longstanding research area. It studies how to learn human-level and even surpassing-human-level agents when playing video games. 
The Arcade Learning Environment  \citep[ALE]{ale} is a universal experiment platform for empirically assessing the general competency of agents across a wide range of games. 
In addition, ALE offers an interface to a diverse set of Atari 2600 game environments designed to engage and challenge human players.  
Agents are expected to perform well in as many games as possible without the use of game-specific information. 

Since Deep Q Network \citep[DQN]{dqn} firstly achieves human level control of Atari games, reinforcement learning (RL) has brought the dawn of solving challenges of ALE and surpassing the human level control, which inspires researchers to pursuit more state-of-the-art(SOTA) performance. 
At the beginning, massive variants of DQN achieve new SOTA results. 
Double DQN \citep{doubledqn} introduces independent target network to alleviating overestimation problem. 
Dueling DQN \citep{dueling_q} adopts the dueling neural network architecture and achieved a new SOTA. 
RAINBOW \citep{rainbow} combines various effective extensions of DQN and improves the learning efficiency and the final performance.
Retrace($\lambda$) \citep{retrace} takes the per-step importance sampling, off policy Q($\lambda$), and tree-backup($\lambda$) \citep{sutton} to estimate $Q(s,a)$, resulting in a low variance estimation of $Q(s, a)$:
\begin{equation}
    Q^{\tilde{\pi}}\left(s_{t}, a_{t}\right)=\mathbf{E}_{\mu}[Q\left(s_{t}, a_{t}\right)+\sum_{k \geq 0} \gamma^{k} c_{[t+1: t+k]} \delta_{t+k}^{Q} Q]
\end{equation}
where $c_{t}=\min \left\{\frac{\pi_{t}}{\mu_{t}}, \bar{c}\right\}$, $c_{[t: t+k]}=\prod_{i=0}^{k} c_{t+i}$ and $\delta_{t}^{Q} Q \stackrel{\text { def }}{=} r_{t}+\gamma Q\left(s_{t+1}, a_{t+1}\right)-Q\left(s_{t}, a_{t}\right)$.

At the same time, PG methods is also booming, wherein AC framework is one of the brightest pearls. 
Asynchronous advantage actor-critic \citep[A3C]{a3c} introduces a novel asynchronous training with several actors, wherein an entropy regularization term 
is introduced into the objective function to encourage the exploration.
Importance-Weighted Actor Learner Architecture  \citep[IMPALA]{impala} is a novel large scale distributed training framework, which achieves stable learning by combining decoupled acting and learning with a novel V-trace off-policy correction method to estimate $V(s)$:
\begin{equation}
    V^{\tilde{\pi}}\left(s_{t}\right)=\mathbf{E}_{\mu}[V\left(s_{t}\right)+\sum_{k \geq 0} \gamma^{k} c_{[t: t+k-1]} \rho_{t+k} \delta_{t+k}^{V} V]
\end{equation}
where $\rho_{t}=\min \left\{\frac{\pi_{t}}{\mu_{t}}, \bar{\rho}\right\}$, $\delta_{t}^{V} V \stackrel{\text { def }}{=} r_{t}+\gamma V\left(s_{t+1}\right)-V\left(s_{t}\right)$. 
IMPALA reaches a new SOTA of policy-based methods on ALE. 
However, there still exist some hard-to-explore games with long horizon and sparse reward, like Montezuma’s Revenge, which need better exploration ability, namely, a breakthrough on the method.

Go-Explore \citep{goexplore} learns exploration and robustification separately, and achieves huge breakthroughs on games which acquire massive exploration.
However, there still exist some extremely hard games like Skiing where the average human performs better than RL agents. 
Agent57 \citep{agent57} firstly surpasses the average human performance in all 57 games, which is marked as a new milestone on ALE. 
Nevertheless, the breakthrough is achieved at the expense of tremendous training samples, called the low learning efficiency problem, which hinders the application of the method into real-world problems. 

For solving the low learning efficiency problem, model-based methods are regarded as one solution. 
MuZero \citep{muzero} is based on the frameworks of AlphaZero, which combines MCTS with a learned model
to make planning. 
It extends model-based RL to a range of logically complex and visually complex domains, and achieves a SOTA performance. 

Unfortunately, both value-based SOTA method RAINBOW, policy-based SOTA method IMPALA, model-free SOTA method Agent57 and the model-based SOTA method MuZero fail to synchronously guarantee the learning efficiency and the final performance. 

We concluded the SOTA results on the Atari benchmark and the corresponding learning efficiency in Figure \ref{fig: efficiency mean HNS time}. 
It's seen that our method reaches a new SOTA on both mean HNS and learning efficiency.
Our final performance is competitive with the best model-free algorithm Agent57, and simultaneously achieves a better learning efficiency than the best model-based algorithm Muzero.


\subsection{RL Benchmarks on HNS}
\label{app: RL Benchmarks on HNS}
We report several milestones of Atari benchmarks on HNS, including DQN \citep{dqn}, RAINBOW \citep{rainbow}, IMPALA \citep{impala}, LASER \citep{laser}, R2D2 \citep{r2d2}, NGU \citep{ngu}, Agent57 \citep{agent57}, Go-Explore \citep{goexplore}, MuZero \citep{muzero}, DreamerV2 \citep{dreamerv2}, SimPLe \citep{modelbasedatari} and Musile \citep{muesli}. We summary mean HNS and median HNS of these algorithms marked with their game time (year), learning efficiency and training scale in Fig \ref{fig: year mean HNS time},  \ref{fig: efficiency mean HNS time} and \ref{fig: scale mean HNS time}.

\begin{figure*}[!t]
    \centering
	\subfigure{
		\includegraphics[width=0.45\textwidth]{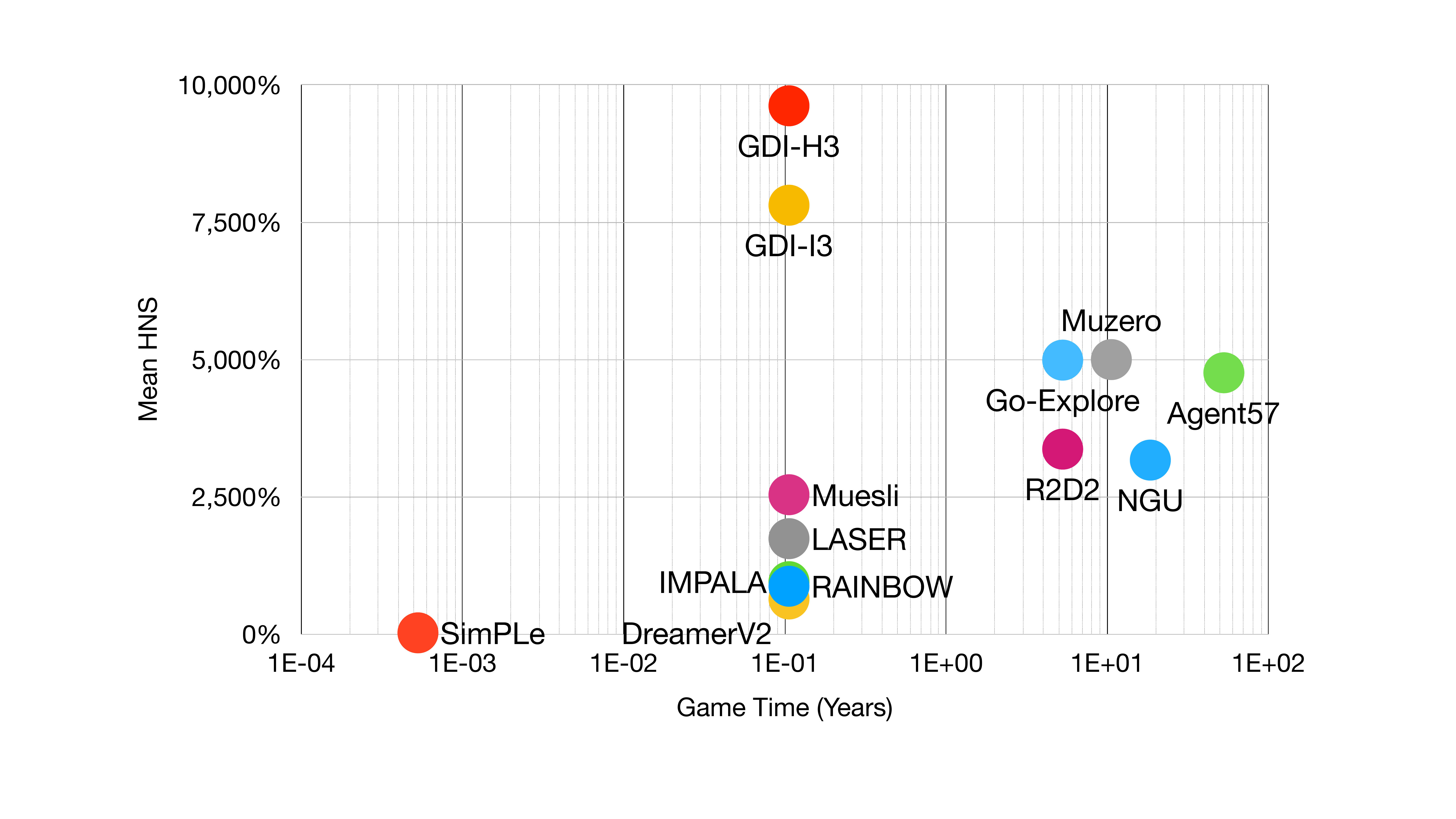}
	}
	\subfigure{
		\includegraphics[width=0.45\textwidth]{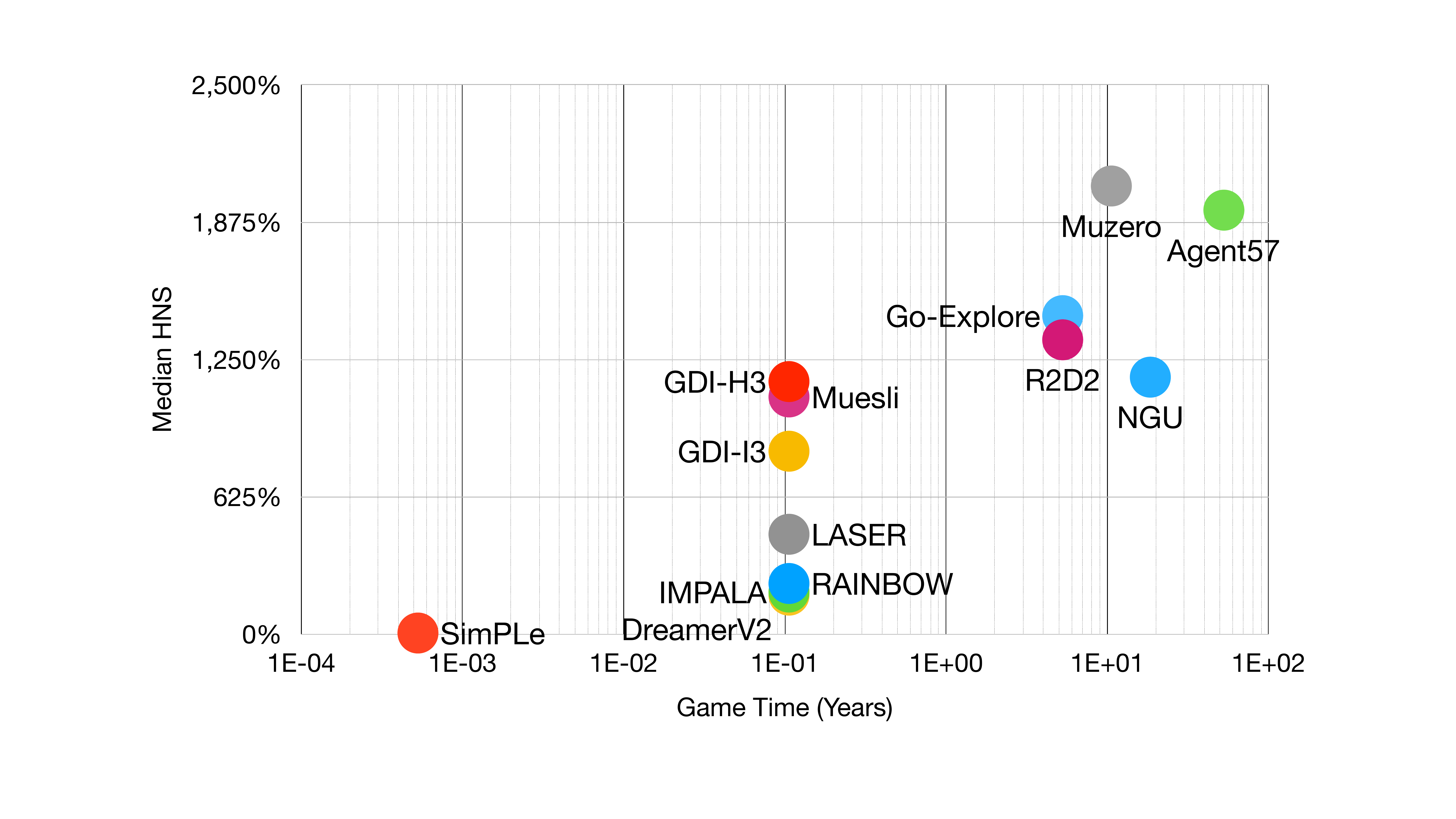}
	}
	\centering
	\caption{SOTA algorithms of Atari 57 games on mean and median HNS (\%) and game time (year).}
	\label{fig: year mean HNS time}
\end{figure*}

\begin{figure*}[!t]
    \centering
	\subfigure{
		\includegraphics[width=0.45\textwidth]{photo/tongjitu/learning efficiency/meanHNS.pdf}
	}
	\subfigure{
		\includegraphics[width=0.45\textwidth]{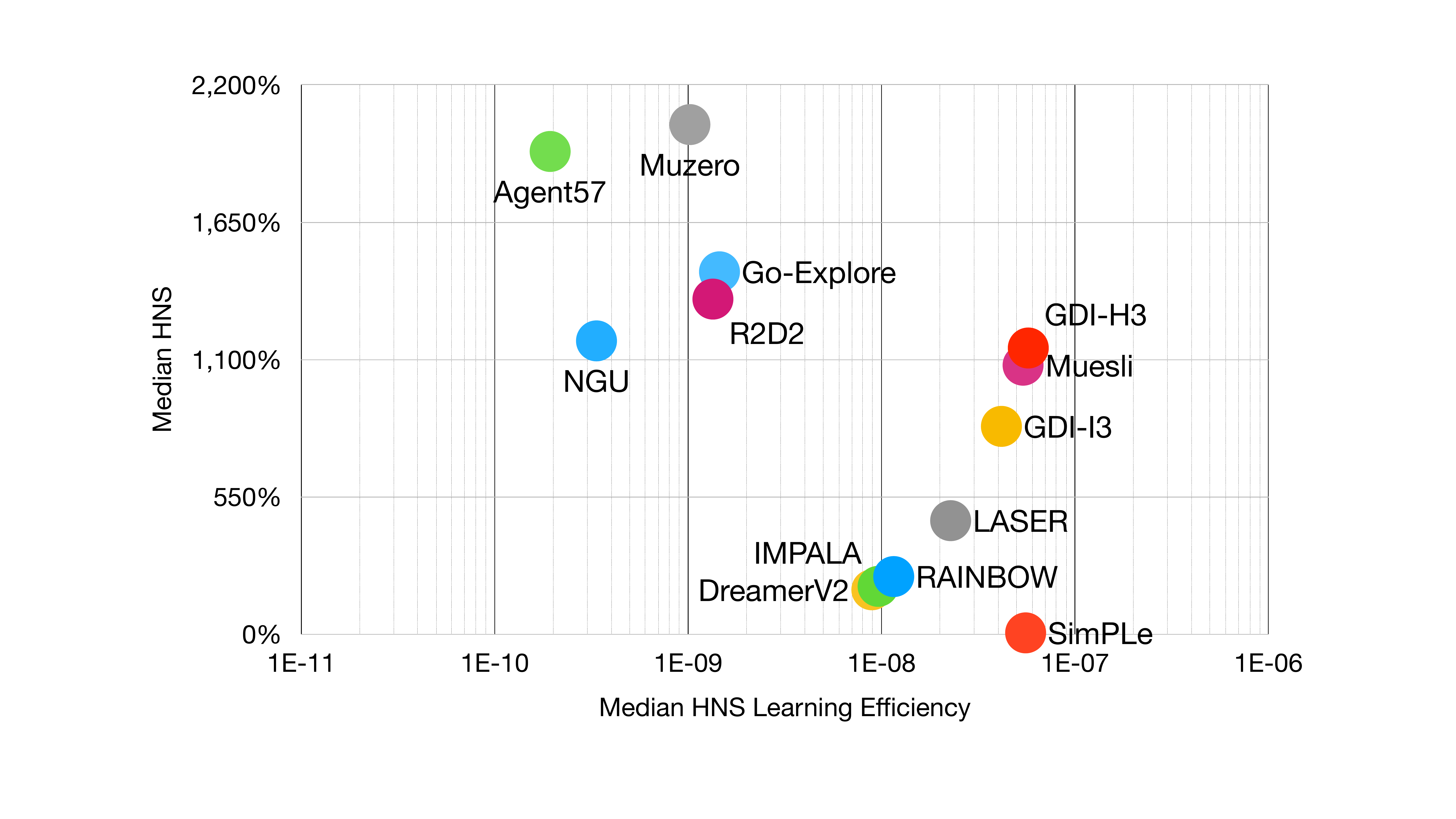}
	}
	\centering
	\caption{SOTA algorithms of Atari 57 games on mean and median HNS (\%) and corresponding learning efficiency calculated by $\frac{\text{MEAN HNS/MEDIAN HNS}}{\text{TRAINING FRAMES}}$.}
	\label{fig: efficiency mean HNS time}
\end{figure*}

\begin{figure*}[!t]
    \centering
	\subfigure{
		\includegraphics[width=0.45\textwidth]{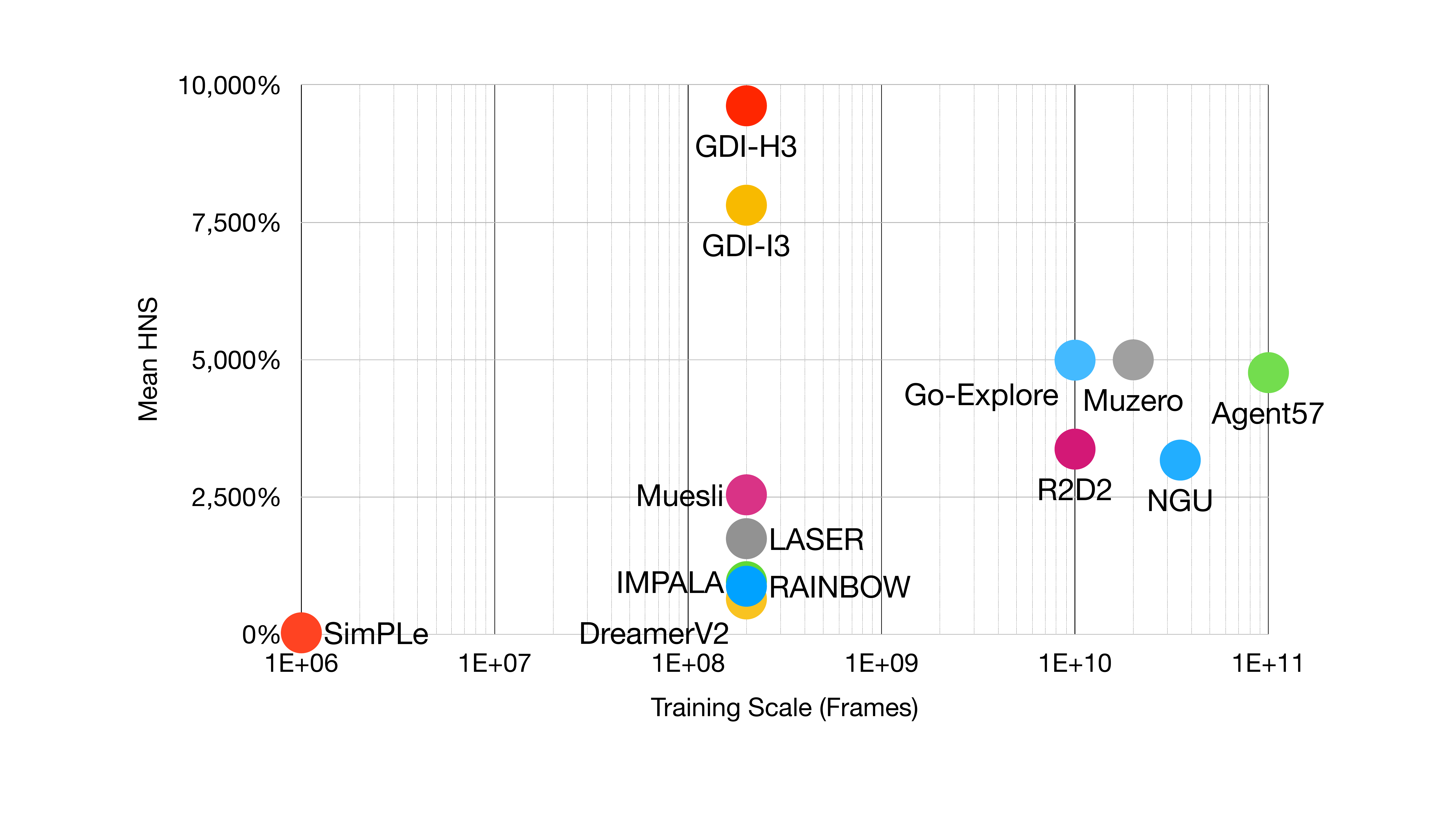}
	}
	\subfigure{
		\includegraphics[width=0.45\textwidth]{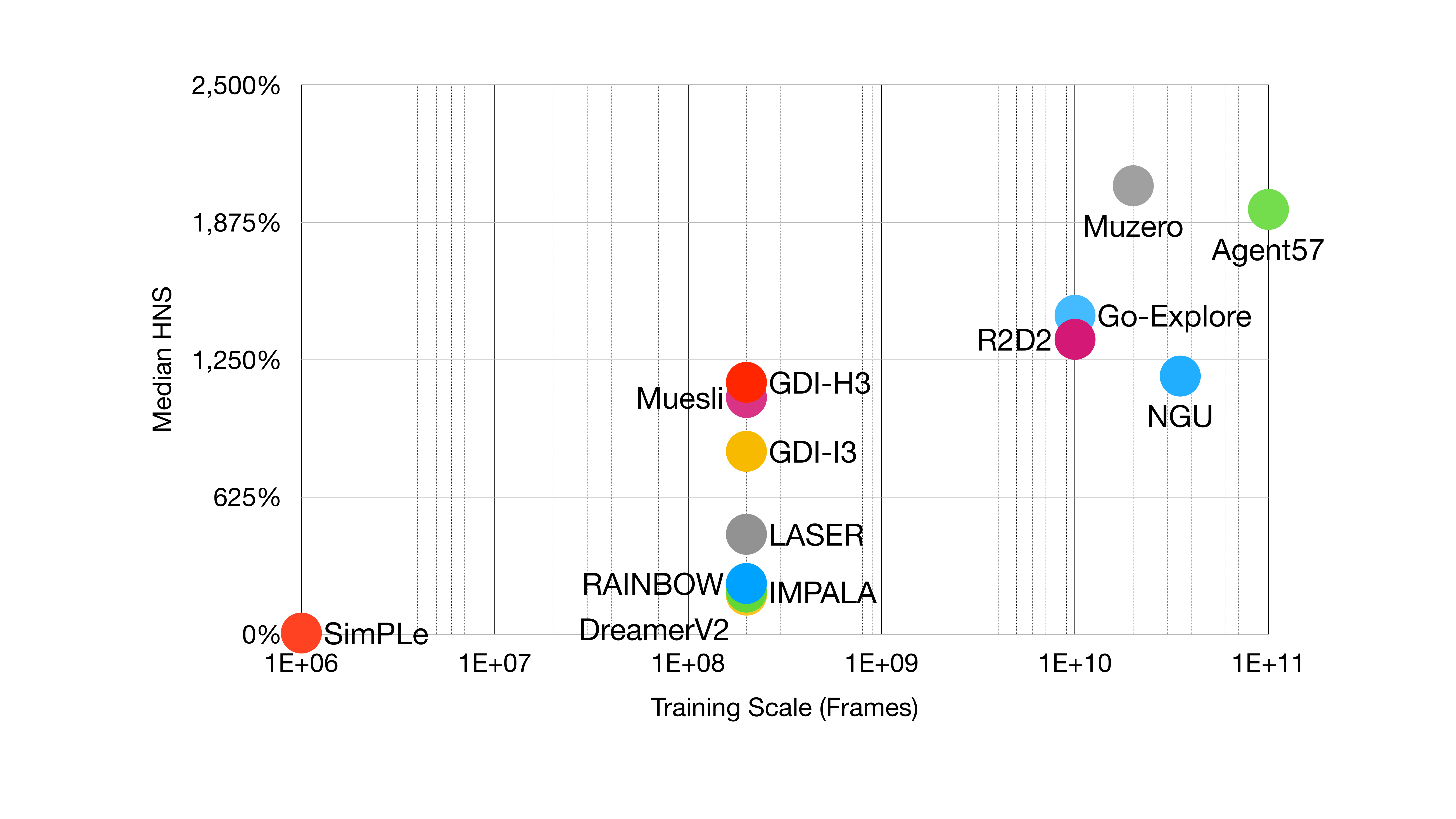}
	}
	\centering
	\caption{SOTA algorithms of Atari 57 games on mean and median HNS (\%) and corresponding training scale.}
	\label{fig: scale mean HNS time}
\end{figure*}

\subsection{RL Benchmarks on HWRNS}
\label{app: RL Benchmarks on HWRNS}

We report several milestones of Atari benchmarks on Human World Records Normalized Score (HWRNS), including DQN \citep{dqn}, RAINBOW \citep{rainbow}, IMPALA \citep{impala}, LASER \citep{laser}, R2D2 \citep{r2d2}, NGU \citep{ngu}, Agent57 \citep{agent57}, Go-Explore \citep{goexplore}, MuZero \citep{muzero}, DreamerV2 \citep{dreamerv2}, SimPLe \citep{modelbasedatari} and Musile \citep{muesli}. We summary mean HWRNS and median HWRNS of these algorithms marked with their game time (year), learning efficiency and training scale in Fig \ref{fig: year mean HWRNS time},  \ref{fig: efficiency mean HWRNS time} and \ref{fig: scale mean HWRNS time}.

\begin{figure*}[!t]
    \centering
	\subfigure{
		\includegraphics[width=0.45\textwidth]{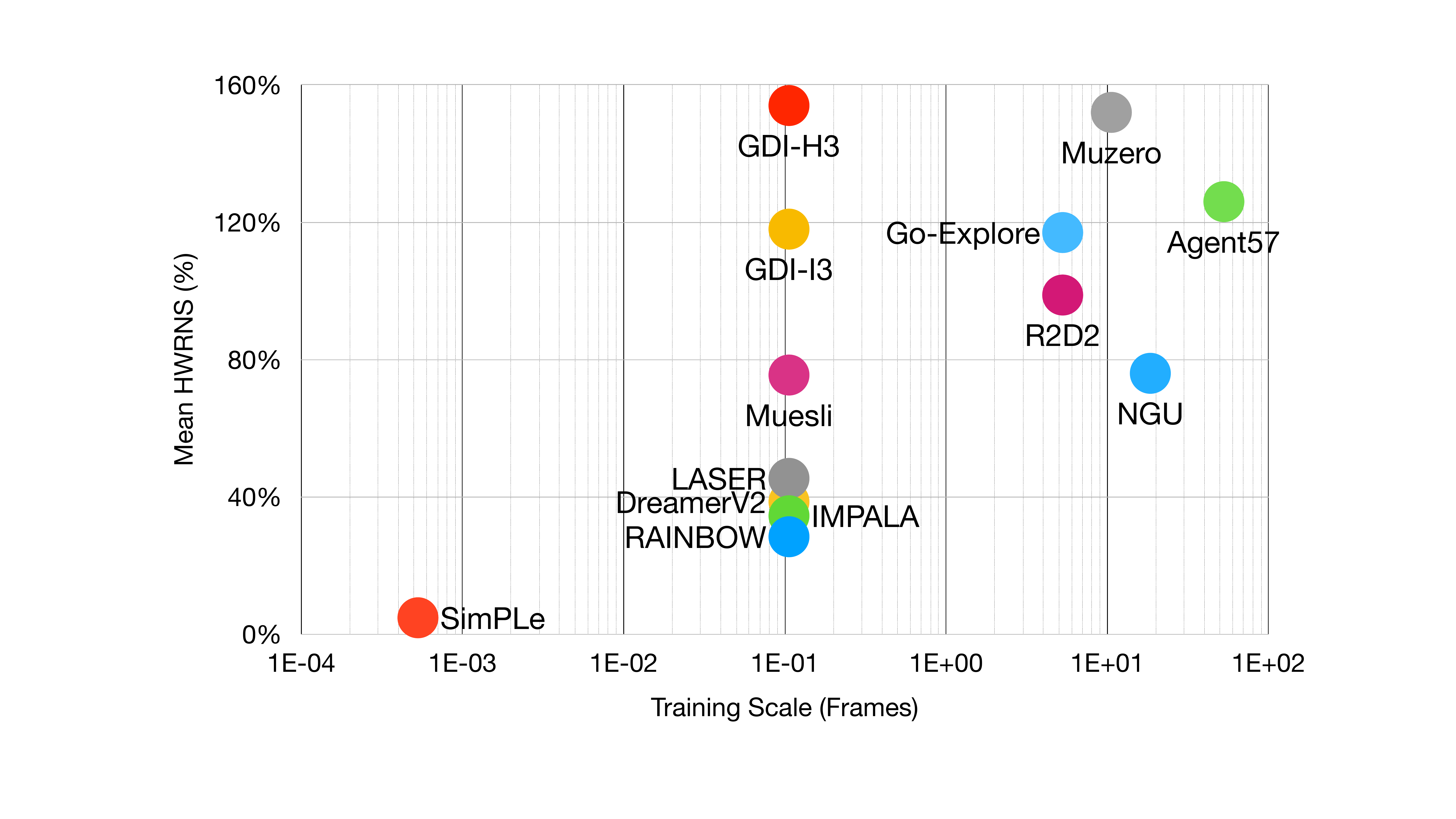}
	}
	\subfigure{
		\includegraphics[width=0.45\textwidth]{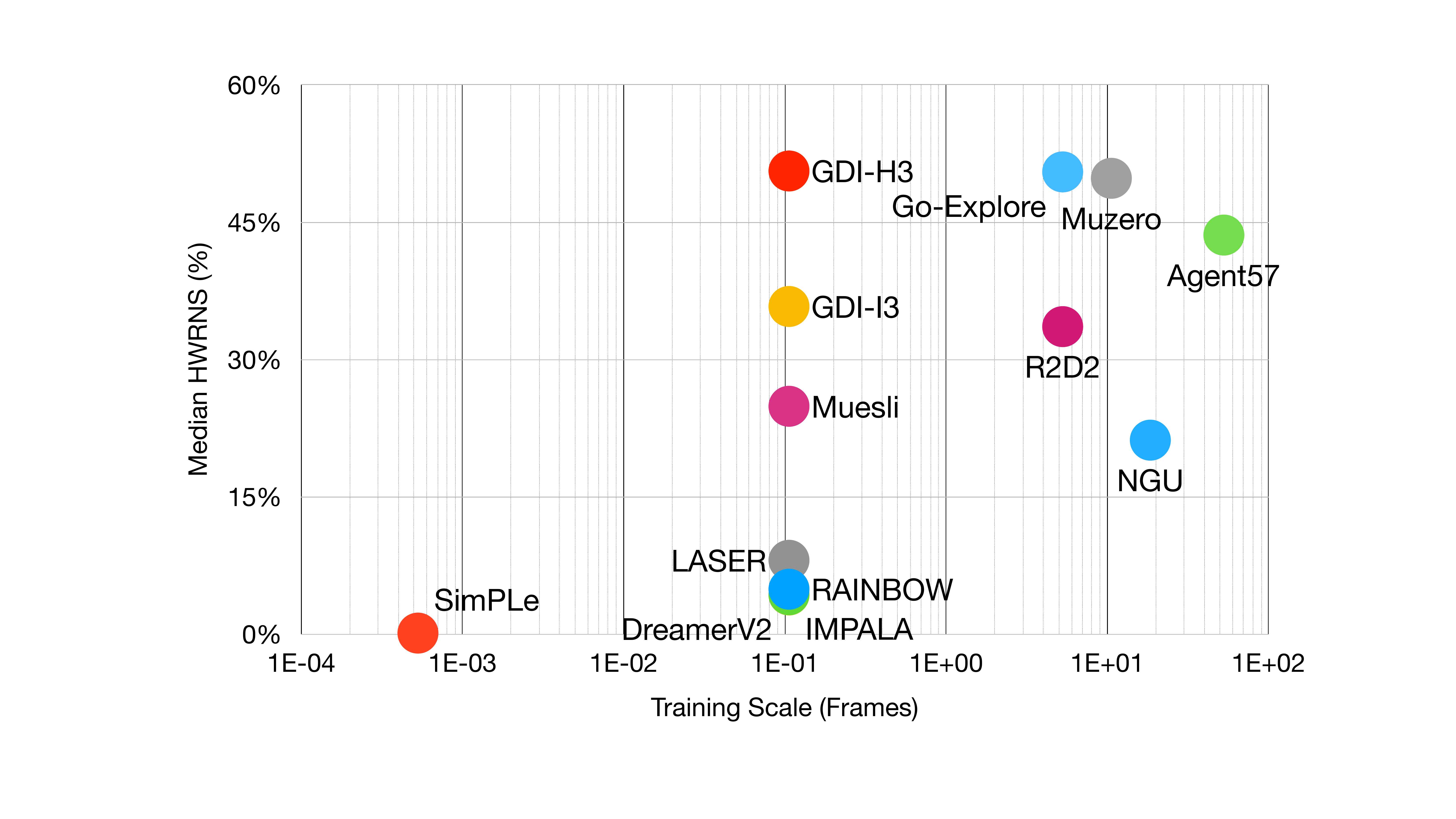}
	}
	\centering
	\caption{SOTA algorithms of Atari 57 games on mean and median HWRNS (\%) and corresponding game time (year).}
	\label{fig: year mean HWRNS time}
\end{figure*}

\begin{figure*}[!t]
    \centering
	\subfigure{
		\includegraphics[width=0.45\textwidth]{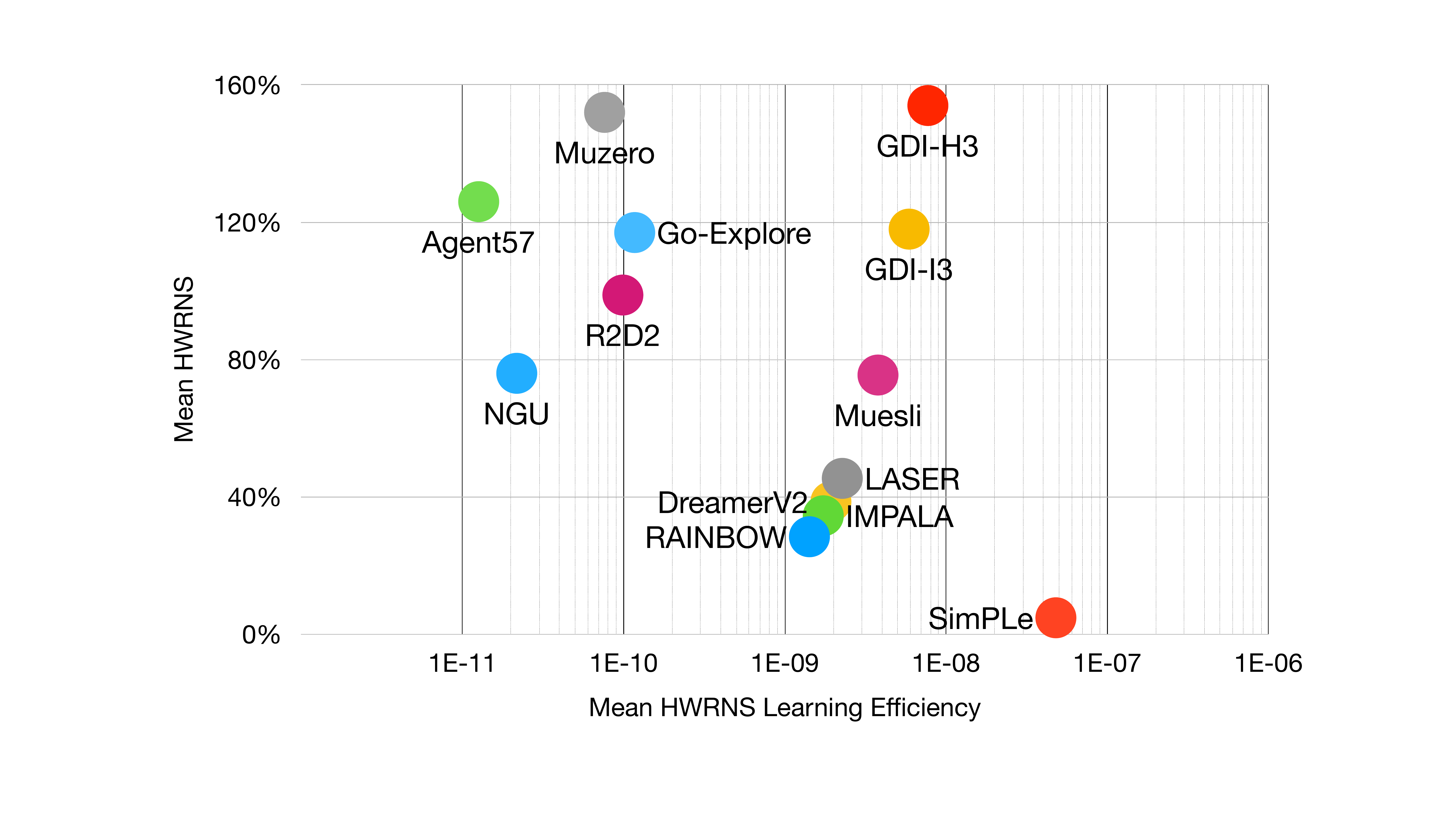}
	}
	\subfigure{
		\includegraphics[width=0.45\textwidth]{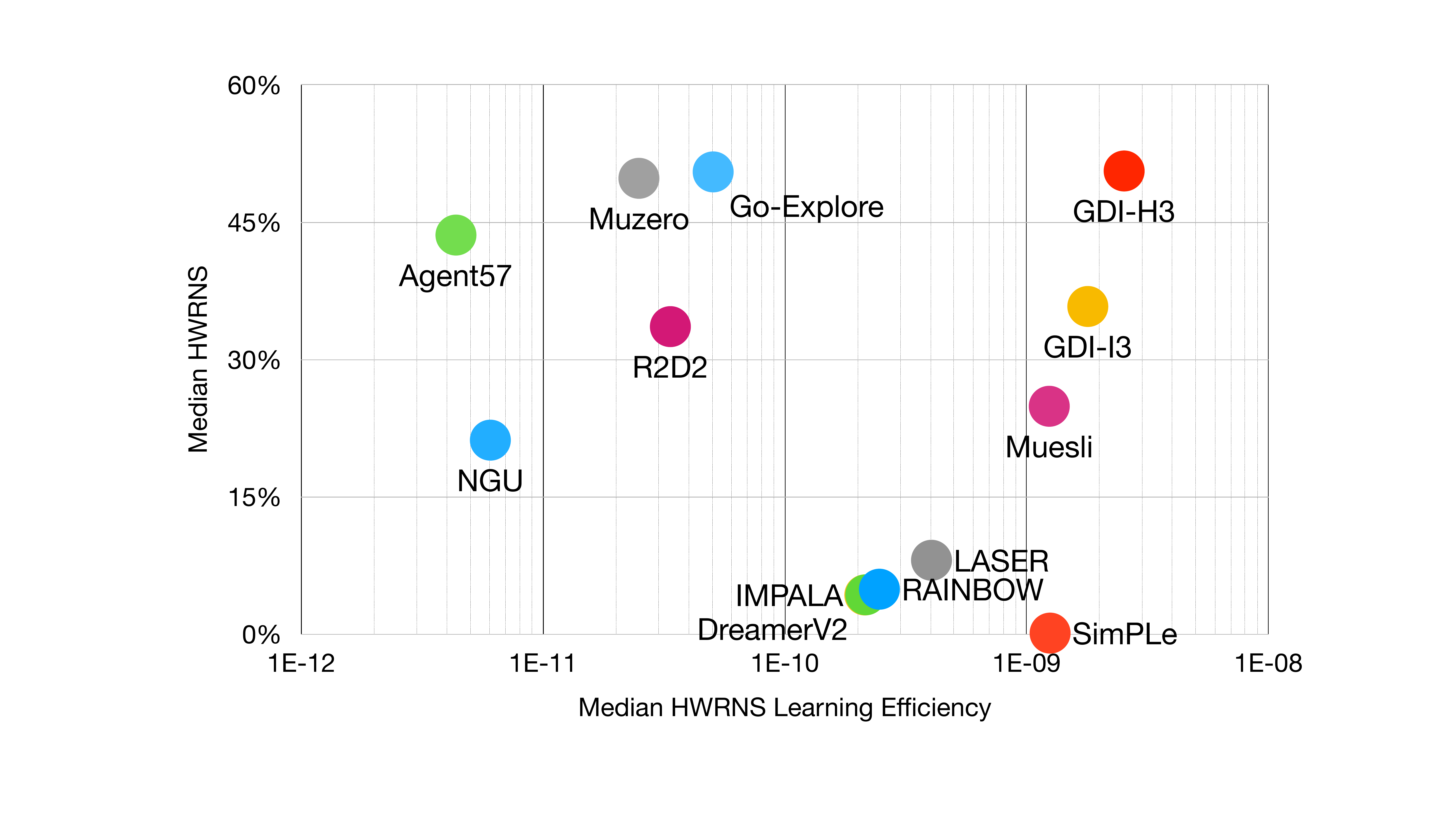}
	}
	\centering
	\caption{SOTA algorithms of Atari 57 games on mean and median HWRNS (\%) and corresponding learning efficiency calculated by $\frac{\text{MEAN HWRNS/MEDIAN HWRNS}}{\text{TRAINING FRAMES}}$.}
	\label{fig: efficiency mean HWRNS time}
\end{figure*}

\begin{figure*}[!t]
    \centering
	\subfigure{
		\includegraphics[width=0.45\textwidth]{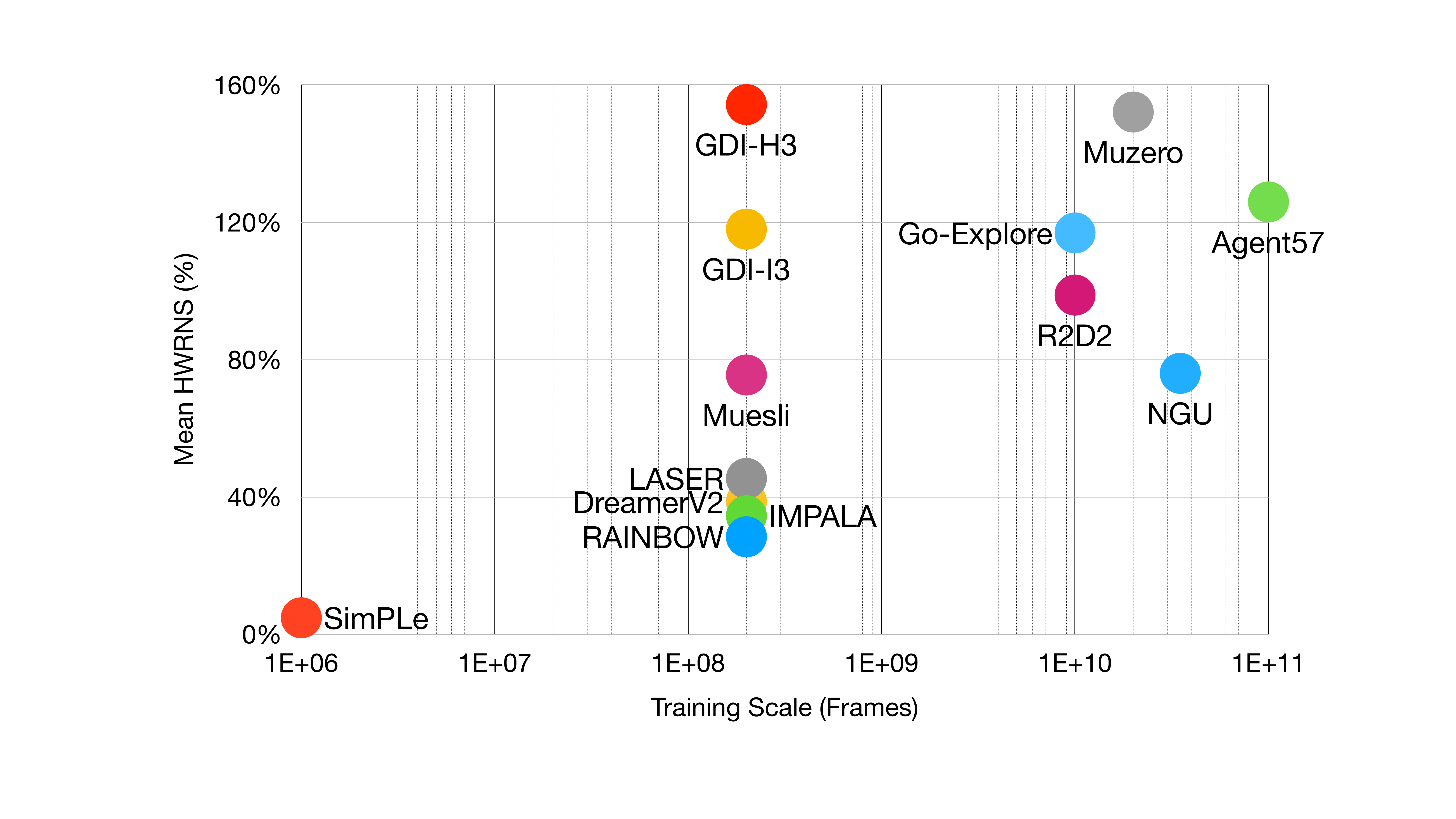}
	}
	\subfigure{
		\includegraphics[width=0.45\textwidth]{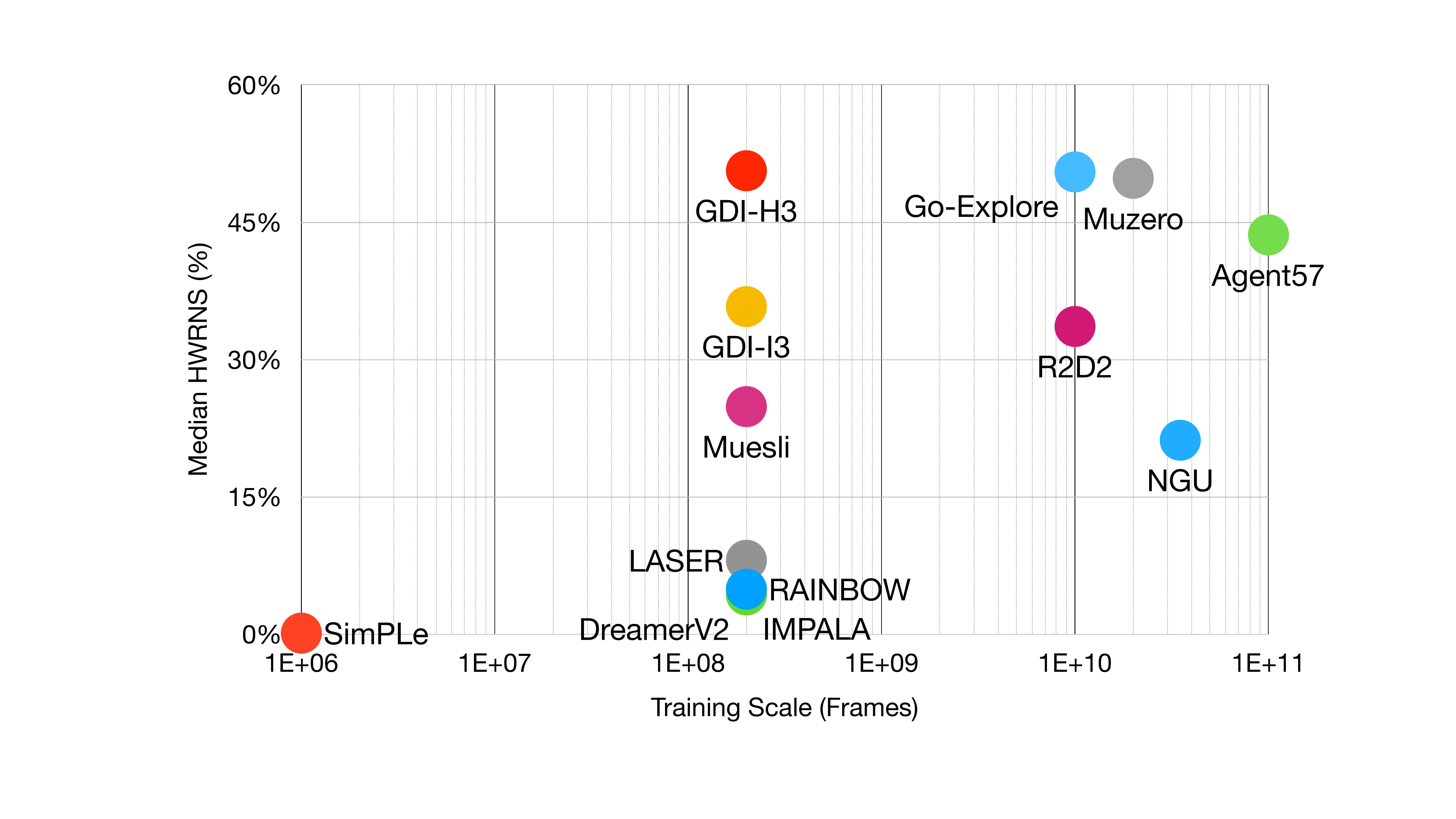}
	}
	\centering
	\caption{SOTA algorithms of Atari 57 games on mean and median HWRNS (\%) and corresponding training scale.}
	\label{fig: scale mean HWRNS time}
\end{figure*}

\subsection{RL Benchmarks on SABER}
\label{app: RL Benchmarks on SABER}

We report several milestones of Atari benchmarks on Standardized Atari BEnchmark for RL (SABER), including DQN \citep{dqn}, RAINBOW \citep{rainbow}, IMPALA \citep{impala}, LASER \citep{laser}, R2D2 \citep{r2d2}, NGU \citep{ngu}, Agent57 \citep{agent57}, Go-Explore \citep{goexplore}, MuZero \citep{muzero}, DreamerV2 \citep{dreamerv2}, SimPLe \citep{modelbasedatari} and Musile \citep{muesli}. We summary mean SABER and median SABER of these algorithms marked with their game time (year), learning efficiency and training scale in Fig \ref{fig: year mean SABER time},  \ref{fig: efficiency mean SABER time} and \ref{fig: scale mean SABER time}.

\begin{figure*}[!t]
    \centering
	\subfigure{
		\includegraphics[width=0.45\textwidth]{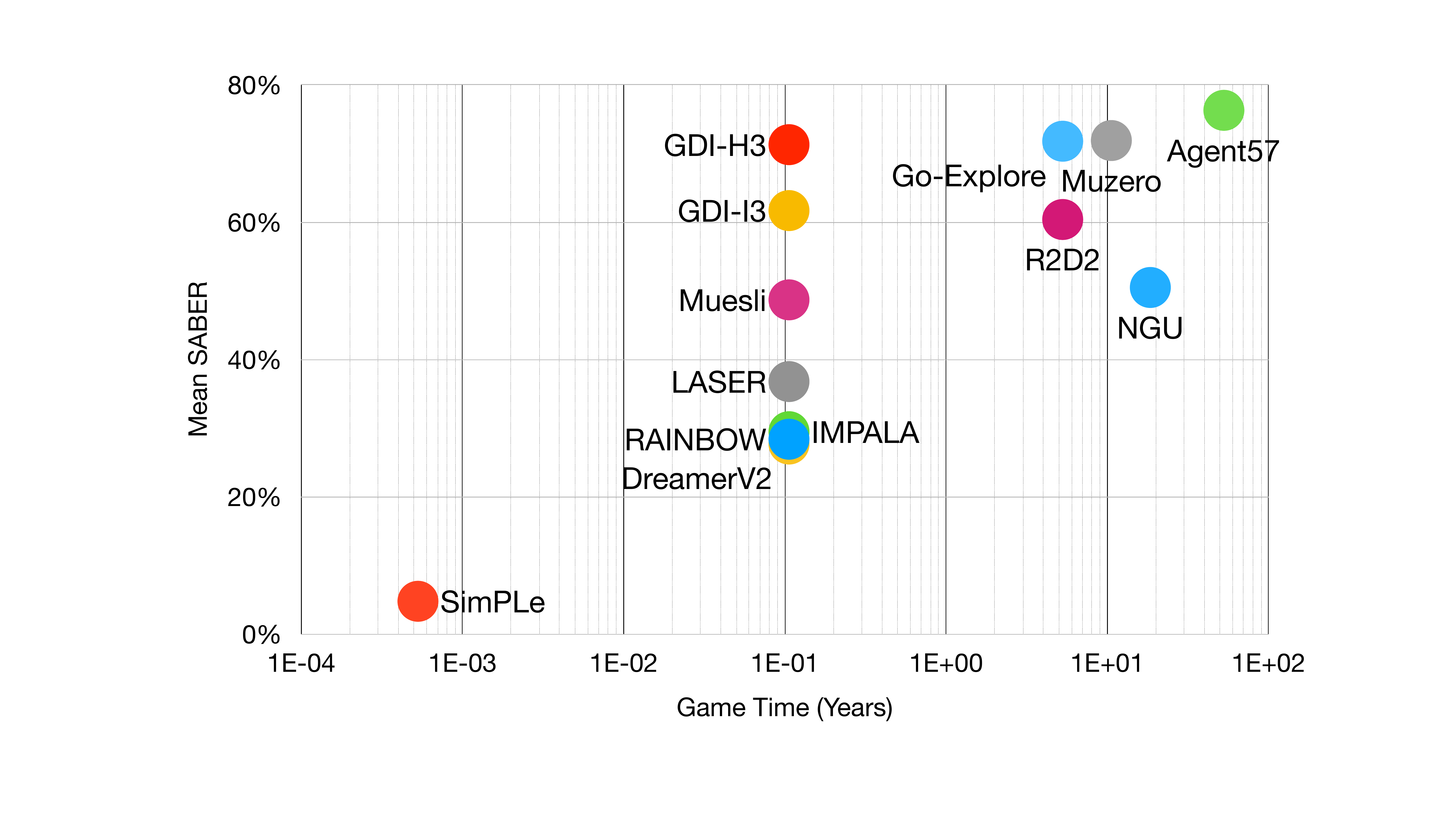}
	}
	\subfigure{
		\includegraphics[width=0.45\textwidth]{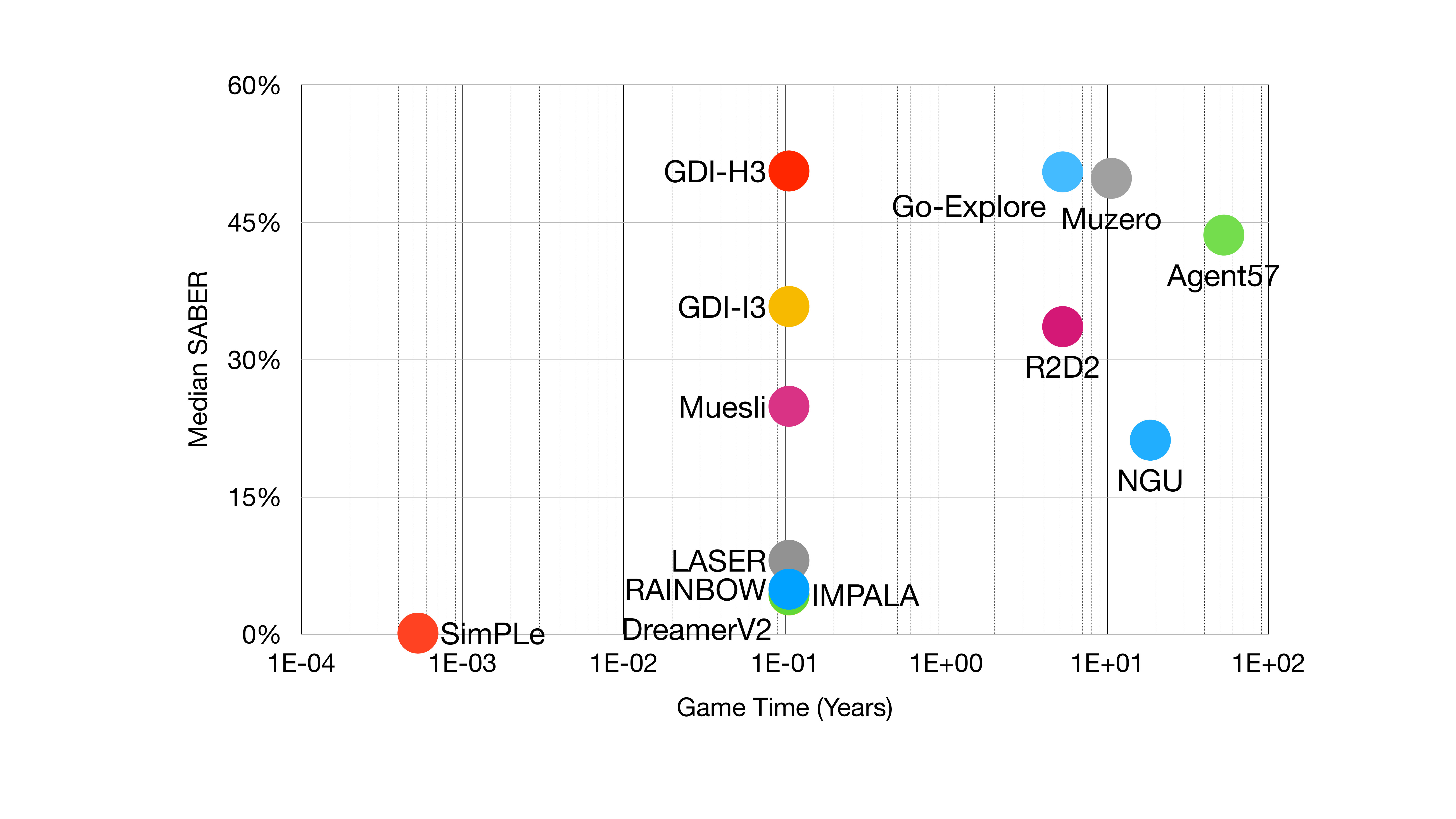}
	}
	\centering
	\caption{SOTA algorithms of Atari 57 games on mean and median SABER (\%) and corresponding game time (year).}
	\label{fig: year mean SABER time}
\end{figure*}

\begin{figure*}[!t]
    \centering
	\subfigure{
		\includegraphics[width=0.45\textwidth]{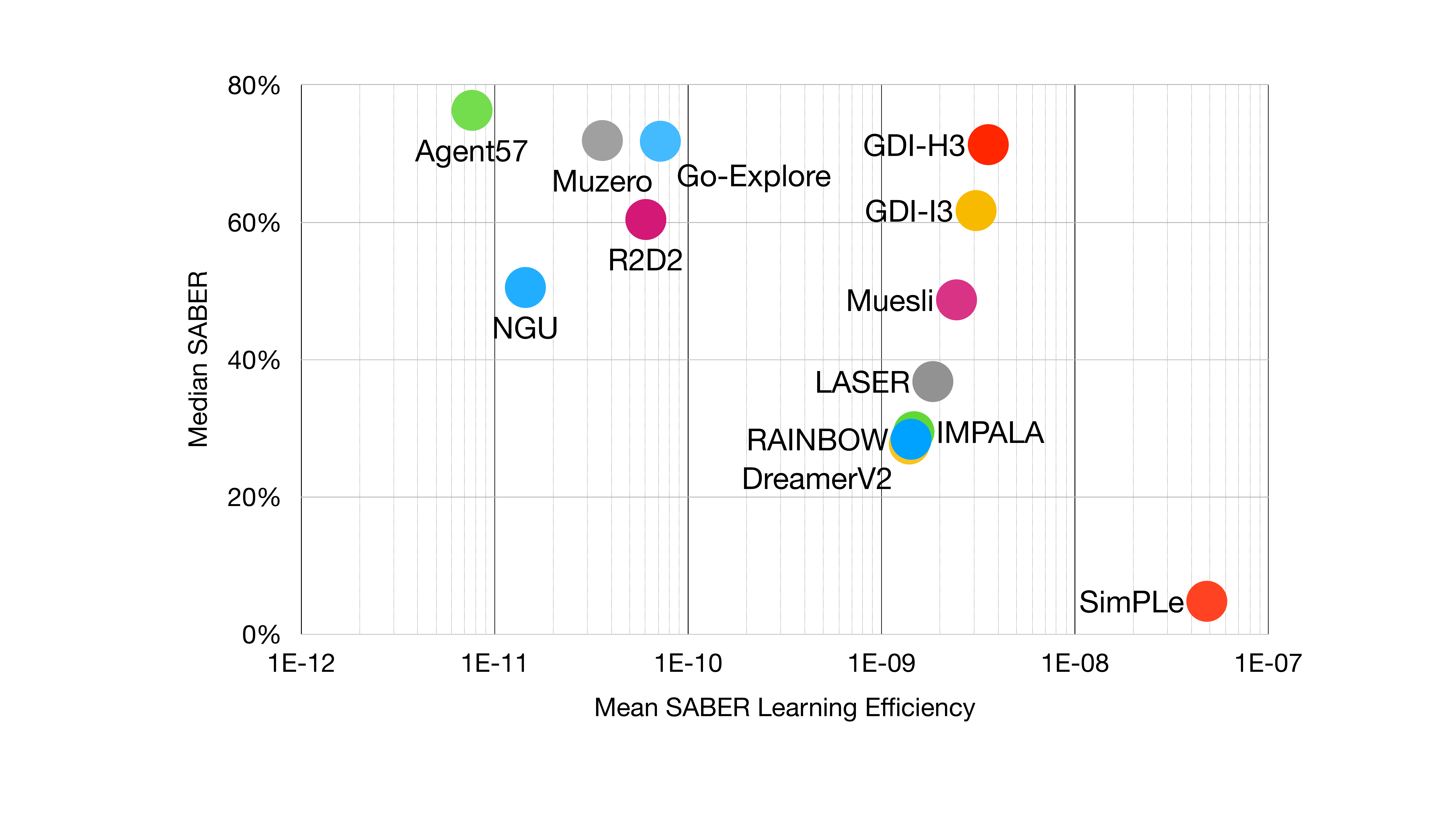}
	}
	\subfigure{
		\includegraphics[width=0.45\textwidth]{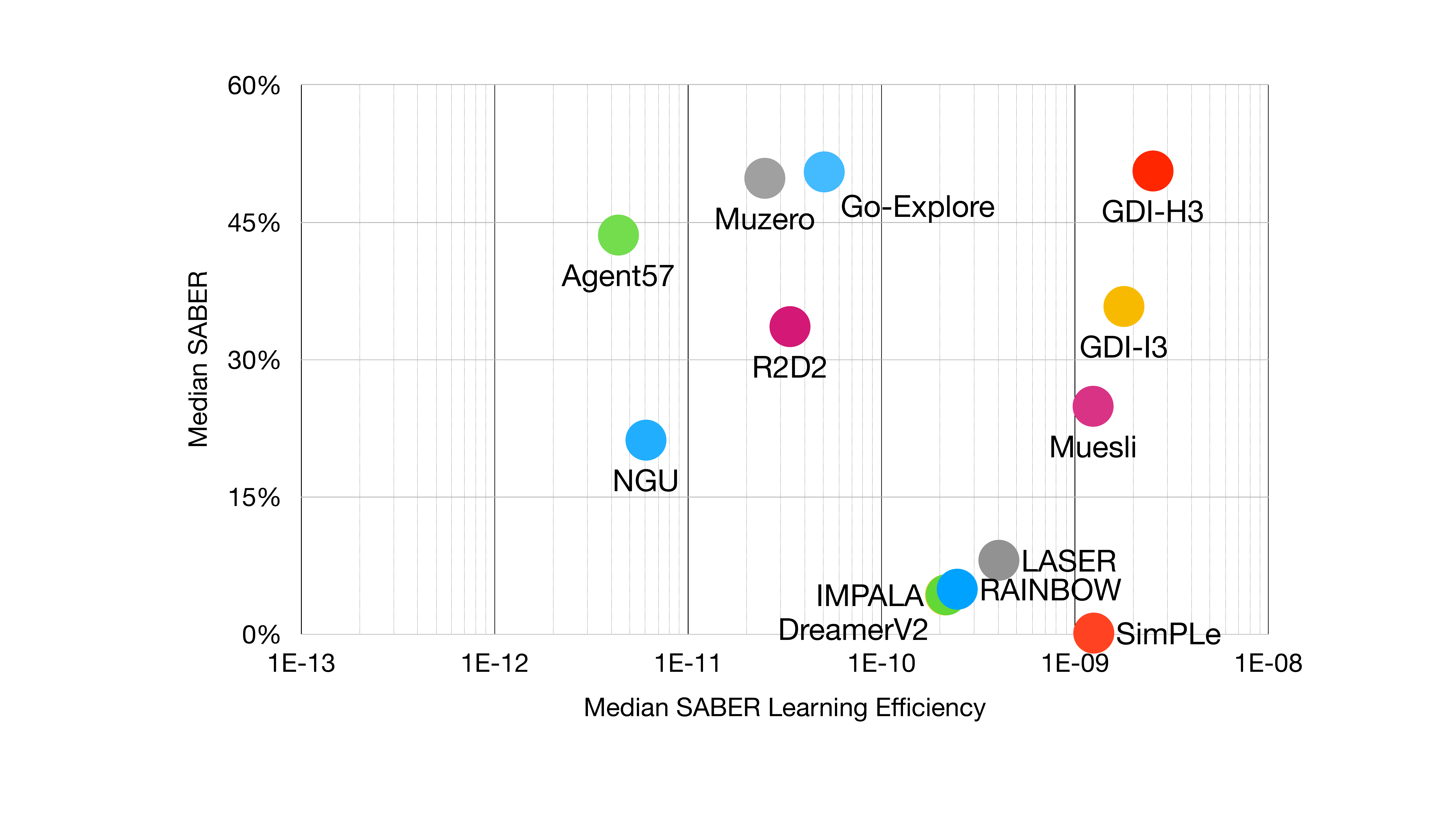}
	}
	\centering
	\caption{SOTA algorithms of Atari 57 games on mean and median SABER (\%) and corresponding learning efficiency calculated by $\frac{\text{MEAN SABER/MEDIAN SABER}}{\text{TRAINING FRAMES}}$.}
	\label{fig: efficiency mean SABER time}
\end{figure*}

\begin{figure*}[!t]
    \centering
	\subfigure{
		\includegraphics[width=0.45\textwidth]{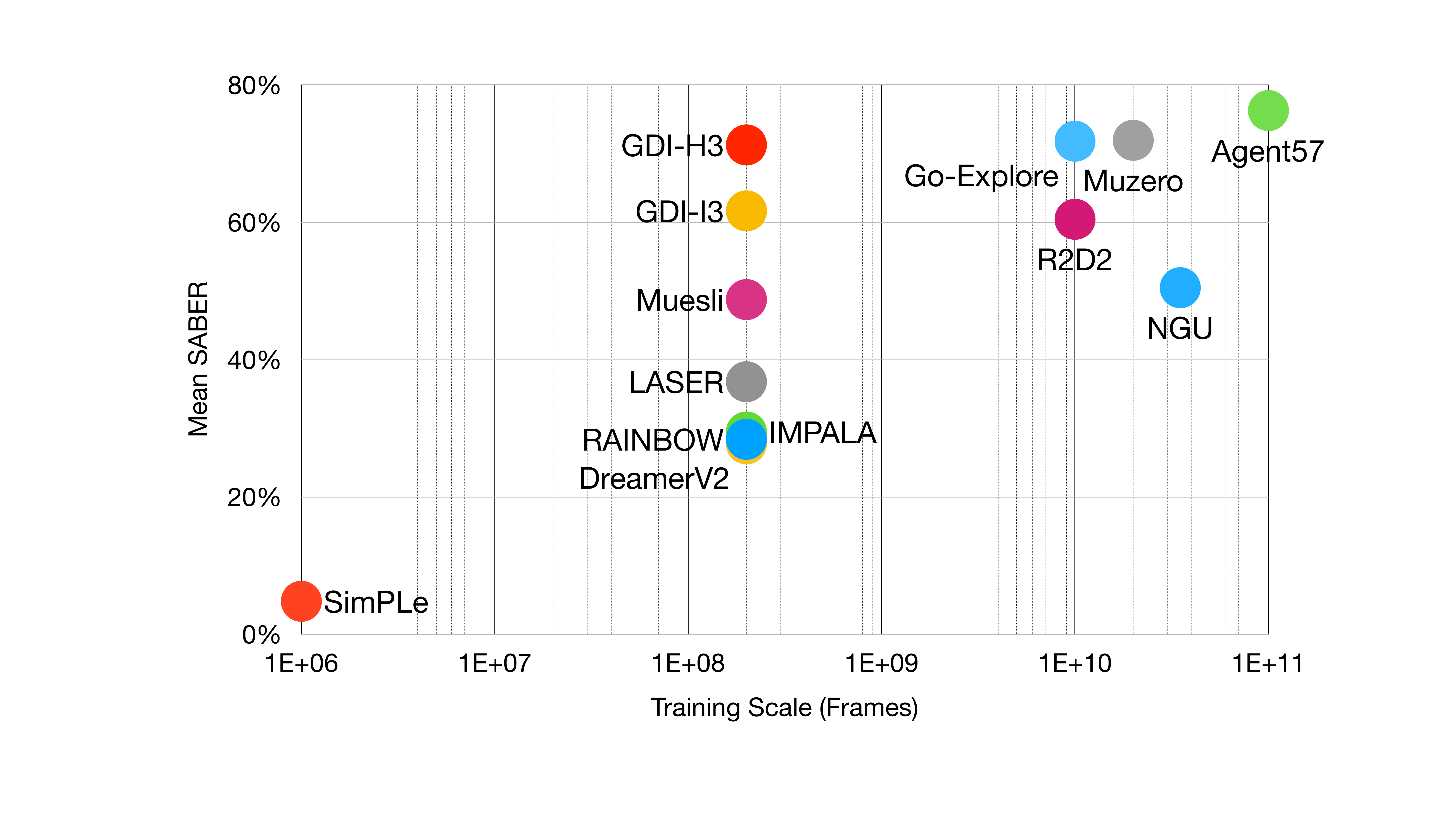}
	}
	\subfigure{
		\includegraphics[width=0.45\textwidth]{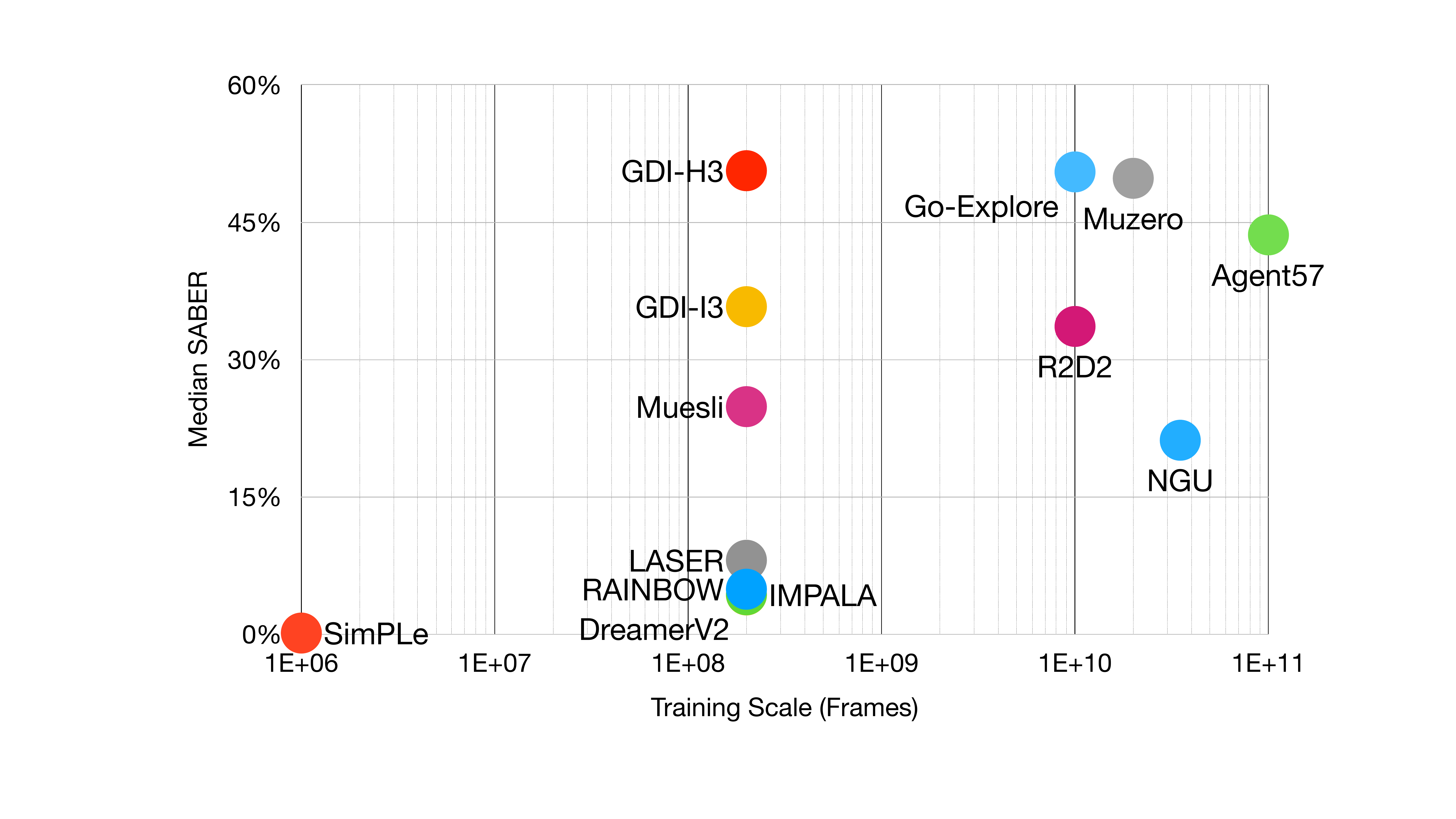}
	}
	\centering
	\caption{SOTA algorithms of Atari 57 games on mean and median SABER (\%) and corresponding training scale.}
	\label{fig: scale mean SABER time}
\end{figure*}

\subsection{RL Benchmarks on HWRB}
\label{app: RL Benchmarks on HWRB}

We report several milestones of Atari benchmarks on HWRB, including DQN \citep{dqn}, RAINBOW \citep{rainbow}, IMPALA \citep{impala}, LASER \citep{laser}, R2D2 \citep{r2d2}, NGU \citep{ngu}, Agent57 \citep{agent57}, Go-Explore \citep{goexplore}, MuZero \citep{muzero}, DreamerV2 \citep{dreamerv2}, SimPLe \citep{modelbasedatari} and Musile \citep{muesli}. We summary HWRB of these algorithms marked with their game time (year), learning efficiency and training scale in Fig \ref{fig: HWRB time}.

\begin{figure*}[!t]
    \centering
	\subfigure{
		\includegraphics[width=0.45\textwidth]{photo/tongjitu/game time year/HWRB.pdf}
	}
	\subfigure{
		\includegraphics[width=0.45\textwidth]{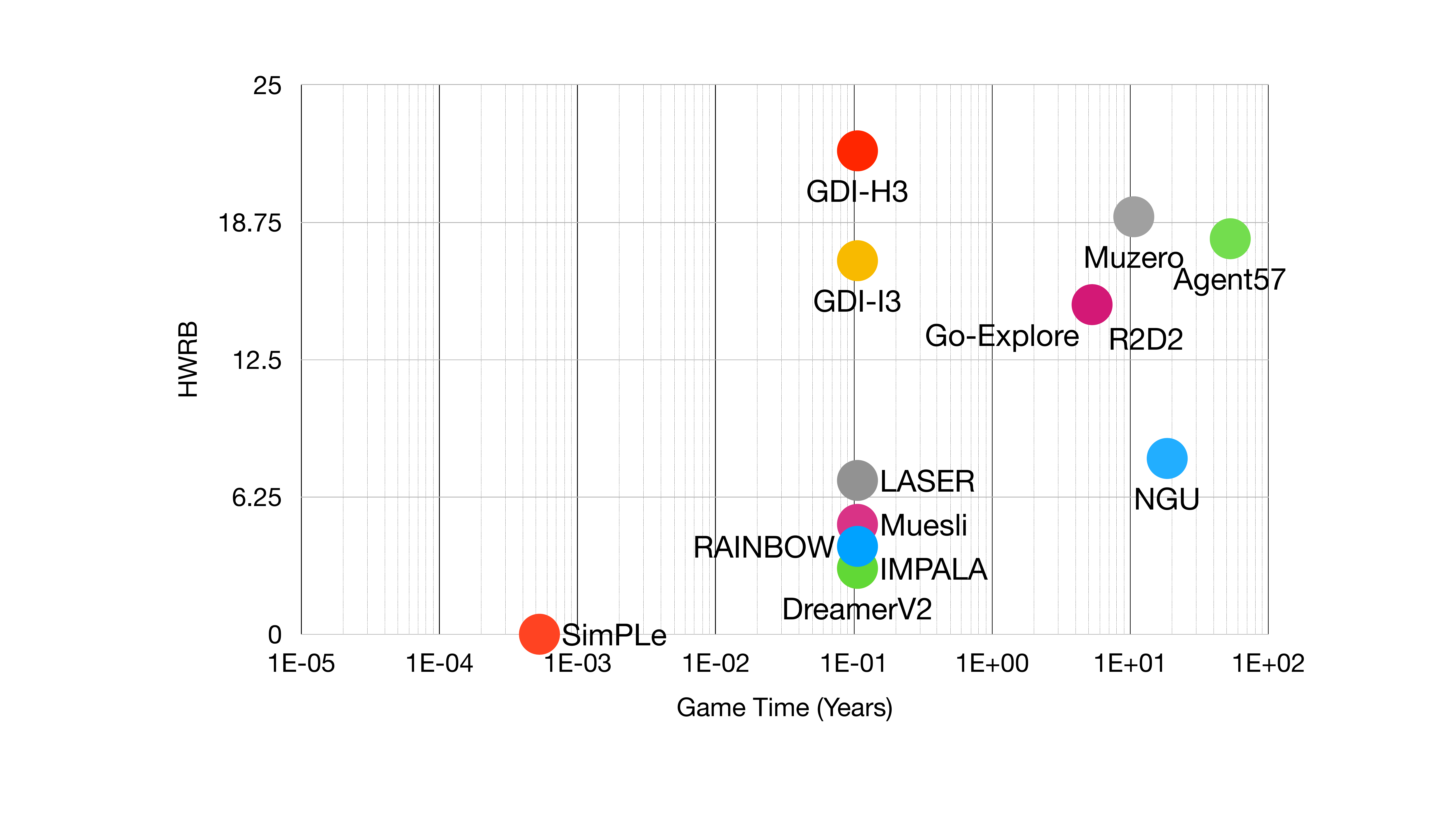}
	}
	\subfigure{
		\includegraphics[width=0.45\textwidth]{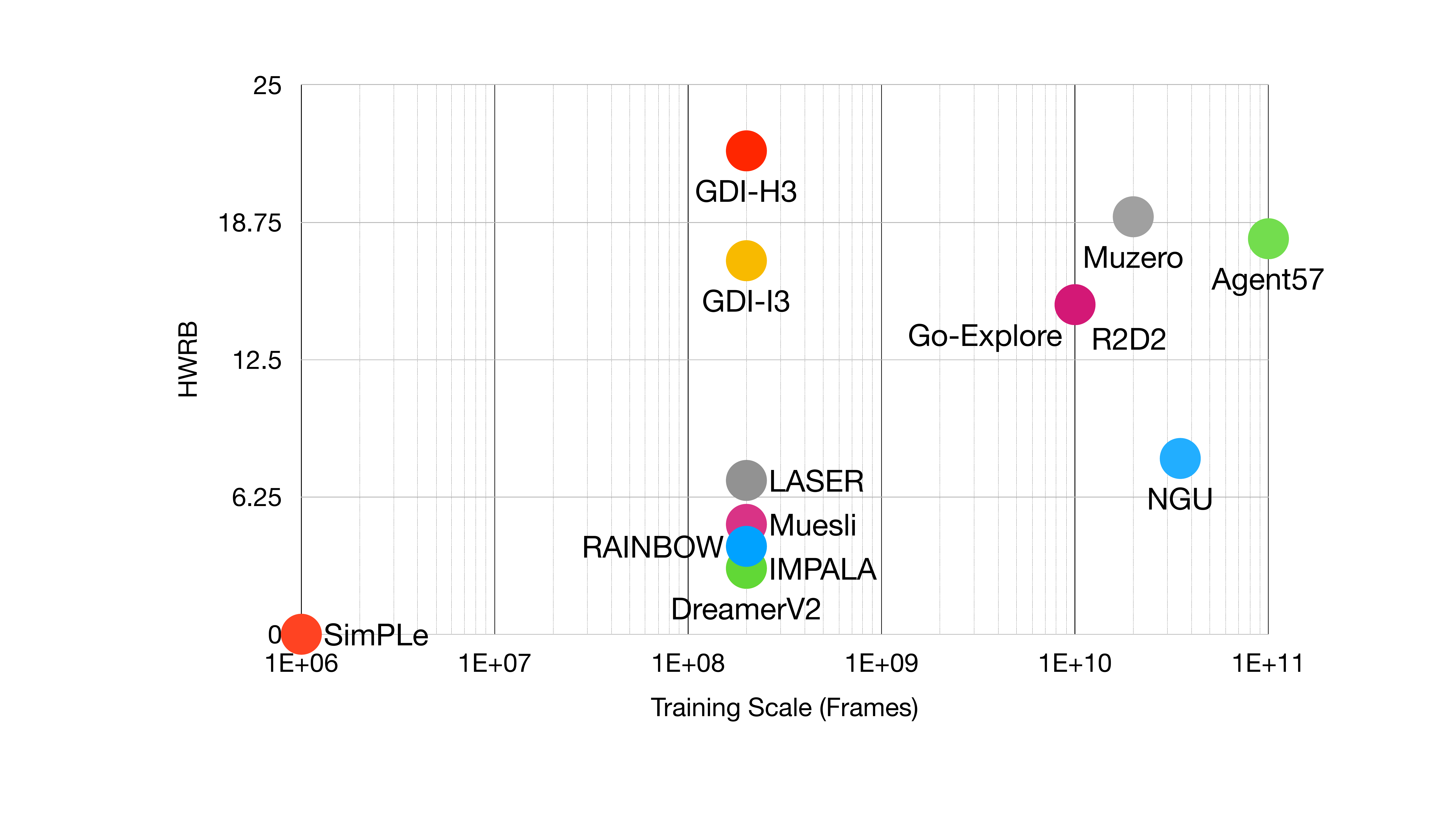}
	}
	\centering
	\caption{SOTA algorithms of Atari 57 games on HWRB.
	HWRB of SimPLe is 0, so it's not shown in the up-right figure.}
	\label{fig: HWRB time}
\end{figure*}

\clearpage

\section{Theoretical Proof}
\label{App: proof}

For a monotonic sequence of numbers which satisfies $a = x_0 < x_1 < \dots < x_n < b$,
we call it a split of interval $[a, b]$.

\begin{Lemma}[Discretized Upper Triangular Transport Inequality for Increasing Functions in $\mathbf{R}^1$]
Assume $\mu$ is a continuous probability measure supported on $[0, 1]$. 
Let $0 = x_0 < x_1 < \dots < x_n < 1$ to be any split of $[0, 1]$.
Define $\Tilde{\mu}(x_i) = \mu([x_i, x_{i+1}))$. 
Define 
$$\Tilde{\beta}(x_i) = \Tilde{\mu}(x_i) \exp(x_i) / Z,\  Z = \sum_{i} \Tilde{\mu}(x_i) \exp (x_i).$$ 
Then there exists a probability measure $\gamma: \{x_i\}_{i=0,\dots, n} \times \{x_i\}_{i=0,\dots, n} \rightarrow [0, 1]$, s.t. 
\begin{equation}
    \left\{
    \begin{aligned}
        &\sum_j \gamma (x_i, y_j) = \Tilde{\mu} (x_i),& &\ i = 0, \dots, n; \\
        &\sum_i \gamma (x_i, y_j) = \Tilde{\beta} (y_j),& &\ j = 0, \dots, n; \\
        &\gamma (x_i, y_j) = 0,& &\ i > j. \\
    \end{aligned}
    \right.
\label{eq:UTTC_R1}
\end{equation}
Then for any monotonic increasing function $f:\{x_i\}_{i=0,\dots, n} \rightarrow $ \textbf{R},
we have
$$
\textbf{E}_{\Tilde{\mu}} [f] \leq \textbf{E}_{\Tilde{\beta}} [f].
$$
\label{lemma:dct_inc_R1}
\end{Lemma}

\begin{proof}[Proof of Lemma \ref{lemma:dct_inc_R1}]

For any couple of measures $(\mu, \beta)$, we say the couple satisfies Upper Triangular Transport Condition (UTTC), if there exists $\gamma$ s.t. \eqref{eq:UTTC_R1} holds.

Given $0 = x_0 < x_1 < \dots < x_n < 1$, we prove the existence of $\gamma$ by induction.

Define 
\begin{equation*}
    \Tilde{\mu}_m(x_i) = 
\left\{
\begin{aligned}
    &\mu ([x_i, x_{i+1})),& &i < m, \\
    &\mu ([x_i, 1)),& &i = m, \\
    & 0,& &i > m.
\end{aligned}
\right.
\end{equation*}

Define
$$
\Tilde{\beta}_m (x_i) = \Tilde{\mu}_m (x_i) \exp(x_i) / Z_m,
\  Z_m = \sum_{i} \Tilde{\mu}_m (x_i) \exp (x_i).$$ 

Noting if we prove that $(\Tilde{\mu}_m, \Tilde{\beta}_m)$ satisfies UTTC for $m=n$, it's equivalent to prove the existence of $\gamma$ in \eqref{eq:UTTC_R1}.

To clarify the proof, we use $x_i$ to represent the point for $\Tilde{\mu}$-axis in coupling and $y_j$ to represent the point for $\Tilde{\beta}$-axis, but they are actually identical, i.e. $x_i = y_j$ when $i = j$. 

When $m = 0$, it's obvious that $(\Tilde{\mu}_0, \Tilde{\beta}_0)$ satisfies UTTC, as
$$
\gamma_0 (x_i, y_j) = \left\{
\begin{aligned}
    &1, \ i = 0, j = 0, \\
    &0, \ else.
\end{aligned}
\right.
$$

Assume UTTC holds for $m$, i.e. there exists $\gamma_m$ s.t. $(\Tilde{\mu}_m, \Tilde{\beta}_m)$ satisfies UTTC, we want to prove it also holds for $m+1$.

By definition of $\Tilde{\mu}_m$, we have
\begin{equation*}
    \left\{
    \begin{aligned}
        &\Tilde{\mu}_m (x_i) = \Tilde{\mu}_{m+1} (x_i),& &i < m, \\
        &\Tilde{\mu}_m (x_i) = \Tilde{\mu}_{m+1} (x_i) + \Tilde{\mu}_{m+1} (x_{i+1}),& &i=m, \\
        &\Tilde{\mu}_m (x_{m+1}) = \Tilde{\mu}_m (x_i) = \Tilde{\mu}_{m+1} (x_i) = 0,& &i>m+1. \\
    \end{aligned}
    \right.
\end{equation*}

By definition of $\Tilde{\beta}_m$, we have
\begin{equation*}
    \left\{
    \begin{aligned}
        &\Tilde{\beta}_m (x_i) = \Tilde{\beta}_{m+1} (x_i) \cdot \frac{Z_{m+1}}{Z_m},& &i < m, \\
        &\Tilde{\beta}_m (x_i) = \left( \Tilde{\beta}_{m+1} (x_i)  + \Tilde{\beta}_{m+1} (x_{i+1}) \exp(x_i - x_{i+1}) \right) \cdot \frac{Z_{m+1}}{Z_m},& &i=m, \\
        &\Tilde{\beta}_m (x_{m+1}) = \Tilde{\beta}_m (x_i) = \Tilde{\beta}_{m+1} (x_i) = 0,& &i>m+1. \\
    \end{aligned}
    \right.
\end{equation*}

Multiplying $\gamma_m$ by $\frac{Z_m}{Z_{m+1}}$, we get the following UTTC
\begin{equation*}
    \left\{
    \begin{aligned}
        &\sum_j \frac{Z_m}{Z_{m+1}} \gamma_m (x_i, y_j) = \frac{Z_m}{Z_{m+1}} \Tilde{\mu}_{m+1} (x_i),& &\ i < m; \\
        &\sum_j \frac{Z_m}{Z_{m+1}} \gamma_m (x_i, y_j) = \frac{Z_m}{Z_{m+1}} (\Tilde{\mu}_{m+1} (x_i) + \Tilde{\mu}_{m+1} (x_{i+1})),& &\ i = m; \\
        &\sum_j \frac{Z_m}{Z_{m+1}} \gamma_m (x_i, y_j) = 0,& &\ i = m+1; \\
        &\sum_j \frac{Z_m}{Z_{m+1}} \gamma_m (x_i, y_j) = \Tilde{\mu}_{m+1} (x_i) = 0,& &\ i > m+1; \\
        &\sum_i \frac{Z_m}{Z_{m+1}} \gamma_m (x_i, y_j) = \Tilde{\beta}_{m+1} (y_j),& &\ j < m; \\
        &\sum_i \frac{Z_m}{Z_{m+1}} \gamma_m (x_i, y_j) = \Tilde{\beta}_{m+1} (y_i)  + \Tilde{\beta}_{m+1} (y_{j+1}) \exp(y_j - y_{j+1}),& &\ j = m; \\
        &\sum_i \frac{Z_m}{Z_{m+1}} \gamma_m (x_i, y_j) = 0,& &\ j = m+1; \\
        &\sum_i \frac{Z_m}{Z_{m+1}} \gamma_m (x_i, y_j) = \Tilde{\beta}_{m+1} (y_j) = 0,& &\ j > m+1; \\
        &\frac{Z_m}{Z_{m+1}} \gamma_m (x_i, y_j) = 0,& &\ i > j. \\
    \end{aligned}
    \right.
\end{equation*}

By definition of $Z_m$,
\begin{equation}
    Z_{m+1} - Z_m = \Tilde{\mu}_{m+1} (x_{m+1}) (\exp(x_{m+1}) - \exp(x_m)) > 0,
\label{eq:Zm_inc}
\end{equation}
so we have $\frac{Z_m}{Z_{m+1}} \Tilde{\mu}_{m+1} (x_i) < \Tilde{\mu}_{m+1} (x_i)$.

Noticing that $\Tilde{\beta}_{m+1} (y_{i+1}) \exp(y_i - y_{i+1}) <  \Tilde{\beta}_{m+1} (y_{i+1})$ and $\frac{Z_m}{Z_{m+1}} \Tilde{\mu}_{m+1} (x_i) < \Tilde{\mu}_{m+1} (x_i)$, we decompose the measure of $\frac{Z_m}{Z_{m+1}} \gamma_m$ at $(x_i, y_m)$ to $(x_i, y_m), (x_i, y_{m+1})$ for $i = 0, \dots, m-1$, and complement a positive measure at $(x_i, y_{m+1})$ to make up the difference between $\frac{Z_m}{Z_{m+1}} \Tilde{\mu}_{m+1} (x_i)$ and $\Tilde{\mu}_{m+1} (x_i)$.
For $i=m$, we decompose the measure at $(x_m, y_m)$ to $(x_m, y_m), (x_m, y_{m+1}), (x_{m+1}, y_{m+1})$ and also complement a proper positive measure.

Now we define $\gamma_{m+1}$ by
\begin{equation*}
    \left\{
    \begin{aligned}
        &\gamma_{m+1} (x_i, y_j) = \frac{Z_m}{Z_{m+1}} \gamma_m (x_i, y_j),\qquad\qquad\qquad\qquad\quad i < m\ and\  j < m, \\
        &\gamma_{m+1} (x_i, y_j) = \left( \frac{Z_m}{Z_{m+1}} \gamma_m (x_i, y_j) + \frac{Z_{m+1} - Z_m}{Z_{m+1}} \Tilde{\mu}_{m+1} (x_i) \right) \\
        &\ \quad\qquad\qquad\quad \cdot \frac{\Tilde{\beta}_{m+1} (y_{j})}{\Tilde{\beta}_{m+1} (y_j)  + \Tilde{\beta}_{m+1} (y_{j+1})},  \qquad\qquad\quad i<m\ and\ j=m, \\
        &\gamma_{m+1} (x_i, y_j) = \left( \frac{Z_m}{Z_{m+1}} \gamma_m (x_i, y_j) + \frac{Z_{m+1} - Z_m}{Z_{m+1}} \Tilde{\mu}_{m+1} (x_i) \right) \\
        &\ \quad\qquad\qquad\quad \cdot \frac{\Tilde{\beta}_{m+1} (y_{j+1})}{\Tilde{\beta}_{m+1} (y_j)  + \Tilde{\beta}_{m+1} (y_{j+1})},\quad\qquad\  i<m\ and\ j=m+1, \\
        &\gamma_{m+1} (x_i, y_j) = 0,\quad\qquad\qquad\qquad\qquad i>j\ or\ i>m+1\ or\ j>m+1, \\
        &\gamma_{m+1} (x_m, y_m) = u, \\
        &\gamma_{m+1} (x_{m}, y_{m+1}) = v, \\
        &\gamma_{m+1} (x_{m+1}, y_{m+1}) = w, \\
    \end{aligned}
    \right.
\end{equation*}
where we assume $u, v, w$ to be the solution of the following equations
\begin{equation}
    \left\{
    \begin{aligned}
        &u + v + w = \Tilde{\mu}_{m+1} (x_m) + \Tilde{\mu}_{m+1} (x_{m+1}), \\
        &\frac{w}{u+v} = \frac{\Tilde{\mu}_{m+1} (x_{m+1})}{\Tilde{\mu}_{m+1} (x_m)}, \\
        &\frac{v+w}{u} = \frac{\Tilde{\beta}_{m+1} (x_{m+1})}{\Tilde{\beta}_{m+1} (x_m)}, \\
        &u, v, w \geq 0. \\
    \end{aligned}
    \right.
\label{eq:uvw_R1}
\end{equation}

It's obvious that 
\begin{equation*}
    \left\{
    \begin{aligned}
        &\sum_j \gamma_{m+1} (x_i, y_j) = \Tilde{\mu}_{m+1} (x_i) = 0,& &\ i>m+1, \\
        &\sum_i \gamma_{m+1} (x_i, y_j) = \Tilde{\beta}_{m+1} (y_j) = 0,& &\ j>m+1,\\
        &\gamma (x_i, y_j) = 0,& &\ i > j. \\
    \end{aligned}
    \right.
\end{equation*}

For $j < m$, since $\sum_i \frac{Z_m}{Z_{m+1}} \gamma_m (x_i, y_j) = \Tilde{\beta}_{m+1} (y_j)$, we have 
$$\sum_i \gamma_{m+1} (x_i, y_j) = \Tilde{\beta}_{m+1} (y_j), \ j < m.$$

For $i < m$, since $\sum_j \frac{Z_m}{Z_{m+1}} \gamma_m (x_i, y_j) = \frac{Z_m}{Z_{m+1}} \Tilde{\mu}_{m+1} (x_i) < \Tilde{\mu}_{m+1} (x_i)$, we add $\frac{Z_{m+1} - Z_m}{Z_{m+1}} \Tilde{\mu}_{m+1} (x_i) \frac{\Tilde{\beta}_{m+1} (y_{m})}{\Tilde{\beta}_{m+1} (y_m)  + \Tilde{\beta}_{m+1} (y_{m+1})}$, $\frac{Z_{m+1} - Z_m}{Z_{m+1}} \Tilde{\mu}_{m+1} (x_i) \frac{\Tilde{\beta}_{m+1} (y_{m+1})}{\Tilde{\beta}_{m+1} (y_m)  + \Tilde{\beta}_{m+1} (y_{m+1})}$ to $\gamma_{m+1} (x_i, y_{m})$, $\gamma_{m+1} (x_i, y_{m+1})$, respectively.
So we have 
$$\sum_j \gamma_{m+1} (x_i, y_j) = \Tilde{\mu}_{m+1} (x_i), \ i < m.$$

For $i = m, m+1$, since assumption \eqref{eq:uvw_R1} holds, we have $u + v + w = \Tilde{\mu}_{m+1} (x_m) + \Tilde{\mu}_{m+1} (x_{m+1}), \frac{w}{u+v} = \frac{\Tilde{\mu}_{m+1} (x_{m+1})}{\Tilde{\mu}_{m+1} (x_m)}$, it's obvious that $u + v = \Tilde{\mu}_{m+1} (x_m), w = \Tilde{\mu}_{m+1} (x_{m+1})$, which is 
$$\sum_j \gamma_{m+1} (x_i, y_j) = \Tilde{\mu}_{m+1} (x_i), \ i = m, m+1.$$

For $j = m, m+1$, we firstly have
\begin{equation*}
    \begin{aligned}
        \sum_{j=m, m+1} \sum_i \gamma_{m+1} (x_i, y_j)
        &= \sum_{j} \sum_i \gamma_{m+1} (x_i, y_j) - \sum_{j \neq m, m+1} \sum_i \gamma_{m+1} (x_i, y_j) \\
        &=  \sum_{i} \sum_j \gamma_{m+1} (x_i, y_j) - \sum_{j \neq m, m+1} \Tilde{\beta}_{m+1}(y_j) \\
        &= \sum_{i} \Tilde{\mu}_{m+1}(x_i) - \sum_{j \neq m, m+1} \Tilde{\beta}_{m+1}(y_j) \\
        &= 1 - (1 - \Tilde{\beta}_{m+1}(y_m) - \Tilde{\beta}_{m+1}(y_{m+1})) \\
        &= \Tilde{\beta}_{m+1}(y_m) + \Tilde{\beta}_{m+1}(y_{m+1}). \\
    \end{aligned}
\end{equation*}
By definition of $\gamma_{m+1}$, we know $\frac{\gamma_{m+1} (x_i, y_m)}{\gamma_{m} (x_i, y_m)} = \frac{\Tilde{\beta}_{m+1} (x_{m+1})}{\Tilde{\beta}_{m+1} (x_m)}$ for $i < m$.
By assumption \eqref{eq:uvw_R1}, we know $\frac{v+w}{u} = \frac{\Tilde{\beta}_{m+1} (x_{m+1})}{\Tilde{\beta}_{m+1} (x_m)}$.
Combining three equations above together, we have
$$\sum_i \gamma_{m+1} (x_i, y_j) = \Tilde{\beta}_{m+1} (y_j), \ j = m, m+1.$$

Now we only need to prove assumption \eqref{eq:uvw_R1} holds.
With linear algebra, we solve \eqref{eq:uvw_R1} and have
\begin{equation*}
    \left\{
    \begin{aligned}
        &u = w \frac{1 + \frac{\Tilde{\mu}_{m+1} (x_{m+1})}{\Tilde{\mu}_{m+1} (x_m)}}{\frac{\Tilde{\mu}_{m+1} (x_{m+1})}{\Tilde{\mu}_{m+1} (x_m)} \left(1 + \frac{\Tilde{\beta}_{m+1} (x_{m+1})}{\Tilde{\beta}_{m+1} (x_m)}\right)}, \\
        &v = w \frac{\frac{\Tilde{\beta}_{m+1} (x_{m+1})}{\Tilde{\beta}_{m+1} (x_m)} - \frac{\Tilde{\mu}_{m+1} (x_{m+1})}{\Tilde{\mu}_{m+1} (x_m)}}{\frac{\Tilde{\mu}_{m+1} (x_{m+1})}{\Tilde{\mu}_{m+1} (x_m)} \left(1 + \frac{\Tilde{\beta}_{m+1} (x_{m+1})}{\Tilde{\beta}_{m+1} (x_m)}\right)}, \\
        &w = \frac{\left(\Tilde{\mu}_{m+1} (x_m) + \Tilde{\mu}_{m+1} (x_{m+1})\right)\frac{\Tilde{\mu}_{m+1} (x_{m+1})}{\Tilde{\mu}_{m+1} (x_m)} \left(1 + \frac{\Tilde{\beta}_{m+1} (x_{m+1})}{\Tilde{\beta}_{m+1} (x_m)}\right)}{\left(1 + \frac{\Tilde{\mu}_{m+1} (x_{m+1})}{\Tilde{\mu}_{m+1} (x_m)}\right) \left(1 + \frac{\Tilde{\beta}_{m+1} (x_{m+1})}{\Tilde{\beta}_{m+1} (x_m)}\right)}. \\
    \end{aligned}
    \right.
\end{equation*}

It's obvious that $u, w \geq 0$. 
$v \geq 0$ also holds, because
\begin{equation}
    \begin{aligned}
    \frac{\Tilde{\beta}_{m+1} (x_{m+1})}{\Tilde{\beta}_{m+1} (x_m)} - \frac{\Tilde{\mu}_{m+1} (x_{m+1})}{\Tilde{\mu}_{m+1} (x_m)} 
    &= \frac{\Tilde{\mu}_{m+1} (x_{m+1})\exp(x_{m+1})}{\Tilde{\mu}_{m+1} (x_m)\exp(x_{m})} - \frac{\Tilde{\mu}_{m+1} (x_{m+1})}{\Tilde{\mu}_{m+1} (x_m)} \\
    &= \frac{\Tilde{\mu}_{m+1} (x_{m+1})}{\Tilde{\mu}_{m+1} (x_m)} \left(\exp(x_{m+1} - x_m) - 1 \right) \geq 0.
\end{aligned} 
\label{eq:v_exist}
\end{equation}

So we can find a proper solution of assumption \eqref{eq:uvw_R1}.

So $\gamma_{m+1}$ defined above satisfies UTTC for $(\Tilde{\mu}_{m+1}, \Tilde{\beta}_{m+1})$.

By induction, for any $0 = x_0 < x_1 < \dots < x_n < 1$, there exists $\gamma$ s.t. UTTC \eqref{eq:UTTC_R1} holds for $(\Tilde{\mu}, \Tilde{\beta})$.

Then for any monotonic increasing function, since $\gamma (x_i, y_j) = 0$ when $i > j$, we know $\gamma (x_i, y_j) f(x_i) \leq \gamma (x_i, y_j) f(y_j)$.
Hence we have
\begin{equation*}
    \begin{aligned}
        \textbf{E}_{\Tilde{\mu}} [f] 
        &= \sum_i \Tilde{\mu} (x_i) f(x_i) 
        = \sum_i \sum_j \gamma (x_i, y_j) f(x_i) \\
        &\leq \sum_i \sum_j \gamma (x_i, y_j) f(y_j) \\
        &= \sum_j \sum_i \gamma (x_i, y_j) f(y_j) \\
        &= \sum_j \Tilde{\beta} (y_j) f(y_j) 
        = \textbf{E}_{\Tilde{\beta}} [f]. \\
    \end{aligned}
\end{equation*}

\end{proof}

\begin{Lemma}[Discretized Upper Triangular Transport Inequality for Co-Monotonic Functions in $\mathbf{R}^1$]
Assume $\mu$ is a continuous probability measure supported on $[0, 1]$.
Let $0 = x_0 < x_1 < \dots < x_n < 1$ to be any split of $[0, 1]$.
Let $f, g:\{x_i\}_{i=0,\dots, n} \rightarrow $ \textbf{R} to be two co-monotonic functions that satisfy
$$(f(x_i) - f(x_j)) \cdot (g(x_i) - g(x_j)) \geq 0, \ \forall \ i, j.$$
Define $\Tilde{\mu}(x_i) = \mu([x_i, x_{i+1}))$. 
Define 
$$\Tilde{\beta}(x_i) = \Tilde{\mu}(x_i) \exp(g(x_i)) / Z, \ Z = \sum_{i} \Tilde{\mu}(x_i) \exp (g(x_i)).$$
Then we have
$$
\textbf{E}_{\Tilde{\mu}} [f] \leq \textbf{E}_{\Tilde{\beta}} [f].
$$
\label{lemma:dct_R1}
\end{Lemma}

\begin{proof}[Proof of Lemma \ref{lemma:dct_R1}]

If the Upper Triangular Transport Condition (UTTC) holds for $(\Tilde{\mu}, \Tilde{\beta})$, i.e. there exists a probability measure $\gamma: \{x_i\}_{i=0,\dots, n} \times \{x_i\}_{i=0,\dots, n} \rightarrow [0, 1]$, s.t. 
\begin{equation*}
    \left\{
    \begin{aligned}
        &\sum_j \gamma (x_i, y_j) = \Tilde{\mu} (x_i),& &\ i = 0, \dots, n; \\
        &\sum_i \gamma (x_i, y_j) = \Tilde{\beta} (y_j),& &\ j = 0, \dots, n; \\
        &\gamma (x_i, y_j) = 0,& &\ g(x_i) > g(y_j), \\
    \end{aligned}
    \right.
\end{equation*}
then we finish the proof by
\begin{equation*}
    \begin{aligned}
        \textbf{E}_{\Tilde{\mu}} [f] 
        &= \sum_i \Tilde{\mu} (x_i) f(x_i) 
        = \sum_i \sum_j \gamma (x_i, y_j) f(x_i) \\
        &\leq \sum_i \sum_j \gamma (x_i, y_j) f(y_j) \\
        &= \sum_j \sum_i \gamma (x_i, y_j) f(y_j) \\
        &= \sum_j \Tilde{\beta} (y_j) f(y_j) 
        = \textbf{E}_{\Tilde{\beta}} [f], \\
    \end{aligned}
\end{equation*}
where $\gamma (x_i, y_j) f(x_i) \leq \gamma (x_i, y_j) f(y_j)$ is because of $\gamma (x_i, y_j) = 0, \ g(x_i) > g(y_j)$ and $(f(x_i) - f(x_j)) \cdot (g(x_i) - g(x_j)) \geq 0$.

Now we only need to prove UTTC holds for $(\Tilde{\mu}, \Tilde{\beta})$.

Given $0 = x_0 < x_1 < \dots < x_n < 1$, we prove the existence of $\gamma$ by induction.
With $g$ to be the transition function in the definition of $\Tilde{\beta}$, we mimic the proof of \textbf{Lemma} \ref{lemma:dct_inc_R1} and sort $(x_0, \dots, x_n)$ in the increasing order of $g$,
which is 
$$g(x_{k_0}) \leq g(x_{k_1}) \leq \dots \leq g(x_{k_n}).$$

Define 
\begin{equation*}
    \Tilde{\mu}_m(x_{k_i}) = 
    \left\{
    \begin{aligned}
        &\mu ([x_{k_i}, \min\{1, x_{k_l} |\ x_{k_l} > x_{k_i}, l \leq m \})),& &i \leq m,\  
        x_{k_i} \neq \min\{x_{k_l} |\ l \leq m \}, \\
        &\mu ([0, \min\{1, x_{k_l} |\ x_{k_l} > x_{k_i}, l \leq m \})),& &i \leq m,\  
        x_{k_i} = \min\{x_{k_l} |\ l \leq m \}, \\
        & 0,& &i > m.
    \end{aligned}
    \right.
\end{equation*}

Define 
$$
\Tilde{\beta}_m (x_{k_i}) = \Tilde{\mu}_m (x_{k_i}) \exp(g(x_{k_i})) / Z_m,
\  Z_m = \sum_{i} \Tilde{\mu}_m (x_{k_i}) \exp (g(x_{k_i})).
$$ 

To clarify the proof, we use $x_{k_i}$ to represent the point for $\Tilde{\mu}$-axis in coupling and $y_{k_j}$ to represent the point for $\Tilde{\beta}$-axis, but they are actually identical, i.e. $x_{k_i} = y_{k_j}$ when $i = j$. 

When $m = 0$, it's obvious that $(\Tilde{\mu}_0, \Tilde{\beta}_0)$ satisfies UTTC, as 
\begin{equation*}
    \gamma_0 (x_{k_i}, y_{k_j}) = 
    \left\{
    \begin{aligned}
        &1,& &\ i=0, j=0,\\
        &0,& &\ else. \\
    \end{aligned}
    \right.
\end{equation*}

Assume UTTC holds for $m$, i.e. there exists $\gamma_m$ s.t. $(\Tilde{\mu}_m, \Tilde{\beta}_m)$ satisfies UTTC, we want to prove it also holds for $m+1$.

When $x_{k_{m+1}} > \min\{x_{k_l} |\ l \leq m \}$, let $x_{k^*} = \max \{x_{k_l} | \ x_{k_l} < x_{k_{m+1}}, l \leq m \}$ to be the closest left neighbor of $x_{k_{m+1}}$ in $\{x_{k_l}|\ l \leq m \}$.
Then we have $\Tilde{\mu}_{m} (x_{k^*}) = \Tilde{\mu}_{m+1} (x_{k^*}) + \Tilde{\mu}_{m+1} (x_{k^{m+1}})$.

When $x_{k_{m+1}} < \min\{x_{k_l} |\ l \leq m \}$, let $x_{k^*} = \min\{x_{k_l} |\ l \leq m \}$ to be the leftmost point in $\{x_{k_l}|\  l \leq m \}$. 
Then we have $\Tilde{\mu}_{m} (x_{k^*}) = \Tilde{\mu}_{m+1} (x_{k^*}) + \Tilde{\mu}_{m+1} (x_{k^{m+1}})$.

In either case, we always have $\Tilde{\mu}_{m} (x_{k^*}) = \Tilde{\mu}_{m+1} (x_{k^*}) + \Tilde{\mu}_{m+1} (x_{k_{m+1}})$.
By definition of $\Tilde{\mu}_m$ and $\Tilde{\beta}_m$, we have
\begin{equation*}
    \left\{
    \begin{aligned}
        &\Tilde{\mu}_m (x_{k_i}) = \Tilde{\mu}_{m+1} (x_{k_i}),& &i \leq m,\ k_i \neq k^*, \\
        &\Tilde{\mu}_m (x_{k_i}) = \Tilde{\mu}_{m+1} (x_{k_i}) + \Tilde{\mu}_{m+1} (x_{k_{m+1}}),& &i \leq m,\ k_i = k^*, \\
        &\Tilde{\mu}_m (x_{k_{m+1}}) = \Tilde{\mu}_m (x_{k_i}) = \Tilde{\mu}_{m+1} (x_{k_i}) = 0,& &i>m+1, \\
    \end{aligned}
    \right.
\end{equation*}
\begin{equation*}
    \left\{
    \begin{aligned}
        &\Tilde{\beta}_m (x_{k_i}) = \Tilde{\beta}_{m+1} (x_{k_i}) \cdot \frac{Z_{m+1}}{Z_m},& &i \leq m,\ k_i \neq k^*, \\
        &\Tilde{\beta}_m (x_{k_i}) = \left( \Tilde{\beta}_{m+1} (x_{k_i})  + \Tilde{\beta}_{m+1} (x_{k_{m+1}}) \exp\left(g(x_{k_i}) - g(x_{k_{m+1}})\right) \right) \cdot \frac{Z_{m+1}}{Z_m},& &i \leq m,\ k_i = k^*, \\
        &\Tilde{\beta}_m (x_{m+1}) = \Tilde{\beta}_m (x_i) = \Tilde{\beta}_{m+1} (x_i) = 0,& &i>m+1. \\
    \end{aligned}
    \right.
\end{equation*}

If $g(x_{k^*}) = g(x_{k_{m+1}})$, it's easy to check that 
$\frac{\Tilde{\mu}_{m+1} (x_{k_{m+1}})}{\Tilde{\mu}_{m+1} (x_{k^*})} = \frac{\Tilde{\beta}_{m+1} (x_{k_{m+1}})}{\Tilde{\beta}_{m+1} (x_{k^*})}$, 
we can simply define the following $\gamma_{m+1}$ which achieves UTTC for $(\Tilde{\mu}_{m+1}, \Tilde{\beta}_{m+1})$:
\begin{equation*}
    \left\{
    \begin{aligned}
        &\gamma_{m+1} (x_{k^*}, y_{k_j}) = \gamma_{m} (x_{k^*}, y_{k_j}) \frac{\Tilde{\mu}_{m+1} (x_{k^*})}{\Tilde{\mu}_{m+1} (x_{k^*}) + \Tilde{\mu}_{m+1} (x_{k_{m+1}})},& &\ j \leq m,\ k_j \neq k^*, \\
        &\gamma_{m+1} (x_{k_{m+1}}, y_{k_j}) = \gamma_{m} (x_{k_{m+1}}, y_{k_j}) \frac{\Tilde{\mu}_{m+1} (x_{k_{m+1}})}{\Tilde{\mu}_{m+1} (x_{k^*}) + \Tilde{\mu}_{m+1} (x_{k_{m+1}})},& &\ j \leq m,\ k_j \neq k^*, \\
        &\gamma_{m+1} (x_{k_i}, y_{k^*}) = \gamma_{m} (x_{k_i}, y_{k^*}) \frac{\Tilde{\beta}_{m+1} (y_{k^*})}{\Tilde{\beta}_{m+1} (y_{k^*}) + \Tilde{\beta}_{m+1} (y_{k_{m+1}})},& &\ i \leq m,\ k_i \neq k^*, \\
        &\gamma_{m+1} (x_{k_{i}}, y_{k_{m+1}}) = \gamma_{m} (x_{k_{i}}, y_{k_{m+1}}) \frac{\Tilde{\beta}_{m+1} (y_{k_{m+1}})}{\Tilde{\beta}_{m+1} (y_{k^*}) + \Tilde{\mu}_{m+1} (x_{k_{m+1}})},& &\ i \leq m,\ k_i \neq k^*, \\
        &\gamma_{m+1} (x_{k^*}, y_{k^*}) = \gamma_{m} (x_{k^*}, y_{k^*}) \frac{\Tilde{\mu}_{m+1} (x_{k^*})}{\Tilde{\mu}_{m+1} (x_{k^*}) + \Tilde{\mu}_{m+1} (x_{k_{m+1}})},& & \\
        &\gamma_{m+1} (x_{k_{m+1}}, y_{k_{m+1}}) = \gamma_{m} (x_{k_{m+1}}, y_{k_{m+1}}) \frac{\Tilde{\mu}_{m+1} (x_{k_{m+1}})}{\Tilde{\mu}_{m+1} (x_{k^*}) + \Tilde{\mu}_{m+1} (x_{k_{m+1}})}, \\
        &\gamma_{m+1} (x_{k_i}, y_{k_j}) = 0,& &\ others.
    \end{aligned}
    \right.
\end{equation*}

If $g(x_{k^*}) < g(x_{k_{m+1}})$, recalling the proof of \textbf{Lemma} \ref{lemma:dct_inc_R1}, it's crucial to prove inequalities \eqref{eq:Zm_inc} and \eqref{eq:v_exist}.
Inequality \eqref{eq:Zm_inc} guarantees that $\frac{Z_m}{Z_{m+1}} < 1$, so we can shrinkage $\gamma_m$ entrywise by $\frac{Z_m}{Z_{m+1}}$ and add some proper measure at proper points.
Inequality \eqref{eq:v_exist} guarantees that $(x_m, y_m)$ can be decomposed to $(x_m, y_m)$, $(x_m, y_{m+1})$, $(x_{m+1}, y_{m+1})$.
Following the idea, we check that 
\begin{equation*}
    Z_{m+1} - Z_m = \Tilde{\mu}_{m+1} (x_{k_{m+1}}) \left( \exp (g(x_{k_{m+1}}) - g(x_{k^*})) \right) > 0,
\end{equation*}
\begin{equation*}
    \begin{aligned}
        \frac{\Tilde{\beta}_{m+1} (x_{k_{m+1}})}{\Tilde{\beta}_{m+1} (x_{k^*})} 
        - \frac{\Tilde{\mu}_{m+1} (x_{k_{m+1}})}{\Tilde{\mu}_{m+1} (x_{k^*})} 
        &= \frac{\Tilde{\mu}_{m+1} (x_{k_{m+1}}) \exp(g(x_{k_{m+1}}))}{\Tilde{\mu}_{m+1} (x_{k^*}) \exp (g(x_{k^*}))} 
        - \frac{\Tilde{\mu}_{m+1} (x_{k_{m+1}})}{\Tilde{\mu}_{m+1} (x_{k^*})}  \\
        &= \frac{\Tilde{\mu}_{m+1} (x_{k_{m+1}})}{\Tilde{\mu}_{m+1} (x_{k^*})}
        \left(\exp(g(x_{k_{m+1}}) - g(x_{k^*})) - 1 \right) > 0.
    \end{aligned}
\end{equation*}

Replacing $x_m, x_{m+1}$ in the proof of \textbf{Lemma} \ref{lemma:dct_inc_R1} by $x_{k^*}, x_{k_{m+1}}$, we can construct $\gamma_{m+1}$ all the same way as in the proof of \textbf{Lemma} \ref{lemma:dct_inc_R1}.

By induction, we prove UTTC for $(\Tilde{\mu}, \Tilde{\beta})$.
The proof is done.
\end{proof}

\begin{Theorem}[Upper Triangular Transport Inequality for Co-Monotonic Functions in $\mathbf{R}^1$]
Assume $\mu$ is a continuous probability measure supported on $[0, 1]$.
Let $f, g: [0, 1] \rightarrow $ \textbf{R} to be two co-monotonic functions that satisfy
$$(f(x) - f(y)) \cdot (g(x) - g(y)) \geq 0, \ \forall \ x, y \in [0, 1].$$
$f$ is continuous.
Define 
$$\beta(x) = \mu(x) \exp(g(x)) / Z, \ Z = \int_{[0, 1]} \mu(x) \exp (g(x)).$$
Then we have 
$$\textbf{E}_{\mu} [f] \leq \textbf{E}_{\beta} [f].$$
\label{thm:cts_R1}
\end{Theorem}

\begin{proof}[Proof of Theorem \ref{thm:cts_R1}]

For $\forall \epsilon > 0$, since $f$ is continuous, $f$ is uniformly continuous, so there exists $\delta > 0$ s.t. $|f(x) - f(y)| < \epsilon, \forall x, y \in [0, 1]$.
We can split $[0, 1]$ by $0 < x_0 < x_1 < \dots < x_n < 1$ s.t. $x_{i+1} - x_i < \delta$.
Define $\Tilde{\mu}$ and $\Tilde{\beta}$ as in \textbf{Lemma} \ref{lemma:dct_R1}.
Since $x_{i+1} - x_i < \delta$, by uniform continuity and the definition of the expectation, we have
$$
| \textbf{E}_{\mu} [f] - \textbf{E}_{\Tilde{\mu}} [f] | < \epsilon,\ 
| \textbf{E}_{\beta} [f] - \textbf{E}_{\Tilde{\beta}} [f] | < \epsilon,
$$
By \textbf{Lemma} \ref{lemma:dct_R1}, we have 
$$\textbf{E}_{\Tilde{\mu}} [f] \leq \textbf{E}_{\Tilde{\beta}} [f].$$
So we have
$$
\textbf{E}_{\mu} [f] 
< \textbf{E}_{\Tilde{\mu}} [f] + \epsilon
\leq \textbf{E}_{\Tilde{\beta}} [f] + \epsilon
< \textbf{E}_{\beta} [f] + 2\epsilon.
$$
Since $\epsilon$ is arbitrary, we prove $\textbf{E}_{\mu} [f] \leq \textbf{E}_{\beta} [f].$

\end{proof}

\begin{Lemma}[Discretized Upper Triangular Transport Inequality for Co-Monotonic Functions in $\mathbf{R}^p$]
Assume $\mu$ is a continuous probability measure supported on $[0, 1]^p$. 
Let $0 = x_0^d < x_1^d < \dots < x_n^d < 1$ to be any split of $[0, 1]$, $d = 1, \dots, p$.
Denote $\textbf{x}_{\textbf{i}} \overset{def}{=} (x_{i_1}^1, \dots, x_{i_p}^p)$.
Define $\Tilde{\mu}(\textbf{x}_{\textbf{i}}) = \mu(\prod_{d=1,\dots,p} [x_{i_d}^d, x_{i_d+1}^d))$. 
Let $f, g: \{\textbf{x}_{\textbf{i}}\}_{\textbf{i} \in \{0, \dots, n\}^p} \rightarrow $ \textbf{R} to be two co-monotonic functions that satisfy
$$(f(\textbf{x}_{\textbf{i}})
- f(\textbf{x}_{\textbf{j}})) 
\cdot (g(\textbf{x}_{\textbf{i}}) 
- g(\textbf{x}_{\textbf{j}})) \geq 0, \ \forall \ \textbf{i}, \textbf{j}.$$
Define 
$$\Tilde{\beta}(\textbf{x}_{\textbf{i}}) 
= \Tilde{\mu}(\textbf{x}_{\textbf{i}}) \exp(g(\textbf{x}_{\textbf{i}})) / Z,\  Z = \sum_{\textbf{i}} \Tilde{\mu}(\textbf{x}_{\textbf{i}}) \exp (g(\textbf{x}_{\textbf{i}})).$$ 
Then there exists a probability measure 
$\gamma: \{\textbf{x}_{\textbf{i}}\}_{\textbf{i} \in \{0,\dots,n\}^p} \times \{\textbf{x}_{\textbf{j}}\}_{\textbf{j} \in \{0,\dots,n\}^p} \rightarrow [0, 1]$, s.t. 
$$
\begin{aligned}
    \sum_{\textbf{j}} \gamma (\textbf{x}_{\textbf{i}}, \textbf{y}_{\textbf{j}}) &= \Tilde{\mu} (\textbf{x}_{\textbf{i}}), \ \forall \ \textbf{i}; \\
    \sum_{\textbf{i}} \gamma (\textbf{x}_{\textbf{i}}, \textbf{y}_{\textbf{j}}) &= \Tilde{\beta} (\textbf{y}_{\textbf{j}}), \ \forall \ \textbf{j}; \\
    \gamma (\textbf{x}_{\textbf{i}}, \textbf{y}_{\textbf{j}}) &= 0, \ g(\textbf{x}_{\textbf{i}}) > g(\textbf{y}_{\textbf{j}}). \\
\end{aligned}
$$
Then we have
$$
\textbf{E}_{\Tilde{\mu}} [f] \leq \textbf{E}_{\Tilde{\beta}} [f].
$$
\label{lemma:dct_Rp}
\end{Lemma}

\begin{proof}[Proof of Lemma \ref{lemma:dct_Rp}]

The proof is almost identical to the proof of \textbf{Lemma} \ref{lemma:dct_R1}, except for the definition of $(\Tilde{\mu}_m, \Tilde{\beta}_m)$ in $\textbf{R}^p$.


Given $\{\textbf{x}_{\textbf{i}}\}_{\textbf{i} \in \{0,\dots,n\}^p}$, we sort $\textbf{x}_{\textbf{i}}$ in the increasing order of $g$, which is 
$$
g(\textbf{x}_{\textbf{k}_0}) \leq g(\textbf{x}_{\textbf{k}_1}) \leq \dots \leq g(\textbf{x}_{\textbf{k}_{(n+1)^p - 1}}),
$$
where $\{\textbf{k}_i\}_{i \in \{0, \dots, (n+1)^p - 1\}}$ is a permutation of $\{\textbf{\textbf{i}}\}_{\textbf{i} \in \{0, \dots, n\}^p}$.

For $\textbf{i}, \textbf{j} \in \{0, \dots, n\}^p$, we define the partial order $\textbf{i} < \textbf{j}$ on $\{0, \dots, n\}^p$, if 
$$
\exists 0 \leq d_0 \leq n,\  s.t.\  \textbf{i}_d \leq \textbf{j}_d,\ \forall d < d_0 \ and\  \textbf{i}_{d_0} < \textbf{j}_{d_0}.
$$
It's obvious that 
\begin{equation*}
    \left\{
    \begin{aligned}
        &\forall \textbf{i} \in \{0, \dots, n\}^p,\  \textbf{i} \nless \textbf{i}, \\
        &\forall \textbf{i}, \textbf{j} \in \{0, \dots, n\}^p,\ \textbf{i} < \textbf{j} \Rightarrow \textbf{j} \nless \textbf{i}, \\
        &\forall \textbf{i}, \textbf{j}, \textbf{k} \in \{0, \dots, n\}^p,\ \textbf{i} < \textbf{j},\  \textbf{j} < \textbf{k} \Rightarrow \textbf{i} < \textbf{k}. \\
    \end{aligned}
    \right.
\end{equation*}
We define $\textbf{i} = \textbf{j}$ if $\textbf{i}_d = \textbf{j}_d,\, \forall\, 0 \leq d \leq n$.
So we define the partial order relation, and we can further define the $\min$ function and the $\max$ function on $\{0, \dots, n\}^p$.

Now using this partial order relation, we define 
\begin{equation*}
    \Tilde{\mu}_m (\textbf{x}_{\textbf{k}_i}) = 
    \left\{
    \begin{aligned}
        &\sum_{\textbf{k} \geq \textbf{k}_i, \textbf{k} < \min \{ \textbf{k}_l | \textbf{k}_l > \textbf{k}_i, l \leq m \} } \Tilde{\mu} (\textbf{x}_{\textbf{k}}) 
        ,& &\ i \leq m,\ \textbf{k}_i \neq \min\{ \textbf{k}_l |\ l \leq m \}, \\
        &\sum_{\textbf{k} < \min\{ \textbf{k}_l | \textbf{k}_l > \textbf{k}_i, l \leq m \} } \Tilde{\mu} (\textbf{x}_{\textbf{k}})     
        ,& &\ i \leq m,\ \textbf{k}_i = \min\{\textbf{k}_l |\ l \leq m \}, \\
        &0         ,& &\  i > m. \\
    \end{aligned}
    \right.
\end{equation*}

With this definition of $\Tilde{\mu}_m$, other parts are identical to the proof of \textbf{Lemma} \ref{lemma:dct_R1}.
The proof is done.

\end{proof}

\begin{Theorem}[Upper Triangular Transport Inequality for Co-Monotonic Functions in $\mathbf{R}^p$]
Assume $\mu$ is a continuous probability measure supported on $[0, 1]^p$. 
Denote $\textbf{x} \overset{def}{=} (x^1, \dots, x^p)$.
Let $f, g: [0, 1]^p \rightarrow $ \textbf{R} to be two co-monotonic functions that satisfy
$$(f(\textbf{x})
- f(\textbf{y})) 
\cdot (g(\textbf{x}) 
- g(\textbf{y})) \geq 0, \ \forall \ \textbf{x}, \textbf{y} \in [0, 1]^p.$$
$f$ is continuous.
Define 
$$\beta(\textbf{x}) 
= \mu(\textbf{x}) \exp(g(\textbf{x})) / Z,\  
Z = \int_{[0, 1]^p} \mu(\textbf{x}) \exp (g(\textbf{x})).$$ 
Let $f, g: [0, 1]^p \rightarrow $ \textbf{R} to be two co-monotonic functions that satisfy
$$(f(\textbf{x})
- f(\textbf{y})) 
\cdot (g(\textbf{x}) 
- g(\textbf{y})) \geq 0, \ \forall \ \textbf{x}, \textbf{y} \in [0, 1]^p.$$
Then we have
$$
\textbf{E}_{\mu} [f] \leq \textbf{E}_{\beta} [f].
$$
\label{thm:cts_Rp}
\end{Theorem}

\begin{proof}[Proof of Theorem \ref{thm:cts_Rp}]

For $\forall \epsilon > 0$, since $f$ is continuous, $f$ is uniformly continuous, so there exists $\delta > 0$ s.t. $|f(\textbf{x}) - f(\textbf{y})| < \epsilon, \forall \textbf{x}, \textbf{y} \in [0, 1]^p$.
We can split $[0, 1]$ by $0 < x_0 < x_1 < \dots < x_n < 1$ s.t. $x_{i+1} - x_i < \delta / \sqrt{p}$.
Define $x_i^d = x_i, \ \forall 0 \leq d \leq p$.
Define $\Tilde{\mu}$ and $\Tilde{\beta}$ as in \textbf{Lemma} \ref{lemma:dct_Rp}.
Since $x_{i+1} - x_i < \delta / \sqrt{p}$, 
$|(x_{i+1}^0, \dots, x_{i+1}^p) - (x_{i}^0, \dots, x_{i}^p)| < \delta$,  by uniform continuity and the definition of the expectation, we have
$$
| \textbf{E}_{\mu} [f] - \textbf{E}_{\Tilde{\mu}} [f] | < \epsilon,\ 
| \textbf{E}_{\beta} [f] - \textbf{E}_{\Tilde{\beta}} [f] | < \epsilon,
$$
By \textbf{Lemma} \ref{lemma:dct_Rp}, we have 
$$\textbf{E}_{\Tilde{\mu}} [f] \leq \textbf{E}_{\Tilde{\beta}} [f].$$
So we have
$$
\textbf{E}_{\mu} [f] 
< \textbf{E}_{\Tilde{\mu}} [f] + \epsilon
\leq \textbf{E}_{\Tilde{\beta}} [f] + \epsilon
< \textbf{E}_{\beta} [f] + 2\epsilon.
$$
Since $\epsilon$ is arbitrary, we prove $\textbf{E}_{\mu} [f] \leq \textbf{E}_{\beta} [f].$

\end{proof}

\begin{Lemma}[Performance Difference Lemma]
For any policies $\pi, \pi'$ and any state $s_0$, we have
\begin{equation*}
    V^{\pi} (s_0) - V^{\pi'} (s_0) = \frac{1}{1 - \gamma} \textbf{E}_{s \sim d_{s_0}^\pi} \textbf{E}_{a \sim \pi (\cdot | s)} \left[ A^{\pi'} (s, a) \right].
\end{equation*}
\label{lemma:perfdiff}
\end{Lemma}

\begin{proof}
    See \citep{kakade2002approximately}.
\end{proof}

\clearpage

\section{Algorithm Pseudocode}
\label{App: Algorithm Pseudocode}

For completeness, we provide the implementation pseudocode of GDI-I$^3$, which is shown in \textbf{Algorithm} \ref{alg:i3}.

\begin{equation}
\label{Equ: i3 casa equ}
    \left\{
    \begin{aligned}
        &A=A_{\theta}\left(s_{t}\right),& 
        &V=V_{\theta}\left(s_{t}\right), \\
        &\bar{A}=A-E_{\pi}[A],& 
        &Q=\bar{A}+V. \\
    \end{aligned}
    \right.
\end{equation}

\begin{equation}
\label{Equ: i3 soft entropy}
    \lambda = (\tau_1, \tau_2, \epsilon), \ 
    \pi_{\theta_{\lambda}}=\varepsilon \cdot \underbrace{\operatorname{Softmax}\left(\frac{A}{\tau_{1}}\right)}_{Exploration}+(1-\varepsilon) \cdot \underbrace{\operatorname{Softmax}\left(\frac{A}{\tau_{2}}\right)}_{Exploitation}
\end{equation}

\begin{figure}[ht]
  \centering
  \begin{minipage}{.7\linewidth}
    \begin{algorithm}[H]
      \caption{GDI-I$^3$ Algorithm.}  
          \begin{algorithmic}
            \STATE Initialize Parameter Server (PS) and Data Collector (DC).
            \STATE
            \STATE // LEARNER
            \STATE Initialize $d_{push}$.
            \STATE Initialize $\theta$ as Eq. \eqref{Equ: i3 casa equ} and \eqref{Equ: i3 soft entropy}.
            \STATE Initialize $count = 0$.
            \WHILE{$True$}
                \STATE Load data from DC.
                \STATE Estimate $qs$ and $vs$ by proper off-policy algorithms.
                \STATE \ \ \ \ (For instance, ReTrace \eqref{Equ: retrace} for $qs$ and V-Trace \eqref{Equ: vtrace} for $vs$.)
                \STATE Update $\theta$ via policy gradient and policy evaluation.
                \IF{$count$ mod $d_{push}$ = 0}
                    \STATE Push $\theta$ to PS.
                \ENDIF
                \STATE $count \leftarrow count + 1$.
            \ENDWHILE
            \STATE
            \STATE // ACTOR
            \STATE Initialize $d_{pull}$, $M$.
            \STATE Initialize $\theta$ as Eq. \eqref{Equ: i3 casa equ} and \eqref{Equ: i3 soft entropy}.
            \STATE Initialize $\{\mathcal{B}_m\}_{m=1,...,M}$ and sample $\lambda$ as in \textbf{Algorithm} \ref{alg:bva}.
            \STATE Initialize $count = 0$, $G = 0$.
            \WHILE{$True$}
                \STATE Calculate $\pi_{\theta_{\lambda}}(\cdot | s)$.
                \STATE Sample $a \sim \pi_{\theta_{\lambda}}(\cdot | s)$.
                \STATE $s, r, done \sim p(\cdot | s, a)$.
                \STATE $G \leftarrow G + r$.
                \IF{$done$}
                    \STATE Update $\{\mathcal{B}_m\}_{m=1,...,M}$ with $(\lambda, G)$ as in \textbf{Algorithm} \ref{alg:bva}.
                    \STATE Send data to DC and reset the environment.
                    \STATE $G \leftarrow 0$.
                    \STATE Sample $\lambda$ as in \textbf{Algorithm} \ref{alg:bva}
                \ENDIF
                \IF{$count \mod d_{pull}$ = 0}
                    \STATE Pull $\theta$ from PS and update $\theta$.
                \ENDIF
                \STATE $count \leftarrow count + 1$.
            \ENDWHILE
          \end{algorithmic}
        \label{alg:i3}
    \end{algorithm}
  \end{minipage}
\end{figure}

\clearpage

For completeness, we provide the implementation pseudocode of GDI-H$^3$, which is shown in \textbf{Algorithm} \ref{alg:h3}.

\begin{equation}
\label{Equ: h3 casa equ}
    \begin{aligned}
    \left\{
    \begin{aligned}
        &A_1=A_{\theta_1}\left(s_{t}\right),& 
        &V_1=V_{\theta_1}\left(s_{t}\right), \\
        &\bar{A}_1=A_1-E_{\pi}[A_1],& 
        &Q_1=\bar{A}_1+V_1. \\
    \end{aligned}
    \right.\\
    \left\{
    \begin{aligned}
        &A_2=A_{\theta_2}\left(s_{t}\right),& 
        &V_2=V_{\theta_2}\left(s_{t}\right), \\
        &\bar{A}_2=A_2-E_{\pi}[A_2],& 
        &Q_2=\bar{A}_2+V_2. \\
    \end{aligned}
    \right.
    \end{aligned}
\end{equation}

\begin{equation}
\label{Equ: h3 soft entropy}
    \lambda = (\tau_1, \tau_2, \epsilon), \ 
    \pi_{\theta_{\lambda}}=\varepsilon \cdot \operatorname{Softmax}\left(\frac{A_1}{\tau_{1}}\right)+(1-\varepsilon) \cdot \operatorname{Softmax}\left(\frac{A_2}{\tau_{2}}\right)
\end{equation}

\begin{figure}[ht]
  \centering
  \begin{minipage}{.7\linewidth}
    \begin{algorithm}[H]
      \caption{GDI-H$^3$ Algorithm.}  
          \begin{algorithmic}
            \STATE Initialize Parameter Server (PS) and Data Collector (DC).
            \STATE
            \STATE // LEARNER
            \STATE Initialize $d_{push}$.
            \STATE Initialize $\theta$  as Eq. \eqref{Equ: h3 casa equ} and \eqref{Equ: h3 soft entropy}.
            \STATE Initialize $count = 0$.
            \WHILE{$True$}
                \STATE Load data from DC.
                \STATE Estimate $qs_1, qs_2$ and $vs_1, vs_2$ by proper off-policy algorithms.
                \STATE \ \ \ \ (For instance, ReTrace \eqref{Equ: retrace} for $qs1, qs_2$ and V-Trace \eqref{Equ: vtrace} for $vs_1, vs_2$.)
                \STATE Update $\theta_1, \theta_2$ via policy gradient and policy evaluation, respectively.
                \IF{$count$ mod $d_{push}$ = 0}
                    \STATE Push $\theta_1, \theta_2$ to PS.
                \ENDIF
                \STATE $count \leftarrow count + 1$.
            \ENDWHILE
            \STATE
            \STATE // ACTOR
            \STATE Initialize $d_{pull}$, $M$.
            \STATE Initialize $\theta_1, \theta_2$ as Eq. \eqref{Equ: h3 casa equ} and \eqref{Equ: h3 soft entropy}.
            \STATE Initialize $\{\mathcal{B}_m\}_{m=1,...,M}$ and sample $\lambda$ as in \textbf{Algorithm} \ref{alg:bva}.
            \STATE Initialize $count = 0$, $G = 0$.
            \WHILE{$True$}
                \STATE Calculate $\pi_{\theta_{\lambda}}(\cdot | s)$.
                \STATE Sample $a \sim \pi_{\theta_{\lambda}}(\cdot | s)$.
                \STATE $s, r, done \sim p(\cdot | s, a)$.
                \STATE $G \leftarrow G + r$.
                \IF{$done$}
                    \STATE Update $\{\mathcal{B}_m\}_{m=1,...,M}$ with $(\lambda, G)$ as in \textbf{Algorithm} \ref{alg:bva}.
                    \STATE Send data to DC and reset the environment.
                    \STATE $G \leftarrow 0$.
                    \STATE Sample $\lambda$ as in \textbf{Algorithm} \ref{alg:bva}
                \ENDIF
                \IF{$count \mod d_{pull}$ = 0}
                    \STATE Pull $\theta$ from PS and update $\theta$.
                \ENDIF
                \STATE $count \leftarrow count + 1$.
            \ENDWHILE
          \end{algorithmic}
        \label{alg:h3}
    \end{algorithm}
  \end{minipage}
\end{figure}

\clearpage

\section{Adaptive Controller Formalism}
\label{Sec: appendix MAB}

In practice, we use a Bandits Controller (BC) to adaptively control the behavior sampling distribution. More details on Bandits can be found in \citep{garivier2008upper}. The whole algorithm is shown in \textbf{Algorithm} \ref{alg:bva}. As the behavior policy can be parameterized and thus sampling behaviors from the policy space is equivalent to sampling parameters $x$ from parameter space. 

Let's firstly define a bandit as $B = Bandit(mode, l, r, lr, d, acc, ta, to, \textbf{w}, \textbf{N})$.
\begin{itemize}
    \item $mode$ is the mode of sampling, with two choices, $argmax$ and $random$, wherein $argmax$ greedily chooses the behaviors with top estimated value from the policy space, and $random$ samples behaviors according to a distribution calculated by $Softmax(V)$.
    \item $l$ is the left boundary of the parameter space, and each $x$ is clipped to $x = \max \{x, l\}$.
    \item $r$ is the right boundary of the parameter space, and each $x$ is clipped to $x = \min \{x, r\}$.
    \item $acc$ is the accuracy of space to be optimized, where each $x$ is located in the $\lfloor (\min\{\max\{x, l\}, r\} - l) / acc \rfloor$th block.
    \item tile coding is a representation method of continuous space \citep{sutton}, and each kind of tile coding can be uniquely determined by $l$, $r$, $to$ and $ta$, wherein $to$ represents the tile offset and $ta$ represents the accuracy of the tile coding.
    \item $to$ is the offset of each tile coding, which represents the relative offset of the basic coordinate system (normally we select the space to be optimized as basic coordinate system).
    \item $ta$ is the accuracy of each tile coding, where each $x$ is located in the $\lfloor (\min\{\max\{x-to, l\}, r\} - l) / ta \rfloor$th tile.
    \item $M_{btt}$ represents block-to-tile, which is a mapping from the block of the original space to the tile coding space.
    \item $M_{ttb}$ represents tile-to-block, which is a mapping from the tile coding space to the block of the original space.
    \item $\textbf{w}$ is a vector in $\mathbf{R}^{\lfloor (r-l) / ta \rfloor}$, which represents the weight of each tile.
    \item $\textbf{N}$ is a vector in $\mathbf{R}^{\lfloor (r-l) / ta \rfloor}$, which counts the number of sampling of each tile.
    \item $lr$ is the learning rate.
    \item $d$ is an integer, which represents how many candidates is provided by each bandit when sampling.
\end{itemize}

During the evaluation process, we evaluate the value of the $i$th tile by
\begin{equation}
\label{eq:bandit_eval}
V_i = \frac{\sum_{k}^{M_{btt}(block_i)} \textbf{w}_k}{len(M_{btt}(block_i))}
\end{equation}

During the training process, for each sample $(x, g)$, where $g$ is the target value. Since $x$ locates in the $j$th tile of $k$th tile\_coding, we update $B$ by
\begin{equation}
\label{eq:bandit_update}
\left\{
\begin{aligned}
&j = \lfloor (\min\{\max\{x-to_{k}, l\}, r\} - l) / ta_{k} \rfloor, \\
&\textbf{w}_j 
\leftarrow \textbf{w}_j + lr * \left(g - V_i\right)\\
& \textbf{N}_j \leftarrow \textbf{N}_j + 1
\end{aligned}
\right.
\end{equation}

During the sampling process, we firstly evaluate $\mathcal{B}$ by \eqref{eq:bandit_eval} and get $(V_1, ..., V_{\lfloor (r-l) / acc \rfloor})$.
We calculate the score of $i$th tile by
\begin{equation}
\label{eq:bandit_score}
score_i = \frac{V_i - \mu(\{V_j\}_{j=1,...,\lfloor(r-l)/acc\rfloor})}{\sigma(\{V_j\}_{j=1,...,\lfloor(r-l)/acc\rfloor})} + c \cdot \sqrt{\frac{\log (1 + \sum_j \textbf{N}_j)}{1 + \textbf{N}_i}}.
\end{equation}
For different $mode$s, we sample the candidates by the following mechanism,
\begin{itemize}
    \item if $mode$ = $argmax$, find blocks with top-$d$ $score$s, then sample $d$ candidates from these blocks, one uniformly from a block;
    \item if $mode$ = $random$, sample $d$ blocks with $score$s as the logits without replacement, then sample $d$ candidates from these blocks, one uniformly from a block;
\end{itemize}

In practice, we define a set of bandits $\mathcal{B}_m = \{B_m\}_{m=1,...,M}$.
At each step, we sample $d$ candidates $\{c_{m, i}\}_{i=1,...,d}$ from each $B_m$, so we have a set of $m \times d$ candidates $\{c_{m, i}\}_{m=1,...,M; i=1,...,d}$.
Then we sample uniformly from these $m \times d$ candidates to get $x$. 
At last, we transform the selected $x$ to $\alpha=\{\tau_1,\tau_2,\epsilon\}$ by $\tau_{1,2} = \frac{1}{\exp (x_{1,2}) - 1}$ and $\epsilon = x_{3}$
When we receive $(\alpha, g)$,  we transform $\alpha$ to $x$ by $x_{1,2} = \log (1 + 1 / \tau_{1,2})$, and  $x_{3} = \epsilon$.
Then we update each $B_m$ by \eqref{eq:bandit_update}.

\begin{figure}[ht]
  \centering
  \begin{minipage}{.9\linewidth}
    \begin{algorithm}[H]
      \caption{Bandits Controller}  
          \begin{algorithmic}
            \FOR{$m=1,...,M$}
                \STATE Sample $mode \sim \{argmax, random\}$ and other initialization parameters
                \STATE Initialize $B_m = Bandit(mode, l, r, lr, d, acc, to, ta, \textbf{w}, \textbf{N})$
                \STATE Ensemble $B_m$ to constitute $\mathcal{B}_m$ 
            \ENDFOR
            \WHILE{$True$}  
                \FOR{$m=1,...,M$}
                    \STATE Evaluate $\mathcal{B}_m$ by \eqref{eq:bandit_eval}.
                    \STATE Sample candidates $c_{m, 1}, ..., c_{m, d}$  from $\mathcal{B}_m$ via \eqref{eq:bandit_score} following its $mode$.
                \ENDFOR
                \STATE Sample $x$ from $\{c_{m, i}\}_{m=1,...,M; i=1,...,d}$.
                \STATE Execute $x$ and receive the return $G$.
                \FOR{$m=1,...,M$}
                    \STATE Update $\mathcal{B}_m$ with $(x, G)$ by \eqref{eq:bandit_update}.
                \ENDFOR
            \ENDWHILE
          \end{algorithmic}  
        \label{alg:bva}
    \end{algorithm}
  \end{minipage}
\end{figure}

\clearpage

\section{Experiment Details}
\label{sec:app Experiment Details}

We evaluated all agents on 57 Atari 2600 games from the arcade learning environment  \citep[ALE]{ale} by recording the average score of the population of agents during training. We demonstrate our multivariate evaluation system in Tab. \ref{Tab: multivariate evaluation system}, and we will describe more details in the following. Besides, all the experiment is accomplished using a single CPU with 92 cores and a single Tesla-V100-SXM2-32GB GPU.

Noting that episodes will be truncated at 100K frames (or 30 minutes of simulated play) as other baseline algorithms \citep{rainbow,agent57,laser,ngu,r2d2}, and thus we calculate the mean game time over 57 games which is called Game Time. In addition to comparing the mean and median human normalized scores (HNS), we also report the performance based on human world records among these algorithms and the related learning efficiency to further emphasize the significance of our algorithm. Inspired by \citep{atarihuman}, human world records normalized score (HWRNS) and SABER are  better descriptors for evaluating algorithms on human top level on Atari games, which simultaneously give rise to more challenges and lead the related research into a new journey to train the superhuman agent instead of  just paying attention to  the human average level. The learning is the ratio of the related evaluation criterion (such as HWRNS, HNS or SABER) and training frames.

\begin{table}[!hb]
\renewcommand\arraystretch{2}
	\centering
	\caption{Multivariate evaluation system}
	\label{Tab: multivariate evaluation system}
	\begin{tabular}{c c}
		\toprule
		\textbf{Evaluation Criterion} &\textbf{Computing Formula}\\
		\midrule
		Game Time    & $\frac{\text{Num.Frames}}{\text{100000*2*24*365}}$\\
		
		HNS         & $\frac{G-G_{\text{random}}}{G_{\text{human}}-G_{\text{random}}}$\\
		
		Human World Record Breakthrough&$\sum_{i=1}^{57}(G_{i}>=G_{\text{Human World Records}})$\\
        
        Learning Efficiency & $\frac{\text{Related Evaluation Criterion}}{\text{Num.Frames}}$\\
        
		HWRNS          & $\frac{G-G_{\text{random}}}{G_{\text{Human World Records}}-G_{\text{random}}}$\\
		
		SABER          & $\min\{\frac{G-G_{\text{random}}}{G_{\text{Human World Records}}-G_{\text{random}}},200\%\}$\\
		
		\bottomrule
	\end{tabular} 
\end{table}

\clearpage
\section{Hyperparameters}
\label{Sec: appendix hyperparameters}
\subsection{Atari pre-processing hyperparameters}
\label{Sec: Atari pre-processing hyperparameters}

In this section we detail the hyperparameters we use to pre-process the environment frames received from the Arcade Learning Environment. The hyperparameters that we used in all experiments are almost the same as Agent57  \citep{agent57}, NGU \citep{ngu}, MuZero \citep{muzero} and R2D2 \citep{r2d2}.
In Tab. \ref{tab:ale_process} we detail these hyperparameters. 

\begin{table}[H]
\begin{center}
\begin{tabular}{l@{\hspace{.43cm}}l@{\hspace{.22cm}}}
\toprule
\textbf{Hyperparameter} & \textbf{Value}  \\
\midrule
Random modes and difficulties & No \\
Sticky action probability  & 0.0 \\
Life information & Not allowed \\
Image Size & (84, 84) \\
Num. Action Repeats & 4 \\
Num. Frame Stacks & 4 \\
Action Space & Full \\
Max episode length   & 100000 \\
Random noops range  & 30\\
Grayscaled/RGB      & Grayscaled\\
\bottomrule
\end{tabular}
\caption{Atari pre-processing hyperparameters.}
\label{tab:ale_process}
\end{center}
\end{table}

\clearpage

\subsection{Hyperparameters Used}
\label{app: Hyperparameters Used}
In this section we detail the hyperparameters we used , which is demonstrated in Tab. \ref{tab:fixed_model_hyperparameters_atari}. We also include the hyperparameters we use for the  UCB bandit.

\begin{table}[H]
\begin{center}
\begin{tabular}{l@{\hspace{.43cm}}l@{\hspace{.22cm}}}
\toprule
\textbf{Parameter} & \textbf{Value}  \\
\midrule
Num. Frames & 200M \\
Replay & 2 \\
Num. Environments & 160 \\
GDI-I$^3$ Reward Shape & $\log (abs (r) + 1.0) \cdot (2 \cdot 1_{\{r \geq 0\}} - 1_{\{r < 0\}})$ \\
GDI-H$^3$ Reward Shape 1 & $\log (abs (r) + 1.0) \cdot (2 \cdot 1_{\{r \geq 0\}} - 1_{\{r < 0\}})$ \\
GDI-H$^3$ Reward Shape 2 & $sign(r) \cdot ((abs (r) + 1.0)^{0.25} - 1.0) + 0.001 \cdot r$ \\
Reward Clip & No \\
Intrinsic Reward & No \\
Entropy Regularization & No \\
Burn-in & 40 \\
Seq-length & 80 \\
Burn-in Stored Recurrent State & Yes \\
Bootstrap & Yes \\
Batch size & 64 \\
Discount ($\gamma$) & 0.997 \\
$V$-loss Scaling ($\xi$) & 1.0 \\
$Q$-loss Scaling ($\alpha$) & 10.0 \\
$\pi$-loss Scaling ($\beta$) & 10.0 \\
Importance Sampling Clip $\Bar{c}$ & 1.05 \\
Importance Sampling Clip $\Bar{\rho}$ & 1.05 \\
Backbone & IMPALA,deep \\
LSTM Units & 256 \\
Optimizer & Adam Weight Decay \\
Weight Decay Rate & 0.01 \\
Weight Decay Schedule & Anneal linearly to 0 \\
Learning Rate & 5e-4 \\
Warmup Steps & 4000 \\
Learning Rate Schedule & Anneal linearly to 0 \\
AdamW $\beta_1$ & 0.9 \\
AdamW $\beta_2$ & 0.98 \\
AdamW $\epsilon$ & 1e-6 \\
AdamW Clip Norm & 50.0 \\
Auxiliary Forward Dynamic Task & Yes \\
Auxiliary Inverse Dynamic Task & Yes \\
Learner Push Model Every $N$ Steps & 25 \\
Actor Pull Model Every $N$ Steps & 64 \\
Num. Bandits & 7 \\
Bandit Learning Rate & Uniform([0.05, 0.1, 0.2]) \\
Bandit Tiling Width & Uniform([2, 3, 4]) \\
Num. Bandit Candidates & 3 \\
Offset of Tile coding & Uniform([0, 60]) \\
Accuracy of Tile coding & Uniform([2, 3, 4]) \\
Accuracy of Search Range for [$1/\tau_1$,$1/\tau_2$,$\epsilon$]& [1.0, 1.0, 0.1]\\
Fixed Selection for [$1/\tau_1$,$1/\tau_2$,$\epsilon$]&[1.0,0.0,1.0]\\
Bandit UCB Scaling & 1.0 \\
Bandit Search Range for $1/\tau_1$ & [0.0, 50.0] \\
Bandit Search Range for $1/\tau_2$ & [0.0, 50.0] \\
Bandit Search Range for $\epsilon$ & [0.0, 1.0] \\
\bottomrule
\end{tabular}
\caption{Hyperparameters for Atari experiments.}
\label{tab:fixed_model_hyperparameters_atari}
\end{center}
\end{table}
\clearpage

\section{Experimental Results}
\label{appendix: experiment results}

In this section, we report the performance of GDI-H$^3$, GDI-I$^3$ and many well-known SOTA algorithms including both the model-based and model-free methods. First of all, we summarize the performance of all the algorithms over all the evaluation criteria of our evaluation system in App. \ref{app: Full Performance Comparison} which is mentioned in App. \ref{sec:app Experiment Details}. In the next three parts, we visualize the performance of GDI-H$^3$, GDI-I$^3$ over HNS in App. \ref{app: Figure of HNS}, HWRNS in App. \ref{app: Figure of HWRNS}, SABER in App. \ref{app: Figure of SABER} via histogram. Furthermore, we details all the original scores of all the algorithms and provide raw data that calculates those  evaluation criteria, wherein we first provide all the human world records in 57 Atari games and calculate the HNS in App. \ref{app: Atari Games Table of Scores Based on Human Average Records}, HWRNS in App. \ref{app: Atari Games Table of Scores Based on Human World Records} and SABER in App. \ref{app: Atari Games Table of Scores Based on SABER} of all 57 Atari games. We further provide all the evaluation curve of GDI-H$^3$, GDI-I$^3$ over 57 Atari games in App. \ref{app: Atari Games Learning Curves}.

\subsection{Full Performance Comparison}
\label{app: Full Performance Comparison}
In this part, we summarize the performance of all mentioned algorithms over all the evaluation criteria in Tab. \ref{Tab: full performance comparison}. In the following sections we will detail the performance of each algorithm on all Atari games one by one. 

\begin{table}[H]
\scriptsize
\centering
\setlength{\tabcolsep}{1.0pt}
\begin{tabular}{ c c c c c c c c c c}
\toprule
Algorithms & Num. Frames & Game Time & HWRB & Mean HNS & Median HNS & Mean HWRNS & Median HWRNS & Mean SABER  & Median SABER\\
\midrule
GDI-I$^3$        &2E+8    & 0.114     & \textbf{\GDIIHWRB} & \textbf{\GDIImeanhns} & \textbf{\GDIImedianhns}      &\textbf{\GDIImeanHWRNS} & \textbf{\GDIImedianHWRNS}& \textbf{\GDIImeanSABER} & \textbf{\GDIImedianSABER}\\
Rainbow     &2E+8    & 0.114     & 4  & 873.97  & 230.99     &28.39  & 4.92 & 28.39 & 4.92\\
IMPALA      &2E+8    & 0.114     & 3  & 957.34  & 191.82     &34.52  & 4.31 & 29.45 & 4.31\\
LASER       &2E+8    & 0.114     & 7  & 1741.36 & 454.91     &45.39  & 8.08 & 36.78 & 8.08\\
\midrule
\midrule
GDI-I$^3$        &2E+8    & 0.114    & \GDIIHWRB & \textbf{\GDIImeanhns} & \GDIImedianhns     &\GDIImeanHWRNS &\GDIImedianHWRNS & \GDIImeanSABER & \GDIImedianSABER\\
R2D2        &1E+10   & 5.7     & 15 & 3374.31 & 1342.27    &98.78  & 33.62& 60.43 & 33.62\\
NGU         &3.5E+10 & 19.9    & 8  & 3169.9  & 1208.11    &76.00  & 21.19& 50.47 & 21.19\\
Agent57     &1E+11   & 57      & \textbf{18} & 4763.69 & \textbf{1933.49}    &\textbf{125.92} & \textbf{43.62}& \textbf{76.26} & \textbf{43.62}\\
\midrule
\midrule
GDI-I$^3$        &2E+8    & 0.114     & \GDIIHWRB & \textbf{\GDIImeanhns} & \GDIImedianhns     &\GDIImeanHWRNS &\GDIImedianHWRNS & \GDIImeanSABER & \GDIImedianSABER\\
SimPLe      &1E+6    & 0.0005  & 0  & 25.30   & 5.55       &4.67   & 0.13 & 4.67  & 0.13\\
DreamerV2   &2E+8    & 0.114     & 3  & 631.18  & 161.96     &37.90  & 4.22 & 27.22 & 4.22\\
MuZero      &2E+10   & 11.4     & \textbf{19} & 4996.20 & \textbf{2041.12}    &\textbf{152.10} & \textbf{49.80}& \textbf{71.94} & \textbf{49.80} \\
\midrule
\midrule
GDI-I$^3$        &2E+8    & 0.114    & \textbf{\GDIIHWRB} & \textbf{\GDIImeanhns} & \GDIImedianhns     &\GDIImeanHWRNS &\GDIImedianHWRNS & \GDIImeanSABER & \GDIImedianSABER\\
Muesli      &2E+8    & 0.114     & 5           &  2538.66         & 1077.47    & 75.52          & 24.86  & 48.74 & 24.86 \\
Go-Explore  &1E+10   & 5.7     & 15 & 4989.94 & \textbf{1451.55}    &116.89 & \textbf{50.50}& \textbf{71.80} & \textbf{50.50}\\
\midrule
\midrule
GDI-H$^3$   &2E+8    & 0.114     &\textbf{ \GDIHHWRB} & \textbf{\GDIHmeanhns} & \textbf{\GDIHmedianhns}     &\textbf{\GDIHmeanHWRNS} &\textbf{\GDIHmedianHWRNS} & \textbf{\GDIHmeanSABER} & \textbf{\GDIHmedianSABER}\\
Rainbow     &2E+8    & 0.114     & 4           & 873.97               & 230.99              &28.39                & 4.92             & 28.39          & 4.92\\
IMPALA      &2E+8    & 0.114     & 3           & 957.34               & 191.82              &34.52                & 4.31             & 29.45          & 4.31\\
LASER       &2E+8    & 0.114     & 7           & 1741.36              & 454.91              &45.39                & 8.08             & 36.78          & 8.08\\
\midrule
\midrule
GDI-H$^3$   &2E+8    & 0.114     & \textbf{\GDIHHWRB} & \textbf{\GDIHmeanhns} & \GDIHmedianhns     &\textbf{\GDIHmeanHWRNS} &\textbf{\GDIHmedianHWRNS} & \GDIHmeanSABER &\textbf{ \GDIHmedianSABER}\\
R2D2        &1E+10   & 5.7     & 15            & 3374.31              & 1342.27             &98.78                & 33.62          & 60.43             & 33.62\\
NGU         &3.5E+10 & 19.9    & 8             & 3169.9               & 1208.11             &76.00                & 21.19          & 50.47             & 21.19\\
Agent57     &1E+11   & 57      & 18            & 4763.69              & \textbf{1933.49}    &125.92               & 43.62 & \textbf{76.26}    & 43.62\\
\midrule
\midrule
GDI-H$^3$   &2E+8    & 0.114     & \textbf{\GDIHHWRB} & \textbf{\GDIHmeanhns} & \GDIHmedianhns     &\textbf{\GDIHmeanHWRNS} &\textbf{\GDIHmedianHWRNS} & \GDIHmeanSABER & \textbf{\GDIHmedianSABER}\\
SimPLe      &1E+6    & 0.0005  & 0             & 25.30                & 5.55                &4.67                  & 0.13              & 4.67              & 0.13\\
DreamerV2   &2E+8    & 0.114     & 3             & 631.18               & 161.96              &37.90                 & 4.22              & 27.22             & 4.22\\
MuZero      &2E+10   & 11.4     & 19   & 4996.20              & \textbf{2041.12}    &152.10       & 49.80    & \textbf{71.94}    & 49.80 \\
\midrule
\midrule
GDI-H$^3$   &2E+8    & 0.114     & \textbf{\GDIHHWRB} & \textbf{\GDIHmeanhns} & \GDIHmedianhns     &\textbf{\GDIHmeanHWRNS} &\textbf{\GDIHmedianHWRNS} & \GDIHmeanSABER & \textbf{\GDIHmedianSABER}\\
Muesli      &2E+8    & 0.114     & 5             &  2538.66              & 1077.47               & 75.52             & 24.86             & 48.74             & 24.86 \\
Go-Explore  &1E+10   & 5.7    & 15            & 4989.94               & \textbf{1451.55}      &116.89             & 50.50    & \textbf{71.80}    & 50.50\\
\bottomrule
\end{tabular}
\caption{Full performance comparison on Atari.}
\label{Tab: full performance comparison}
\end{table}

\normalsize
\clearpage

\subsection{Figure of HNS}
\label{app: Figure of HNS}
In this part, we begin to visualize the HNS using GDI-H$^3$ and  GDI-I$^3$ in all 57 games. The HNS histogram of GDI-I$^3$ is illustrated in Fig. \ref{fig: HNS of GDII}. The HNS histogram of GDI-H$^3$ is illustrated in Fig. \ref{fig: HNS of GDIH}. In addition, we mark the error bars in the histogram with respect to the random seed after running experiments multiple times.

\begin{figure*}[!ht]
	\subfigure{
		\includegraphics[width=\textwidth,height=0.5\textheight]{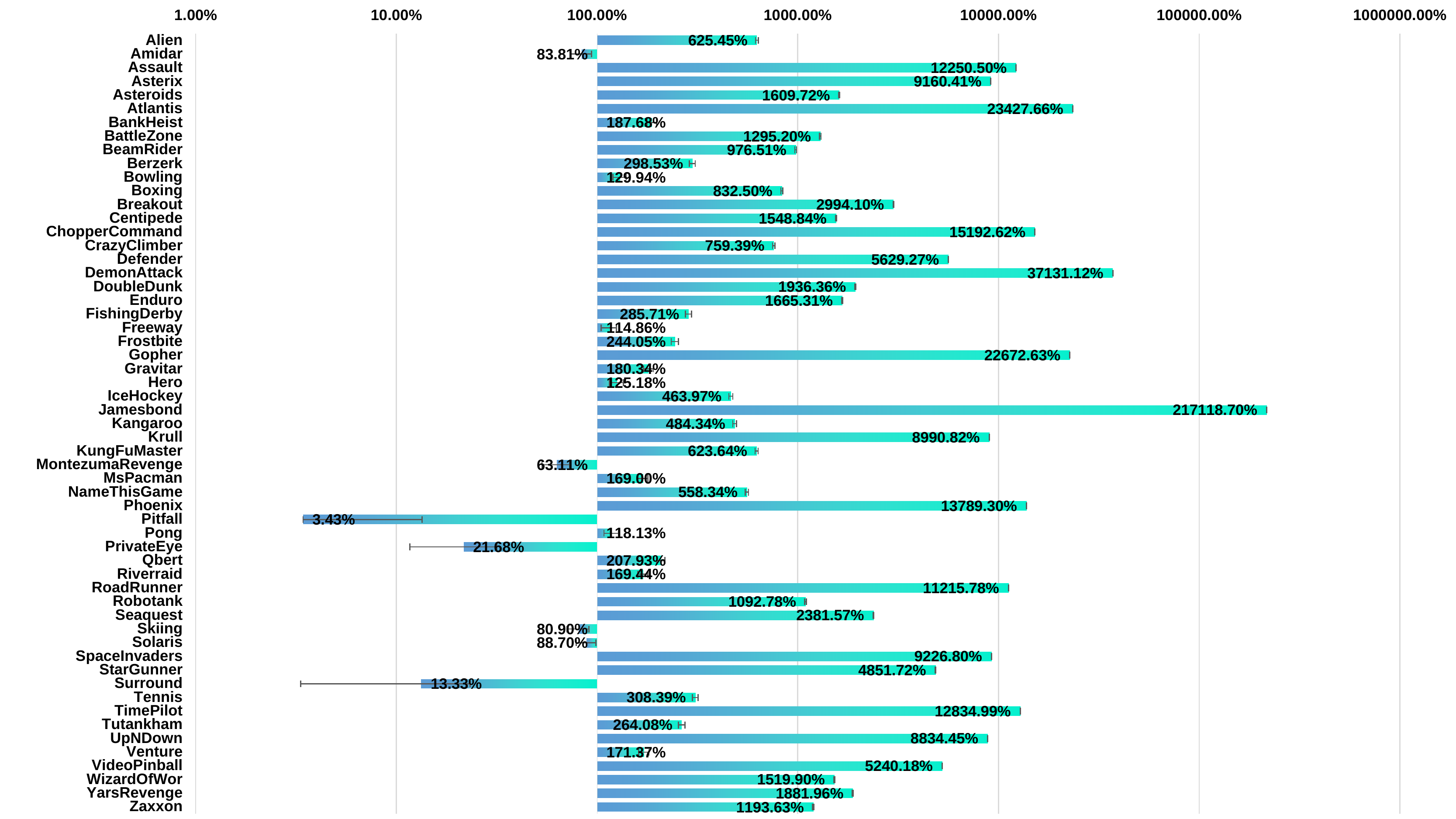}
	}
	\caption{HNS (\%) of Atari 57 games using GDI-I$^3$.}
	\label{fig: HNS of GDII}
\end{figure*}

\begin{figure*}[!ht]
	\subfigure{
		\includegraphics[width=\textwidth,height=0.5\textheight]{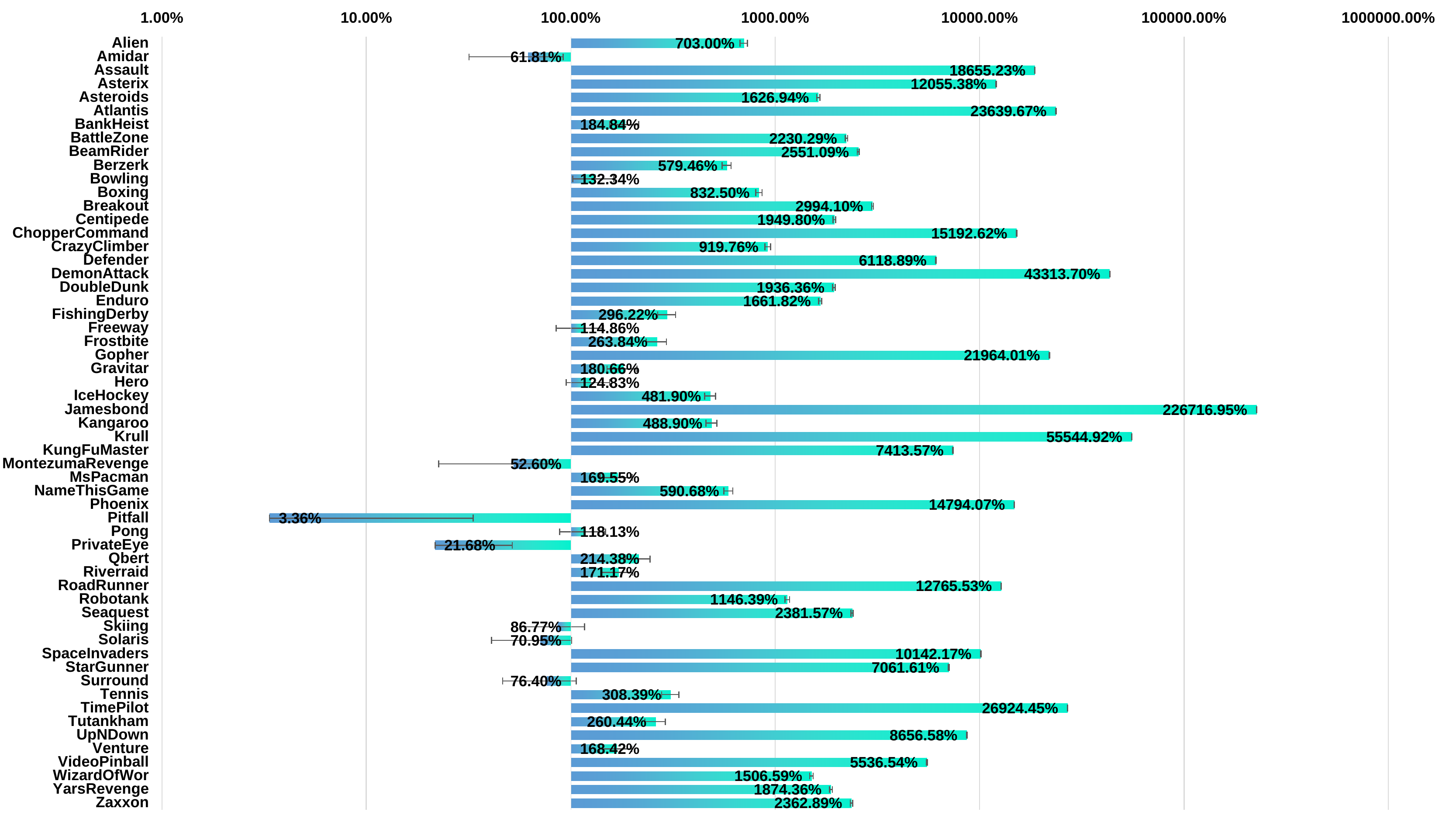}
	}
	\caption{HNS (\%) of Atari 57 games using GDI-H$^3$.}
	\label{fig: HNS of GDIH}
\end{figure*}

\clearpage

\subsection{Figure of HWRNS}
\label{app: Figure of HWRNS}
In this part, we begin to visualize the HWRNS \citep{dreamerv2,atarihuman} using GDI-H$^3$ and  GDI-I$^3$ in all 57 games. The HWRNS histogram of GDI-I$^3$ is illustrated in Fig. \ref{fig: HWRNS of GDII}. The HWRNS histogram of GDI-H$^3$ is illustrated in Fig. \ref{fig: HWRNS of GDIH}. In addition, we mark the error bars in the histogram with respect to the random seed after running experiments multiple times.

\begin{figure*}[!ht]
	\subfigure{
		\includegraphics[width=\textwidth,height=0.5\textheight]{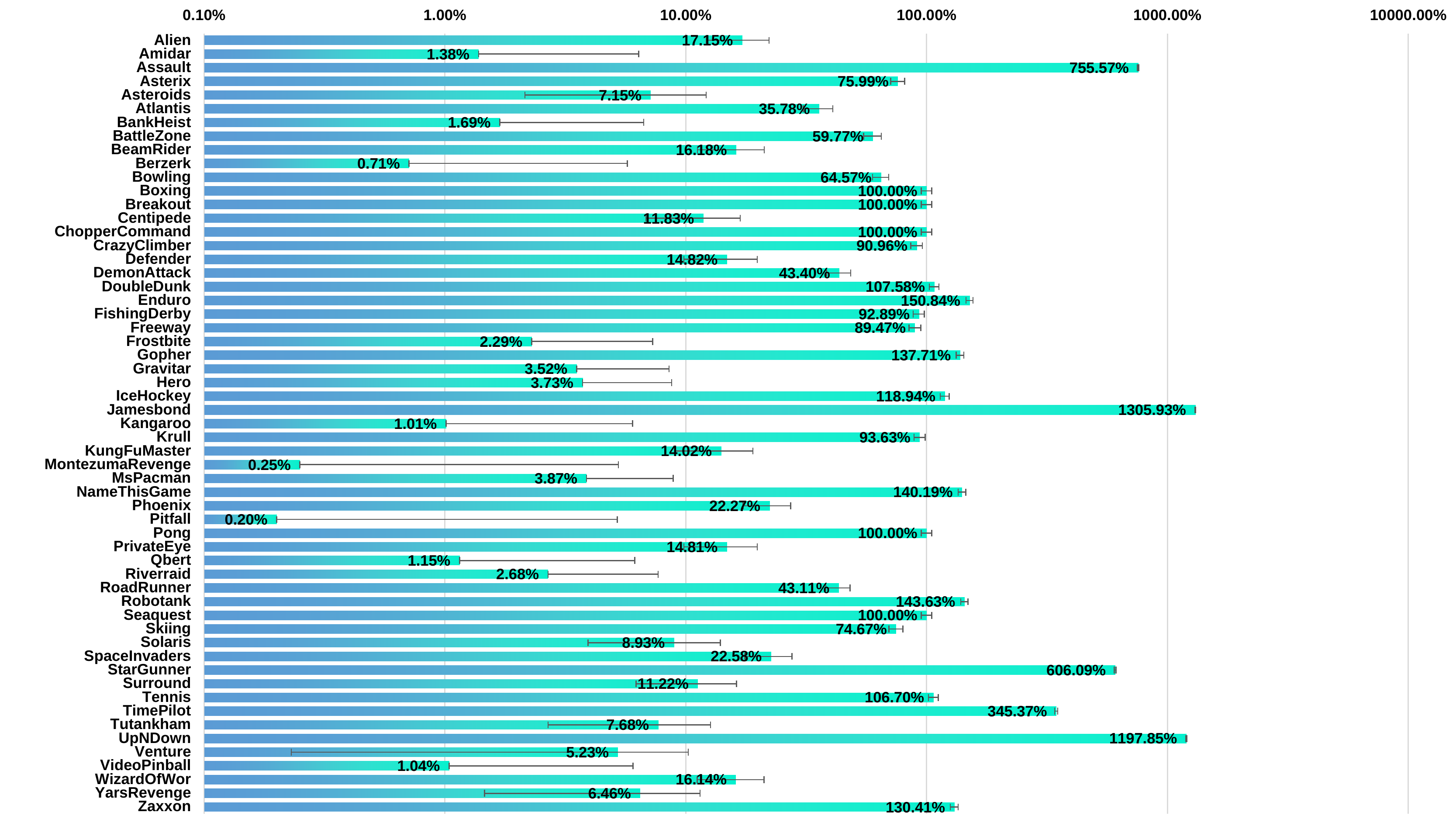}
	}
	\caption{HWRNS (\%) of Atari 57 games using GDI-I$^3$.}
	\label{fig: HWRNS of GDII}
\end{figure*}

\begin{figure*}[!ht]
	\subfigure{
		\includegraphics[width=\textwidth,height=0.5\textheight]{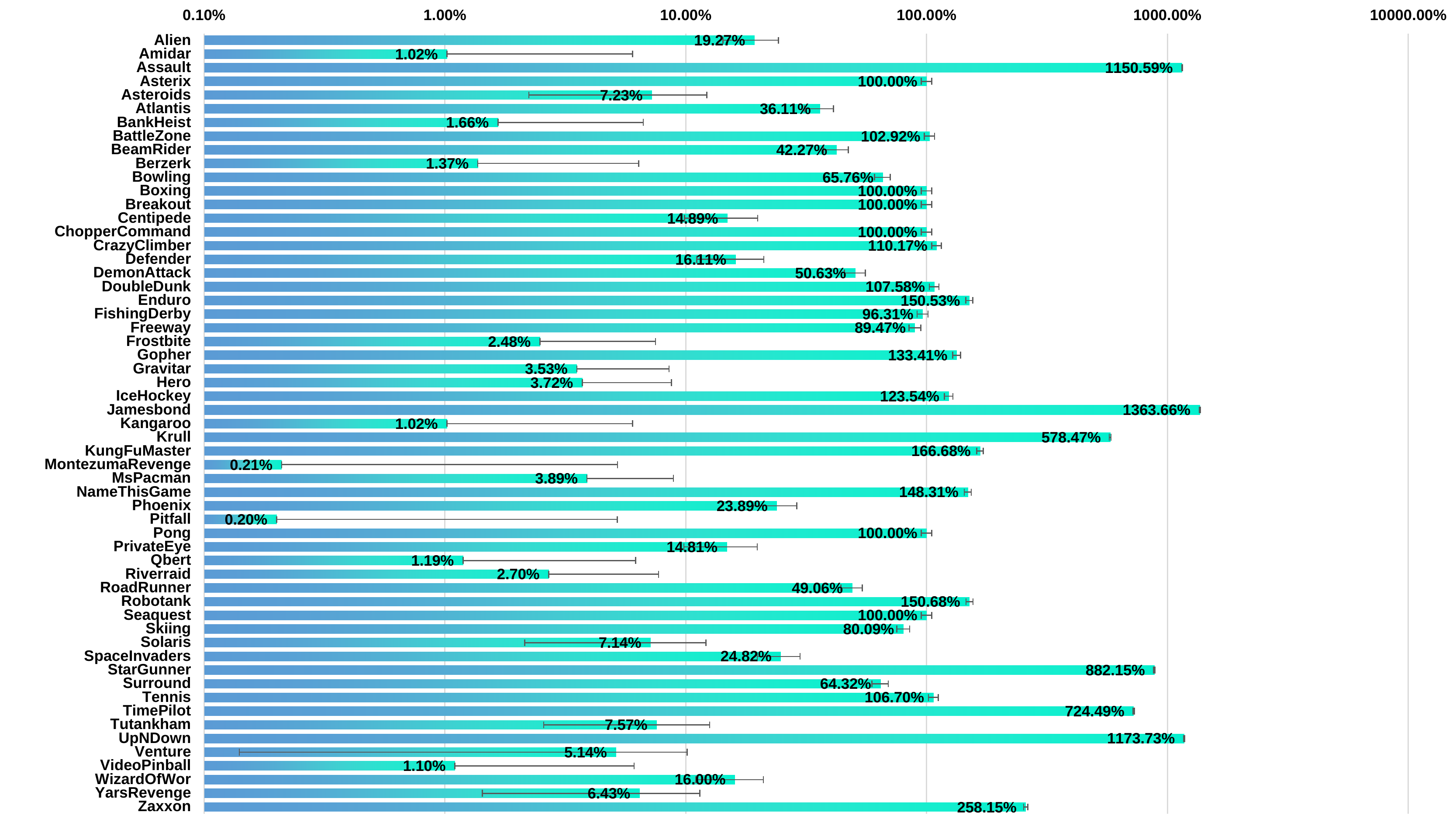}
	}
	\caption{HWRNS (\%) of Atari 57 games using GDI-H$^3$.}
	\label{fig: HWRNS of GDIH}
\end{figure*}

\clearpage

\subsection{Figure of SABER}
\label{app: Figure of SABER}

In this part, we begin to visualize the HWRNS \citep{dreamerv2,atarihuman} using GDI-H$^3$ and  GDI-I$^3$ in all 57 games. The HWRNS histogram of GDI-I$^3$ is illustrated in Fig. \ref{fig: SABER of GDII}. The HWRNS histogram of GDI-H$^3$ is illustrated in Fig. \ref{fig: SABER of GDIH}. In addition, we mark the error bars in the histogram with respect to the random seed after running experiments multiple times.

\begin{figure*}[!ht]
	\subfigure{
		\includegraphics[width=\textwidth,height=0.5\textheight]{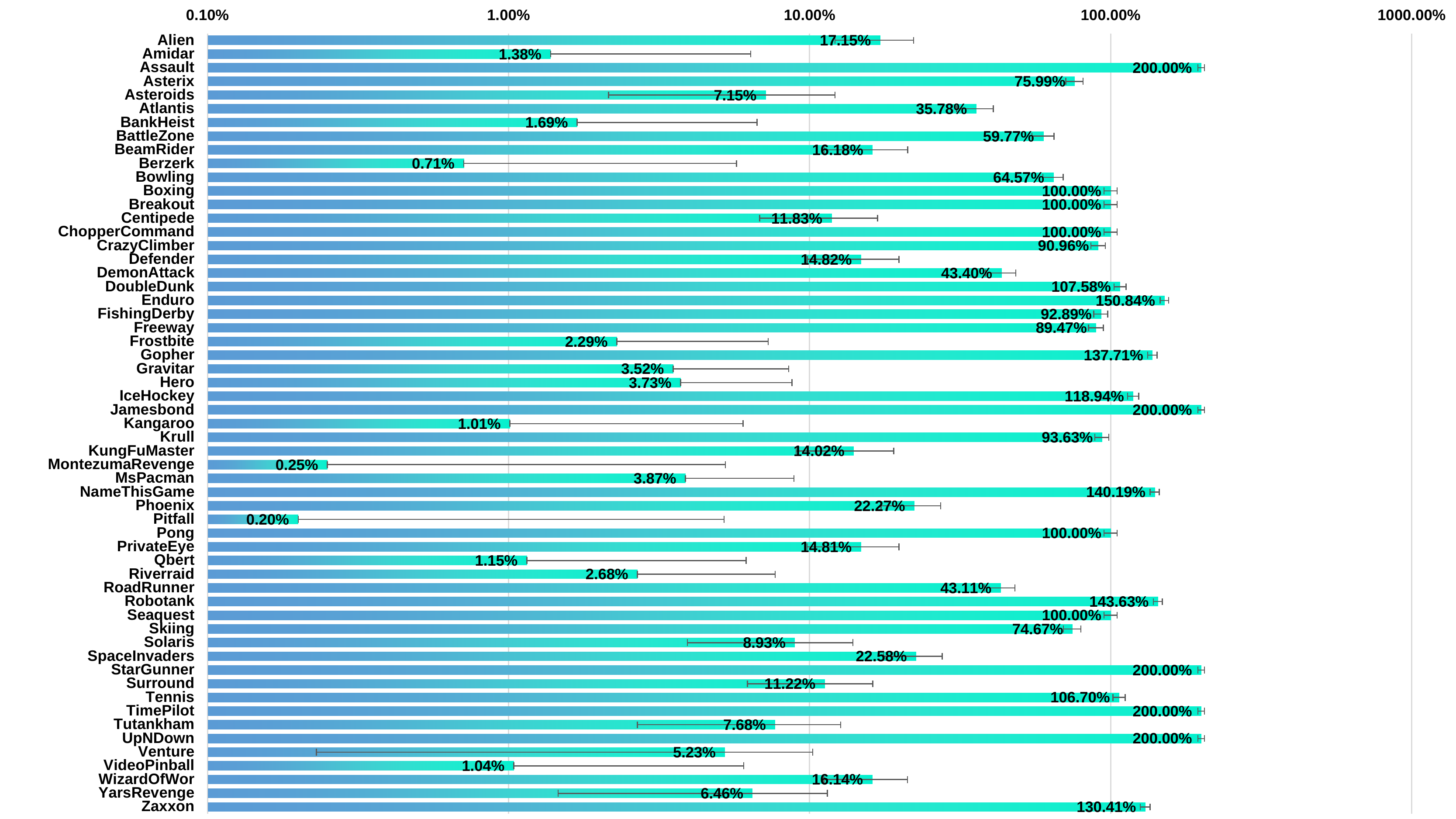}
	}
	\caption{SABER (\%) of Atari 57 games using GDI-I$^3$.}
	\label{fig: SABER of GDII}
\end{figure*}

\begin{figure*}[!ht]
	\subfigure{
		\includegraphics[width=\textwidth,height=0.5\textheight]{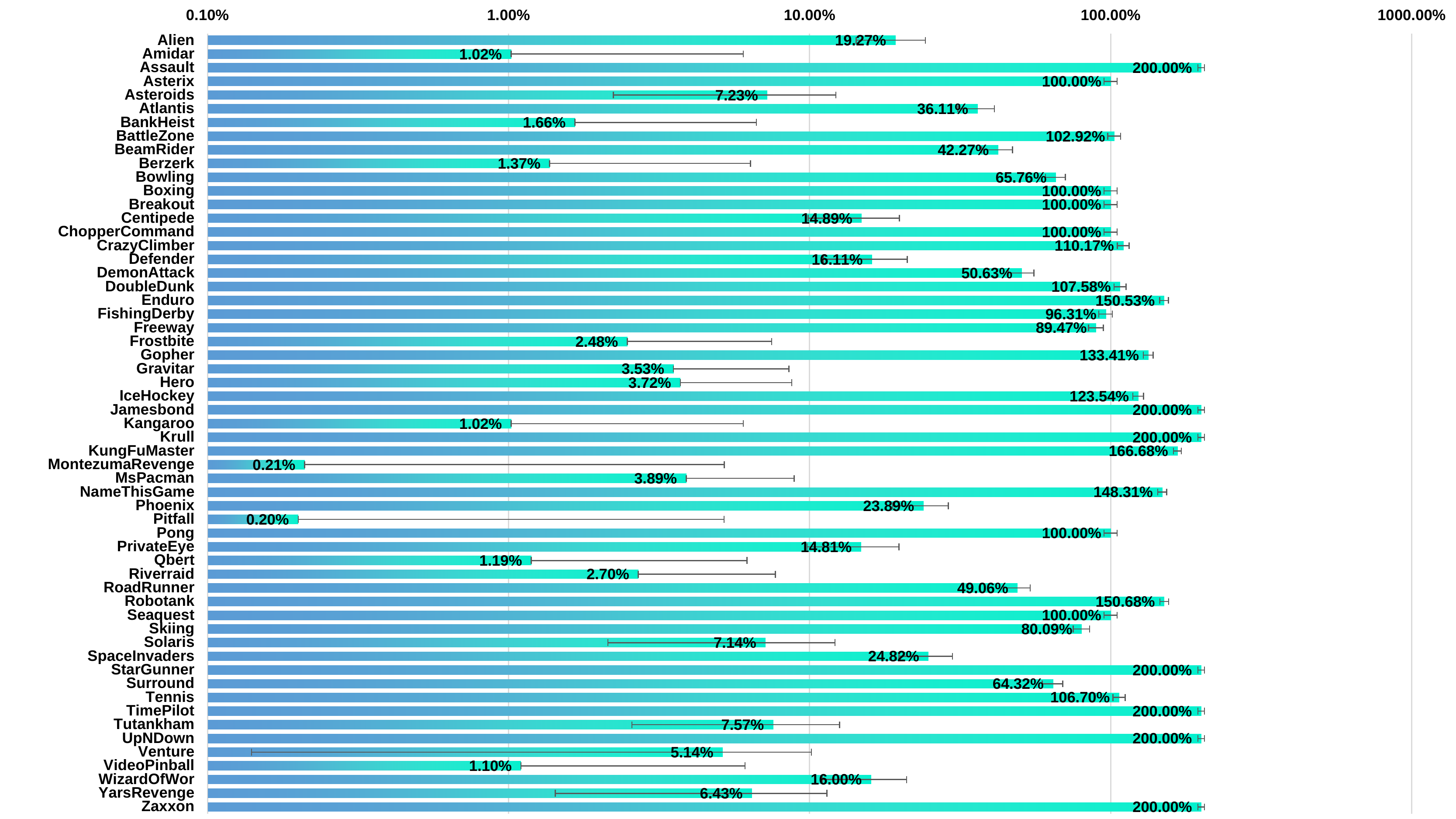}
	}
	\caption{SABER (\%) of Atari 57 games using GDI-H$^3$. }
	\label{fig: SABER of GDIH}
\end{figure*}
\clearpage

\subsection{Atari Games Table of Scores Based on Human Average Records}
\label{app: Atari Games Table of Scores Based on Human Average Records}
In this part, we detail the raw score of several representative SOTA algorithms including the SOTA 200M model-free algorithms, SOTA 10B+ model-free algorithms, SOTA model-based algorithms and other SOTA algorithms.\footnote{200M and 10B+ represent the training scale.} Additionally, we calculate the Human Normalized Score (HNS) of each game with each algorithms. First of all, we demonstrate the sources of the scores that we used.
Random scores and average human's scores are from \citep{agent57}.
Rainbow's scores are from \citep{rainbow}.
IMPALA's scores are from \citep{impala}.
LASER's scores are from \citep{laser}, no sweep at 200M.
As there are many versions of R2D2 and NGU, we use original papers'.
R2D2's scores are from \citep{r2d2}.
NGU's scores are from \citep{ngu}.
Agent57's scores are from \citep{agent57}.
MuZero's scores are from \citep{muzero}.
DreamerV2's scores are from \citep{dreamerv2}.
SimPLe's scores are form \citep{modelbasedatari}.
Go-Explore's scores are form \citep{goexplore}.
Muesli's scores are form \citep{muesli}.
In the following we detail the raw scores and HNS of each algorithms on 57 Atari games.
\clearpage

\subsubsection{Comparison with SOTA 200M Model-Free Algorithms on HNS}

\begin{table}[!hb]
\scriptsize
\begin{center}
\setlength{\tabcolsep}{1.0pt}
\begin{tabular}{ c c c c c c c c c c c c c}
\toprule
Games & RND & HUMAN & RAINBOW & HNS(\%) & IMPALA & HNS(\%) & LASER & HNS(\%) & GDI-I$^3$ & HNS(\%) & GDI-H$^3$ & HNS(\%)\\
\midrule
Scale  &     &       & 200M   &       &  200M    &        & 200M   &         &  200M   &  &  200M   &\\
\midrule
 alien  & 227.8 & 7127.8            & 9491.7 & 134.26 & 15962.1  & 228.03 & 35565.9 & 512.15                        &43384             &625.45         &\textbf{48735}             &\textbf{703.00}      \\
 amidar & 5.8   & 1719.5            & \textbf{5131.2} & \textbf{299.08} & 1554.79  & 90.39  & 1829.2  & 106.4       &1442              &83.81          &1065                       &61.81       \\
 assault & 222.4 & 742              & 14198.5 & 2689.78 & 19148.47 & 3642.43  & 21560.4 & 4106.62                   &63876             &12250.50       &\textbf{97155}             &\textbf{18655.23}    \\
 asterix & 210   & 8503.3           & 428200 & 5160.67 & 300732   & 3623.67  & 240090  & 2892.46                    &759910            &9160.41        &\textbf{999999}            &\textbf{12055.38}    \\
 asteroids & 719 & 47388.7          & 2712.8 & 4.27   & 108590.05 & 231.14  & 213025  &  454.91                     &751970            &1609.72        &\textbf{760005 }           &\textbf{1626.94}     \\
 atlantis & 12850 & 29028.1         & 826660 & 5030.32 & 849967.5 & 5174.39 & 841200 & 5120.19                      &3803000           &23427.66       &\textbf{3837300}           &\textbf{23639.67}     \\
 bank heist & 14.2 & 753.1          & 1358   & 181.86  & 1223.15  & 163.61  & 569.4  & 75.14                        &\best{1401}       &\best{187.68}  &1380                       &184.84       \\
 battle zone & 236 & 37187.5        & 62010 & 167.18  & 20885    & 55.88  & 64953.3 & 175.14                        &478830            &1295.20        &\textbf{824360}            &\textbf{2230.29}      \\
 beam rider & 363.9 & 16926.5       & 16850.2 & 99.54 & 32463.47 & 193.81 & 90881.6 & 546.52                        &162100            &976.51         &\textbf{422890}            &\textbf{2551.09}      \\
 berzerk & 123.7 & 2630.4           & 2545.6   & 96.62  & 1852.7   & 68.98  & \textbf{25579.5} & \textbf{1015.51}   &7607              &298.53         &14649             &579.46       \\
 bowling & 23.1 & 160.7             & 30   & 5.01        & 59.92    & 26.76  & 48.3    & 18.31                      &201.9             &129.94         &\textbf{205.2}             &\textbf{132.34}       \\
 boxing  & 0.1  & 12.1              & 99.6 & 829.17      & 99.96    & 832.17 & \textbf{100} & \textbf{832.5}        &\best{100}        &\best{832.50}  &\textbf{100}               &\textbf{832.50}       \\
 breakout & 1.7 & 30.5              & 417.5 & 1443.75    & 787.34   & 2727.92 & 747.9 & 2590.97                     &\best{864}        &\best{2994.10} &\textbf{864}               &\textbf{2994.10}      \\
 centipede & 2090.9 & 12017         & 8167.3 & 61.22   & 11049.75 & 90.26   & \textbf{292792} & \textbf{2928.65}    &155830            &1548.84        &195630                     &1949.80      \\
 chopper command & 811 & 7387.8     & 16654 & 240.89 & 28255  & 417.29  & 761699 & 11569.27                         &\best{999999}     &\best{15192.62}&\textbf{999999}            &\textbf{15192.62}     \\
 crazy climber & 10780.5 & 36829.4  & 168788.5 & 630.80 & 136950 & 503.69 & 167820  & 626.93                        &201000            &759.39         &\textbf{241170}            &\textbf{919.76}  \\
 defender & 2874.5 & 18688.9        & 55105 & 330.27 & 185203 & 1152.93 & 336953  & 2112.50                         &893110            &5629.27        &\textbf{970540}            &\textbf{6118.89}     \\
 demon attack & 152.1 & 1971        & 111185 & 6104.40 & 132826.98 & 7294.24 & 133530 & 7332.89                     &675530                 &37131.12   &\textbf{787985}                     &\textbf{43313.70}     \\
 double dunk & -18.6 & -16.4        & -0.3   & 831.82  & -0.33     & 830.45  & 14     & 1481.82                     &\best{24}         &\best{1936.36} &\textbf{24}                &\textbf{1936.36}     \\
 enduro      & 0   & 860.5          & 2125.9 & 247.05  & 0       & 0.00     & 0    & 0.00                           &\best{14330}      &\best{1665.31 }&14300                      &1661.82      \\
 fishing derby & -91.7 & -38.8      & 31.3 & 232.51  & 44.85   & 258.13    & 45.2   & 258.79                        &59         &285.71                &\textbf{65}                &\textbf{296.22}   \\
 freeway       & 0     & 29.6       & \textbf{34} & \textbf{114.86}  & 0     & 0.00       & 0    & 0.00             &\best{34}          &\best{114.86   }&\textbf{34}               &\textbf{114.86}   \\
 frostbite     & 65.2  & 4334.7     & 9590.5 & 223.10 & 317.75 & 5.92     & 5083.5 & 117.54                         &10485              &244.05          &\textbf{11330}            &\textbf{263.84}   \\
 gopher  & 257.6 & 2412.5           & 70354.6 & 3252.91    & 66782.3 & 3087.14 & 114820.7 & 5316.40                 &\best{488830}     &\best{22672.63} &473560           &21964.01    \\
 gravitar & 173 & 3351.4            & 1419.3  & 39.21   & 359.5      & 5.87    & 1106.2   & 29.36                   &5905       &180.34                 &\textbf{5915}             &\textbf{180.66}   \\
 hero   & 1027 & 30826.4            & \textbf{55887.4} & \textbf{184.10}   & 33730.55  & 109.75  & 31628.7 & 102.69 &38330             &125.18          &38225            &124.83    \\
 ice hockey & -11.2 & 0.9           & 1.1    & 101.65   & 3.48      & 121.32   & 17.4    & 236.36                   &44.94      &463.97&\textbf{47.11}           &\textbf{481.90}   \\
 jamesbond  & 29    & 302.8         & 19809 & 72.24   & 601.5     & 209.09   & 37999.8 & 13868.08                   &594500     &217118.70 &\textbf{620780}          &\textbf{226716.95}    \\
 kangaroo   & 52    & 3035          & \textbf{14637.5} & \textbf{488.05} & 1632    & 52.97    & 14308   & 477.91    &14500             &484.34           &14636           &488.00    \\
 krull     & 1598   & 2665.5        & 8741.5  & 669.18 & 8147.4  & 613.53   & 9387.5  &  729.70                     &97575      &8990.82   &\textbf{594540}          &\textbf{55544.92}     \\
 kung fu master & 258.5 & 22736.3   & 52181 & 230.99 & 43375.5 & 191.82 & 607443 & 2701.26        &140440            &623.64           &\textbf{1666665}          &\textbf{7413.57}         \\
 montezuma revenge&0&\textbf{4753.3}& 384   & 8.08   & 0       & 0.00   & 0.3    & 0.01                             &3000              &63.11            &2500            &52.60   \\
 ms pacman  & 307.3 & 6951.6        & 5380.4  & 76.35   & 7342.32 & 105.88 & 6565.5   & 94.19                       &11536      &169.00    &\textbf{11573}           &\textbf{169.55}    \\
 name this game & 2292.3 & 8049     & 13136 & 188.37   & 21537.2 & 334.30 & 26219.5 & 415.64                        &34434      &558.34    &\textbf{36296}           &\textbf{590.68}    \\
 phoenix & 761.5 & 7242.6  & 108529 & 1662.80   & 210996.45  & 3243.82 & 519304 & 8000.84                           &894460     &13789.30  &\textbf{959580}          &\textbf{14794.07}  \\
 pitfall & -229.4 & \textbf{6463.7} & 0      & 3.43      & -1.66      & 3.40    & -0.6   & 3.42                     &0                 &3.43             &-4.345            &3.36 \\
 pong    & -20.7  & 14.6   & 20.9   & 117.85    & 20.98      & 118.07  & \textbf{21}     &  \textbf{118.13}         &\best{21   }      &\best{118.13}    &\textbf{21}              &\textbf{118.13}    \\
 private eye & 24.9&\textbf{69571.3}& 4234 & 6.05     & 98.5       & 0.11    & 96.3   & 0.10                        &15100             &21.68            &15100           &21.68       \\
 qbert  & 163.9 & 13455.0 & 33817.5 & 253.20   & \textbf{351200.12}  & \textbf{2641.14} & 21449.6 & 160.15          &27800             &207.93           &28657           &214.38    \\
 riverraid & 1338.5 & 17118.0       & 22920.8 & 136.77 & 29608.05  & 179.15  & \textbf{40362.7} & \textbf{247.31}   &28075             &169.44           &28349           &171.17    \\
 road runner & 11.5 & 7845          & 62041   & 791.85 & 57121     & 729.04  & 45289   & 578.00                     &878600            &11215.78         &\textbf{999999} &\textbf{12765.53} \\
 robotank   & 2.2   & 11.9          & 61.4   & 610.31    & 12.96     & 110.93  & 62.1    & 617.53                   &108.2             &1092.78          &\textbf{113.4}           &\textbf{1146.39}  \\
 seaquest  & 68.4 & 42054.7         & 15898.9 & 37.70    & 1753.2    & 4.01    & 2890.3  & 6.72                     &943910	           &2247.98  &\textbf{1000000}          &\textbf{2381.57}\\
 skiing & -17098  & \textbf{-4336.9}& -12957.8 & 32.44  & -10180.38 & 54.21   & -29968.4 & -100.86                  &-6774             &80.90            &-6025	          &86.77   \\
 solaris & 1236.3 & \textbf{12326.7}& 3560.3  & 20.96  & 2365      & 10.18   & 2273.5   & 9.35                      &11074             &88.70            &9105            &70.95   \\
 space invaders & 148 & 1668.7      & 18789 & 1225.82 & 43595.78 & 2857.09 & 51037.4 & 3346.45                      &140460            &9226.80          &\textbf{154380} &\textbf{10142.17}   \\
 star gunner & 664 & 10250          & 127029    & 1318.22 & 200625   & 2085.97 & 321528  & 3347.21                  &465750            &4851.72          &\textbf{677590} &\textbf{7061.61}    \\
 surround    & -10 & 6.5            & \textbf{9.7}       & \textbf{119.39}  & 7.56     & 106.42  & 8.4     & 111.52 &-7.8              &13.33            &2.606           &76.40    \\
 tennis  & -23.8   & -8.3           & 0        & 153.55    & 0.55     & 157.10  & 12.2    & 232.26                  &\best{24       }  &\best{308.39   } &\textbf{24}              &\textbf{308.39}  \\
 time pilot & 3568 & 5229.2         & 12926 & 563.36     & 48481.5  & 2703.84 & 105316  & 6125.34                   &216770     &12834.99         &\textbf{450810}          &\textbf{26924.45}   \\
 tutankham  & 11.4 & 167.6          & 241   & 146.99     & 292.11   & 179.71  & 278.9   & 171.25                    &\best{423.9 }     &\best{264.08   } &418.2           &260.44  \\
 up n down  & 533.4 & 11693.2       & 125755 & 1122.08 & 332546.75 & 2975.08 & 345727 & 3093.19                     &\best{986440}     &\best{8834.45 }  &966590          &8656.58    \\
 venture    & 0     & 1187.5        & 5.5    & 0.46    & 0         & 0.00    & 0      & 0.00                        &\best{2035     }  &\best{171.37   } &2000            &168.42   \\
 video pinball & 0 & 17667.9        & 533936.5 & 3022.07 & 572898.27 & 3242.59 & 511835 & 2896.98                   &925830     &5240.18                 &\textbf{978190} &\textbf{5536.54}     \\
 wizard of wor & 563.5 & 4756.5     & 17862.5 & 412.57 & 9157.5    & 204.96  & 29059.3 & 679.60                     &\best{64239 }     &\best{1519.90 }  &63735           &1506.59     \\
 yars revenge & 3092.9 & 54576.9    & 102557 & 193.19 & 84231.14  & 157.60 & 166292.3  & 316.99                     &\textbf{972000}     &\textbf{1881.96}   &968090          &1874.36     \\
 zaxxon       & 32.5   & 9173.3     & 22209.5 & 242.62 & 32935.5   & 359.96 & 41118    & 449.47                     &109140     &1193.63   &\textbf{216020} &\textbf{2362.89}     \\
\hline
MEAN HNS(\%)        &     0.00 & 100.00   &         & 873.97 &         & 957.34  &        & 1741.36 &      & \GDIImeanhns       &      & \textbf{\GDIHmeanhns} \\
Learning Efficiency &     0.00 & N/A   &         & 4.37E-08 &         & 4.79E-08  &        & 8.71E-08 &      & 3.91E-07       &      & \textbf{4.70E-07} \\
\hline
MEDIAN HNS(\%)      & 0.00   & 100.00   &         & 230.99 &         & 191.82  &        & 454.91  &      & \GDIImedianhns        &      & \textbf{\GDIHmedianhns} \\
Learning Efficiency & 0.00   & N/A   &         & 1.15E-08 &         & 9.59E-09  &        & 2.27E-08 &      & 4.16E-08       &      & \textbf{5.73E-08} \\
\bottomrule
\end{tabular}
\caption{Score table of SOTA 200M model-free algorithms on HNS.}
\end{center}
\end{table}
\clearpage

\subsubsection{Comparison with SOTA 10B+ model-free algorithms on HNS}

\begin{table}[!hb]
\scriptsize
\begin{center}
\setlength{\tabcolsep}{1.0pt}
\begin{tabular}{ c c c c c c c c c c c}
\toprule
 Games & R2D2 & HNS(\%) & NGU & HNS(\%) & AGENT57 & HNS(\%) & GDI-I$^3$ & HNS(\%) & GDI-H$^3$ & HNS(\%) \\
\midrule
Scale  & 10B   &        & 35B &         & 100B     &        & 200M &     &  200M   &\\
\midrule
 alien  & 109038.4 & 1576.97 & 248100 & 3592.35 & \textbf{297638.17} & \textbf{4310.30}             &43384             &625.45                &48735             &703.00       \\
 amidar & 27751.24 & 1619.04 & 17800  & 1038.35 & \textbf{29660.08}  & \textbf{1730.42}             &1442              &83.81                 &1065              &61.81        \\
 assault &  90526.44 & 17379.53 & 34800 & 6654.66 & 67212.67 & 12892.66            &63876           &12250.50          &\textbf{97155}        &\textbf{18655.23}     \\
 asterix &  999080   & 12044.30 & 950700 & 11460.94 & 991384.42 & 11951.51         &759910          &9160.41           &\textbf{999999}       &\textbf{12055.38}     \\
 asteroids & 265861.2 & 568.12 & 230500 & 492.36   & 150854.61 & 321.70                             &751970            &1609.72               &\textbf{760005}            &\textbf{1626.94}      \\
 atlantis & 1576068   & 9662.56 & 1653600 & 10141.80 & 1528841.76 & 9370.64                         &3803000           &23427.66              &\textbf{3837300}           &\textbf{23639.67}     \\
 bank heist & \textbf{46285.6} & \textbf{6262.20} & 17400   & 2352.93  & 23071.5& 3120.49           &1401              &187.68                &1380              &184.84       \\
 battle zone & 513360 & 1388.64 & 691700  & 1871.27  & \textbf{934134.88}& \textbf{2527.36}         &478830            &1295.20               &824360            &2230.29      \\
 beam rider & 128236.08 & 772.05 & 63600  & 381.80   & 300509.8 & 1812.19         &162100            &976.51                &\textbf{422390}            &\textbf{2548.07}      \\
 berzerk & 34134.8      & 1356.81 & 36200 & 1439.19  & \textbf{61507.83} & \textbf{2448.80}         &7607              &298.53                &14649             &579.46       \\
 bowling & 196.36       & 125.92  & 211.9 & 137.21   & \textbf{251.18}   & \textbf{165.76}          &201.9             &129.94                &205.2             &132.34       \\
 boxing  & 99.16        & 825.50  & 99.7  & 830.00   & \textbf{100}      & \textbf{832.50}          &\best{100}        &\best{832.50}         &\textbf{100}               &\textbf{832.50}       \\
 breakout & 795.36      & 2755.76 & 559.2 & 1935.76  & 790.4 & 2738.54                  &\best{864}        &\best{2994.10}                    &\textbf{864}             &\textbf{2994.10}      \\
 centipede & 532921.84  & 5347.83 & \textbf{577800} & \textbf{5799.95} & 412847.86& 4138.15         &155830            &1548.84               &195630            &1949.80\\
 chopper command&960648&14594.29&999900&15191.11&999900&15191.11                                    &\best{999999}&\best{15192.62}            &\textbf{999999}            &\textbf{15192.62}\\
 crazy climber & 312768   & 1205.59  & 313400 & 1208.11&\textbf{565909.85}&\textbf{2216.18}         &201000            &759.39                &241170            &919.76\\
 defender & 562106        & 3536.22  & 664100 & 4181.16  & 677642.78 & 4266.80                      &893110     &5629.27                      &\textbf{970540}            &\textbf{6118.89}\\
 demon attack & 143664.6  & 7890.07  & 143500 & 7881.02  & 143161.44 & 7862.41                      &675530     &37131.12       &\textbf{787985}                     &\textbf{43313.70}\\
 double dunk & 23.12      & 1896.36  & -14.1  & 204.55   & 23.93& 1933.18         &\textbf{24}      &\textbf{1936.36}                         &\textbf{24}                &\textbf{1936.36}\\
 enduro      & 2376.68    & 276.20   & 2000   & 232.42   & 2367.71   & 275.16                       &\best{14330}      &\best{1665.31}        &14300             &1661.82\\
 fishing derby & 81.96    & 328.28   & 32     & 233.84   & \textbf{86.97}& \textbf{337.75}          &59                &285.71                &65               &296.22\\
 freeway       & \textbf{34}       & \textbf{114.86}   & 28.5   & 96.28    & 32.59& 110.10          &\best{34}         &\best{114.86}         &\textbf{34}               &\textbf{114.86}\\
 frostbite    & 11238.4  & 261.70   & 206400 & 4832.76&\textbf{541280.88}&\textbf{12676.32}         &10485             &244.05                &11330	           &263.84\\
 gopher  & 122196        & 5658.66  & 113400 & 5250.47  & 117777.08 & 5453.59                       &\best{488830}     &\best{22672.63}       &473560           &21964.01\\
 gravitar & 6750         & 206.93   & 14200  & 441/32  &\textbf{19213.96}&\textbf{599.07}           &5905              &180.34                &5915             &180.66\\
 hero   & 37030.4        & 120.82   & 69400  & 229.44&\textbf{114736.26}&\textbf{381.58}            &38330             &125.18                &38225	            &124.83\\
 ice hockey & \textbf{71.56}      & \textbf{683.97}   &-4.1   & 58.68    & 63.64& 618.51            &44.94             &463.97                &47.11           &481.90\\
 jamesbond  & 23266      & 8486.85  & 26600  & 9704.53  & 135784.96 & 49582.16                      &594500            &217118.70             &\textbf{620780	}          &\textbf{226716.95}\\
 kangaroo   & 14112      & 471.34   & \textbf{35100}  & \textbf{1174.92}&24034.16& 803.96           &14500             &484.34                &14636           &488.90\\
 krull     & 145284.8    & 13460.12 & 127400&11784.73& 251997.31&23456.61         &97575             &8990.82                                 &\textbf{594540}          &\textbf{55544.92}\\
 kung fu master & 200176 & 889.40   & 212100 & 942.45   & 206845.82 & 919.07      &140440            &623.64                &\textbf{1666665}	         &\textbf{7413.57}\\
 montezuma revenge & 2504 & 52.68   & \textbf{10400}  & \textbf{218.80} &9352.01& 196.75            &3000              &63.11                 &2500            &52.60\\
 ms pacman  & 29928.2     & 445.81  & 40800  & 609.44& \textbf{63994.44}&\textbf{958.52}            &11536             &169.00                &11573           &169.55\\
 name this game & 45214.8 & 745.61  & 23900  & 375.35&\textbf{54386.77}&\textbf{904.94}             &34434             &558.34                &36296           &590.68\\
 phoenix & 811621.6       & 125.11  & 959100 &14786.66 &908264.15&14002.29         &894460            &13789.30              &\textbf{959580}	          &\textbf{14794.07}\\
 pitfall & 0              & 3.43    & 7800   & 119.97&\textbf{18756.01}&\textbf{283.66}             &0                 &3.43                  &-4.3            &3.36\\
 pong    & \textbf{21}             & \textbf{118.13}  & 19.6   & 114.16   & 20.67& 117.20           &\best{21}         &\best{118.13}         &\textbf{21}     &\textbf{118.13}\\
 private eye & 300        & 0.40    & \textbf{100000} & \textbf{143.75}& 79716.46&114.59            &15100             &21.68                 &15100           &21.68\\
 qbert  & 161000          & 1210.10 & 451900 & 3398.79&\textbf{580328.14}&\textbf{4365.06}          &27800             &207.93                &28657           &214.38\\
 riverraid & 34076.4      & 207.47  & 36700  & 224.10 & \textbf{63318.67}&\textbf{392.79}           &28075             &169.44                &28349           &171.17\\
 road runner & 498660     & 6365.59 & 128600 & 1641.52  & 243025.8&3102.24                          &878600            &11215.78       &\textbf{999999} &\textbf{12765.53}\\
 robotank   & \textbf{132.4}       & \textbf{1342.27} & 9.1    & 71.13 &127.32 &1289.90             &108.2             &1092.78               &113.4           &1146.39\\
 seaquest  & 999991.84    & 2381.55 & \textbf{1000000} & \textbf{2381.57}&999997.63&2381.56         &943910	           &2247.98        &\textbf{1000000}          &\textbf{2381.57}\\
 skiing & -29970.32       & -100.87 & -22977.9 & -46.08 & \textbf{-4202.6}  &\textbf{101.05}        &-6774             &80.90                 &-6025	       &86.77\\
 solaris & 4198.4         & 26.71   & 4700     & 31.23  & \textbf{44199.93}& \textbf{387.39}        &11074             &88.70                 &9105            &70.95\\
 space invaders & 55889   & 3665.48 & 43400    & 2844.22 & 48680.86 & 3191.48                       &140460            &9226.80               &\textbf{154380} &\textbf{10142.17}\\
 star gunner & 521728     & 5435.68 & 414600   &4318.13&\textbf{839573.53}&\textbf{8751.40}         &465750            &4851.72               &677590	         &7061.61\\
 surround    & \textbf{9.96}       & \textbf{120.97}  & -9.6     & 2.42    & 9.5&118.18             &-7.8              &13.33                 &2.606           &76.40\\
 tennis  & \textbf{24}             & \textbf{308.39}  & 10.2     & 219.35 & 23.84& 307.35           &\best{24}         &\best{308.39}         &\textbf{24}     &\textbf{308.39}             \\
 time pilot & 348932    & 20791.28 & 344700 & 20536.51&405425.31&24192.24         &216770            &12834.99              &\textbf{450810}	          &\textbf{26924.45}\\
 tutankham  & 393.64    & 244.71   & 191.1   & 115.04   & \textbf{2354.91}&\textbf{1500.33}         &423.9             &264.08                &418.2           &260.44\\
 up n down  & 542918.8  & 4860.17  & 620100  & 5551.77  & 623805.73 & 5584.98                       &\best{986440}     &\best{8834.45}        &966590          &8656.58\\
 venture    & 1992      & 167.75   & 1700    & 143.16   &\textbf{2623.71}  &\textbf{220.94}         &2035              &171.37                &2000	            &168.42\\
 video pinball & 483569.72 & 2737.00 & 965300 & 5463.58 &\textbf{992340.74}&\textbf{5616.63}        &925830            &5240.18               &978190          &5536.54\\
 wizard of wor & 133264 & 3164.81  & 106200  & 2519.35  &\textbf{157306.41}&\textbf{3738.20}        &64293             &1519.90               &63735           &1506.59\\
 yars revenge & 918854.32 & 1778.73 & 986000 & 1909.15  &\textbf{998532.37}&\textbf{1933.49}        &972000            &1881.96               &968090          &1874.36\\
 zaxxon & 181372        & 1983.85  & 111100  & 1215.07  &\textbf{249808.9} &\textbf{2732.54}        &109140            &1193.63               &216020	         &2362.89\\
\hline
MEAN HNS(\%)            &               & 3374.31  &         &  3169.90 &           & 4763.69  &     & \GDIImeanhns &      & \textbf{\GDIHmeanhns} \\
Learning Efficiency     &               &3.37E-09  &         & 9.06E-10  &        &  4.76E-10 &      & 3.91E-07       &      & \textbf{\GDIHmeanhnsle} \\
\hline
MEDIAN HNS(\%) &            & 1342.27   &         & 1208.11  &           & \textbf{1933.49}  &     & \GDIImedianhns &      & \GDIHmedianhns \\
Learning Efficiency     &    & 1.34E-09 &         & 3.45E-10  &        & 1.93E-10&      & 4.16E-08       &      & \textbf{\GDIHmedianhnsle} \\
\bottomrule
\end{tabular}
\caption{Score table of SOTA  model-free algorithms on HNS.}
\end{center}
\end{table}
\clearpage

\subsubsection{Comparison with SOTA Model-Based Algorithms on HNS}
SimPLe \citep{modelbasedatari} and DreamerV2\citep{dreamerv2} haven't evaluated all 57 Atari Games in their paper. For fairness, we set the score on those games as N/A, which will not be considered when calculating the median and mean HNS.
\begin{table}[!hb]
\scriptsize
\begin{center}
\setlength{\tabcolsep}{1.0pt}
\begin{tabular}{c c c c c c c c c c c}
\toprule
 Games              & MuZero         & HNS(\%)      & DreamerV2 & HNS(\%)    & SimPLe             & HNS(\%)          & GDI-I$^3$     & HNS(\%) & GDI-H$^3$ & HNS(\%) \\
\midrule
Scale               & 20B            &              & 200M      &            & 1M               &                  & 200M     &      &  200M   &\\
\midrule    
 alien              & \textbf{741812.63}      & \textbf{10747.61}     &3483       & 47.18      &616.9     & 5.64    & 43384       & 625.45                    &48735	             &703.00      \\
 amidar             & \textbf{28634.39 }      & \textbf{1670.57    }  &2028       & 118.00     &74.3      & 4.00    & 1442        & 83.81                     &1065              &61.81 \\
 assault            & \textbf{143972.03}      & \textbf{27665.44}     &7679       & 1435.07    &527.2     & 58.66   & 63876       & 12250.50                  &97155	             &18655.23\\
 asterix            & 998425                  & 12036.40     &25669      & 306.98     &1128.3    & 11.07   & 759910      & 9160.41                   &\textbf{999999}   &\textbf{12055.38} \\
 asteroids          & 678558.64             & 1452.42      &3064       & 5.02       &793.6     & 0.16               &751970& 1609.72           &\textbf{760005}            &\textbf{1626.94}  \\
 atlantis           & 1674767.2             & 10272.64     &989207     & 6035.05    &20992.5   & 50.33              &3803000&23427.66          &\textbf{3837300}           &\textbf{23639.67}   \\
 bank heist         & 1278.98               & 171.17       &1043       & 139.23     &34.2      & 2.71               &\best{1401}  & \best{187.68  }           &1380              &184.84 \\
 battle zone        & \textbf{848623         }& \textbf{2295.95}      &31225      & 83.86      &4031.2    & 10.27   & 478830      & 1295.20                   &824360            &2230.29\\
 beam rider         & \textbf{454993.53}      & \textbf{2744.92}      &12413      & 72.75      &621.6     & 1.56    & 162100      & 976.51                    &422390            &2548.07\\
 berzerk            & \textbf{85932.6        }& \textbf{3423.18}      &751        & 25.02      &N/A       & N/A     & 7607        & 298.53                    &14649             &579.46\\
 bowling            & \textbf{260.13         }& \textbf{172.26 }      &48         & 18.10      &30        & 5.01    & 202         & 129.94                    &205.2             &132.34\\
 boxing             & \textbf{100}                   & \textbf{832.50}       &87         & 724.17     &7.8       & 64.17   & \best{100}  & \best{832.50  }    &\textbf{100}      &\textbf{832.50}  \\
 breakout           & \textbf{864}                   & \textbf{2994.10}      &350        & 1209.38    &16.4      & 51.04   & \best{864}  & \best{2994.10    } &\textbf{864}      &\textbf{2994.10}\\
 centipede          & \textbf{1159049.27}     & \textbf{11655.72}     &6601       & 45.44      &N/A       & N/A     & 155830      & 1548.84                   &195630            &1949.80\\
 chopper command    & 991039.7              & 15056.39     &2833       & 30.74      & 979.4    & 2.56    & \best{999999}& \best{15192.62}                     &\textbf{999999}   &\textbf{15192.62}\\
 crazy climber      & \textbf{458315.4}       & \textbf{1786.64    }  &141424     & 521.55     & 62583.6  & 206.81  & 201000      & 759.39                    &241170	            &919.76\\
 defender           & 839642.95             & 5291.18      & N/A       & N/A        & N/A      & N/A     & 893110      & 5629.27               &\textbf{970540}   &\textbf{6118.89}\\
 demon attack       & 143964.26             & 7906.55      & 2775     &144.20      & 208.1    & 3.08    & 675530      & 37131.12                &\textbf{787985}                     &\textbf{43313.70}\\
 double dunk        & 23.94          & 1933.64      & 22        &1845.45     & N/A      & N/A     & \textbf{24}          & \textbf{1936.36}                   &\textbf{24 }      &\textbf{1936.36}\\
 enduro             & 2382.44               & 276.87       & 2112     &245.44      & N/A      & N/A     & \best{14330}       & \best{1665.31}                 &14300             &1661.82\\
 fishing derby      & \textbf{91.16}          & \textbf{345.67     }  & 60        &286.77      &-90.7     & 1.89    & 59          & 285.71                    &65               &296.22\\
 freeway            & 33.03                 & 111.59       & \textbf{34}        &\textbf{114.86}      &16.7      & 56.42   & \best{34}      & \best{114.86 }  &\textbf{34}      &\textbf{114.86}\\
 frostbite          & \textbf{631378.53}      & \textbf{14786.59}     & 15622    &364.37      &236.9     & 4.02    & 10485       & 244.05                     &11330	            &263.84\\
 gopher             & 130345.58             & 6036.85      & 53853    &2487.14     &596.8     & 15.74   & \best{488830}      & \best{22672.6}                 &473560           &21964.01\\
 gravitar           & \textbf{6682.7     }    & \textbf{204.81     }  & 3554     &106.37      &173.4     & 0.01    & 5905        & 180.34                     &5915             &180.66\\
 hero               & \textbf{49244.11}       & \textbf{161.81     }  & 30287    &98.19       &2656.6    & 5.47    & 38330       & 125.18                     &38225	            &124.83\\
 ice hockey         & \textbf{67.04      }    & \textbf{646.61     }  & 29        &332.23      &-11.6     & -3.31   & 44.94       & 463.97                    &47.11           &481.90\\
 jamesbond          & 41063.25              & 14986.94     &  9269     &3374.73     &100.5     & 26.11   & 594500      & 217118.70              &\textbf{620780	}  &\textbf{226716.95}\\
 kangaroo           & \textbf{16763.6        }& \textbf{560.23     }  & 11819     &394.47      &51.2      & -0.03   & 14500       & 484.34                    &14636           &488.90\\
 krull              & 269358.27      & 25082.93     & 9687     &757.75      &2204.8    & 56.84   & 97575       & 8990.82                    &\textbf{594540}          &\textbf{55544.92}\\
 kung fu master     & 204824         & 910.08       & 66410    &294.30      &14862.5   & 64.97   & 140440      & 623.64                     &\textbf{1666665}	          &\textbf{7413.57}\\
 montezuma revenge  & 0                     & 0.00         & 1932     &40.65       &N/A       & N/A     & \best{3000}        & \best{63.11  }                 &2500            &52.60\\
 ms pacman          & \textbf{243401.1 }      & \textbf{3658.68    }  & 5651     &80.43       &1480      & 17.65   & 11536       & 169.00                     &11573           &169.55\\
 name this game     & \textbf{157177.85}      & \textbf{2690.53    }  & 14472    &211.57      &2420.7    & 2.23    & 34434       & 558.34                     &36296           &590.68\\
 phoenix            & 955137.84      & 14725.53     & 13342     &194.11      &N/A       & N/A     & 894460      & 13789.30                  &\textbf{959580 }         &	\textbf{14794.07}\\
 pitfall            &\textbf{ 0}                     & \textbf{3.43}         & -1        &3.41        &N/A       & N/A     & \best{0}       & \best{3.43  }   &-4.3            &3.36\\
 pong               & \textbf{21}                    & \textbf{118.13}       & 19        &112.46      & 12.8     & 94.90   & \best{21}      & \best{118.13}   &\textbf{21}     &\textbf{118.13}   \\
 private eye        & \textbf{15299.98 }      & \textbf{21.96  }      & 158       &0.19        & 35       & 0.01    & 15100       & 21.68                     &15100           &21.68\\
 qbert              & \textbf{72276          }& \textbf{542.56 }      & 162023    &1217.80     & 1288.8   & 8.46    & 27800       & 207.93                    &28657           &214.38\\
 riverraid          & \textbf{323417.18}      & \textbf{2041.12}      & 16249    &94.49       & 1957.8   & 3.92    & 28075       & 169.44                     &28349           &171.17\\
 road runner        & 613411.8              & 7830.48      & 88772    &1133.09     & 5640.6   & 71.86   & 878600      & 11215.78                              &\textbf{999999	} &\textbf{12765.53}\\
 robotank           & \textbf{131.13}         & \textbf{1329.18}      & 65        &647.42      & N/A      & N/A     & 108         & 1092.78                   &113.4           &1146.39\\
 seaquest           & 999976.52             & 2381.51      & 45898    &109.15      & 683.3    & 1.46    &943910	           &2247.98             &\textbf{1000000}          &\textbf{2381.57}\\
 skiing             & -29968.36      & -100.86     & -8187    &69.83       & N/A      & N/A     & -6774       & 80.90                       &\textbf{-6025}  &\textbf{86.77}\\
 solaris            & 56.62                 & -10.64       & 883       &-3.19       & N/A      & N/A     & \best{11074}       & \best{88.70   }               &9105            &70.95\\
 space invaders     & 74335.3               & 4878.50      & 2611      &161.96      & N/A      & N/A     & 140460     & 9226.80                 &\textbf{154380}          &\textbf{10142.17}\\
 star gunner        & 549271.7       & 5723.01      & 29219    &297.88      & N/A      & N/A     & 465750      & 4851.72                    &\textbf{677590}          &\textbf{7061.61}\\
 surround           & \textbf{9.99       }    & \textbf{121.15     }  & N/A       &N/A         & N/A      & N/A     & -7.8          & 13.33                   &2.606           &76.40 \\
 tennis             & 0       & 153.55  & 23        &301.94      & N/A      & N/A     & \textbf{24}          & \textbf{308.39}                                &\textbf{24}              &\textbf{308.39}  \\
 time pilot         & \textbf{476763.9}       & \textbf{28486.90}     & 32404    &1735.96     & N/A      & N/A     & 216770      & 12834.99                   &450810	          &26924.45 \\
 tutankham          & \textbf{491.48     }    & \textbf{307.35     }  & 238       &145.07      & N/A      & N/A     & 424         & 264.08                    &418.2           &260.44 \\
 up n down          & 715545.61             & 6407.03      & 648363   &5805.03     & 3350.3   & 25.24   & \best{986440}      & \best{8834.45}                 &966590          &8656.58\\
 venture            & 0.4                   & 0.03         & 0         &0.00        & N/A      & N/A     & \best{2035}        & \best{171.37 }                &2000	            &168.42\\
 video pinball      & \textbf{981791.88}      & \textbf{5556.92}      & 22218    &125.75      & N/A      & N/A     & 925830      & 5240.18                    &978190          &5536.54\\
 wizard of wor      & \textbf{197126         }& \textbf{4687.87}      & 14531    &333.11      & N/A      & N/A     & 64439       & 1523.38                    &63735           &1506.59\\
 yars revenge       & 553311.46             & 1068.72      & 20089    &33.01       & 5664.3   & 4.99    & \best{972000}      & \best{1881.96}                 &968090          &1874.36\\
 zaxxon             & \textbf{725853.9}       & \textbf{7940.46}      & 18295    &199.79      & N/A      & N/A     & 109140      & 1193.63                    &216020	          &2362.89\\
\hline    
MEAN HNS(\%)        &                & 4996.20      &           &  631.17   &           & 25.3    &             & \GDIImeanhns&      & \textbf{\GDIHmeanhns} \\
Learning Efficiency     &     &2.50E-09  &         & 3.16E-08  &        & 2.53E-07 &      & 3.91E-07       &      & \textbf{\GDIHmeanhnsle} \\
\hline
MEDIAN HNS(\%)      &                & \textbf{2041.12}      &           & 161.96    &           & 5.55    &             & \GDIImedianhns &      & \GDIHmedianhns \\
Learning Efficiency     &     & 1.02E-09 &         & 8.10E-09  &        & 5.55E-08&      & 4.16E-08       &      & \textbf{\GDIHmedianhnsle} \\
\bottomrule
\end{tabular}
\caption{Score table of SOTA model-based algorithms on HNS.}
\end{center}
\end{table}
\clearpage

\subsubsection{Comparison with Other SOTA algorithms on HNS}
In this section, we report the performance of our algorithm compared with other SOTA algorithms, Go-Explore \citep{goexplore} and Muesli \citep{muesli}.

\begin{table}[!hb]
\scriptsize
\begin{center}
\setlength{\tabcolsep}{1.0pt}
\begin{tabular}{c c c c c c c c c}            
\toprule
 Games        & Muesli & HNS(\%)      & Go-Explore              & HNS(\%)                     & GDI-I$^3$ & HNS(\%)               & GDI-H$^3$ & HNS(\%) \\
\midrule
Scale         &  200M   &            & 10B                     &                             & 200M              &                &  200M   &\\
\midrule    
 alien        &139409          &2017.12                   &\textbf{959312}       &\textbf{13899.77}              & 43384             &625.45            &48735	             &703.00              \\
 amidar       &\textbf{21653}  &\textbf{1263.18}          &19083                 &1113.22                        & 1442              &83.81             &1065              &61.81           \\
 assault      &36963           &7070.94                   &30773                 &5879.64                        & 63876      &12250.50   &\textbf{97155	}    &\textbf{18655.23}       \\
 asterix      &316210          &3810.30                   &999500       &12049.37              & 759910            &9160.41           &\textbf{999999}    &\textbf{12055.38}\\
 asteroids    &484609          &1036.84                   &112952                &240.48                         & 751970     &1609.72    &\textbf{760005}            &\textbf{1626.94}       \\
 atlantis     &1363427         &8348.18                   &286460                &1691.24                        & 3803000    &23427.66   &\textbf{3837300}           &\textbf{23639.67}       \\
 bank heist   &1213            &162.24                    &\textbf{3668}         &\textbf{494.49}                & 1401              &187.68            &1380              &184.84\\
 battle zone  &414107          &1120.04                   &\textbf{998800}       &\textbf{2702.36}               & 478830            &1295.20           &824360            &2230.29\\
 beam rider   &288870          &1741.91                   &371723       &2242.15               & 162100            &976.51            &\textbf{422390}   &\textbf{2548.07}\\
 berzerk      &44478           &1769.43                   &\textbf{131417}       &\textbf{5237.69}               & 7607              &298.53            &14649             &579.46\\
 bowling      &191             &122.02                    &\textbf{247}           &\textbf{162.72}                & 202               &129.94           &205.2             &132.34\\
 boxing       &99              &824.17                    &91                     &757.50                         & \best{100}        &\best{832.50}    &\textbf{100}      &\textbf{832.50}        \\
 breakout     &791             &2740.63                   &774                    &2681.60                        & \best{864}        &\best{2994.10}   &\textbf{864}      &\textbf{2994.10}        \\
 centipede    &\textbf{869751} &\textbf{8741.20}          &613815                &6162.78               & 155830            &1548.84                    &195630            &1949.80\\
 chopper command &101289       &1527.76            &996220                &15135.16                       & \best{999999}     &\best{15192.62}          &\textbf{999999}   &\textbf{15192.62}\\
 crazy climber   &175322       &656.88             &235600       &897.52                & 201000            &759.39                   &\textbf{241170}	            &\textbf{919.76}\\
 defender        &629482       &3962.26            &N/A                    &N/A                            & 893110     &5629.27                        &\textbf{970540}   &\textbf{6118.89}\\
 demon attack    &129544       &7113.74            &239895                 &13180.65                       & 675530     &37131.12         &\textbf{787985}                     &\textbf{43313.70}\\
 double dunk     &-3           &709.09             &\textbf{24}                     &\textbf{1936.36}                        & \best{24}         &\best{1936.36}          &\textbf{24}       &\textbf{1936.36}\\
 enduro          &2362         &274.49             &1031                   &119.81                         & \best{14330}      &\best{1665.31}          &14300             &1661.82\\
 fishing derby   &51           &269.75             &\textbf{67}            &\textbf{300.00}                & 59                &285.71                  &65               &296.22\\
 freeway         &33           &111.49             &\textbf{34}            &\textbf{114.86}                & \best{34}         &\best{114.86}           &\textbf{34}        &\textbf{114.86}\\
 frostbite       &301694       &7064.73            &\textbf{999990}       &\textbf{23420.19}              & 10485             &244.05                   &11330	            &263.84\\
 gopher          &104441       &4834.72            &134244                &6217.75                        & \best{488830}     &\best{22672.63}          &473560           &21964.01\\
 gravitar        &11660        &361.41             &\textbf{13385}        &\textbf{415.68}                & 5905              &180.34                   &5915             &180.66\\
 hero            &37161        &121.26            &37783                  &123.34                         & \textbf{38330}      &\textbf{125.18}            &38225	   &124.83\\
 ice hockey      &25           &299.17             &33                     &365.29                         & 44.94         &463.97        &\textbf{47.11}           &\textbf{481.90}    \\
 jamesbond       &19319        &7045.29            &200810                &73331.26                       & 594500     &217118.70         &\textbf{620780	}          &\textbf{226716.95}\\
 kangaroo        &14096        &470.80             &\textbf{24300}        &\textbf{812.87}                & 14500             &484.34                   &14636           &488.90\\
 krull           &34221        &3056.02            &63149                 &5765.90                        & 97575      &8990.82           &\textbf{594540}          &\textbf{55544.92}\\
 kung fu master  &134689       &598.06             &24320                 &107.05                         & 140440     &623.64            &\textbf{1666665	}          &\textbf{7413.57}\\
 montezuma revenge  &2359      &49.63               &\textbf{24758}        &\textbf{520.86}                & 3000              &63.11                   &2500            &52.60\\
 ms pacman          &65278     &977.84              &\textbf{456123}       &\textbf{6860.25}               & 11536             &169.00                  &11573           &169.55\\
 name this game     &105043    &1784.89              &\textbf{212824}       &\textbf{3657.16}               & 34434             &558.34                 &36296           &590.68\\
 phoenix        &805305        &12413.69                    &19200                 &284.50                  & 894460     &13789.30 &\textbf{959580	}          &\textbf{14794.07}   \\
 pitfall        &0             &3.43                    &\textbf{7875}          &\textbf{121.09}                & 0                 &3.43               &-4.3            &3.36\\
 pong           &20            &115.30                  &\textbf{21}            &\textbf{118.13}                & \best{21}         &\best{118.13}      &\textbf{21}              &\textbf{118.13}      \\
 private eye    &10323         &14.81                   &\textbf{69976}        &\textbf{100.58}                & 15100             &21.68               &15100           &21.68\\
 qbert          &157353        &1182.66                 &\textbf{999975}       &\textbf{7522.41}               & 27800             &207.93              &28657           &214.38\\
 riverraid      &\textbf{47323}&\textbf{291.42}         &35588                 &217.05                & 28075             &169.44                       &28349           &171.17\\
 road runner    &327025        &4174.55                 &999900        &12764.26              & 878600            &11215.78           &\textbf{999999}	          &\textbf{12765.53}\\
 robotank       &59            &585.57                  &\textbf{143}           &\textbf{1451.55}               & 108               &1092.78            &113.4           &1146.39\\
 seaquest       &815970        &1943.26                 &539456                &1284.68               &943910	           &2247.98               &\textbf{1000000}          &\textbf{2381.57}\\
 skiing         &-18407        &-10.26                  &\textbf{-4185}        &\textbf{101.19}                & -6774             &80.90               &-6025	         &86.77\\
 solaris        &3031          &16.18                   &\textbf{20306}        &\textbf{171.95}                & 11074             &88.70               &9105            &70.95\\
 space invaders &59602         &3909.65                &93147                 &6115.54                        & 140460     &9226.80       &\textbf{154380}          &\textbf{10142.17}\\
 star gunner    &214383        &2229.49                &609580       &6352.14               & 465750     &4851.72                     &\textbf{677590}	          &\textbf{7061.61}\\
 surround       &\textbf{9}    &\textbf{115.15}                 &N/A                    &N/A                  & -8        &13.33                        &2.606           &76.40\\
 tennis         &12            &230.97                 &\best{24}              &\best{308.39}                  & \best{24}         &\best{308.39}       &\textbf{24}    &\textbf{308.39}     \\
 time pilot     &\textbf{359105} &\textbf{21403.71}                   &183620                &10839.32       & 216770     &12834.99                     &450810	          &26924.45\\
 tutankham      &252           &154.03                 &\textbf{528}           &\textbf{330.73}                & 424               &264.08              &418.2           &260.44\\
 up n down      &649190        &5812.44                &553718                &4956.94                        & \best{986440}     &\best{8834.45}       &966590          &8656.58    \\
 venture        &2104          &177.18                 &\textbf{3074}         &\textbf{258.86}                & 2035              &171.37               &2000	            &168.42\\
 video pinball  &685436        &3879.56                &\textbf{999999}       &\textbf{5659.98}               & 925830            &5240.18              &978190          &5536.54\\
 wizard of wor  &93291         &2211.48                &\textbf{199900}       &\textbf{4754.03}               & 64293             &1519.90              &63735           &1506.59\\
 yars revenge   &557818        &1077.47                &\textbf{999998}       &\textbf{1936.34}               & 972000            &1881.96              &968090          &1874.36\\
 zaxxon         &65325         &714.30                 &18340                 &200.28                         & 109140     &1193.63       &\textbf{216020	}          &\textbf{2362.89}    \\
\hline    
MEAN HNS(\%)        & & 2538.66           &                       & 4989.94                       &            &  \GDIImeanhns &      & \textbf{\GDIHmeanhns} \\
Learning Efficiency & & 1.27E-07          &                       & 4.99E-09                      &      & 3.91E-07       &      & \textbf{\GDIHmeanhnsle} \\
\hline
MEDIAN HNS(\%)      & & 1077.47           &                       & \textbf{1451.55}              &            & \GDIImedianhns   &      & \GDIHmedianhns \\
Learning Efficiency & & 5.39E-08          &                       & 1.45E-09                      &      & 4.16E-08       &      & \textbf{\GDIHmedianhnsle} \\
\bottomrule
\end{tabular}
\caption{Score table of other SOTA algorithms on HNS.}
\end{center}
\end{table}

\clearpage

\subsection{Atari Games Table of Scores Based on Human World Records}
\label{app: Atari Games Table of Scores Based on Human World Records}
In this part, we detail the raw score of several representative SOTA algorithms including the SOTA 200M model-free algorithms, SOTA 10B+ model-free algorithms, SOTA model-based algorithms and other SOTA algorithms.\footnote{200M and 10B+ represent the training scale.} Additionally, we calculate the human world records normalized world score (HWRNS) of each game with each algorithms. First of all, we demonstrate the sources of the scores that we used.
Random scores  are from \citep{agent57}.
Human world records (HWR) are form \citep{dreamerv2,atarihuman}.
Rainbow's scores are from \citep{rainbow}.
IMPALA's scores are from \citep{impala}.
LASER's scores are from \citep{laser}, no sweep at 200M.
As there are many versions of R2D2 and NGU, we use original papers'.
R2D2's scores are from \citep{r2d2}.
NGU's scores are from \citep{ngu}.
Agent57's scores are from \citep{agent57}.
MuZero's scores are from \citep{muzero}.
DreamerV2's scores are from \citep{dreamerv2}.
SimPLe's scores are form \citep{modelbasedatari}.
Go-Explore's scores are form \citep{goexplore}.
Muesli's scores are form \citep{muesli}.
In the following we detail the raw scores and HWRNS of each algorithms on 57 Atari games.
\clearpage

\subsubsection{Comparison with SOTA 200M Algorithms on HWRNS}

\begin{table}[!hb]
\scriptsize
\begin{center}
\setlength{\tabcolsep}{1.0pt}
\begin{tabular}{c c c c c c c c c c c c c}
\toprule
Games               & RND       & HWR       & RAINBOW  & HWRNS(\%) & IMPALA  & HWRNS(\%)  & LASER  & HWRNS(\%) & GDI-I$^3$ & HWRNS(\%)  & GDI-H$^3$ & HWRNS(\%) \\
\midrule
Scale               &           &           & 200M     &           &  200M   &            & 200M    &           &  200M    &            &    200M   &\\
\midrule
 alien              & 227.8     & \textbf{251916}    & 9491.7   &3.68    & 15962.1    & 6.25       & 976.51  & 14.04     &43384             &17.15  &48735             &19.27   \\
 amidar             & 5.8       & \textbf{104159}    & 5131.2   &4.92    & 1554.79    & 1.49       & 1829.2  & 1.75      &1442              &1.38   &1065              &1.02          \\
 assault            & 222.4     & 8647             & 14198.5  &165.90  & 19148.47   & 224.65     & 21560.4 & 253.28    &63876        &755.57  &\textbf{97155}             &\textbf{1150.59} \\
 asterix            & 210       & \textbf{1000000}   & 428200   &42.81   & 300732     & 30.06      & 240090  & 23.99     &759910            &75.99  &999999            &100.00 \\
 asteroids          & 719       & \textbf{10506650}  & 2712.8   &0.02    & 108590.05  & 1.03       & 213025  & 2.02      &751970            &7.15   &760005            &7.23\\
 atlantis           & 12850     & \textbf{10604840}  & 826660   &7.68    & 849967.5   & 7.90       & 841200  & 7.82      &3803000           &35.78  &3837300           &36.11\\
 bank heist         & 14.2      & \textbf{82058}     & 1358     &1.64    & 1223.15    & 1.47       & 569.4   & 0.68      &1401              &1.69   &1380              &1.66  \\
 battle zone        & 236       & 801000    & 62010    &7.71    & 20885      & 2.58       & 64953.3 & 8.08      &478830            &59.77  &\textbf{824360}            &102.92 \\
 beam rider         & 363.9     & \textbf{999999}    & 16850.2  &1.65    & 32463.47   & 3.21       & 90881.6 & 9.06      &162100            &16.18  &422390            &42.22   \\
 berzerk            & 123.7     & \textbf{1057940}   & 2545.6   &0.23    & 1852.7     & 0.16       & 25579.5 & 2.41      &7607                       &0.71   &14649             &1.37          \\
 bowling            & 23.1      & \textbf{300}       & 30       &2.49    & 59.92      & 13.30      & 48.3    & 9.10      &201.9             &64.57  &205.2             &65.76   \\
 boxing             & 0.1       & \textbf{100}       & 99.6     &99.60   & 99.96      & 99.96      & \textbf{100}     & \textbf{100.00}    &\best{100}        &\best{100.00 } &\textbf{100}               &\textbf{100.00}    \\
 breakout           & 1.7       & \textbf{864}       & 417.5    &48.22   & 787.34     & 91.11      & 747.9   & 86.54     &\best{864}        &\best{100.00 } &\textbf{864}             &\textbf{100.00 }  \\
 centipede          & 2090.9    & \textbf{1301709}   & 8167.3   &0.47    & 11049.75   & 0.69       & 292792  & 22.37     &155830            &11.83          &195630            &14.89 \\
 chopper command    & 811       & \textbf{999999}    & 16654    &1.59    & 28255      & 2.75       & 761699  & 76.15     &\best{999999}     &\best{100.00 } &\textbf{999999}            &\textbf{100.00}\\
 crazy climber      & 10780.5   & 219900    & 168788.5 &75.56   & 136950     & 60.33      & 167820  & 75.10     &201000            &90.96          &\textbf{241170}            &\textbf{110.17}\\
 defender           & 2874.5    & \textbf{6010500}   & 55105    &0.87    & 185203     & 3.03       & 336953  & 5.56      &893110            &14.82          &970540            &16.11\\
 demon attack       & 152.1     & \textbf{1556345}   & 111185   &7.13    & 132826.98  & 8.53       & 133530  & 8.57      &675530            &43.40          &\textbf{787985}   &\textbf{50.63}\\
 double dunk        & -18.6     & 21        & -0.3     &46.21   & -0.33      & 46.14      & 14      & 82.32     &\best{24}                  &\best{107.58 } &\textbf{24}                &\textbf{107.58}  \\
 enduro             & 0         & 9500      & 2125.9   &22.38   & 0          & 0.00       & 0       & 0.00      &\best{14330}               &\best{150.84  }&14300             &150.53  \\
 fishing derby      & -91.7     & \textbf{71}        & 31.3     &75.60   & 44.85      & 83.93      & 45.2    & 84.14     &59                &92.89          &65               &96.31\\
 freeway            & 0         & \textbf{38}        & 34       &89.47   & 0          & 0.00       & 0       & 0.00      &34                &89.47          &34               &89.47\\
 frostbite          & 65.2      & \textbf{454830}    & 9590.5   &2.09    & 317.75     & 0.06       & 5083.5  & 1.10      &10485             &2.29           &11330            &2.48 \\
 gopher             & 257.6     & 355040    & 70354.6  &19.76   & 66782.3    & 18.75      & 114820.7& 32.29     &\best{488830}              &\best{137.71 } &473560           &133.41 \\
 gravitar           & 173       & \textbf{162850}    & 1419.3   &0.77    & 359.5      & 0.11       & 1106.2  & 0.57      &5905              &3.52           &5915             &3.53\\
 hero               & 1027      & \textbf{1000000}   & 55887.4  &5.49    & 33730.55   & 3.27       & 31628.7 & 3.06      &38330                      &3.73           &38225            &3.72  \\
 ice hockey         & -11.2     & 36        & 1.1      &26.06   & 3.48       & 31.10      & 17.4    & 60.59     &44.92              &118.94      &\textbf{47.11}           &\textbf{123.54} \\
 jamesbond          & 29        & 45550     & 19809    &43.45   & 601.5      & 1.26       & 37999.8 & 83.41     &594500              &1305.93 &\textbf{620780}          &\textbf{1363.66}\\
 kangaroo           & 52        & \textbf{1424600}   & 14637.5  &1.02    & 1632       & 0.11       & 14308   & 1.00      &14500                      &1.01           &14636           &1.02  \\
 krull              & 1598      & 104100    & 8741.5   &6.97    & 8147.4     & 6.39       & 9387.5  & 7.60      &97575             &93.63          &\textbf{594540}          &\textbf{578.47}\\
 kung fu master     & 258.5     & 1000000   & 52181    &5.19    & 43375.5    & 4.31       & 607443  & 60.73     &140440            &14.02          &\textbf{1666665}          &\textbf{166.68}\\
 montezuma revenge  &0          & \textbf{1219200}   & 384      &0.03    & 0          & 0.00       & 0.3     & 0.00      &3000              &0.25           &2500            &0.21\\
 ms pacman          & 307.3     & \textbf{290090}    & 5380.4   &1.75    & 7342.32    & 2.43       & 6565.5  & 2.16      &11536             &3.87           &11573           &3.89\\
 name this game     & 2292.3    & 25220     & 13136    &47.30   & 21537.2    & 83.94      & 26219.5 & 104.36    &34434               &140.19  &\textbf{36296}  &\textbf{148.31}    \\
 phoenix            & 761.5     & \textbf{4014440}   & 108529   &2.69    & 210996.45  & 5.24       & 519304  & 12.92     &894460            &22.27          &959580          &23.89\\
 pitfall            & -229.4    & \textbf{114000}    & 0        &0.20    & -1.66      & 0.20       & -0.6    & 0.20      &\best{0    }      &0.20           &-4.3            &0.20\\
 pong               & -20.7     & \textbf{21}        & 20.9     &99.76   & 20.98      & 99.95      & \textbf{21}      & \textbf{100.00}    &\best{21   }      &\best{100.00 } &\textbf{21}     &\textbf{100.00}    \\
 private eye        & 24.9      & \textbf{101800}    & 4234     &4.14    & 98.5       & 0.07       & 96.3    & 0.07      &15100             &14.81          &15100           &14.81\\
 qbert              & 163.9     & \textbf{2400000}   & 33817.5  &1.40    & 351200.12  & 14.63      & 21449.6 & 0.89      &27800             &1.15           &28657           &1.19\\
 riverraid          & 1338.5    & \textbf{1000000}   & 22920.8  &2.16    & 29608.05   & 2.83       & 40362.7 & 3.91      &28075             &2.68           &28349           &2.70\\
 road runner        & 11.5      & \textbf{2038100}   & 62041    &3.04    & 57121      & 2.80       & 45289   & 2.22      &878600            &43.11          &999999          &49.06\\
 robotank           & 2.2       & 76        & 61.4     &80.22   & 12.96      & 14.58      & 62.1    & 81.17     &108.2               &143.63  &\textbf{113.4}           &\textbf{150.68}\\
 seaquest           & 68.4      & 999999    & 15898.9  &1.58    & 1753.2     & 0.17       & 2890.3  & 0.28      &943910	             &94.39&\textbf{1000000}          &\textbf{100.00}\\
 skiing             & -17098    & \textbf{-3272}     & -12957.8 &29.95   & -10180.38  & 50.03      & -29968.4& -93.09    &-6774             &74.67          &-6025	          &86.77\\
 solaris            & 1236.3    & \textbf{111420}    & 3560.3   &2.11    & 2365       & 1.02       & 2273.5  & 0.94      &11074             &8.93           &9105            &7.14\\
 space invaders     & 148       & \textbf{621535 }   & 18789    &3.00    & 43595.78   & 6.99       & 51037.4 & 8.19      &140460            &22.58          &154380          &24.82\\
 star gunner        & 664       & 77400     & 127029   &164.67  & 200625     & 260.58     & 321528  & 418.14    &465750              &606.09  &\textbf{677590} &\textbf{882.15}\\
 surround           & -10       & 9.6       & \textbf{9.7}      &\textbf{100.51}  & 7.56       & 89.59      & 8.4     & 93.88  &-7.8        &11.22          &2.606           &64.32\\
 tennis             & -23.8     & 21        & 0        &53.13   & 0.55       & 54.35      & 12.2    & 80.36     &\best{24       }           &\best{106.70    } &\textbf{24}  &\textbf{106.70}\\
 time pilot         & 3568      & 65300     & 12926    &15.16   & 48481.5    & 72.76      & 105316  & 164.82    &216770              &345.37         &\textbf{450810} &\textbf{724.49}\\
 tutankham          & 11.4      & \textbf{5384}      & 241      &4.27    & 292.11     & 5.22       & 278.9   & 4.98      &423.9             &7.68           &418.2           &7.57\\
 up n down          & 533.4     & 82840     & 125755   &152.14  & 332546.75  & 403.39     & 345727  & 419.40    &\best{986440}              &\best{1197.85  } &966590          &1173.73  \\
 venture            & 0         & \textbf{38900}     & 5.5      &0.01    & 0          & 0.00       & 0       & 0.00      &2000              &5.23           &2000            &5.14\\
 video pinball      & 0         & \textbf{89218328}  & 533936.5 &0.60    & 572898.27  & 0.64       & 511835  & 0.57      &925830            &1.04           &978190          &1.10\\
 wizard of wor      & 563.5     & \textbf{395300}    & 17862.5  &4.38    & 9157.5     & 2.18       & 29059.3 & 7.22      &64439             &16.14          &63735           &16.00\\
 yars revenge       & 3092.9    & \textbf{15000105}  & 102557   &0.66    & 84231.14   & 0.54       & 166292.3& 1.09      &972000            &6.46           &968090          &6.43\\
 zaxxon             & 32.5      & 83700     & 22209.5  &26.51   & 32935.5    & 39.33      & 41118   & 49.11     &109140              &130.41  &\textbf{216020}          &\textbf{258.15}\\
\hline
MEAN HWRNS(\%)      & 0.00      & 100.00    &          & 28.39  &            & 34.52  &        & 45.39 &      & \GDIImeanHWRNS  &      & \textbf{\GDIHmeanHWRNS} \\
Learning Efficiency &     0.00 & N/A  &         & 1.42E-09 &         & 1.73E-09  &        & 2.27E-09 &      & 5.90E-09       &      & \textbf{\GDIHmeanHWRNSle} \\
\hline   
MEDIAN HWRNS(\%)    & 0.00      & \textbf{100.00}    &          & 4.92   &            & 4.31  &        & 8.08  &      &  \GDIImedianHWRNS  &      & \GDIHmedianHWRNS  \\
Learning Efficiency & 0.00   & N/A   &         & 2.46E-10 &         & 2.16E-10  &        & 4.04E-10 &      & 1.79E-09       &      & \textbf{\GDIHmedianHWRNSle} \\
\hline   
HWRB    & 0      &\textbf{57}    &          & 4   &            & 3  &        & 7  &      &  \GDIIHWRB  &      & \GDIHHWRB\\
\bottomrule
\end{tabular}
\caption{Score table of SOTA 200M model-free algorithms on HWRNS}
\end{center}
\end{table}
\clearpage

\subsubsection{Comparison with SOTA 10B+ Model-Free Algorithms on HWRNS}

\begin{table}[!hb]
\scriptsize
\begin{center}
\setlength{\tabcolsep}{1.0pt}
\begin{tabular}{c c c c c c c c c c c }
\toprule
 Games & R2D2 & HWRNS(\%) & NGU & HWRNS(\%) & AGENT57 & HWRNS(\%) & GDI-I$^3$ & HWRNS(\%) & GDI-H$^3$ & HWRNS(\%) \\
\midrule
Scale  & 10B   &        & 35B &         & 100B     &        & 200M &  &    200M   &\\
\midrule
 alien              & 109038.4          & 43.23       & 248100          & 98.48          & \textbf{297638.17}   &\textbf{118.17}         &43384             &17.15   &48735             &19.27    \\
 amidar             & 27751.24          & 26.64       & 17800           & 17.08          & \textbf{29660.08}    &\textbf{28.47}          &1442              &1.38    &1065              &1.02    \\
 assault            & 90526.44          & 1071.91     & 34800           & 410.44         & 67212.67             &795.17         &63876             &755.57  &\textbf{97155}             &\textbf{1150.59}    \\
 asterix            & 999080   & 99.91       & 950700          & 95.07          & 991384.42            &99.14          &759910            &75.99   &\textbf{999999}            &\textbf{100.00}    \\
 asteroids          & 265861.2          & 2.52        & 230500          & 2.19           & 150854.61            &1.43           &751970     &7.15    &\textbf{760005}            &\textbf{7.23}    \\
 atlantis           & 1576068           & 14.76       & 1653600         & 15.49          & 1528841.76           &14.31          &3803000    &35.78   &\textbf{3837300}           &\textbf{36.11}    \\
 bank heist         & \textbf{46285.6}  & \textbf{56.40}       & 17400           & 21.19          & 23071.5              &28.10          &1401              &1.69    &1380              &1.66    \\
 battle zone        & 513360            & 64.08       & 691700          & 86.35          & \textbf{934134.88}   &\textbf{116.63}         &478830            &59.77   &824360            &102.92    \\
 beam rider         & 128236.08         & 12.79       & 63600           & 6.33           & 300509.8    &30.03          &162100            &16.18   &\textbf{422390}            &\textbf{42.22}    \\
 berzerk            & 34134.8           & 3.22        & 36200           & 3.41           & \textbf{61507.83}    &\textbf{5.80 }          &7607              &0.71    &14649             &1.37    \\
 bowling            & 196.36            & 62.57       & 211.9           & 68.18          & \textbf{251.18}      &\textbf{82.37}          &201.9             &64.57   &205.2             &65.76    \\
 boxing             & 99.16             & 99.16       & 99.7            & 99.70          & \textbf{100}         &\textbf{100.00}         &\best{100}        &\textbf{100.00} &\textbf{100}       &\textbf{100.00}     \\
 breakout           & 795.36            & 92.04       & 559.2           & 64.65          & 790.4                &91.46          &\best{864}        &\textbf{100.00}  &\textbf{864}             &\textbf{100.00}    \\
 centipede          & 532921.84         & 40.85       & \textbf{577800} & \textbf{44.30}          & 412847.86            &31.61          &155830            &11.83   &195630            &14.89    \\
 chopper command    &960648             & 96.06       &999900           & 99.99          &999900                &99.99          &\best{999999}     &\textbf{100.00}  &\textbf{999999}            &\textbf{100.00}    \\
 crazy climber      & 312768            & 144.41      & 313400          & 144.71         &\textbf{565909.85}    &\textbf{265.46}         &201000            &90.96   &241170            &110.17    \\
 defender           & 562106            & 9.31        & 664100          & 11.01          & 677642.78            &11.23          &893110     &14.82   &\textbf{970540}            &\textbf{16.11}    \\
 demon attack       & 143664.6          & 9.22        & 143500          & 9.21           & 143161.44            &9.19           &675530     &43.40    &\textbf{787985}   &\textbf{50.63}   \\
 double dunk        & 23.12             & 105.35      & -14.1           & 11.36          & 23.93       &107.40         &\textbf{24}                &\textbf{107.58}  &\textbf{24}                &\textbf{107.58}    \\
 enduro             & 2376.68           & 25.02       & 2000            & 21.05          & 2367.71              &24.92          &\best{14330}      &\textbf{150.84}  &14300             &150.53    \\
 fishing derby      & 81.96             & 106.74      & 32              & 76.03          & \textbf{86.97}       &\textbf{109.82}         &59                &92.89   &65                &96.31    \\
 freeway            & \textbf{34}       & \textbf{89.47}       & 28.5            & 75.00          & 32.59                &85.76          &\best{34}         &\textbf{89.47} &\textbf{34}         &\textbf{89.47}       \\
 frostbite          & 11238.4           & 2.46        & 206400          & 45.37          &\textbf{541280.88}    &\textbf{119.01}         &10485             &2.29    &11330            &2.48    \\
 gopher             & 122196            & 34.37       & 113400          & 31.89          & 117777.08            &33.12          &\best{488830}     &\textbf{137.71}  &473560           &133.41    \\
 gravitar           & 6750              & 4.04        & 14200           & 8.62           &\textbf{19213.96}     &\textbf{11.70}          &5905              &3.52    &5915             &3.53    \\
 hero               & 37030.4           & 3.60        & 69400           & 6.84           &\textbf{114736.26}    &\textbf{11.38}          &38330             &3.73    &38225            &3.72    \\
 ice hockey         & \textbf{71.56}    & \textbf{175.34}      &-4.1             & 15.04          & 63.64                &158.56         &37.89             &118.94  &47.11           &123.54    \\
 jamesbond          & 23266             & 51.05       & 26600           & 58.37          & 135784.96            &298.23         &594500              &1305.93 &\textbf{620780}          &\textbf{1363.66}    \\
 kangaroo           & 14112             & 0.99        & \textbf{35100}  & \textbf{2.46}           &24034.16              &1.68           &14500             &1.01    &14636           &1.02    \\
 krull              & 145284.8          & 140.18      & 127400          & 122.73         & 251997.31   &244.29         &97575             &93.63   &\textbf{594540}          &\textbf{578.47}    \\
 kung fu master     & 200176            & 20.00       & 212100          & 21.19          & 206845.82            &20.66          &140440            &14.02            &\textbf{1666665}	          &\textbf{166.68}\\
 montezuma revenge  & 2504              & 0.21        & \textbf{10400}  & \textbf{0.85}           &9352.01               &0.77           &3000              &0.25    &2500            &0.21    \\
 ms pacman          & 29928.2           & 10.22       & 40800           & 13.97          & \textbf{63994.44}    &\textbf{21.98}          &11536             &3.87    &11573           &3.89    \\
 name this game     & 45214.8           & 187.21      & 23900           & 94.24          &\textbf{54386.77}     &\textbf{227.21}         &34434             &140.19  &36296           &148.31    \\
 phoenix            & 811621.6          & 20.20       & \textbf{959100} & \textbf{23.88}          &908264.15             &22.61          &894460            &22.27   &959580          &23.89    \\
 pitfall            & 0                 & 0.20        & 7800            & 7.03           &\textbf{18756.01}     &\textbf{16.62}          &0                 &0.20    &-4.3            &0.20    \\
 pong               & \textbf{21}       & \textbf{100.00}      & 19.6            & 96.64          & 20.67                &99.21          &\best{21}         &\textbf{100.00} &\textbf{21}      &\textbf{100.00}      \\
 private eye        & 300               & 0.27        & \textbf{100000} & \textbf{98.23}          & 79716.46             &78.30          &15100             &14.81   &15100           &14.81    \\
 qbert              & 161000            & 6.70        & 451900          & 18.82          &\textbf{580328.14}    &\textbf{24.18}          &27800             &1.15    &28657           &1.19   \\
 riverraid          & 34076.4           & 3.28        & 36700           & 3.54           & \textbf{63318.67}    &\textbf{6.21}           &28075             &2.68    &28349           &2.70    \\
 road runner        & 498660            & 24.47       & 128600          & 6.31           & 243025.8             &11.92          &878600     &43.11   &\textbf{999999}          &\textbf{49.06}    \\
 robotank           & \textbf{132.4}    & \textbf{176.42}      & 9.1             & 9.35           &127.32                &169.54         &108               &143.63  &113.4           &150.68    \\
 seaquest           & 999991.84         & 100.00      & \textbf{1000000}& \textbf{100.00}         &999997.63             &100.00         &943910	             &94.39 &\textbf{1000000}  &\textbf{100.00}    \\
 skiing             & -29970.32         & -93.10      & -22977.9        & -42.53         & \textbf{-4202.6}     &\textbf{93.27}          &-6774             &74.67           &-6025	          &86.77\\
 solaris            & 4198.4            & 2.69        & 4700            & 3.14           & \textbf{44199.93}    &\textbf{38.99}          &11074             &8.93            &9105            &7.14\\
 space invaders     & 55889             & 8.97        & 43400           & 6.96           & 48680.86             &7.81           &140460     &22.58           &\textbf{154380} &\textbf{24.82}\\
 star gunner        & 521728            & 679.03      & 414600          & 539.43         &\textbf{839573.53}    &\textbf{1093.24}        &465750            &606.09          &677590          &882.15\\
 surround           & \textbf{9.96}     & \textbf{101.84}      & -9.6            & 2.04           & 9.5                  &99.49          &-7.8              &11.22           &2.606           &64.32\\
 tennis             & \textbf{24}       & \textbf{106.70}      & 10.2            & 75.89          & 23.84                &106.34         &\textbf{24}                &\textbf{106.70}          &\textbf{24}              &\textbf{106.70}\\
 time pilot         & 348932            & 559.46      & 344700          & 552.60         &\textbf{405425.31}    &\textbf{650.97}         &216770            &345.37          &450810          &724.49\\
 tutankham          & 393.64            & 7.11        & 191.1           & 3.34           & \textbf{2354.91}     &\textbf{43.62}          &423.9             &7.68            &418.2           &7.57\\
 up n down          & 542918.8          & 658.98      & 620100          & 752.75         & 623805.73            &757.26         &\best{986440}     &\textbf{1197.85}         &966590          &1173.73\\
 venture            & 1992              & 5.12        & 1700            & 4.37           &\textbf{2623.71}      &\textbf{6.74}           &2000              &5.23            &2000            &5.14\\
 video pinball      & 483569.72         & 0.54        & 965300          & 1.08           &\textbf{992340.74}    &\textbf{1.11}           &925830            &1.04            &978190          &1.10\\
 wizard of wor      & 133264            & 33.62       & 106200          & 26.76          &\textbf{157306.41}    &\textbf{39.71}          &64439             &16.14           &63735           &16.00\\
 yars revenge       & 918854.32         & 6.11        & 986000          & 6.55           &\textbf{998532.37}    &\textbf{6.64}           &972000            &6.46            &968090          &6.43\\
 zaxxon             & 181372            & 216.74      & 111100          & 132.75         &\textbf{249808.9}     &\textbf{298.53}         &109140            &130.41          &216020          &258.15\\
\hline
MEAN HWRNS(\%)        &                   & 98.78       &                 &  76.00         &                      &125.92           &                  & \GDIImeanHWRNS        &         & \textbf{\GDIHmeanHWRNS} \\
Learning Efficiency &     & 9.88E-11 &         & 2.17E-11  &        & 1.26E-11 &      & 5.90E-09       &      & \textbf{\GDIHmeanHWRNSle} \\
\hline
MEDIAN HWRNS(\%)      &                   & 33.62       &                 &  21.19         &                      &43.62   &                  & \GDIImedianHWRNS          &         & \textbf{\GDIHmedianHWRNS} \\
Learning Efficiency & & 3.36E-11 &         & 6.05E-12  &        & 4.36E-12 &      & 1.79E-09       &      & \textbf{\GDIHmedianHWRNSle} \\
\hline
HWRB                  &                   & 15          &                 &  8             &                      &18               &                  & \GDIIHWRB             &       & \textbf{\GDIHHWRB} \\
\bottomrule
\end{tabular}
\caption{Score table of SOTA 10B+ model-free algorithms on HWRNS.}
\end{center}
\end{table}
\clearpage

\subsubsection{Comparison with SOTA Model-Based Algorithms on HWRNS}
SimPLe \citep{modelbasedatari} and DreamerV2\citep{dreamerv2} haven't evaluated all 57 Atari Games in their paper. For fairness, we set the score on those games as N/A, which will not be considered when calculating the median and mean HWRNS and human world record breakthrough (HWRB).  

\begin{table}[!hb]
\scriptsize
\begin{center}
\setlength{\tabcolsep}{1.0pt}
\begin{tabular}{ c c c c c c c c c c c }
\toprule
 Games              & MuZero         & HWRNS(\%)      & DreamerV2 & HWRNS(\%)    & SimPLe             & HWRNS(\%)          & GDI-I$^3$     & HWRNS(\%) & GDI-H$^3$ & HWRNS(\%) \\
\midrule
Scale               & 20B            &              & 200M      &            & 1M               &                  & 200M     & &    200M   &\\
\midrule
 alien              & \textbf{741812.63}      & \textbf{294.64  }     &3483       & 1.29     &616.9     & 0.15    & 43384       &  17.15            &48735             &19.27       \\
 amidar             & \textbf{28634.39 }      & \textbf{27.49   }  &2028       & 1.94     &74.3      & 0.07    & 1442        &  1.38                &1065              &1.02     \\
 assault            & \textbf{143972.03}      & \textbf{1706.31 }     &7679       & 88.51       &527.2     & 3.62       & 63876       &  755.57     &97155             &1150.59    \\
 asterix            & 998425         & 99.84        &25669      & 2.55     &1128.3    & 0.09       & 759910      &  75.99         &\textbf{999999}            &\textbf{100.00} \\
 asteroids          & 678558.64               & 6.45        &3064                & 0.02    &793.6              & 0.00    &751970         & 7.15   &\textbf{760005}            &\textbf{7.23}  \\
 atlantis           & 1674767.2               & 15.69           &989207              & 9.22       &20992.5            & 0.08       &3803000        & 35.78  &\textbf{3837300}           &\textbf{36.11}  \\
 bank heist         & 1278.98                 & 1.54        &1043                & 1.25    &34.2               & 0.02    &\best{1401}           & \best{1.69  }  &1380              &1.66 \\
 battle zone        & \textbf{848623         }& \textbf{105.95}      &31225      & 3.87     &4031.2    & 0.47       & 478830      &  59.77      &824360            &102.92     \\
 beam rider         & \textbf{454993.53}      & \textbf{45.48}      &12413      & 1.21     &621.6     & 0.03    & 162100      &  16.18          &422390            &42.22 \\
 berzerk            & \textbf{85932.6        }& \textbf{8.11}      &751        & 0.06     &N/A       & N/A     & 7607        &  0.71            &14649             &1.37 \\
 bowling            & \textbf{260.13         }& \textbf{85.60}      &48         & 8.99     &30        & 2.49    & 202         &  64.57          &205.2             &65.76   \\
 boxing             & \textbf{100}                     & \textbf{100.00}       &87                  & 86.99       &7.8         & 7.71       & \best{100}  & \best{100.00 } &\textbf{100}               &\textbf{100.00} \\
 breakout           & \textbf{864}                     & \textbf{100.00}          &350                 & 40.39       &16.4     & 1.70       & \best{864}  & \best{100.00}  &\textbf{864}             &100.00 \\
 centipede          & \textbf{1159049.27}     & \textbf{89.02}     &6601       & 0.35     &N/A       & N/A     & 155830      & 11.83                     &195630            &14.89\\
 chopper command    & 991039.7                & 99.10     &2833                & 0.20     & 979.4             & 0.02    & \best{999999}        & \best{100.00} &\textbf{999999}            &\textbf{100.00}  \\
 crazy climber      & \textbf{458315.4}       & \textbf{214.01    }  &141424     & 62.47       & 62583.6  & 24.77   & 201000      & 90.96                      &241170            &110.17 \\
 defender           & 839642.95               & 13.93      & N/A                & N/A        & N/A               & N/A      & 893110 & 14.82 &\textbf{970540}            &\textbf{16.11}   \\
 demon attack       & 143964.26               & 9.24      & 2775              &0.17      & 208.1             & 0.00    & 675530      & 43.40  &\textbf{787985}   &\textbf{50.63} \\
 double dunk        & 23.94          & 107.42  & 22        &102.53        & N/A      & N/A     & \textbf{24}          & \textbf{107.58}                    &\textbf{24}                &\textbf{107.58}    \\
 enduro             & 2382.44                 & 25.08       & 2112              &22.23         & N/A               & N/A     & \best{14330} & \best{150.84}&14300             &150.53   \\
 fishing derby      & \textbf{91.16}          & \textbf{112.39     }              &93.24          &286.77      &-90.7     & 0.61    & 59          & 92.89  &65               &96.31        \\
 freeway            & 33.03                   & 86.92       & \textbf{34}                 &\textbf{89.47}          &16.7               & 43.95   & \best{34}      & \best{89.47} &\textbf{34}               &\textbf{89.47}   \\
 frostbite          & \textbf{631378.53}      & \textbf{138.82}     & 15622    &3.42       &236.9     & 0.04    & 10485       & 2.29                           &11330            &2.48\\
 gopher             & 130345.58               & 36.67      & 53853             &15.11         &596.8              & 0.10       & \best{488830} & \best{137.71} &473560           &133.41  \\
 gravitar           & \textbf{6682.7     }    & \textbf{4.00    }  & 3554     &2.08      &173.4     & 0.00    & 5905        & 3.52     &5915             &3.53       \\
 hero               & \textbf{49244.11}       & \textbf{4.83    }  & 30287    &2.93      &2656.6    & 0.16        & 38330       & 3.73 &38225            &3.72           \\
 ice hockey         & \textbf{67.04      }    & \textbf{165.76  }  & 29        &85.17         &-11.6     & -0.85       & 38          & 118.94 &47.11           &123.54         \\
 jamesbond          & 41063.25                & 90.14     & 9269              &20.30         &100.5     & 0.16    &594500              &1305.93 &\textbf{620780}          &\textbf{ 1363.66}  \\
 kangaroo           & \textbf{16763.6        }& \textbf{1.17}  & 11819     &0.83      &51.2      & 0.00    & 14500       & 1.01          &14636           &1.02   \\
 krull              & 269358.27               & 261.22     & 9687     &7.89       &2204.8    & 0.59    & 97575       & 93.63    &\textbf{594540}          &\textbf{578.47}       \\
 kung fu master     & 204824         & 20.46  & 66410    &6.62       &14862.5   & 1.46    & 140440      & 14.02        &\textbf{1666665}          &\textbf{166.68 } \\
 montezuma revenge  & 0                       & 0.00         & 1932              &0.16       &N/A       & N/A     & \best{3000}          & \best{0.25  } &2500            &0.21  \\
 ms pacman          & \textbf{243401.1 }      & \textbf{83.89   }  & 5651     &1.84       &1480      & 0.40    & 11536       & 3.87       &11573           &3.89     \\
 name this game     & \textbf{157177.85}      & \textbf{675.54  }  & 14472    &53.12          &2420.7    & 0.56    & 34434       & 140.19 &36296           &148.31        \\
 phoenix            & \textbf{955137.84}      & \textbf{23.78   }  & 13342     &0.31       &N/A       & N/A     & 894460      & 22.27     &959580          &23.89     \\
 pitfall            & \textbf{0}                       & \textbf{0.20}               & -1                 &0.20       &N/A       & N/A     & \best{0} & \best{0.20} &-4.3            &0.20     \\
 pong               & \textbf{21}                      & 100.00             & 19                 &95.20          & 12.8     & 80.34   & \best{21} & \best{100.00} &\textbf{21}              &\textbf{100.00}   \\
 private eye        & \textbf{15299.98 }      & \textbf{15.01}     & 158       &0.13       & 35       & 0.01    & 15100       & 14.81             &15100           &14.81     \\
 qbert              & \textbf{72276          }& \textbf{3.00}      & 162023    &6.74       & 1288.8   & 0.05    & 27800       & 1.15              &28657           &1.19 \\
 riverraid          & \textbf{323417.18}      & \textbf{32.25}     & 16249    &1.49       & 1957.8   & 0.06    & 28075       & 2.68               &28349           &2.70\\
 road runner        & 613411.8                & 30.10              & 88772             &4.36       & 5640.6   & 0.28       & 878600 & 43.11   &\textbf{999999}          &\textbf{49.06}   \\
 robotank           & \textbf{131.13}         & \textbf{174.70}    & 65        &85.09          & N/A      & N/A     & 108         & 143.63        &113.4           &150.68 \\
 seaquest           & 999976.52               & 100.00             & 45898             &4.58       & 683.3             & 0.06    &943910	             &94.39 &\textbf{1000000}          &\textbf{100.00}\\
 skiing             & -29968.36      & -93.09      & -8187    &64.45          & N/A      & N/A     & -6774       & 74.67        &\textbf{-6025}	         &\textbf{86.77} \\
 solaris            & 56.62                   & -1.07              & 883                &-0.32          & N/A               & N/A     & \best{11074}         & \best{8.93 } &9105            &7.14\\
 space invaders     & 74335.3                 & 11.94                 & 2611               &0.40       & N/A               & N/A     & 140460        & 22.58  &\textbf{154380}          &\textbf{24.82}   \\
 star gunner        & 549271.7       & 714.93     & 29219    &37.21          & N/A      & N/A     & 465750      & 606.09      &\textbf{677590}          &\textbf{882.15}\\
 surround           & \textbf{9.99       }    & \textbf{101.99}     & N/A       &N/A         & N/A      & N/A     & -8          & 11.22         &2.606           &64.32 \\
 tennis             & 0                       & 53.13      & 23        &104.46        & N/A      & N/A     & \textbf{24}          & \textbf{106.70} &\textbf{24}              &\textbf{106.70}      \\
 time pilot         & \textbf{476763.9}       & \textbf{766.53}         & 32404    &46.71         & N/A      & N/A     & 216770      & 345.37       &450810          &724.49   \\
 tutankham          & \textbf{491.48     }    & \textbf{8.94}   & 238       &4.22         & N/A      & N/A     & 424         & 7.68                 &418.2           &7.57 \\
 up n down          & 715545.61               & 868.72            & 648363            &787.09        & 3350.3            & 3.42   & \best{986440}   & \best{1197.85} &966590          &1173.73    \\
 venture            & 0.4                     & 0.00        & 0                  &0.00       & N/A               & N/A     & \best{2030}          & \best{5.23}      &2000            &5.14  \\
 video pinball      & \textbf{981791.88}      & \textbf{1.10}      & 22218    &0.02     & N/A      & N/A     & 925830      & 1.04                                    &978190          &1.10\\
 wizard of wor      & \textbf{197126         }& \textbf{49.80}      & 14531    &3.54     & N/A      & N/A     & 64439       & 16.14                                  &63735           &16.00\\
 yars revenge       & 553311.46               & 3.67      & 20089             &0.11     & 5664.3            & 0.02    & \best{972000}        & \best{6.46}           &968090          &6.43\\
 zaxxon             & \textbf{725853.9}       & \textbf{867.51}      & 18295    &21.83          & N/A      & N/A     & 109140      & 130.41                          &216020          &258.15\\
\hline
MEAN HWRNS(\%) &               &152.1  &         &  37.9 &                                          & 4.67  &     & \GDIImeanHWRNS &                  & \textbf{\GDIHmeanHWRNS} \\
Learning Efficiency &      &7.61E-11 &         & 1.89E-09 &        & \textbf{4.67E-08} &      & 5.90E-09       &      & \GDIHmeanHWRNSle \\
\hline
MEDIAN HWRNS(\%) &            & 49.8  &         & 4.22  &                                          & 0.13  &     & \textbf{\GDIImedianHWRNS} &                  & \GDIHmedianHWRNS \\
Learning Efficiency &  & 2.49E-11 &         & 2.11E-10  &        & 1.25E-09 &      & 1.79E-09       &      & \textbf{\GDIHmedianHWRNSle} \\
\hline
HWRB &            & 19   &         & 3    &                                          & 0  &     & \GDIIHWRB & & \textbf{\GDIHHWRB} \\
\bottomrule
\end{tabular}
\caption{Score table of SOTA model-based algorithms on HWRNS.}
\end{center}
\end{table}
\clearpage

\subsubsection{Comparison with Other SOTA algorithms on HWRNS}
In this section, we report the performance of our algorithm compared with other SOTA algorithms, Go-Explore \citep{goexplore} and Muesli \citep{muesli}.

\begin{table}[!hb] 
\scriptsize
\begin{center}
\setlength{\tabcolsep}{1.0pt}
\begin{tabular}{ c c c c c c c c c }
\toprule
 Games      & Muesli              & HWRNS(\%)        & Go-Explore              & HWRNS(\%)                     & GDI-I$^3$ & HWRNS(\%) & GDI-H$^3$ & HWRNS(\%) \\
\midrule
Scale       &                     & 200M        & 10B                     &                             & 200M      &    &    200M   &\\
\midrule    
 alien        &139409          &55.30               &\textbf{959312}       &\textbf{381.06}                & 43384             &17.15     &48735             &19.27            \\
 amidar       &\textbf{21653}  &\textbf{20.78}      &19083        &18.32                & 1442              &1.38                         &1065              &1.02          \\
 assault      &36963           &436.11           &30773                 &362.64                         & 63876      &755.57&\textbf{97155}    &\textbf{1150.59}           \\
 asterix      &316210          &31.61            &999500       &99.95                 & 759910            &75.99        &\textbf{999999}            &\textbf{100.00}    \\
 asteroids    &484609          &4.61             &112952                &1.07                           & 751970     &7.15  &\textbf{760005}            &\textbf{7.23}        \\
 atlantis     &1363427         &12.75            &286460                &2.58                           & 3803000    &35.78 &\textbf{3837300}           &\textbf{36.11}         \\
 bank heist   &1213            &1.46          &\textbf{3668}         &\textbf{4.45  }                & 1401              &1.69            &1380              &1.66    \\
 battle zone  &414107          &51.68            &\textbf{998800}       &\textbf{124.70}                & 478830            &59.77        &824360            &102.92     \\
 beam rider   &288870          &28.86         &371723       &37.15                & 162100            &16.18           &\textbf{422390}            &\textbf{42.22}     \\
 berzerk      &44478           &4.19             &\textbf{131417}      &\textbf{12.41 }                & 7607              &0.71          &14649             &1.37       \\
 bowling      &191             &60.64        &\textbf{247}           &\textbf{80.86 }                & 202               &64.57            &205.2             &65.76 \\
 boxing       &99              &99.00        &91                     &90.99                          & \best{100}        &\best{100.00}    &\textbf{100}      &\textbf{100.00}        \\
 breakout     &791             &91.53          &774                    &89.56                          & \best{864}        &\best{100.00}  &\textbf{864}              &\textbf{100.00}       \\
 centipede    &\textbf{869751} &\textbf{66.76}          &613815      &47.07                 & 155830            &11.83                    &195630            &14.89  \\
 chopper command &101289       &10.06       &996220                &99.62                          & \best{999999}     &\best{100.00}     &\textbf{999999}            &\textbf{100.00}    \\
 crazy climber   &175322       &78.68     &235600       &107.51                & 201000            &90.96               &\textbf{241170}            &\textbf{110.17}\\
 defender        &629482       &10.43        &N/A                    &N/A                            & 893110     &14.82    &\textbf{970540}                &\textbf{16.11}      \\
 demon attack    &129544       &8.31        &239895                 &15.41                          & 675530     &43.40     &\textbf{787985}   &\textbf{50.63}    \\
 double dunk     &-3           &39.39       &\textbf{24}                     &\textbf{107.58}       & \best{24}       &\best{107.58}      &\textbf{24}      &\textbf{107.58}      \\
 enduro          &2362         &24.86                   &1031                   &10.85              & \best{14330}      &\best{150.84}    &14300             &150.53        \\
 fishing derby   &51           &87.71       &\textbf{67}            &\textbf{97.54 }                & 59                &92.89            &65               &96.31\\
 freeway         &33           &86.84                   &\textbf{34}            &\textbf{89.47 }    & \best{34}         &\best{89.47}     &\textbf{34}      &\textbf{89.47}      \\
 frostbite       &301694       &66.33            &\textbf{999990}       &\textbf{219.88}                & 10485             &2.29         &11330            &2.48   \\
 gopher          &104441       &29.37            &134244                &37.77                          & \best{488830}     &\best{137.71} &473560           &133.41       \\
 gravitar        &11660        &7.06           &\textbf{13385}        &\textbf{8.12}                  & 5905              &3.52           &5915             &3.53  \\
 hero            &37161        &3.62                     &37783                  &3.68                           & \best{38330}      &\best{3.73}   &38225           &3.72           \\
 ice hockey      &25           &76.69          &33                     &93.64                          & 45         &118.94           &\textbf{47.11}  &\textbf{123.54} \\
 jamesbond       &19319        &42.38            &200810                &441.07                         &594500              &1305.93         &\textbf{620780}       &\textbf{ 1363.66}\\
 kangaroo        &14096        &0.99                   &\textbf{24300}        &\textbf{1.70}                  & 14500             &1.01             &14636           &1.02\\
 krull           &34221        &31.83                     &63149                 &60.05                          & 97575      &93.63  &\textbf{594540}          &\textbf{578.47}        \\
 kung fu master  &134689       &13.45      &24320                 &2.41                           & 140440     &14.02                 &\textbf{1666665}          &\textbf{166.68}\\
 montezuma revenge  &2359      &0.19             &\textbf{24758}        &\textbf{2.03}                  & 3000              &0.25                   &2500            &0.21\\
 ms pacman          &65278     &22.42     &\textbf{456123}       &\textbf{157.30}                & 11536             &3.87                          &11573           &3.89\\
 name this game     &105043    &448.15     &\textbf{212824}       &\textbf{918.24}               & 34434             &140.19                        &36296           &148.31\\
 phoenix        &805305        &20.05                                &19200                 &0.46                           & 894460     &22.27 &\textbf{959580}          &\textbf{23.89}   \\
 pitfall        &0             &0.20                   &\textbf{7875}          &\textbf{7.09   }               & 0                 &0.2             &-4.3            &0.20 \\
 pong           &20            &97.60                       &\textbf{21}            &\textbf{100.00 }               & \best{21}         &\best{100} &\textbf{21}              &\textbf{100.00}          \\
 private eye    &10323         &10.12                &\textbf{69976}        &\textbf{68.73  }               & 15100             &14.81               &15100           &14.81 \\
 qbert          &157353        &6.55                            &\textbf{999975}       &\textbf{41.66  }               & 27800             &1.15     &28657           &1.19       \\
 riverraid      &\textbf{47323}&\textbf{4.60}               &35588        &3.43               & 28075             &2.68                              &28349           &2.70\\
 road runner    &327025        &16.05                        &999900        &49.06                 & 878600           &43.11       &\textbf{999999}          &\textbf{49.06}    \\
 robotank       &59            &76.96                             &\textbf{143}           &\textbf{190.79 }               & 108         &143.63      &113.4           &150.68          \\
 seaquest       &815970        &81.60                &539456       &53.94                 &943910	             &94.39       &\textbf{1000000}          &\textbf{100.00}    \\
 skiing         &-18407        &-9.47                      &\textbf{-4185}        &\textbf{93.40  }               & -6774             &74.67         &-6025	          &86.77      \\
 solaris        &3031          &1.63                         &\textbf{20306}        &\textbf{17.31  }               & 11074             &8.93        &9105            &7.14        \\
 space invaders &59602         &9.57                 &93147                 &14.97                          & 140460     &22.58        &\textbf{154380}          &\textbf{24.82} \\
 star gunner    &214383        &278.51   &609580                 &793.52                         &465750     &606.09                  &\textbf{677590}          &\textbf{882.15}\\
 surround       &\textbf{9}    &\textbf{96.94}             &N/A                    &N/A                            & -8         &11.22               &2.606           &64.32\\
 tennis         &12            &79.91                           &\best{24}              &\best{106.7}                   & \best{24}         &\best{106.70   }  &\textbf{24}              &\textbf{106.70}          \\
 time pilot     &359105 &575.94    &183620                &291.67                         & 216770     & 345.37                    &\textbf{450810}          &\textbf{724.49}\\
 tutankham      &252           &4.48         &\textbf{528}           &\textbf{9.62}                  & 424               &7.68                       &418.2           &7.57\\
 up n down      &649190        &788.10                &553718                &672.10                         & \best{986440}     &\best{1197.85}     &966590          &1173.73      \\
 venture        &2104          &5.41            &\textbf{3074}         &\textbf{7.90}                & 2035              &5.23                       &2000            &5.14\\
 video pinball  &685436        &0.77                 &\textbf{999999}       &\textbf{1.12}               & 925830            &1.04                   &978190          &1.10\\
 wizard of wor  &93291         &23.49                &\textbf{199900}       &\textbf{50.50}               & 64293             &16.14                 &63735           &16.00\\
 yars revenge   &557818        &3.70               &\textbf{999998}       &\textbf{6.65}               & 972000            &6.46                     &968090          &6.43\\
 zaxxon         &65325         &78.04                 &18340                 &21.88                        & 109140     &130.41        &\textbf{216020} &\textbf{258.15}  \\
\hline    
MEAN HWRNS(\%)  &    & 75.52 &                       & 116.89                       &            &  \GDIImeanHWRNS  &                  & \textbf{\GDIHmeanHWRNS} \\
Learning Efficiency &    & 3.78E-09 &                       & 1.17E-10                      &      & 5.90E-09       &      & \textbf{\GDIHmeanHWRNSle}\\
\hline
MEDIAN HWRNS(\%)&   & 24.68  &                       & 50.5              &            & \GDIImedianHWRNS   &                  & \textbf{\GDIHmedianHWRNS} \\
Learning Efficiency &    & 1.24E-09 &                       & 5.05E-11                       &      & 1.79E-09       &      & \textbf{\GDIHmedianHWRNSle} \\
\hline
HWRB    &       & 5 &               & 15              &            & \GDIIHWRB  & & \textbf{\GDIHHWRB} \\
\bottomrule
\end{tabular}
\caption{Score table of other SOTA  algorithms on HWRNS.}
\end{center}
\end{table}
\clearpage

\subsection{Atari Games Table of Scores Based on SABER}
\label{app: Atari Games Table of Scores Based on SABER}
In this part, we detail the raw score of several representative SOTA algorithms including the SOTA 200M model-free algorithms, SOTA 10B+ model-free algorithms, SOTA model-based algorithms and other SOTA algorithms.\footnote{200M and 10B+ represent the training scale.} Additionally, we calculate the capped human world records normalized world score (CHWRNS) or called SABER \citep{atarihuman} of each game with each algorithms. First of all, we demonstrate the sources of the scores that we used.
Random scores  are from \citep{agent57}.
Human world records (HWR) are form \citep{dreamerv2,atarihuman}.
Rainbow's scores are from \citep{rainbow}.
IMPALA's scores are from \citep{impala}.
LASER's scores are from \citep{laser}, no sweep at 200M.
As there are many versions of R2D2 and NGU, we use original papers'.
R2D2's scores are from \citep{r2d2}.
NGU's scores are from \citep{ngu}.
Agent57's scores are from \citep{agent57}.
MuZero's scores are from \citep{muzero}.
DreamerV2's scores are from \citep{dreamerv2}.
SimPLe's scores are form \citep{modelbasedatari}.
Go-Explore's scores are form \citep{goexplore}.
Muesli's scores are form \citep{muesli}.
In the following we detail the raw scores and SABER of each algorithms on 57 Atari games.
\clearpage
\subsubsection{Comparison with SOTA 200M Algorithms on SABER}

\begin{table}[!hb]
\scriptsize
\begin{center}
\setlength{\tabcolsep}{1.0pt}
\begin{tabular}{ c c c c c c c  c c c c c c }
\toprule
Games & RND & HWR & RAINBOW & SABER(\%) & IMPALA & SABER(\%) & LASER & SABER(\%) & GDI-I$^3$ & SABER(\%) & GDI-H$^3$ & SABER(\%)\\
\midrule
Scale  &     &       & 200M   &       &  200M    &        & 200M   & &  200M   & &  200M   & \\
\midrule
 alien              & 227.8     & \textbf{251916}    & 9491.7   &3.68    & 15962.1    & 6.25       & 976.51  & 14.04     &43384      &17.15   &48735             &19.27  \\
 amidar             & 5.8       & \textbf{104159}    & 5131.2   &4.92    & 1554.79    & 1.49       & 1829.2  & 1.75      &1442              &1.38  &1065              &1.02           \\
 assault            & 222.4     & 8647               & 14198.5  &165.90  & 19148.47   & 200.00     & 21560.4 & 200.00    &63876      &200.00  &\textbf{97155}     &\textbf{200.00} \\
 asterix            & 210       & \textbf{1000000}   & 428200   &42.81   & 300732     & 30.06      & 240090  & 23.99     &759910     &75.99   &999999            &100.00\\
 asteroids          & 719       & \textbf{10506650}  & 2712.8   &0.02    & 108590.05  & 1.03       & 213025  & 2.02      &751970     &7.15    &760005            &7.23\\
 atlantis           & 12850     & \textbf{10604840}  & 826660   &7.68    & 849967.5   & 7.90       & 841200  & 7.82      &3803000    &35.78   &3837300           &36.11\\
 bank heist         & 14.2      & \textbf{82058}     & 1358     &1.64    & 1223.15    & 1.47       & 569.4   & 0.68      &1401       &1.69    &1380              &1.66 \\
 battle zone        & 236       &801000    & 62010    &7.71    & 20885      & 2.58       & 64953.3 & 8.08      &478830     &59.77   &\textbf{824360}            &\textbf{102.92} \\
 beam rider         & 363.9     & \textbf{999999}    & 16850.2  &1.65    & 32463.47   & 3.21       & 90881.6 & 9.06      &162100     &16.18   &422390            &42.22 \\
 berzerk            & 123.7     & \textbf{1057940}            & 2545.6   &0.23    & 1852.7     & 0.16       & 25579.5 & 2.41      &7607              &0.71 &14649             &1.37             \\
 bowling            & 23.1      & \textbf{300}       & 30       &2.49    & 59.92      & 13.30      & 48.3    & 9.10      &201.9      &64.57   &205.2             &65.76\\
 boxing             & 0.1       & \textbf{100}                & 99.6     &99.60   & 99.96      & 99.96      & \textbf{100}     & \textbf{100.00}    &\best{100}        &\best{100.00 } &\textbf{100}       &\textbf{100.00}    \\
 breakout           & 1.7       & \textbf{864}                & 417.5    &48.22   & 787.34     & 91.11      & 747.9   & 86.54     &\best{864}        &\best{100.00 } &\textbf{864}     &\textbf{100.00}   \\
 centipede          & 2090.9    & \textbf{1301709}   & 8167.3   &0.47    & 11049.75   & 0.69       & 292792  & 22.37     &155830            &11.83          &195630    &14.89\\
 chopper command    & 811       & \textbf{999999}             & 16654    &1.59    & 28255      & 2.75       & 761699  & 76.15     &\best{999999}     &\best{100.00 } &\textbf{999999}    &\textbf{100.00}\\
 crazy climber      & 10780.5   & 219900    & 168788.5 &75.56   & 136950     & 60.33      & 167820  & 75.10     &201000     &90.96                 &\textbf{241170}    &\textbf{110.17}\\
 defender           & 2874.5    & \textbf{6010500}   & 55105    &0.87    & 185203     & 3.03       & 336953  & 5.56      &893110     &14.82                 &970540    &16.11\\
 demon attack       & 152.1     & \textbf{1556345}   & 111185   &7.13    & 132826.98  & 8.53       & 133530  & 8.57      &675530     &43.10                  &787985   &50.63\\
 double dunk        & -18.6     & 21                 & -0.3     &46.21   & -0.33      & 46.14      & 14      & 82.32     &\best{24}      &\best{107.58} &\textbf{24}        &\textbf{107.58}\\
 enduro             & 0         & 9500               & 2125.9   &22.38   & 0          & 0.00       & 0       & 0.00      &\best{14330}      &\best{150.84  }&14300     &150.53\\
 fishing derby      & -91.7     & \textbf{71}        & 31.3     &75.60   & 44.85      & 83.93      & 45.2    & 84.14     &59                &95.08          &65        &96.31\\
 freeway            & 0         & \textbf{38}        & 34       &89.47   & 0          & 0.00       & 0       & 0.00      &34                &89.47          &34        &89.47\\
 frostbite          & 65.2      & \textbf{454830}    & 9590.5   &2.09    & 317.75     & 0.06       & 5083.5  & 1.10      &10485      &2.29                  &11330     &2.48\\          
 gopher             & 257.6     & 355040             & 70354.6  &19.76   & 66782.3    & 18.75      & 114820.7& 32.29     &\best{488830}     &\best{137.71 } &473560    &133.41\\
 gravitar           & 173       & \textbf{162850}    & 1419.3   &0.77    & 359.5      & 0.11       & 1106.2  & 0.57      &5905       &3.52                  &5915      &3.53\\
 hero               & 1027      & \textbf{1000000}            & 55887.4  &5.49    & 33730.55   & 3.27       & 31628.7 & 3.06      &38330             &3.73           &38225     &3.72\\
 ice hockey         & -11.2     & 36                 & 1.1      &26.06   & 3.48       & 31.10      & 17.4    & 60.59     &44.92              &118.94       &\textbf{47.11}           &\textbf{123.54}\\
 jamesbond          & 29        & 45550              & 19809    &43.45   & 601.5      & 1.26       & 37999.8 & 83.41     &594500              &200.00  &\textbf{620780}          &\textbf{200.00}\\
 kangaroo           & 52        & \textbf{1424600}   & 14637.5  &1.02    & 1632       & 0.11       & 14308   & 1.00      &14500             &1.01           &14636           &1.02\\
 krull              & 1598      & 104100    & 8741.5   &6.97    & 8147.4     & 6.39       & 9387.5  & 7.60      &97575     &93.63                  &\textbf{594540}          &\textbf{200.00}\\
 kung fu master     & 258.5     & 1000000   & 52181    &5.19    & 43375.5    & 4.31       & 607443  & 60.73     &140440            &14.02          &\textbf{1666665}          &\textbf{166.68}\\
 montezuma revenge  &0          & \textbf{1219200}   & 384      &0.03    & 0          & 0.00       & 0.3     & 0.00      &3000              &0.25           &2500            &0.21\\
 ms pacman          & 307.3     & \textbf{290090}    & 5380.4   &1.75    & 7342.32    & 2.43       & 6565.5  & 2.16      &11536              &3.87          &11573           &3.89\\
 name this game     & 2292.3    & 25220              & 13136    &47.30   & 21537.2    & 83.94      & 26219.5 & 104.36    &34434      &140.19  &\textbf{36296}  &\textbf{148.31}\\
 phoenix            & 761.5     & \textbf{4014440}   & 108529   &2.69    & 210996.45  & 5.24       & 519304  & 12.92     &894460         &22.27             &959580          &23.89\\
 pitfall            & -229.4    & \textbf{114000}    & \textbf{0}        &\textbf{0.20}    & -1.66      & 0.20       & -0.6    & 0.20      &\best{0    }      &0.20           &-4.3            &0.20\\
 pong               & -20.7     & \textbf{21}                 & 20.9     &99.76   & 20.98      & 99.95      & \textbf{21}      & \textbf{100.00}    &\best{21   }      &\best{100.00 } &\textbf{21}     &\textbf{100.00}    \\
 private eye        & 24.9      & \textbf{101800}    & 4234     &4.14    & 98.5       & 0.07       & 96.3    & 0.07      &15100      &14.81                 &15100           &14.81\\
 qbert              & 163.9     & \textbf{2400000}   & 33817.5  &1.40    & 351200.12  & 14.63      & 21449.6 & 0.89      &27800             &1.03           &28657           &1.19  \\
 riverraid          & 1338.5    & \textbf{1000000}   & 22920.8  &2.16    & 29608.05   & 2.83       & 40362.7 & 3.91      &28075             &2.68           &28349           &2.70  \\
 road runner        & 11.5      & \textbf{2038100}   & 62041    &3.04    & 57121      & 2.80       & 45289   & 2.22      &878600    &43.11                  &999999          &49.06\\
 robotank           & 2.2       & 76                 & 61.4     &80.22   & 12.96      & 14.58      & 62.1    & 81.17     &108.2      &143.63  &\textbf{113.4}  &\textbf{150.68}  \\
 seaquest           & 68.4      & 999999             & 15898.9  &1.58    & 1753.2     & 0.17       & 2890.3  & 0.28      &943910	             &94.39&\textbf{1000000}          &\textbf{100.00}  \\
 skiing             & -17098    & \textbf{-3272}     & -12957.8 &29.95   & -10180.38  & 50.03      & -29968.4& -93.09    &-6774             &74.67          &-6025	         &86.77   \\
 solaris            & 1236.3    & \textbf{111420}    & 3560.3   &2.11    & 2365       & 1.02       & 2273.5  & 0.94      &11074             &8.93           &9105            &7.14  \\
 space invaders     & 148       & \textbf{621535 }   & 18789    &3.00    & 43595.78   & 6.99       & 51037.4 & 8.19      &140460    &22.58                  &154380          &24.82\\
 star gunner        & 664       & 77400              & 127029   &164.67  & 200625     & 200.00     & 321528  & 418.14    &465750     &200.00  &\textbf{677590} &\textbf{200.00}  \\
 surround           & -10       & 9.6                & \textbf{9.7}      &\textbf{100.51}  & 7.56       & 89.59      & 8.4     & 93.88     &-7.8    &11.22  &2.606           &64.32\\
 tennis             & -23.8     & 21                 & 0        &53.13   & 0.55       & 54.35      & 12.2    & 80.36     &\best{24       }  &\best{106.70    }  &\textbf{24} &\textbf{106.70}  \\
 time pilot         & 3568      & 65300              & 12926    &15.16   & 48481.5    & 72.76      & 105316  & 164.82    &216770     &200.00   &\textbf{450810}          &\textbf{200.00}\\
 tutankham          & 11.4      & \textbf{5384}      & 241      &4.27    & 292.11     & 5.22       & 278.9   & 4.98      &423.9     &7.68                    &418.2           &7.57\\
 up n down          & 533.4     & 82840              & 125755   &152.14  & 332546.75  & 200.00     & 345727  & 200.00    &\best{986440}     &\best{200.00  } &966590        &200.00\\
 venture            & 0         & \textbf{38900}     & 5.5      &0.01    & 0          & 0.00       & 0       & 0.00      &2000    &5.14                      &2000            &5.14\\
 video pinball      & 0         & \textbf{89218328}  & 533936.5 &0.60    & 572898.27  & 0.64       & 511835  & 0.57      &925830  &1.04                      &978190          &1.10\\\
 wizard of wor      & 563.5     & \textbf{395300}    & 17862.5  &4.38    & 9157.5     & 2.18       & 29059.3 & 7.22      &64439   &16.18                     &63735           &16.00\\
 yars revenge       & 3092.9    & \textbf{15000105}  & 102557   &0.66    & 84231.14   & 0.54       & 166292.3& 1.09      &972000  &6.46                      &968090          &6.43\\
 zaxxon             & 32.5      & 83700              & 22209.5  &26.51   & 32935.5    & 39.33      & 41118   & 49.11     &109140     &130.41   &\textbf{216020}          &\textbf{200.00}\\
\hline
MEAN SABER(\%) &     0.00 & \textbf{100.00}   &         & 28.39 &         & 29.45  &        & 36.78 &      & \GDIImeanSABER &      &\GDIHmeanSABER\\
Learning Efficiency &     0.00 & N/A   &         & 1.42E-09 &         & 1.47E-09  &        & 1.84E-09 &      & 3.08E-09       &      & \textbf{\GDIHmeanSABERle} \\
\hline
MEDIAN SABER(\%) & 0.00   & \textbf{100.00}   &         & 4.92 &         & 4.31  &        & 8.08  &      & \GDIImedianSABER &      & \GDIHmedianSABER  \\
Learning Efficiency &     0.00 & N/A   &         &  2.46E-10 &         & 2.16E-10  &        &  4.04E-10  &      &2.27E-09      &      & \textbf{\GDIHmedianSABERle} \\
\hline
HWRB & 0   & \textbf{57}   &         & 4 &         & 3  &        & 7  &      & \GDIIHWRB &      & \GDIHHWRB  \\
\bottomrule
\end{tabular}
\caption{Score table of SOTA 200M model-free algorithms on SABER.}
\end{center}
\end{table}
\clearpage

\subsubsection{Comparison with SOTA 10B+ Model-Free Algorithms on SABER}

\begin{table}[!hb]
\scriptsize
\begin{center}
\setlength{\tabcolsep}{1.0pt}
\begin{tabular}{ c c c c c c c c c c c }
\toprule
 Games & R2D2 & SABER(\%) & NGU & SABER(\%) & AGENT57 & SABER(\%) & GDI-I$^3$ & SABER(\%) & GDI-H$^3$ & SABER(\%)\\
\midrule
Scale  & 10B   &        & 35B &         & 100B     &        & 200M & &  200M   & \\
\midrule
 alien              & 109038.4          & 43.23             & 248100          & 98.48          & \textbf{297638.17}   &\textbf{118.17}           &43384             &17.15              &48735  &19.27        \\
 amidar             & 27751.24          & 26.64             & 17800           & 17.08          & \textbf{29660.08}    &\textbf{28.47}            &1442              &1.38               &1065              &1.02     \\
 assault            & 90526.44          & 200.00            & 34800           & 200.00         & 67212.67             &200.00                    &63876             &200.00             &\textbf{97155}     &\textbf{ 200.00}  \\
 asterix            & 999080            & 99.91             & 950700          & 95.07          & 991384.42            &99.14                     &759910            &75.99              &\textbf{999999}     &\textbf{100.00}\\
 asteroids          & 265861.2          & 2.52              & 230500          & 2.19           & 150854.61            &1.43                      &751970            &7.15               &\textbf{760005}     &\textbf{7.23}    \\
 atlantis           & 1576068           & 14.76             & 1653600         & 15.49          & 1528841.76           &14.31                     &3803000           &35.78              &\textbf{3837300}    &\textbf{36.11}            \\
 bank heist         & \textbf{46285.6}  & \textbf{56.40}    & 17400           & 21.19          & 23071.5              &28.10                     &1401              &1.69               &1380       &1.66\\
 battle zone        & 513360            & 64.08             & 691700          & 86.35          & \textbf{934134.88}   &\textbf{116.63}           &478830            &59.77              &824360     &102.92  \\
 beam rider         & 128236.08         & 12.79             & 63600           & 6.33           & 300509.8             &30.03                     &162100            &16.18              &\textbf{422390}     &\textbf{42.22}            \\
 berzerk            & 34134.8           & 3.22              & 36200           & 3.41           & \textbf{61507.83}    &\textbf{5.80}             &7607              &0.71               &14649      &1.37         \\
 bowling            & 196.36            & 62.57             & 211.9           & 68.18          & \textbf{251.18}      &\textbf{82.37}            &201.9             &64.57              &205.2      &65.76   \\
 boxing             & 99.16             & 99.16             & 99.7            & 99.70          & \textbf{100}         &\textbf{100.00}           &\best{100}        &\textbf{100.00}    &\textbf{100} &\textbf{100.00}      \\
 breakout           & 795.36            & 92.04             & 559.2           & 64.65          & 790.4                &91.46                     &\best{864}        &\textbf{100.00}    &\textbf{864}     &\textbf{100}          \\
 centipede          & 532921.84         & 40.85             & \textbf{577800} & 44.30          & 412847.86            &31.61                     &155830            &11.83              &195630    &14.89\\
 chopper command    &960648             & 96.06             &999900           & 99.99          &999900                &99.99                     &\best{999999}     &\textbf{100.00}    &\textbf{999999}  &\textbf{100.00} \\
 crazy climber      & 312768            & 144.41            & 313400          & 144.71         &\textbf{565909.85}    &\textbf{200.00}           &201000            &90.96              &241170    &110.17      \\
 defender           & 562106            & 9.31              & 664100          & 11.01          & 677642.78            &11.23                     &893110            &14.82              &\textbf{970540}&\textbf{16.11}          \\
 demon attack       & 143664.6          & 9.22              & 143500          & 9.21           & 143161.44            &9.19                      &675530            &43.10              &\textbf{787985}   &\textbf{50.63}           \\
 double dunk        & 23.12             & 105.35            & -14.1           & 11.36          & 23.93                &107.40                    &\textbf{24}       &\textbf{107.58}    &\textbf{24}&\textbf{107.58}     \\
 enduro             & 2376.68           & 25.02             & 2000            & 21.05          & 2367.71              &24.92                     &\best{14330}      &\textbf{150.84}    &14300     &150.53          \\
 fishing derby      & 81.96             & 106.74            & 32              & 76.03          & \textbf{86.97}       &\textbf{109.82}           &59                &95.08   &65        &96.31           \\
 freeway            & \textbf{34}       & \textbf{89.47}    & 28.5            & 75.00          & 32.59                &85.76                     &\best{34}         &\textbf{89.47}     &\textbf{34} &\textbf{89.47}           \\
 frostbite          & 11238.4           & 2.46              & 206400          & 45.37          &\textbf{541280.88}    &\textbf{119.01}           &10485             &2.29               &11330     &2.48            \\
 gopher             & 122196            & 34.37             & 113400          & 31.89          & 117777.08            &33.12                     &\best{488830}     &\textbf{137.71}    &473560    &133.41          \\
 gravitar           & 6750              & 4.04              & 14200           & 8.62           &\textbf{19213.96}     &\textbf{11.70}            &5905              &3.52               &5915      &3.53    \\
 hero               & 37030.4           & 3.60              & 69400           & 6.84           &\textbf{114736.26}    &\textbf{11.38}            &38330             &3.73               &38225     &3.72    \\
 ice hockey         & \textbf{71.56}    & \textbf{175.34}   &-4.1             & 15.04          & 63.64                &158.56                    &44.92             &118.94             &\textbf{47.11}  &\textbf{123.54}\\
 jamesbond          & 23266             & 51.05             & 26600           & 58.37          & 135784.96            &200.00                    &594500            &200.00             &\textbf{620780} &\textbf{200.00}   \\
 kangaroo           & 14112             & 0.99              & \textbf{35100}  & 2.46           &24034.16              &1.68                      &14500             &1.01               &14636           &1.02\\
 krull              & 145284.8          & 140.18            & 127400          & 122.73         & 251997.31            &200.00                    &97575             &93.63              &\textbf{594540} &\textbf{200.00}      \\
 kung fu master     & 200176            & 20.00             & 212100          & 21.19          & 206845.82            &20.66                     &140440            &14.02              &\textbf{1666665}&\textbf{166.68}\\
 montezuma revenge  & 2504              & 0.21              & \textbf{10400}  & 0.85           &9352.01               &0.77                      &3000              &0.25               &2500            &0.21   \\
 ms pacman          & 29928.2           & 10.22             & 40800           & 13.97          & \textbf{63994.44}    &\textbf{21.98}            &11536             &3.87               &11573           &3.89 \\
 name this game     & 45214.8           & 187.21            & 23900           & 94.24          &\textbf{54386.77}     &\textbf{200.00}           &34434             &140.19             &36296           &148.31\\
 phoenix            & 811621.6          & 20.20             & 959100          & 23.88          &908264.15             &22.61                     &894460            &22.27              &\textbf{959580} &\textbf{23.89}\\
 pitfall            & \textbf{0}        & \textbf{0.20}     & 7800            & 7.03           &\textbf{18756.01}     &\textbf{16.62}            &0                 &0.20               &-4.3            &0.20\\
 pong               & \textbf{21}       & \textbf{100.00}   & 19.6            & 96.64          & 20.67                &99.21                     &\best{21}         &\textbf{100.00}    &\textbf{21}     &\textbf{100.00}\\
 private eye        & 300               & 0.27              & \textbf{100000} & 98.23          & 79716.46             &78.30                     &15100             &14.81              &15100           &14.81\\
 qbert              & 161000            & 6.70              & 451900          & 18.82          &\textbf{580328.14}    &\textbf{24.18}            &27800             &1.03               &28657           &1.19\\
 riverraid          & 34076.4           & 3.28              & 36700           & 3.54           & \textbf{63318.67}    &\textbf{6.21}             &28075             &2.68               &28349           &2.70\\
 road runner        & 498660            & 24.47             & 128600          & 6.31           & 243025.8             &11.92                     &\best{878600}     &\textbf{43.11}     &999999          &49.06\\
 robotank           & \textbf{132.4}    & \textbf{176.42}   & 9.1             & 9.35           &127.32                &169.54                    &108               &143.63             &113.4           &150.68\\
 seaquest           & 999991.84         & 100.00            & \textbf{1000000}& 100.00         &999997.63             &100.00                    &943910	        &94.39              &\textbf{1000000}&\textbf{100.00}\\
 skiing             & -29970.32         & -93.10            & -22977.9        & -42.53         & \textbf{-4202.6}     &\textbf{93.27}            &-6774             &74.67              &-6025	         &86.77\\
 solaris            & 4198.4            & 2.69              & 4700            & 3.14           & \textbf{44199.93}    &\textbf{38.99}            &11074             &8.93               &9105            &7.14\\
 space invaders     & 55889             & 8.97              & 43400           & 6.96           & 48680.86             &7.81                      &140460            &22.58              &\textbf{154380} &\textbf{24.82}\\
 star gunner        & 521728            & 200.00            & 414600          & 200.00         &\textbf{839573.53}    &\textbf{200.00}           &465750            &200.00             &677590          &200.00\\
 surround           & \textbf{9.96}     & \textbf{101.84}   & -9.6            & 2.04           & 9.5                  &99.49                     &-7.8              &11.22              &2.606           &64.32\\
 tennis             & \textbf{24}       & \textbf{106.70}   & 10.2            & 75.89          & 23.84                &106.34                    &\textbf{24}       &\textbf{106.70}    &\textbf{24}     &\textbf{106.70}\\
 time pilot         & 348932            & 200.00            & 344700          & 200.00         &405425.31             &200.00                    &216770            &200.00             &\textbf{450810} &\textbf{200.00}\\
 tutankham          & 393.64            & 7.11              & 191.1           & 3.34           & \textbf{2354.91}     &\textbf{43.62}            &423.9             &7.68               &418.2           &7.57\\
 up n down          & 542918.8          & 200.00            & 620100          & 200.00         & 623805.73            &200.00                    &\best{986440}     &\textbf{200.00}    &966590          &200.00\\
 venture            & 1992              & 5.12              & 1700            & 4.37           &\textbf{2623.71}      &\textbf{6.74}             &2000              &5.14               &2000            &5.14\\
 video pinball      & 483569.72         & 0.54              & 965300          & 1.08           &\textbf{992340.74}    &\textbf{1.11}             &925830            &1.04               &978190          &1.10\\
 wizard of wor      & 133264            & 33.62             & 106200          & 26.76          &\textbf{157306.41}    &\textbf{39.71}            &64439             &16.18              &63735           &16.00\\
 yars revenge       & 918854.32         & 6.11              & 986000          & 6.55           &\textbf{998532.37}    &\textbf{6.64}             &972000            &6.46               &968090          &6.43\\
 zaxxon             & 181372            & 200.00            & 111100          & 132.75         &\textbf{249808.9}     &\textbf{200.00}           &109140            &130.41             &216020          &200.00\\
\hline
MEAN SABER(\%)        &                   & 60.43             &                 &  50.47         &                      & \textbf{76.26}  &                  & \GDIImeanSABER &      & \GDIHmeanSABER \\
Learning Efficiency &    & 6.04E-11 &         &  1.44E-11  &        & 7.63E-12 &      & 5.90E-09       &      & \textbf{\GDIHmeanSABERle} \\
\hline
MEDIAN SABER(\%)      &                   & 33.62             &                 & 21.19          &                      & 43.62  &                  & \GDIImedianSABER &      & \textbf{\GDIHmedianSABER}  \\
Learning Efficiency &    & 3.36E-11 &         & 6.05E-12  &        & 4.36E-12 &      &2.27E-09      &      & \textbf{\GDIHmedianSABERle}\\
\hline
HWRB                  &                   & 15       &                 & 9              &                      & 18     &                  & \GDIIHWRB &                  & \textbf{\GDIHHWRB}\\
\bottomrule
\end{tabular}
\caption{Score table of SOTA  10B+ model-free algorithms on SABER.}
\end{center}
\end{table}
\clearpage

\subsubsection{Comparison with SOTA Model-Based Algorithms on SABER}
SimPLe \citep{modelbasedatari} and DreamerV2 \citep{dreamerv2} haven't evaluated all 57 Atari Games in their paper. For fairness, we set the score on those games as N/A, which will not be considered when calculating the median and mean SABER.
\begin{table}[!hb]
\scriptsize
\begin{center}
\setlength{\tabcolsep}{1.0pt}
\begin{tabular}{ c c c c c c c c c c c }
\toprule
 Games              & MuZero         & SABER(\%)      & DreamerV2 & SABER(\%)    & SimPLe             & SABER(\%)          & GDI-I$^3$     & SABER(\%) & GDI-H$^3$ & SABER(\%)\\
\midrule
Scale               & 20B            &              & 200M      &            & 1M               &                  & 200M     & &  200M   & \\
\midrule
 alien              & \textbf{741812.63}      & \textbf{200.00  }     &3483       & 1.29     &616.9     & 0.15    & 43384                       &  17.15      &48735    &19.27   \\
 amidar             & \textbf{28634.39 }      & \textbf{27.49   }  &2028       & 1.94     &74.3      & 0.07    & 1442                           &  1.38       &1065     &1.02          \\
 assault            & \textbf{143972.03}      & \textbf{200.00 }     &7679       & 88.51       &527.2     & 3.62       & 63876                  &  200.00     &97155     &200.00    \\
 asterix            & 998425         & 99.84        &25669      & 2.55     &1128.3    & 0.09       & 759910                     &  75.99      &\textbf{999999}            &\textbf{100.00}    \\
 asteroids          & 678558.64               & 6.45        &3064                & 0.02    &793.6              & 0.00           &751970         & 7.15   &\textbf{760005}            &\textbf{7.23}   \\
 atlantis           & 1674767.2               & 15.69           &989207              & 9.22       &20992.5            & 0.08    &3803000    & 35.78  &\textbf{3837300}           &\textbf{36.11} \\
 bank heist         & 1278.98                 & 1.54        &1043                & 1.25    &34.2               & 0.02    &\best{1401}           & \best{1.69  } &1380              &1.66   \\
 battle zone        & \textbf{848623         }& \textbf{105.95}      &31225      & 3.87     &4031.2    & 0.47       & 478830                    &  59.77        &824360            &102.92  \\
 beam rider         & \textbf{454993.53}      & \textbf{45.48}      &12413      & 1.21     &621.6     & 0.03    & 162100                        &  16.18        &422390            &42.22  \\
 berzerk            & \textbf{85932.6        }& \textbf{8.11}      &751        & 0.06     &N/A       & N/A     & 7607                           &  0.71         &14649             &1.37  \\
 bowling            & \textbf{260.13         }& \textbf{85.60}      &48         & 8.99     &30        & 2.49    & 202                           &  64.57        &205.2             &65.76 \\
 boxing             & \textbf{100}                     & \textbf{100.00}       &87                  & 86.99       &7.8         & 7.71       & \best{100}          & \best{100.00 }&\textbf{100}       &\textbf{100.00}\\
 breakout           & \textbf{864}                     & \textbf{100.00}          &350                 & 40.39       &16.4     & 1.70       & \best{864}          & \best{100.00} &\textbf{864}     &100.00\\
 centipede          & \textbf{1159049.27}     & \textbf{89.02}     &6601       & 0.35     &N/A       & N/A     & 155830                         & 11.83         &195630    &14.89\\
 chopper command    & 991039.7                & 99.10     &2833                & 0.20     & 979.4             & 0.02    & \best{999999}         & \best{100.00} &\textbf{999999}    &\textbf{100.00}\\
 crazy climber      & \textbf{458315.4}       & \textbf{200.00    }  &141424     & 62.47       & 62583.6  & 24.77   & 201000                    & 90.96         &241170    &110.17\\
 defender           & 839642.95               & 13.93      & N/A                & N/A        & N/A               & N/A      & 893110     &14.82  &\textbf{970540}    &\textbf{16.11}\\
 demon attack       & 143964.26               & 9.24      & 2775              &0.17      & 208.1             & 0.00    & 675530         & 43.40  &\textbf{787985}   &\textbf{50.63}\\
 double dunk        & 23.94          & 107.42     & 22        &102.53        & N/A      & N/A                                               & \textbf{24}& \textbf{107.58} &\textbf{24}&\textbf{107.58} \\
 enduro             & 2382.44                 & 25.08       & 2112              &22.23         & N/A               & N/A     & \best{14330}    & \best{150.84}  &14300     &150.53\\
 fishing derby      & \textbf{91.16}          & \textbf{112.39     }              &93.24          &200.00      &-90.7     & 0.61    & 59        & 92.89         &65        &96.31\\
 freeway            & 33.03                   & 86.92       & \textbf{34}                 &\textbf{89.47}          &16.7               & 43.95   & \best{34}      & \best{89.47}  &\textbf{34}        &\textbf{89.47}\\
 frostbite          & \textbf{631378.53}      & \textbf{138.82}     & 15622    &3.42       &236.9     & 0.04    & 10485                        & 2.29           &11330     &2.48\\
 gopher             & 130345.58               & 36.67      & 53853             &15.11         &596.8              & 0.10       & \best{488830} & \best{137.71}  &473560    &133.41\\
 gravitar           & \textbf{6682.7     }    & \textbf{4.00    }  & 3554     &2.08      &173.4     & 0.00    & 5905                           & 3.52           &5915      &3.53\\
 hero               & \textbf{49244.11}       & \textbf{4.83    }  & 30287    &2.93      &2656.6    & 0.16        & 38330                      & 3.73           &38225     &3.72\\
 ice hockey         & \textbf{67.04      }    & \textbf{165.76  }  & 29        &85.17         &-11.6     & -0.85       &44.92              &118.94        &47.11           &123.54\\
 jamesbond          & 41063.25                & 90.14     & 9269              &20.30         &100.5     & 0.16    &594500              &200.00  &\textbf{620780}          &\textbf{200.00}\\
 kangaroo           & \textbf{16763.6        }& \textbf{1.17}  & 11819     &0.83      &51.2      & 0.00    & 14500                              & 1.01          &14636           &1.02\\
 krull              & 269358.27      & 200.00 & 9687     &7.89       &2204.8    & 0.59    & 97575                        & 93.63          &\textbf{594540}          &\textbf{200.00}\\
 kung fu master     & \textbf{204824         }& \textbf{20.46}  & 66410    &6.62       &14862.5   & 1.46    & 140440                           & 14.02          &1666665          &166.68\\
 montezuma revenge  & 0                       & 0.00         & 1932              &0.16       &N/A       & N/A     & \best{3000}                & \best{0.25  }  &2500            &0.21\\
 ms pacman          & \textbf{243401.1 }      & \textbf{83.89   }  & 5651     &1.84       &1480      & 0.40    & 11536                         & 3.87           &11573           &3.89\\
 name this game     & \textbf{157177.85}      & \textbf{200.00  }  & 14472    &53.12          &2420.7    & 0.56    & 34434                     & 140.19         &36296           &148.31\\
 phoenix            & \textbf{955137.84}      & \textbf{23.78   }  & 13342     &0.31       &N/A       & N/A     & 894460                        & 22.27         &959580          &23.89\\
 pitfall            & \textbf{0}              & \textbf{0.20}               & -1                 &0.20       &N/A       & N/A     & \best{0}             & \best{0.20}   &-4.3            &0.20\\
 pong               & \textbf{21}             & \textbf{100.00}             & 19                 &95.20          & 12.8     & 80.34   & \best{21}        & \best{100.00} &\textbf{21}   &\textbf{100.00}\\
 private eye        & \textbf{15299.98 }      & \textbf{15.01}     & 158       &0.13       & 35       & 0.01    & 15100                         & 14.81         &15100           &14.81\\
 qbert              & \textbf{72276          }& \textbf{3.00}      & 162023    &6.74       & 1288.8   & 0.05    & 27800                         & 1.15          &28657           &1.19\\
 riverraid          & \textbf{323417.18}      & \textbf{32.25}     & 16249    &1.49       & 1957.8   & 0.06    & 28075                         & 2.68           &28349           &2.70\\
 road runner        & 613411.8                & 30.10              & 88772             &4.36       & 5640.6   & 0.28       & 878600     & 43.11   &\textbf{999999}          &\textbf{49.06}\\
 robotank           & \textbf{131.13}         & \textbf{174.70}    & 65        &85.09          & N/A      & N/A     & 108                       & 143.63        &113.4           &150.68\\
 seaquest           & 999976.52               & 100.00             & 45898             &4.58       & 683.3             & 0.06 &943910	             &94.39&\textbf{1000000}          &\textbf{100.00}\\
 skiing             & \textbf{-29968.36}      & \textbf{-93.09}      & -8187    &64.45          & N/A      & N/A     & -6774                   & 74.67          &-6025	          &86.77\\
 solaris            & 56.62                   & -1.07              & 883                &-0.32          & N/A               & N/A   & \best{11074}  & \best{8.93 }&9105            &7.14\\
 space invaders     & 74335.3                 & 11.94                 & 2611               &0.40       & N/A               & N/A    & 140460 & 22.58&\textbf{154380}          &\textbf{24.82}    \\
 star gunner        &549271.7       & 200.00     & 29219    &37.21          & N/A      & N/A     & 465750      & 200.00                         &\textbf{677590}          &\textbf{200.00}\\
 surround           & \textbf{9.99       }    & \textbf{101.99}     & N/A       &N/A         & N/A      & N/A     & -8          & 11.22                            &2.606           &64.32\\
 tennis             & 0       & 53.13      & 23        &104.46        & N/A      & N/A     & \textbf{24}          & \textbf{106.70}       &\textbf{24}           &\textbf{106.70}\\
 time pilot         & \textbf{476763.9}       & \textbf{200.00}         & 32404    &46.71         & N/A      & N/A     & 216770      & 200.00    &450810          &200.00     \\
 tutankham          & \textbf{491.48     }    & \textbf{8.94}   & 238       &4.22         & N/A      & N/A     & 424         & 7.68              &418.2           &7.57\\
 up n down          & 715545.61               & 200.00            & 648363            &200.00        & 3350.3            & 3.42   & \best{986440}        & \best{200.00} &966590        &200.00    \\
 venture            & 0.4                     & 0.00        & 0                  &0.00       & N/A               & N/A     & \best{2000}          & \best{5.23}   &2000            &5.14    \\
 video pinball      & \textbf{981791.88}      & \textbf{1.10}      & 22218    &0.02     & N/A      & N/A     & 925830      & 1.04                                &978190          &1.10\\
 wizard of wor      & \textbf{197126         }& \textbf{49.80}      & 14531    &3.54     & N/A      & N/A     & 64439       & 16.14                             &63735           &16.00\\
 yars revenge       & 553311.46               & 3.67      & 20089             &0.11     & 5664.3            & 0.02    & \best{972000}        & \best{6.46}      &968090          &6.43\\
 zaxxon             & \textbf{725853.9}       & \textbf{200.00}      & 18295    &21.83          & N/A      & N/A     & 109140      & 130.41                     &216020          &200.00\\
\hline
MEAN SABER(\%) &               &\textbf{71.94}  &         &  27.22 &                                          & 4.67  &     & \GDIImeanSABER &      &\GDIHmeanSABER\\
Learning Efficiency &     &  3.60E-11 &         & 1.36E-09  &        &\textbf{ 4.67E-08}&      & 5.90E-09       &      & \GDIHmeanSABERle\\
\hline
MEDIAN SABER(\%) &            & 49.8   &         & 4.22  &                                          & 0.13  &     & \GDIImedianSABER &      & \textbf{\GDIHmedianSABER}  \\
Learning Efficiency &     & 2.49E-11 &         & 2.11E-10  &        & 1.60E-09&      &2.27E-09      &      & \textbf{\GDIHmedianSABERle}\\
\hline
HWRB &            & 19   &         & 3    &                                          & 0  &      & \GDIIHWRB &                  & \textbf{\GDIHHWRB}\\
\bottomrule
\end{tabular}
\caption{Score table of  SOTA  model-based algorithms on SABER.}
\end{center}
\end{table}
\clearpage

\subsubsection{Comparison with Other SOTA algorithms on SABER}
In this section, we report the performance of our algorithm compared with other SOTA algorithms, Go-Explore \citep{goexplore} and Muesli \citep{muesli}.

\begin{table}[!hb]
\scriptsize
\begin{center}
\setlength{\tabcolsep}{1.0pt}
\begin{tabular}{ c  c c   c c c c c c}
\toprule
 Games           & Muesli              & SABER(\%)          & Go-Explore              & SABER(\%)                     & GDI-I$^3$ & SABER(\%)        & GDI-H$^3$ & SABER(\%)         \\
\midrule
Scale            & 200M           &           & 10B                     &                             & 200M      &        & 200M      &           \\
\midrule    
 alien           &139409          &55.30               &\textbf{959312}       &\textbf{200.00}                & 43384             &17.15     &48735             &19.27         \\
 amidar          &\textbf{21653}  &\textbf{20.78}      &19083                 &18.32                          & 1442              &1.38      &1065              &1.02          \\
 assault         &36963           &200.00              &30773                 &200.00                         & 63876             &200.00  &\textbf{97155}&\textbf{200.00}        \\
 asterix         &316210          &31.61               &999500       &99.95                 & 759910            &75.99      &\textbf{999999}            &\textbf{100.00}     \\
 asteroids       &484609          &4.61                &112952                &1.07                           & 751970     &7.15&\textbf{760005}            &\textbf{7.23}           \\
 atlantis        &1363427         &12.75               &286460                &2.58                           & 3803000    &35.78&\textbf{3837300}           &\textbf{36.11}          \\
 bank heist      &1213            &1.46                &\textbf{3668}         &\textbf{4.45  }                & 1401              &1.69        &1380              &1.66      \\
 battle zone     &414107          &51.68               &\textbf{998800}       &\textbf{124.70}                & 478830            &59.77       &824360            &102.92      \\
 beam rider      &288870          &28.86               &371723       &37.15                 & 162100            &16.18       &\textbf{422390}            &\textbf{42.22}      \\
 berzerk         &44478           &4.19                &\textbf{131417}       &\textbf{12.41 }                & 7607              &0.71        &14649             &1.37      \\
 bowling         &191             &60.64               &\textbf{247}           &\textbf{80.86 }                & 202               &64.57      &205.2             &65.76       \\
 boxing          &99              &99.00               &91                     &90.99                          & \best{100}        &\best{100.00} &\textbf{100}       &\textbf{100.00}           \\
 breakout        &791             &91.53               &774                    &89.56                          & \best{864}        &\best{100.00} &\textbf{864}     &\textbf{100.00}          \\
 centipede       &\textbf{869751} &\textbf{66.76}      &613815                &47.07                          & 155830            &11.83        &195630    &14.89  \\
 chopper command &101289       &10.06               &996220                &99.62                          & \best{999999}     &\best{100.00}   &\textbf{999999}    &\textbf{100.00}       \\
 crazy climber   &175322       &78.68               &235600       &107.51                & 201000            &90.96           &\textbf{241170}    &\textbf{110.17} \\
 defender        &629482       &10.43               &N/A                    &N/A                            & 893110     &14.82   &\textbf{970540}    &\textbf{16.11}        \\
 demon attack    &129544       &8.31                &239895                 &15.41                          & 675530     &43.40   &\textbf{787985}   &\textbf{50.63}       \\
 double dunk     &-3           &39.39               &\textbf{24}            &\textbf{107.58}                         & \best{24}         &\best{107.58}  &\textbf{24}        &\textbf{107.58}         \\
 enduro          &2362         &24.86               &1031                   &10.85                          & \best{14330}      &\best{150.84}  &14300     &150.53         \\
 fishing derby   &51           &87.71               &\textbf{67}            &\textbf{97.54 }                & 59                &92.89          &65        &96.31  \\
 freeway         &33           &86.84               &\textbf{34}            &\textbf{89.47 }                & \best{34}         &\best{89.47}   &\textbf{34}        &\textbf{89.47}         \\
 frostbite       &301694       &66.33               &\textbf{999990}       &\textbf{200.00}                & 10485             &2.29            &11330     &2.48\\
 gopher          &104441       &29.37               &134244                &37.77                          & \best{488830}     &\best{137.71}   &473560    &133.41       \\
 gravitar        &11660        &7.06                &\textbf{13385}        &\textbf{8.12}                  & 5905              &3.52            &5915      &3.53\\
 hero            &37161        &3.62                &37783                  &3.68                           &38330      &3.73    &\textbf{38225}     &\textbf{3.72}       \\
 ice hockey      &25           &76.69               &33                     &93.64                          &44.92              &118.94       &\textbf{47.11}           &\textbf{123.54}          \\
 jamesbond       &19319        &42.38               &200810                &200.00                         &594500              &200.00   &\textbf{620780}          &\textbf{200.00}     \\
 kangaroo        &14096        &0.99                &\textbf{24300}        &\textbf{1.70}                  & 14500             &1.01            &14636              &1.02\\
 krull           &34221        &31.83               &63149                 &60.05                          & 97575      &93.63    &\textbf{594540}    &\textbf{200.00}       \\
 kung fu master  &134689       &13.45               &24320                 &2.41                           & 140440     &14.02    &\textbf{1666665}          &\textbf{166.68}        \\
 montezuma revenge  &2359      &0.19                &\textbf{24758}        &\textbf{2.03}                  & 3000              &0.25            &2500            &0.21\\
 ms pacman          &65278     &22.42               &\textbf{456123}       &\textbf{157.30}                & 11536             &3.87            &11573           &3.89\\
 name this game     &105043    &200.00              &\textbf{212824}       &\textbf{200.00}                & 34434             &140.19          &36296           &148.31\\
 phoenix        &805305        &20.05               &19200                 &0.46                           & 894460     &22.27    &\textbf{959580}          &\textbf{23.89}\\
 pitfall        &0             &0.20                &\textbf{7875}          &\textbf{7.09   }               & 0                 &0.2            &-4.3            &0.20 \\
 pong           &20            &97.60               &\textbf{21}            &\textbf{100.00 }               & \best{21}         &\best{100}     &\textbf{21}              &\textbf{100.00}       \\
 private eye    &10323         &10.12               &\textbf{69976}        &\textbf{68.73  }               & 15100             &14.81           &15100           &14.81  \\
 qbert          &157353        &6.55                &\textbf{999975}       &\textbf{41.66  }               & 27800             &1.15            &28657           &1.19\\
 riverraid      &\textbf{47323}&\textbf{4.60}       &35588                 &3.43                           & 28075             &2.68            &28349           &2.70\\
 road runner    &327025        &16.05               &999900        &49.06                 & 878600            &43.11          &\textbf{999999}          &\textbf{49.06}\\
 robotank       &59            &76.96               &\textbf{143}           &\textbf{190.79 }               & 108               &143.63         &113.4           &150.68\\
 seaquest       &815970        &81.60               &539456                 &53.94                 &943910	             &94.39   &\textbf{1000000}          &\textbf{100.00}\\
 skiing         &-18407        &-9.47               &\textbf{-4185}        &\textbf{93.40  }               & -6774             &74.67           &-6025	         &86.77\\
 solaris        &3031          &1.63                &\textbf{20306}        &\textbf{17.31  }               & 11074             &8.93            &9105            &7.14\\
 space invaders &59602         &9.57                &93147                 &14.97                          & 140460     &22.58    &\textbf{154380}          &\textbf{24.82}\\
 star gunner    &214383        &200.00              &609580      &200.00                         & 465750     &200.00&\textbf{677590}          &\textbf{200.00}     \\
 surround       &\textbf{9}    &\textbf{96.94}      &N/A                    &N/A                            & -8         &11.22                 &2.606           &64.32\\
 tennis         &12            &79.91               &\best{24}              &\best{106.7}                   & \best{24}         &\best{106.70   }&\textbf{24}           &\textbf{106.70}            \\
 time pilot     &359105 &200.00   &183620                &200.00                         & 216770     & 200.00               &\textbf{450810}          &\textbf{200.00}\\
 tutankham      &252           &4.48                &\textbf{528}           &\textbf{9.62}                  & 424               &7.68           &418.2           &7.57\\
 up n down      &649190        &200.00              &553718                &200.00                         & \best{986440}     &\best{11.9785}  &966590        &200.00         \\
 venture        &2104          &5.41                &\textbf{3074}         &\textbf{7.90}                  & 2035              &5.23            &2000            &5.14\\
 video pinball  &685436        &0.77                &\textbf{999999}       &\textbf{1.12}                  & 925830            &1.04            &978190          &1.10\\
 wizard of wor  &93291         &23.49               &\textbf{199900}       &\textbf{50.50}                 & 64293             &16.14           &63735           &16.00\\
 yars revenge   &557818        &3.70                &\textbf{999998}       &\textbf{6.65}                  & 972000            &6.46            &968090          &6.43\\
 zaxxon         &65325         &78.04               &18340                 &21.88                          & 109140         &130.41   &\textbf{216020} &\textbf{200.00}        \\
\hline    
MEAN SABER(\%)  &              & 48.74              &                       & \textbf{71.80}                &                   & \GDIImeanSABER &      &\GDIHmeanSABER\\
Learning Efficiency &    & 2.43E-09 &                       & 7.18E-11                      &      & 3.08E-09       &      & \textbf{\GDIHmeanSABERle}\\
\hline
MEDIAN SABER(\%)&              &  24.86             &                       &50.5                 &                   & \GDIImedianSABER &      & \textbf{\GDIHmedianSABER}  \\
Learning Efficiency &    & 1.24E-09 &                       & 5.05E-11                     &      &1.78E-09      &      & \textbf{\GDIHmedianSABERle}\\
\hline
HWRB           &                       & 5  &                       & 15              &            & \GDIIHWRB &                  & \textbf{\GDIHHWRB}\\
\bottomrule
\end{tabular}
\caption{Score table of other SOTA  algorithms on SABER.}
\end{center}
\end{table}

\clearpage

\subsection{Atari Games Learning Curves}
\label{app: Atari Games Learning Curves}

\subsubsection{Atari Games Learning Curves of GDI-I$^3$}
\renewcommand{\thesubfigure}{\arabic{subfigure}.}
\begin{figure}[!ht] 
    \subfigure[alien]{
    \includegraphics[width=0.3\textwidth]{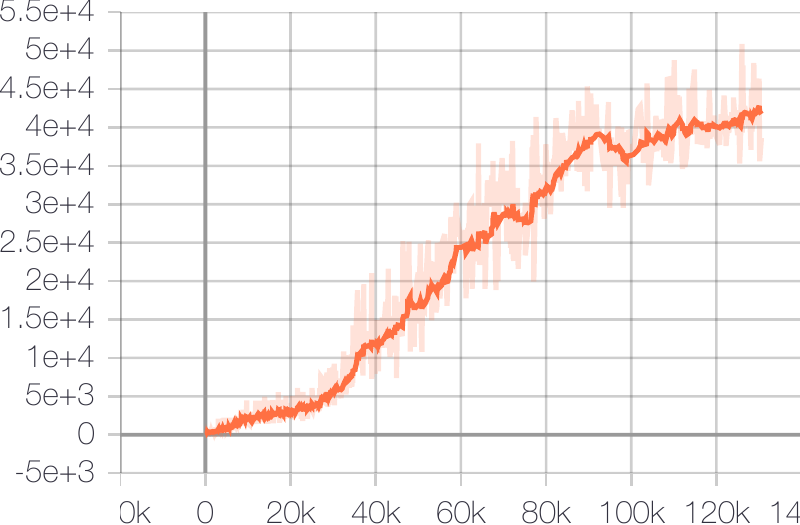}
    }
    \subfigure[amidar]{
    \includegraphics[width=0.3\textwidth]{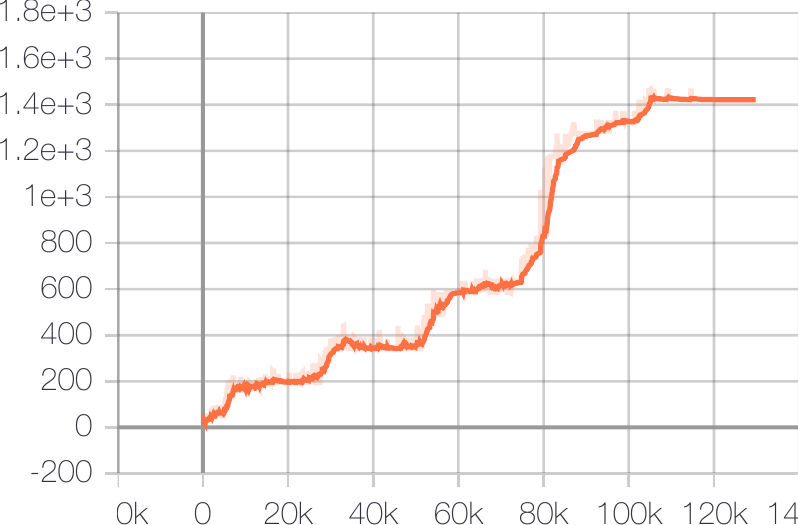}
    }
    \subfigure[assault]{
    \includegraphics[width=0.3\textwidth]{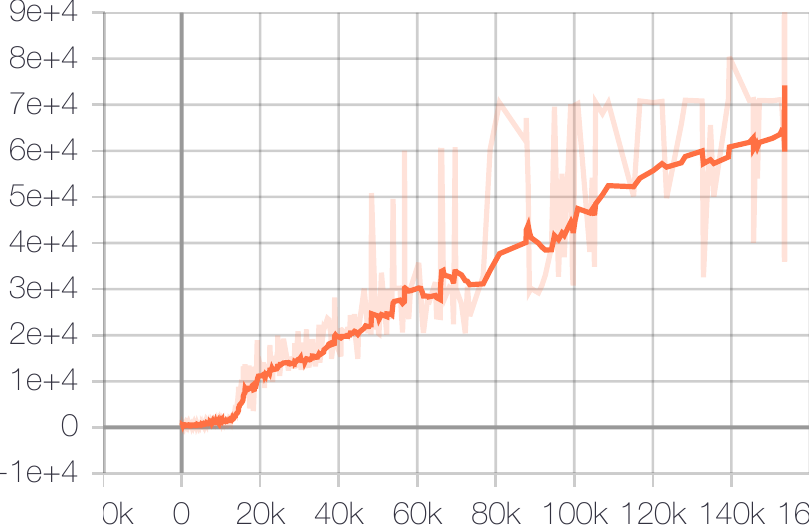}
    }
\end{figure}

\begin{figure}[!ht]
    \subfigure[asterix]{
    \includegraphics[width=0.3\textwidth]{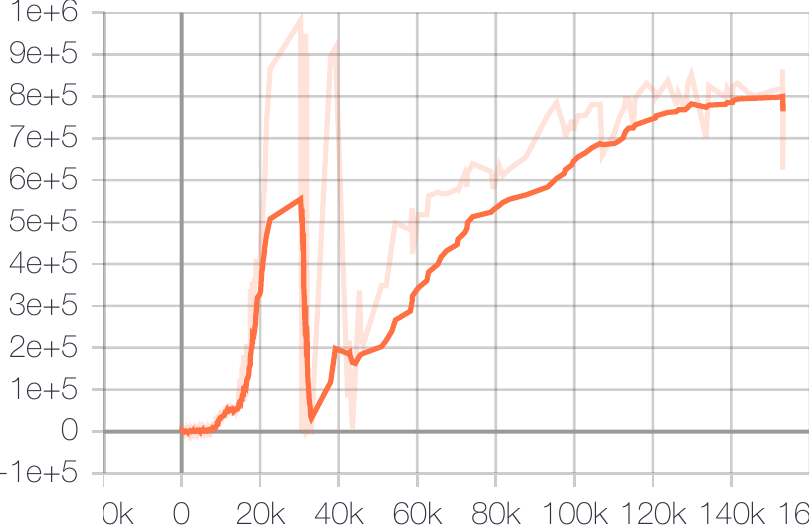}
    }
    \subfigure[asteroids]{
    \includegraphics[width=0.3\textwidth]{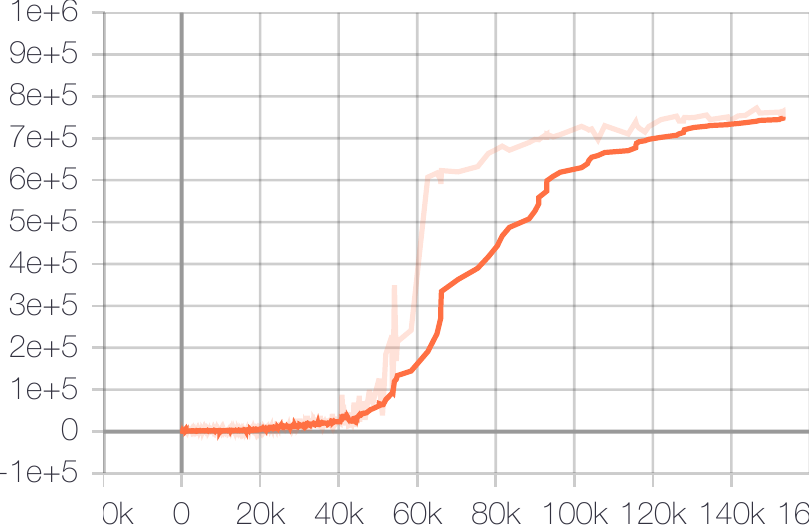}
    }
    \subfigure[atlantis]{
    \includegraphics[width=0.3\textwidth]{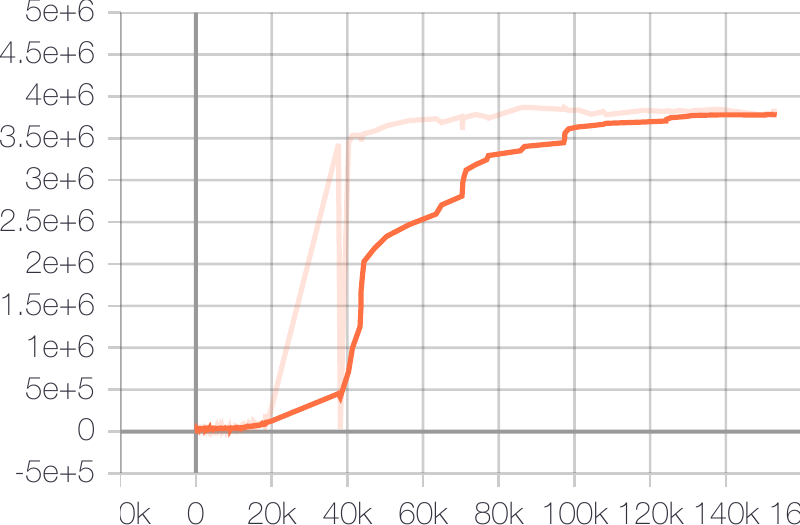}
    }
\end{figure}

\begin{figure}[!ht]
    \subfigure[bank\_heist]{
    \includegraphics[width=0.3\textwidth]{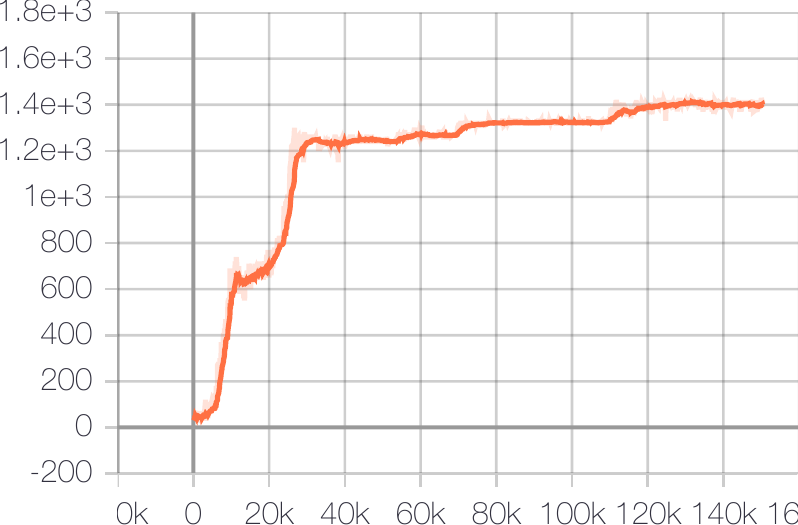}
    }
    \subfigure[battle\_zone]{
    \includegraphics[width=0.3\textwidth]{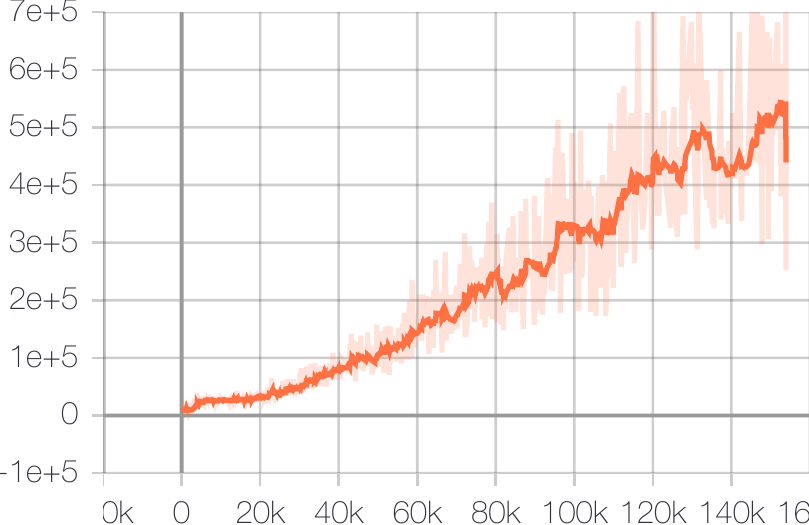}
    }
    \subfigure[beam\_rider]{
    \includegraphics[width=0.3\textwidth]{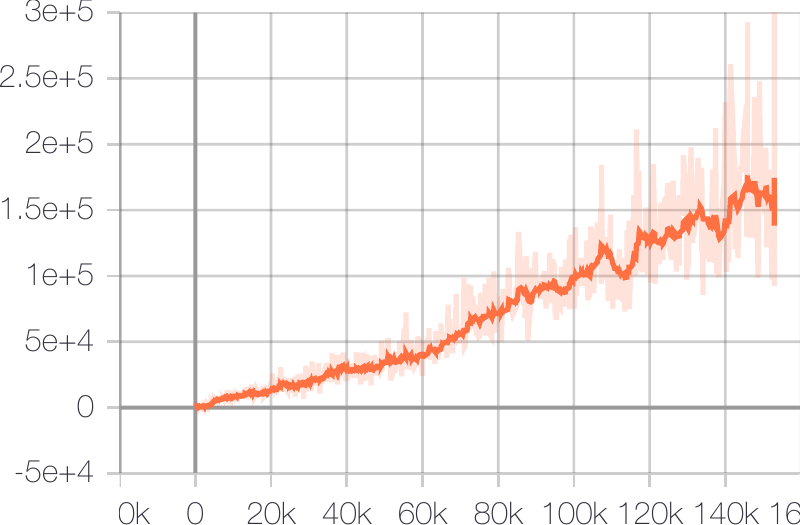}
    }
\end{figure}

\begin{figure}[!ht]
    \subfigure[berzerk]{
    \includegraphics[width=0.3\textwidth]{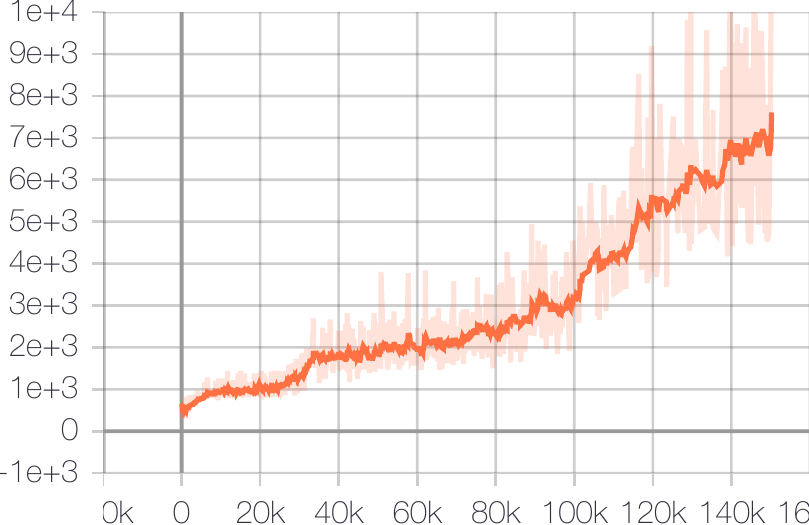}
    }
    \subfigure[bowling]{
    \includegraphics[width=0.3\textwidth]{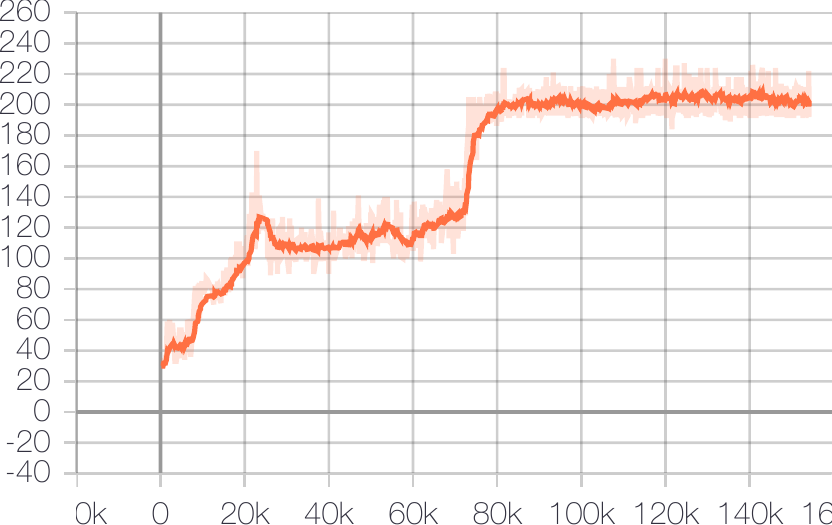}
    }
    \subfigure[boxing]{
    \includegraphics[width=0.3\textwidth]{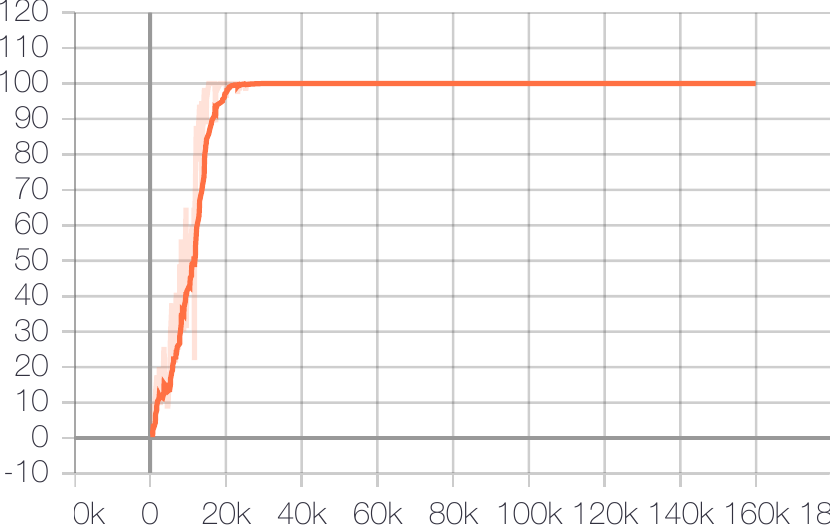}
    }
\end{figure}

\begin{figure}[!ht]
    \subfigure[breakout]{
    \includegraphics[width=0.3\textwidth]{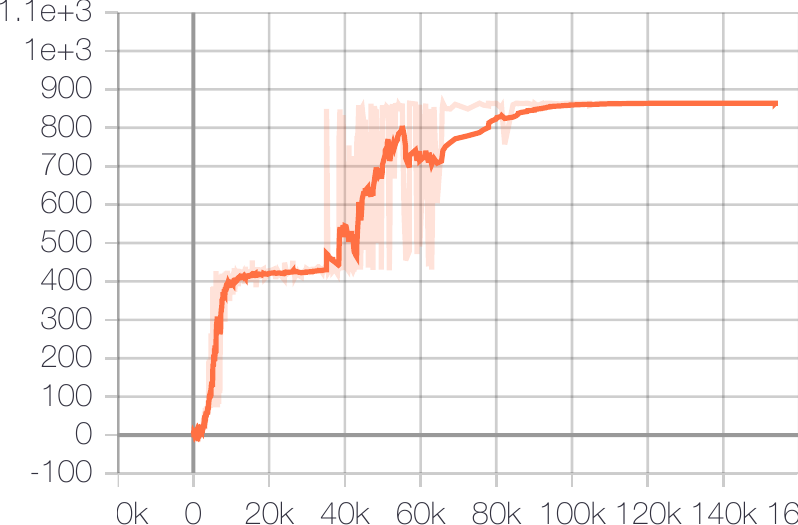}
    }
    \subfigure[centipede]{
    \includegraphics[width=0.3\textwidth]{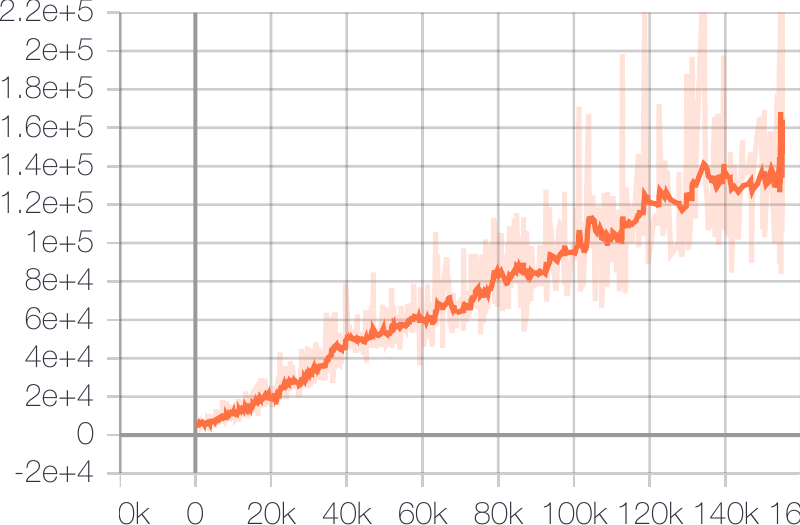}
    }
    \subfigure[chopper\_command]{
    \includegraphics[width=0.3\textwidth]{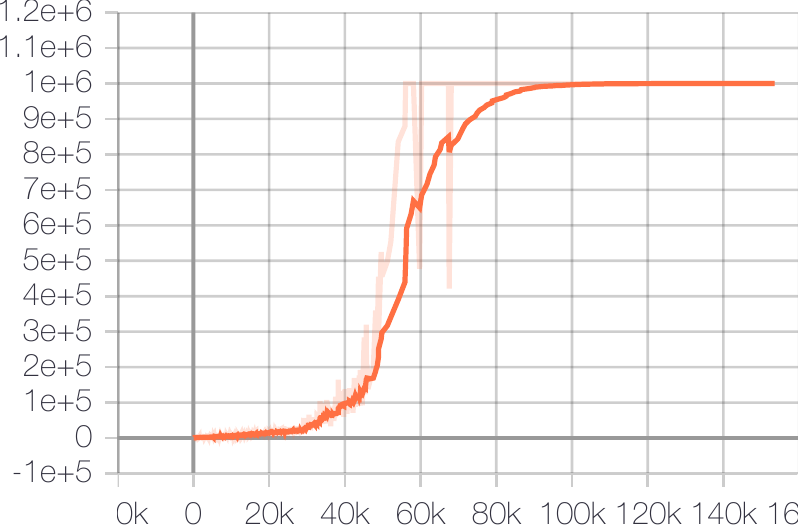}
    }
\end{figure}

\begin{figure}[!ht]
    \subfigure[crazy\_climber]{
    \includegraphics[width=0.3\textwidth]{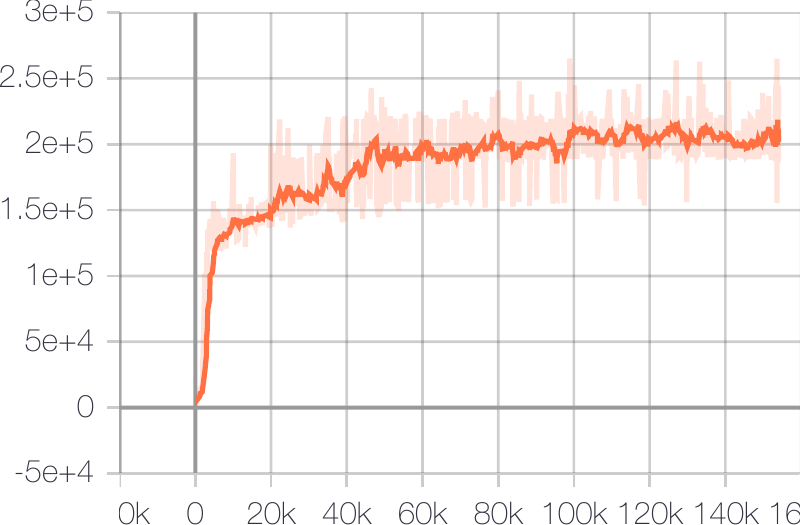}
    }
    \subfigure[defender]{
    \includegraphics[width=0.3\textwidth]{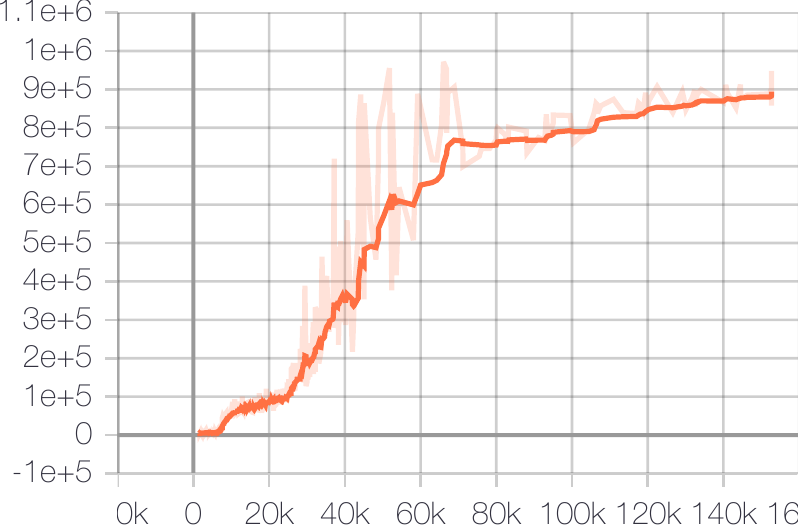}
    }
    \subfigure[demon\_attack]{
    \includegraphics[width=0.3\textwidth]{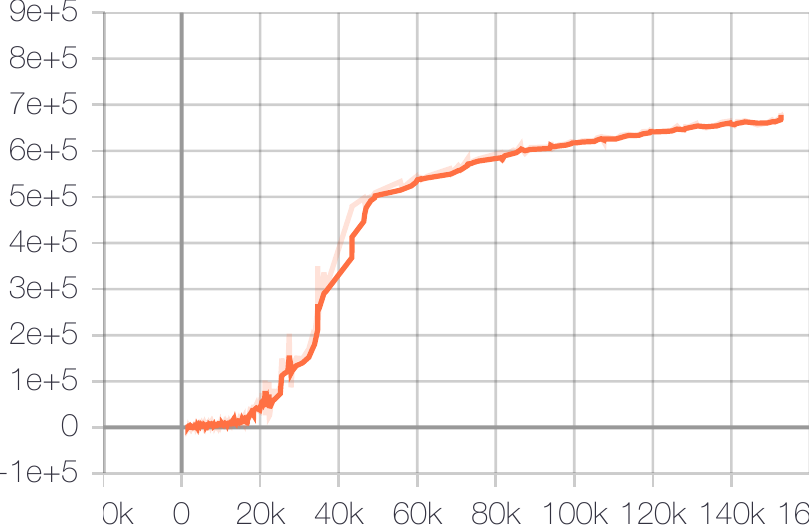}
    }
\end{figure}

\begin{figure}[!ht]
    \subfigure[double\_dunk]{
    \includegraphics[width=0.3\textwidth]{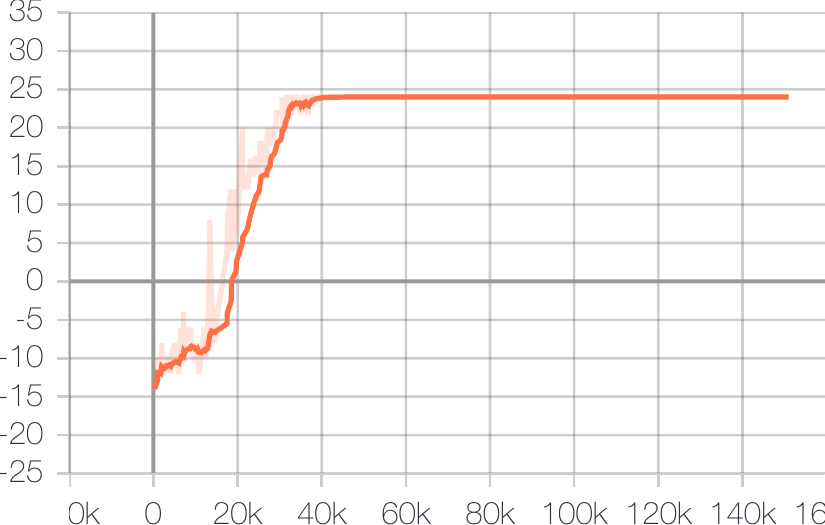}
    }
    \subfigure[enduro]{
     \includegraphics[width=0.3\textwidth]{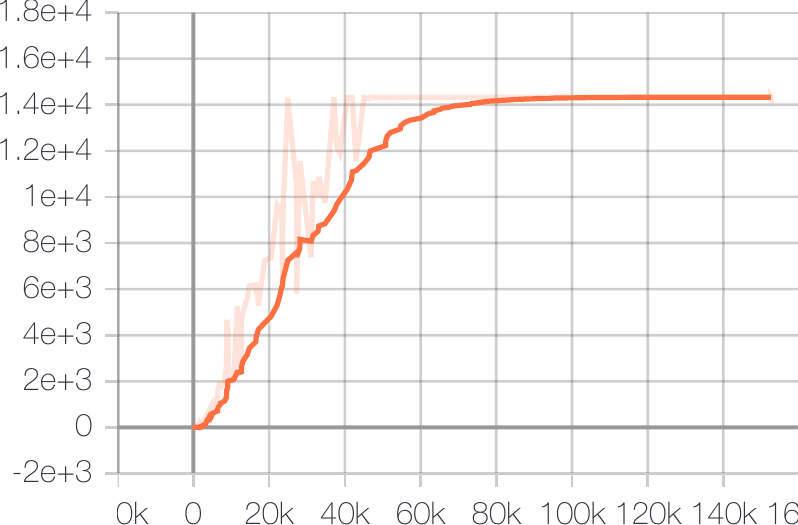}
    }
    \subfigure[fishing\_derby]{
    \includegraphics[width=0.3\textwidth]{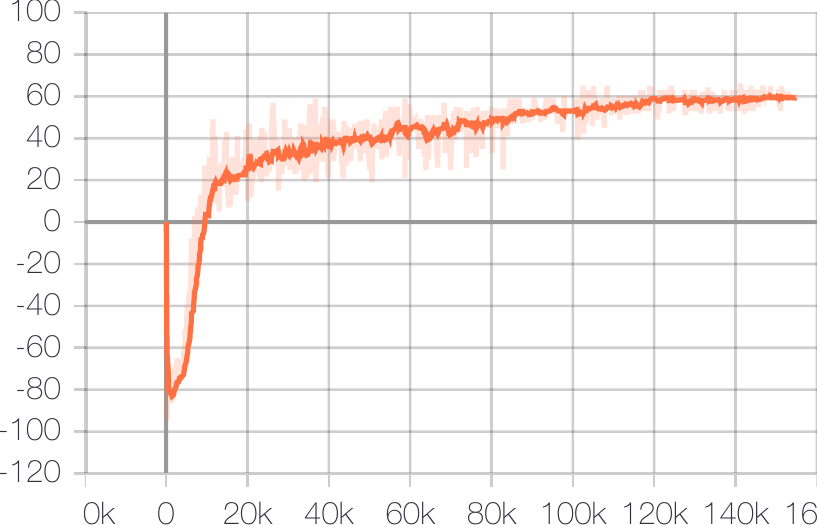}
    }
\end{figure}

\begin{figure}[!ht]
    \subfigure[freeway]{
    \includegraphics[width=0.3\textwidth]{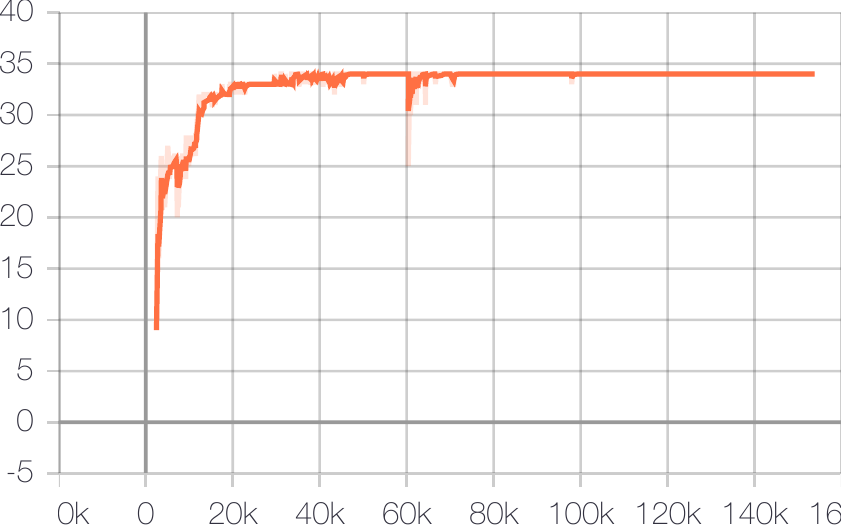}
    }
    \subfigure[frostbite]{
    \includegraphics[width=0.3\textwidth]{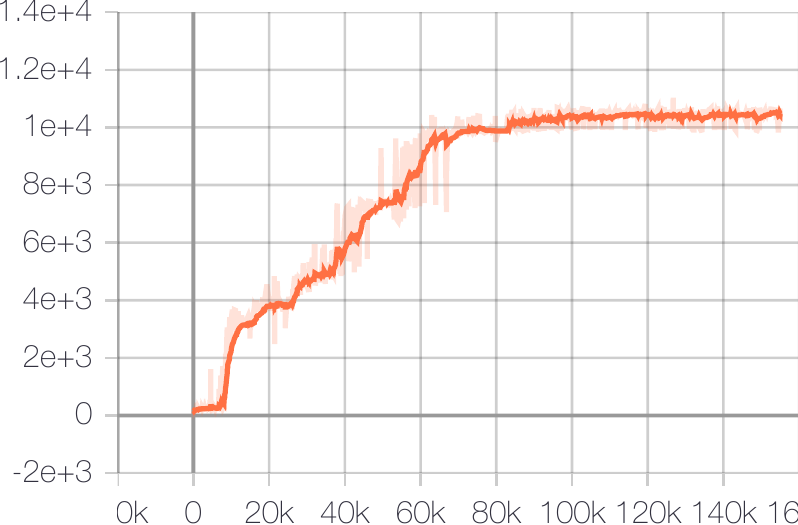}
    }
    \subfigure[gopher]{
    \includegraphics[width=0.3\textwidth]{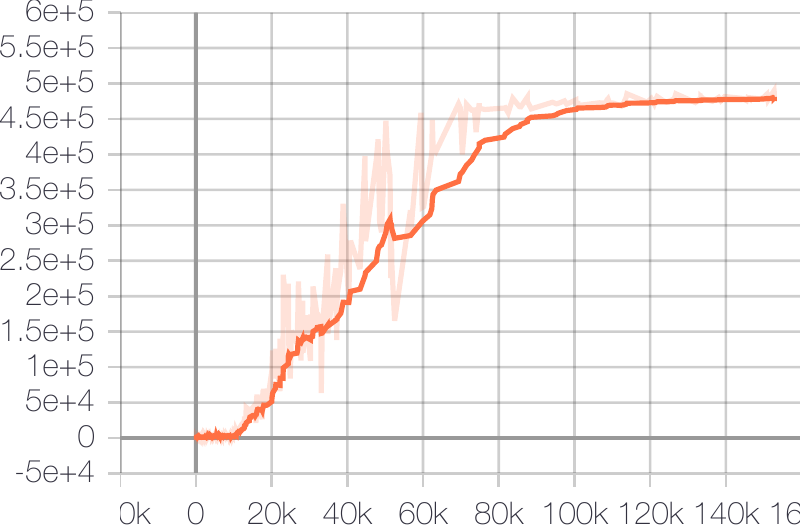}
    }
\end{figure}

\begin{figure}[!ht]
    \subfigure[gravitar]{
    \includegraphics[width=0.3\textwidth]{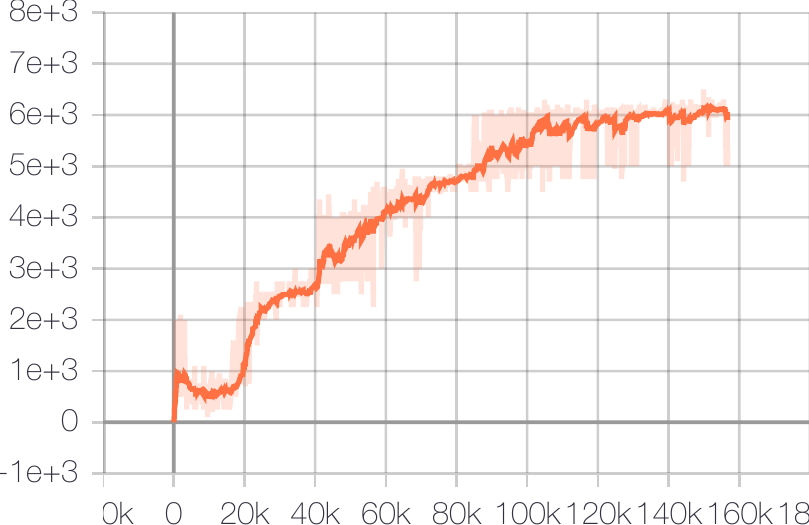}
    }
    \subfigure[hero]{
    \includegraphics[width=0.3\textwidth]{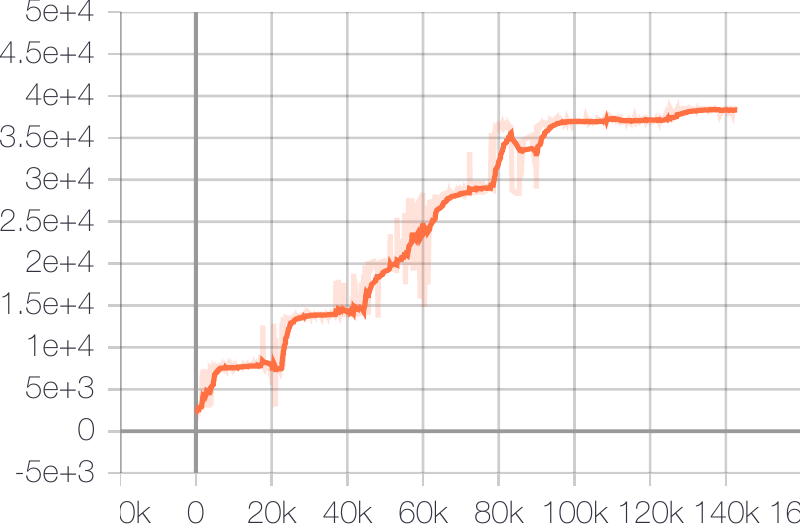}
    }
    \subfigure[ice\_hockey]{
    \includegraphics[width=0.3\textwidth]{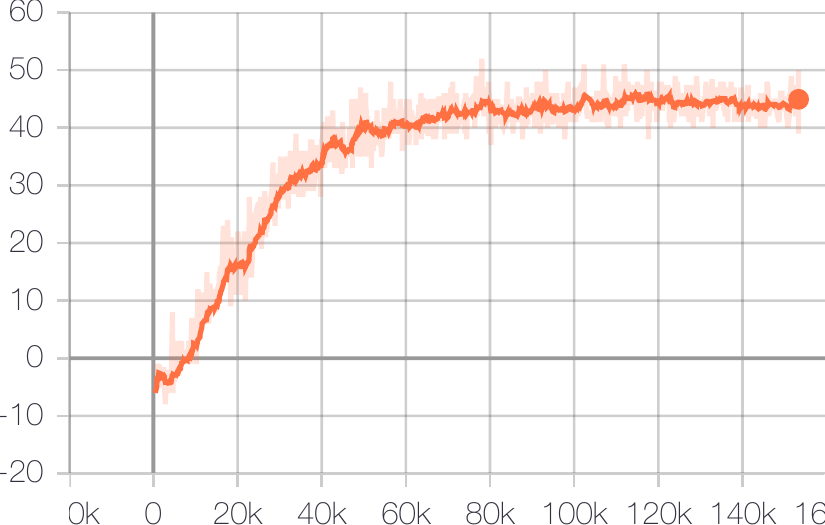}
    }
\end{figure}

\begin{figure}[!ht]
    \subfigure[jamesbond]{
    \includegraphics[width=0.3\textwidth]{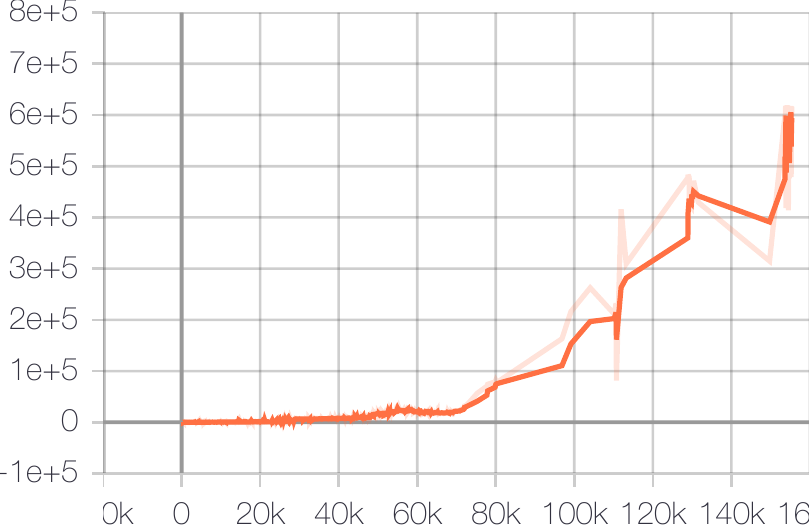}
    }
    \subfigure[kangaroo]{
    \includegraphics[width=0.3\textwidth]{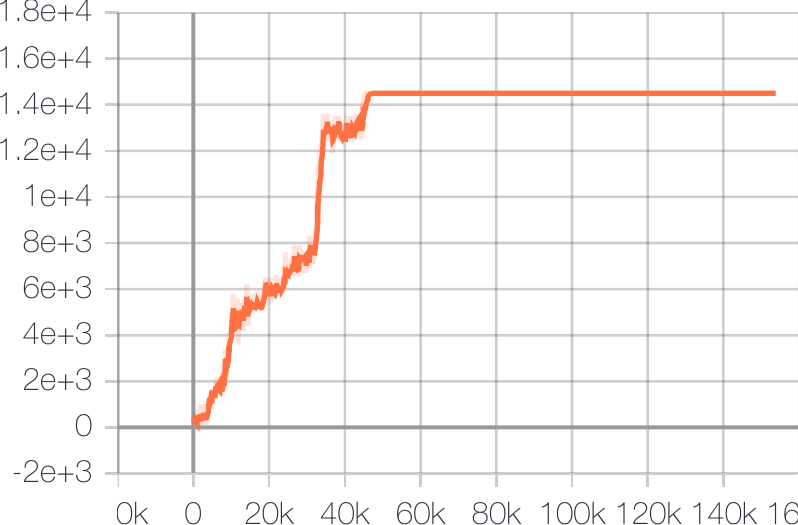}
    }
    \subfigure[krull]{
    \includegraphics[width=0.3\textwidth]{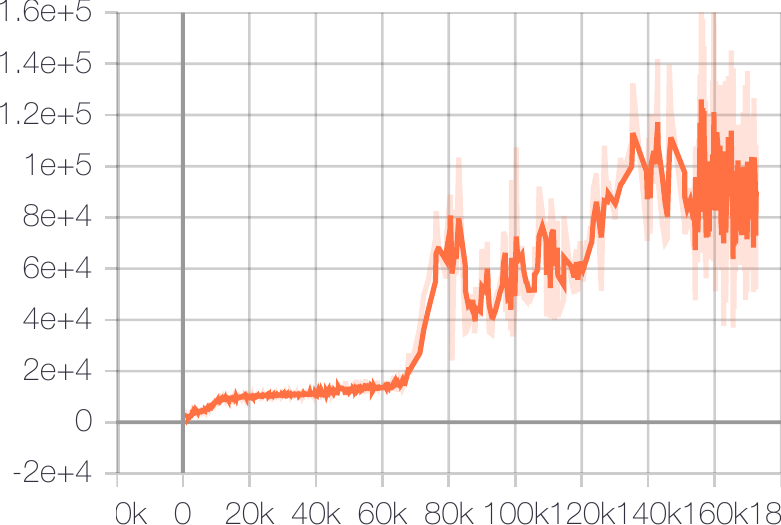}
    }
\end{figure}

\begin{figure}[!ht]
    \subfigure[kung\_fu\_master]{
    \includegraphics[width=0.3\textwidth]{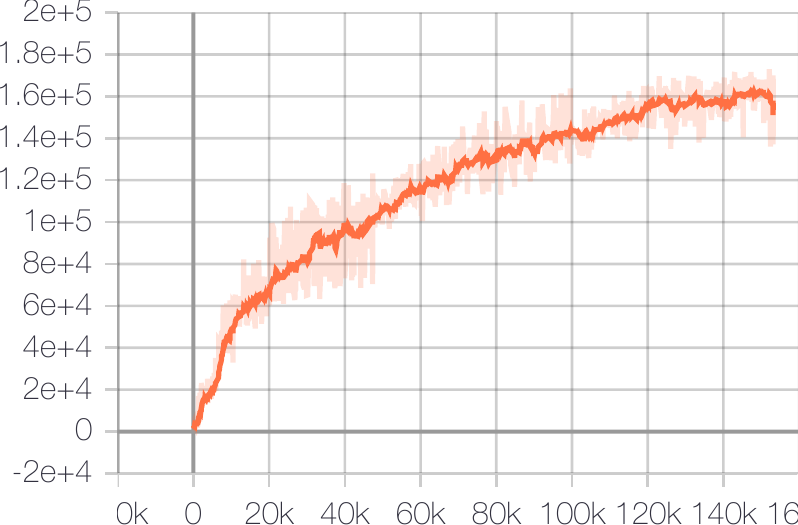}
    }
    \subfigure[montezuma\_revenge]{
    \includegraphics[width=0.3\textwidth]{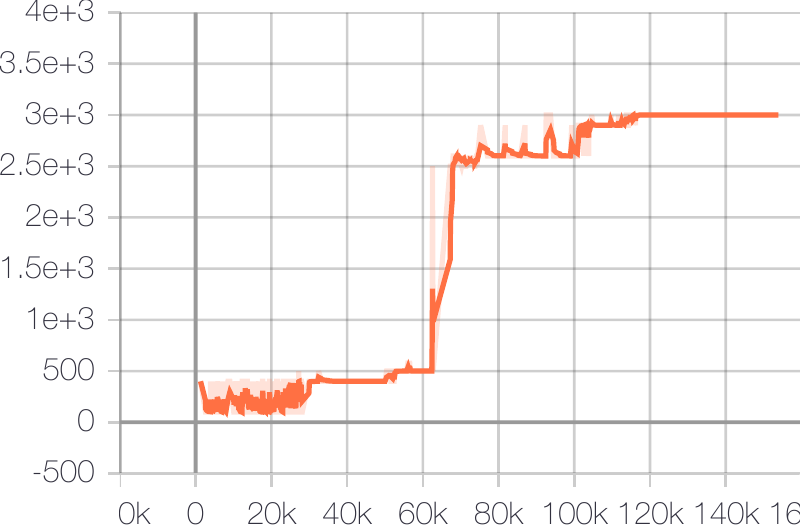}
    }
    \subfigure[ms\_pacman]{
    \includegraphics[width=0.3\textwidth]{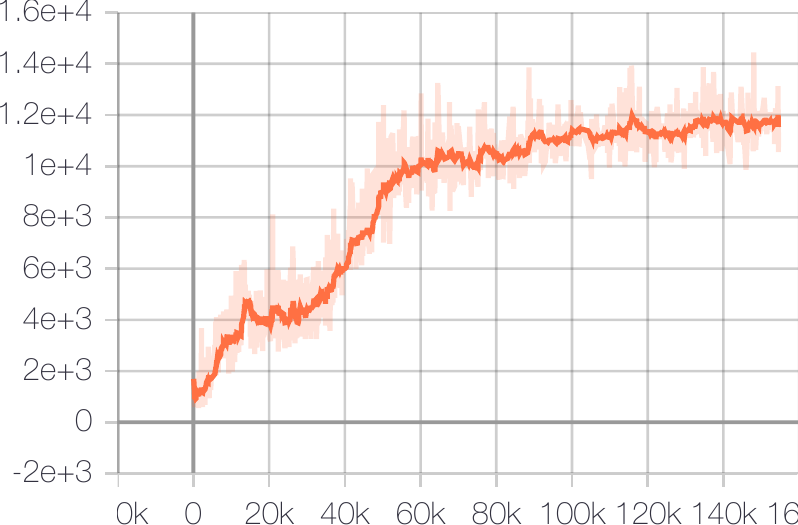}
    }
\end{figure}

\begin{figure}[!ht]
    \subfigure[name\_this\_game]{
    \includegraphics[width=0.3\textwidth]{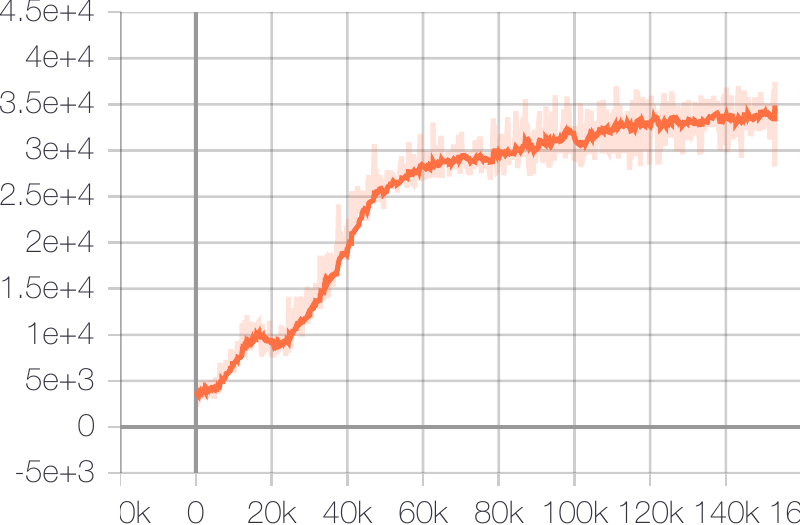}
    }
    \subfigure[phoenix]{
    \includegraphics[width=0.3\textwidth]{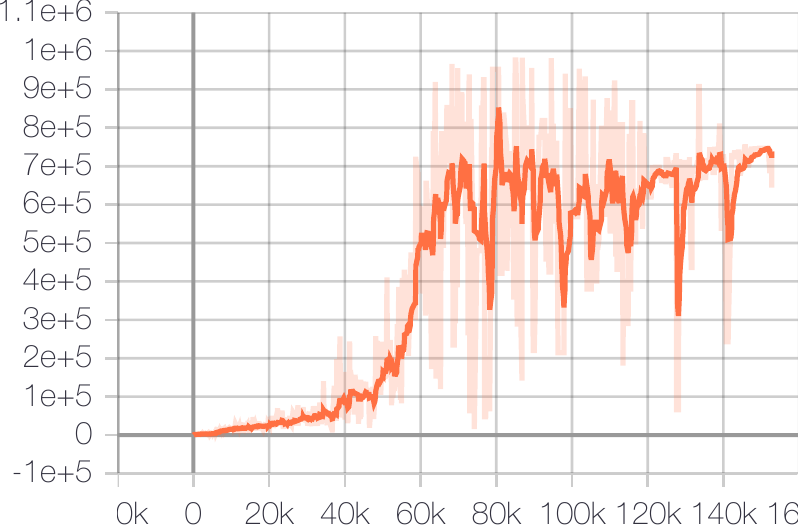}
    }
    \subfigure[pitfall]{
    \includegraphics[width=0.3\textwidth]{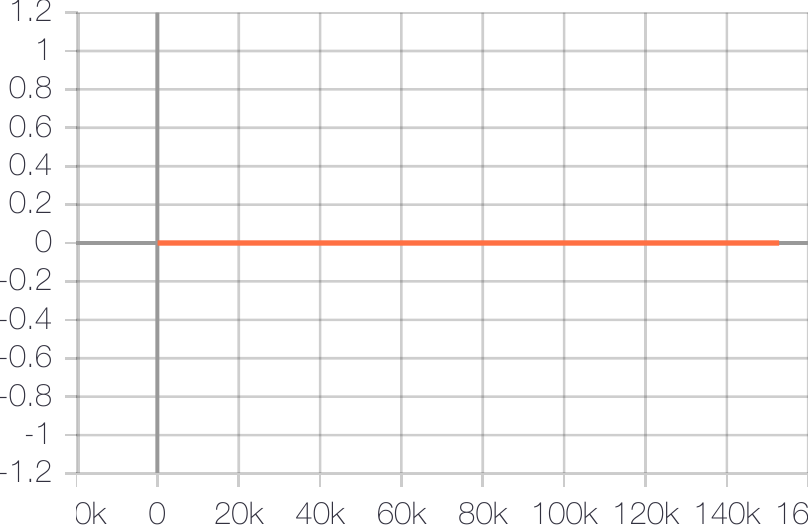}
    }
\end{figure}

\begin{figure}[!ht]
    \subfigure[pong]{
    \includegraphics[width=0.3\textwidth]{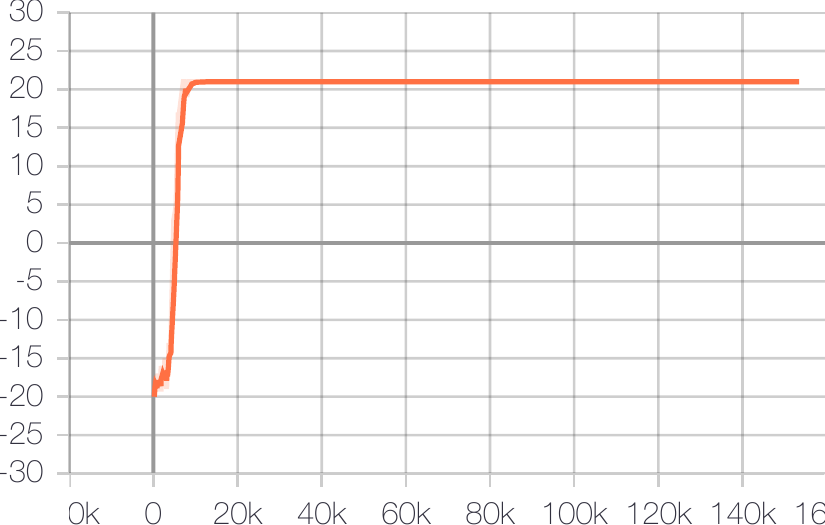}
    }
    \subfigure[private\_eye]{
    \includegraphics[width=0.3\textwidth]{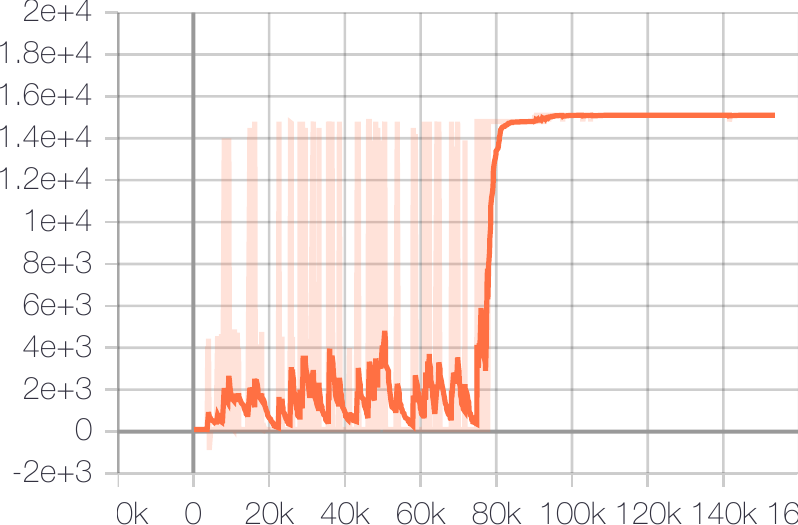}
    }
    \subfigure[qbert]{
    \includegraphics[width=0.3\textwidth]{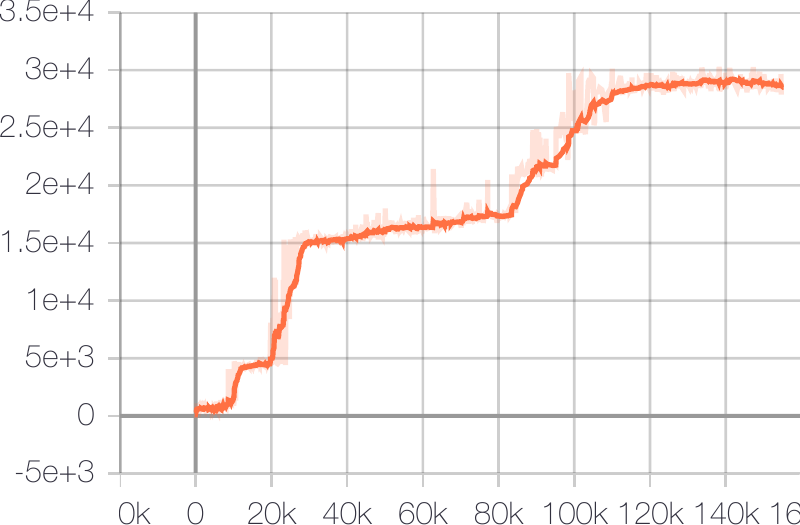}
    }
\end{figure}

\begin{figure}[!ht]
    \subfigure[riverraid]{
    \includegraphics[width=0.3\textwidth]{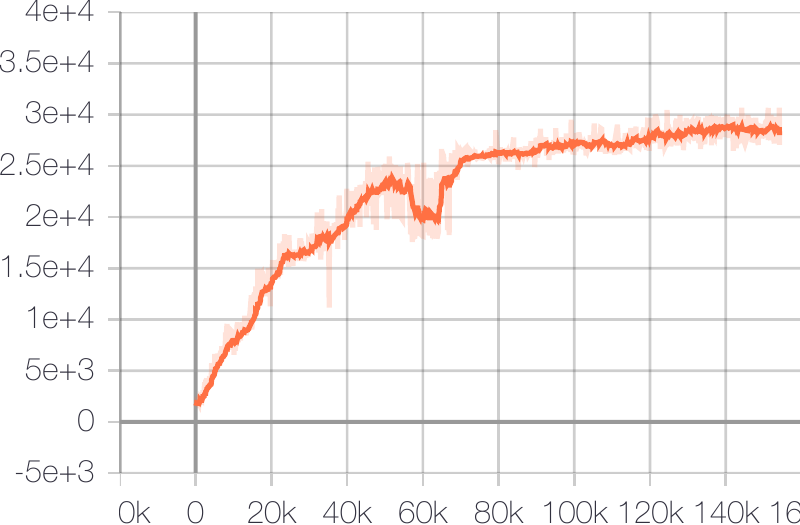}
    }
    \subfigure[road\_runner]{
    \includegraphics[width=0.3\textwidth]{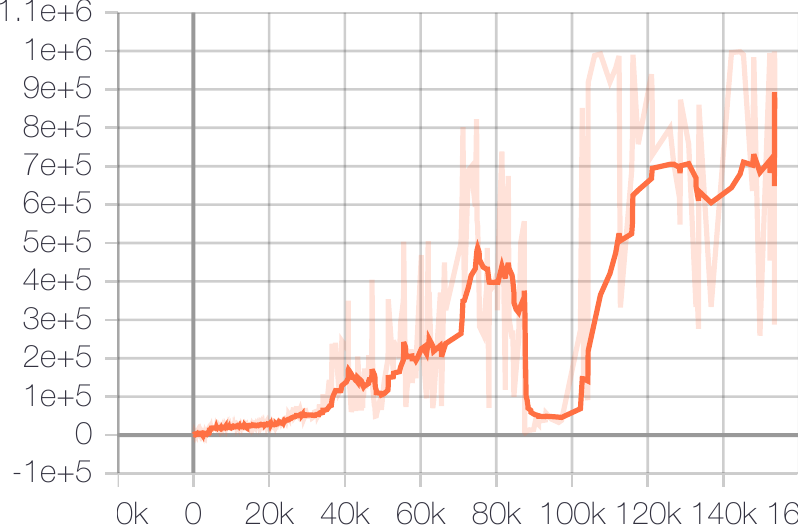}
    }
    \subfigure[robotank]{
    \includegraphics[width=0.3\textwidth]{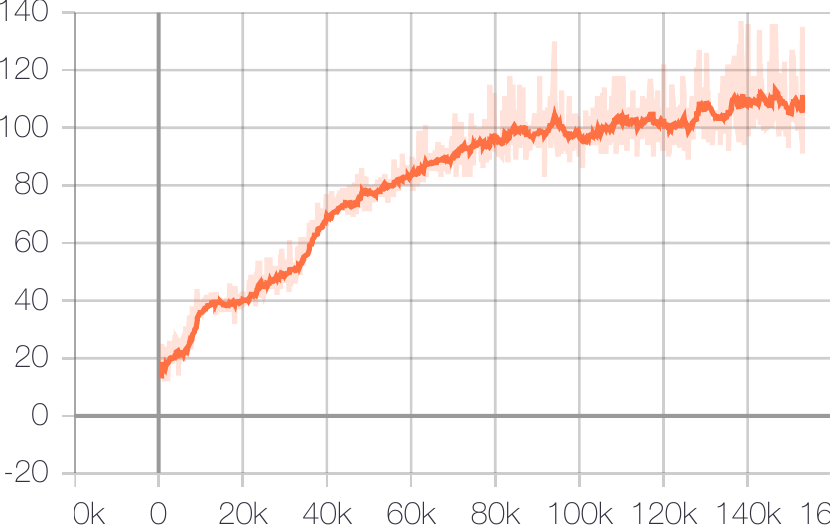}
    }
\end{figure}

\begin{figure}[!ht]
    \subfigure[seaquest]{
    \includegraphics[width=0.3\textwidth]{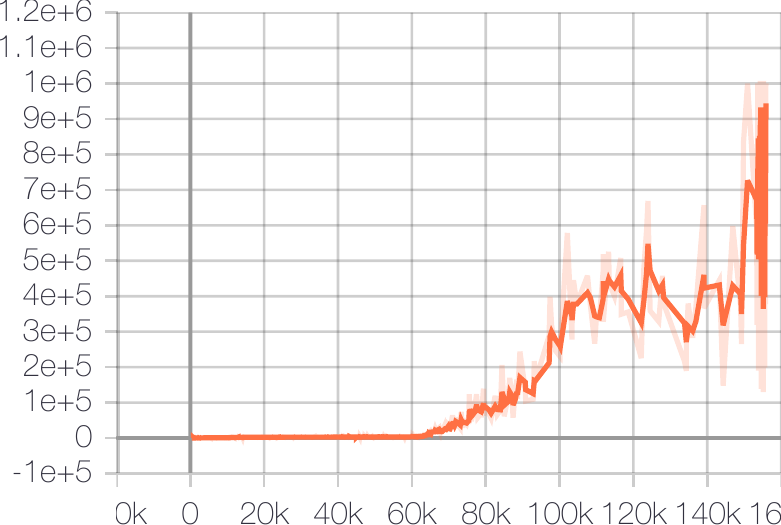}
    }
    \subfigure[skiing]{
    \includegraphics[width=0.3\textwidth]{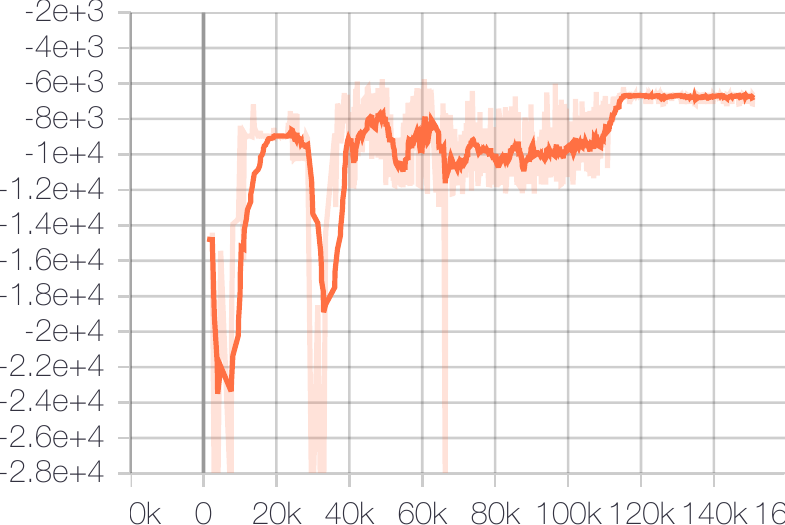}
    }
    \subfigure[solaris]{
    \includegraphics[width=0.3\textwidth]{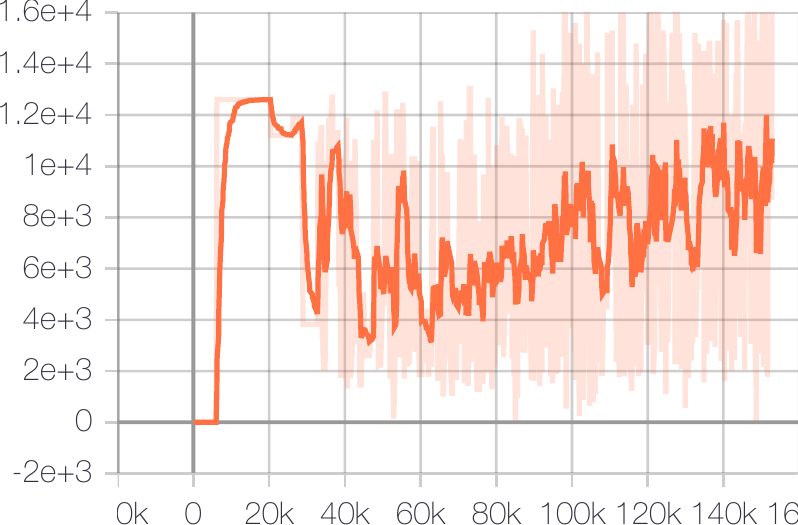}
    }
\end{figure}

\begin{figure}[!ht]
    \subfigure[space\_invaders]{
    \includegraphics[width=0.3\textwidth]{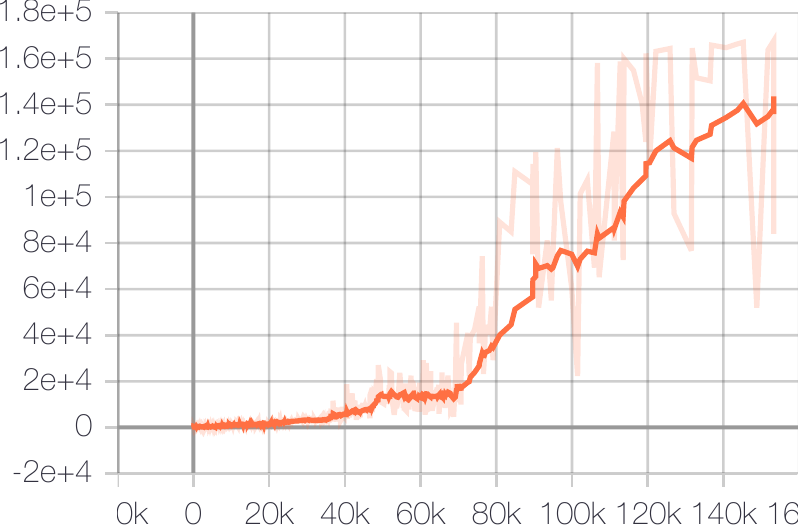}
    }
    \subfigure[star\_gunner]{
    \includegraphics[width=0.3\textwidth]{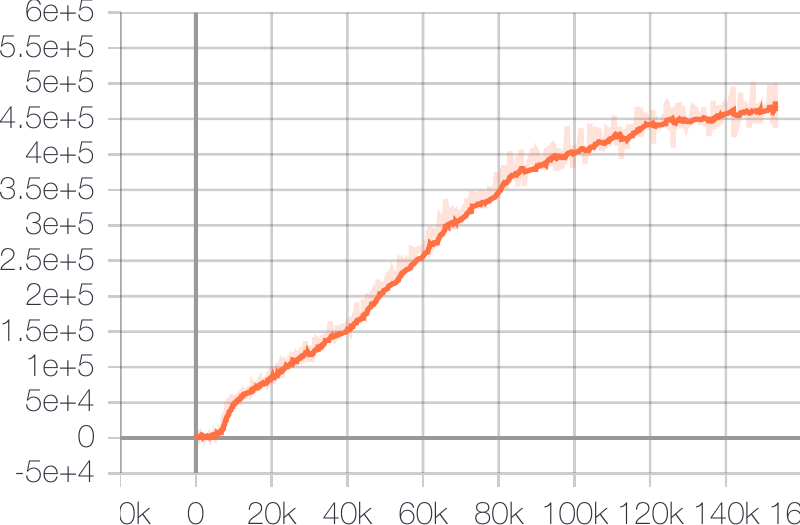}
    }
    \subfigure[surround]{
    \includegraphics[width=0.3\textwidth]{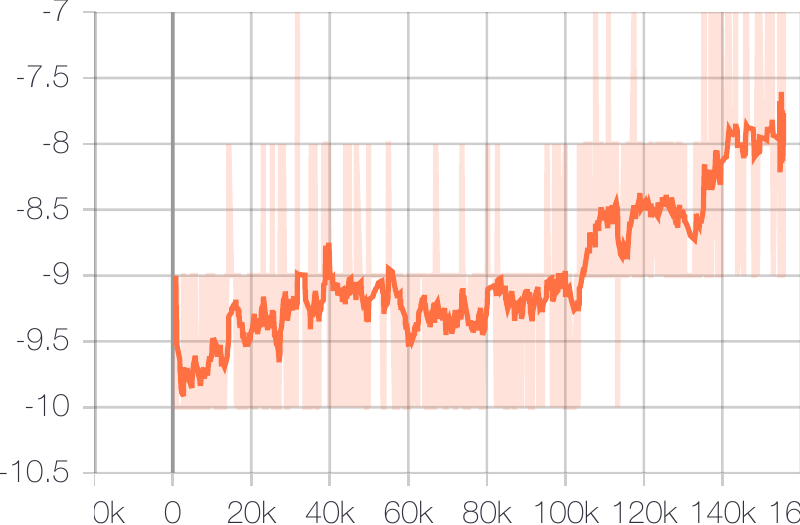}
    }
\end{figure}

\begin{figure}[!ht]
    \subfigure[tennis]{
    \includegraphics[width=0.3\textwidth]{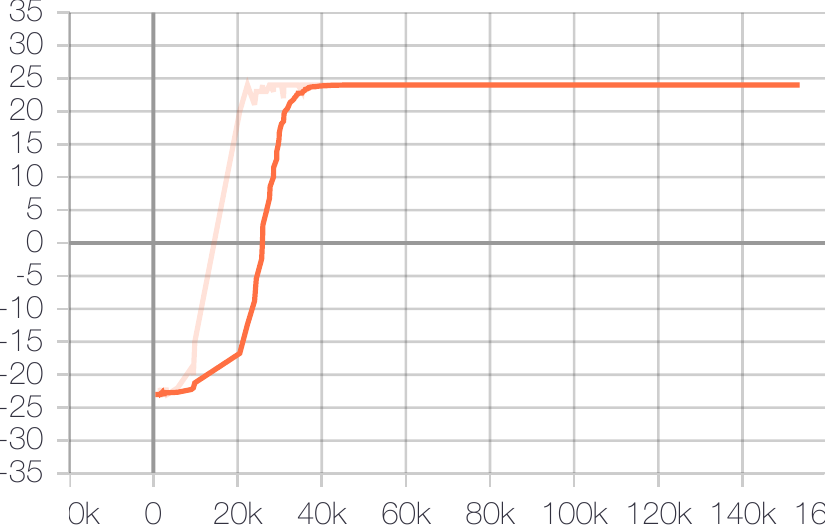}
    }
    \subfigure[time\_pilot]{
    \includegraphics[width=0.3\textwidth]{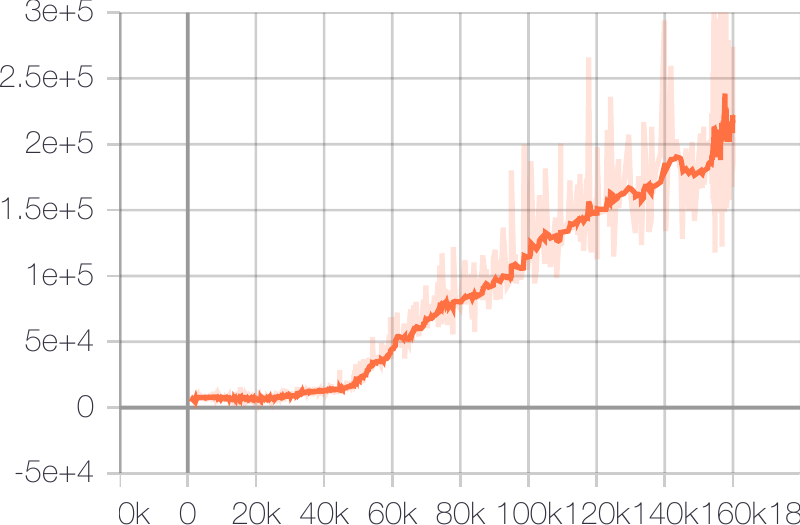}
    }
    \subfigure[tutankham]{
    \includegraphics[width=0.3\textwidth]{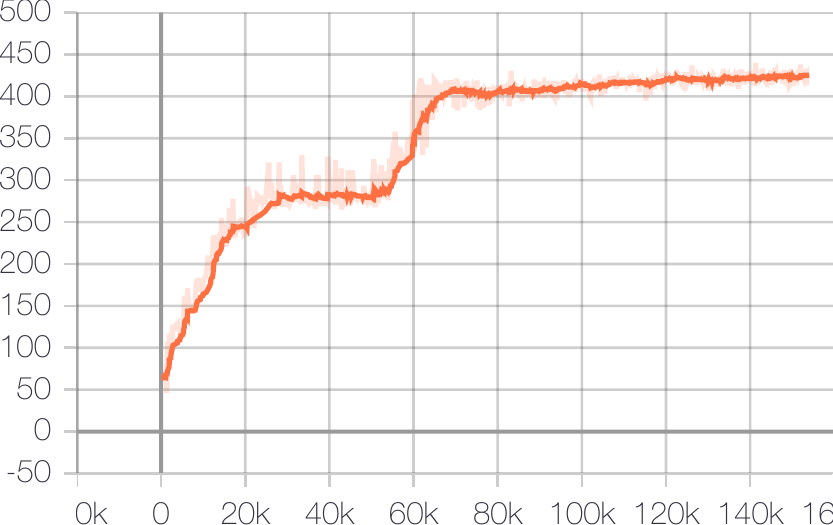}
    }
\end{figure}

\clearpage

\begin{figure}[!ht]
    \subfigure[up\_n\_down]{
    \includegraphics[width=0.3\textwidth]{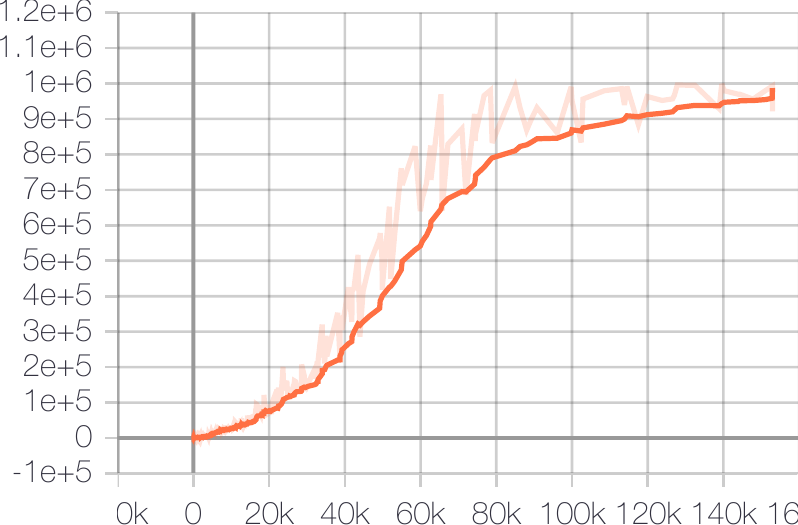}
    }
    \subfigure[venture]{
    \includegraphics[width=0.3\textwidth]{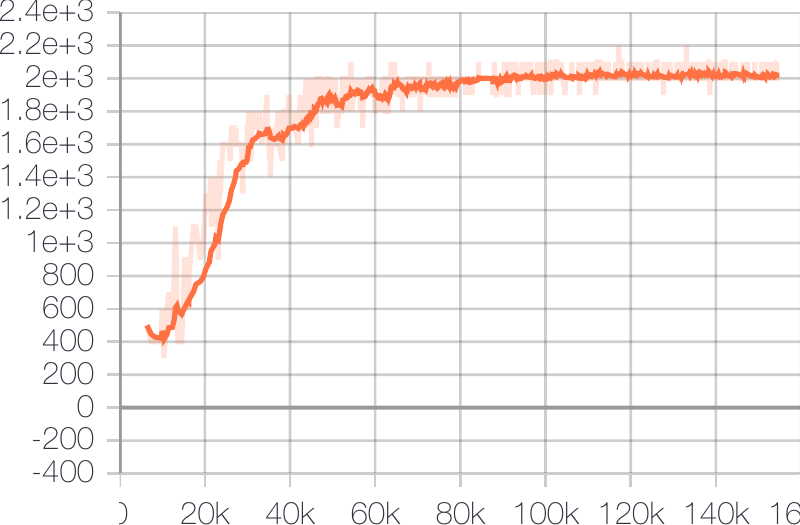}
    }
    \subfigure[video\_pinball]{
    \includegraphics[width=0.3\textwidth]{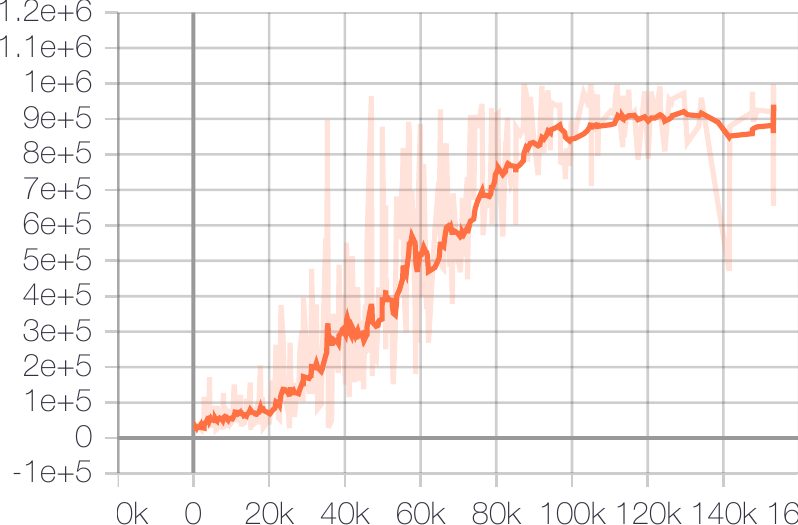}
    }
\end{figure}

\begin{figure}[!ht]
    \subfigure[wizard\_of\_wor]{
    \includegraphics[width=0.3\textwidth]{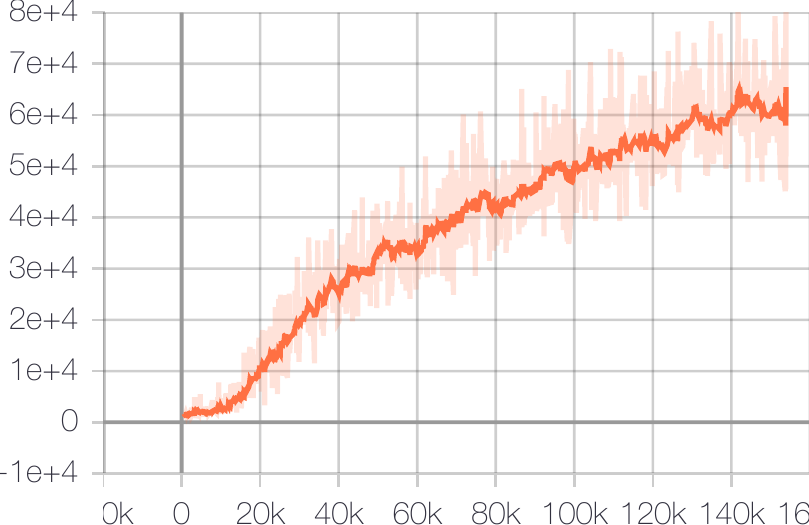}
    }
    \subfigure[yars\_revenge]{
    \includegraphics[width=0.3\textwidth]{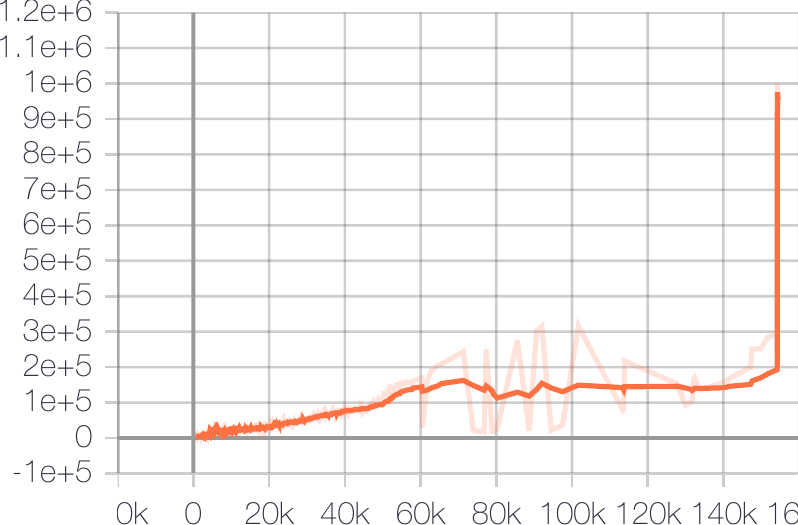}
    }
    \subfigure[zaxxon]{
    \includegraphics[width=0.3\textwidth]{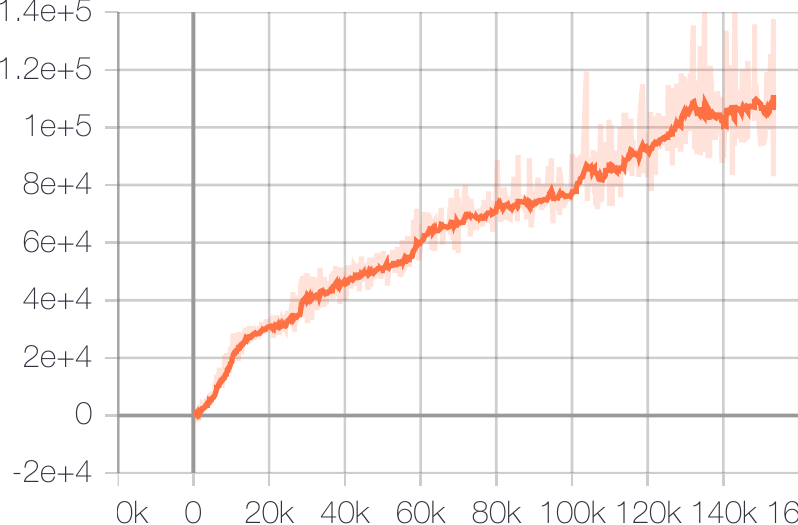}
    }
\end{figure}

\clearpage

\subsubsection{Atari Games Learning Curves of GDI-H$^3$}

\setcounter{subfigure}{0}

\begin{figure}[!ht] 
    \subfigure[alien]{
    \includegraphics[width=0.3\textwidth]{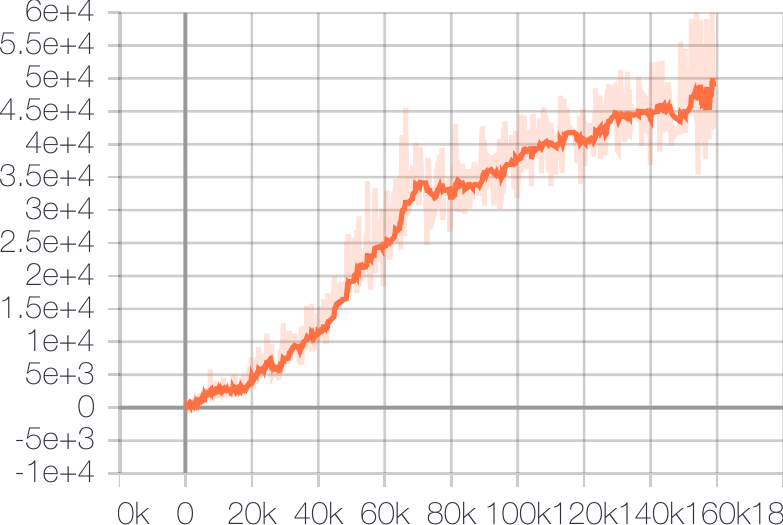}
    }
    \subfigure[amidar]{
    \includegraphics[width=0.3\textwidth]{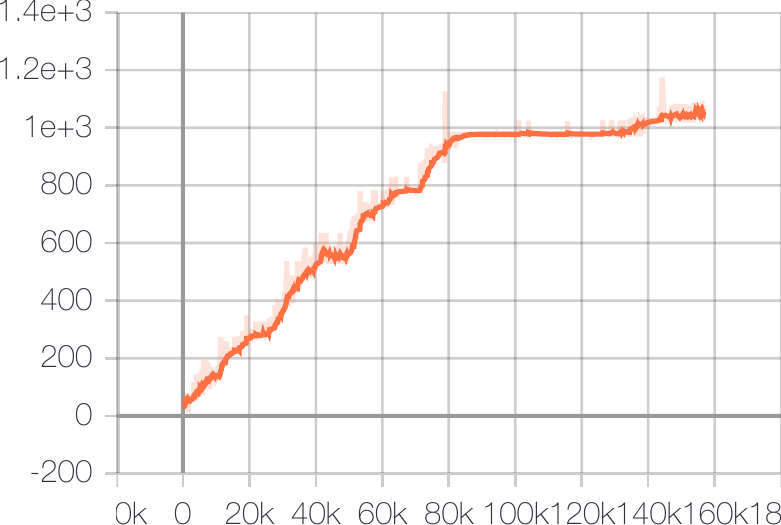}
    }
    \subfigure[assault]{
    \includegraphics[width=0.3\textwidth]{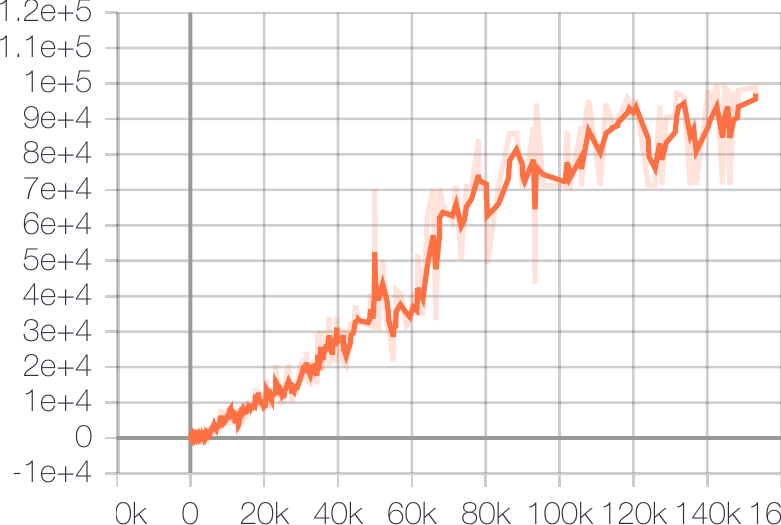}
    }
\end{figure}

\begin{figure}[!ht]
    \subfigure[asterix]{
    \includegraphics[width=0.3\textwidth]{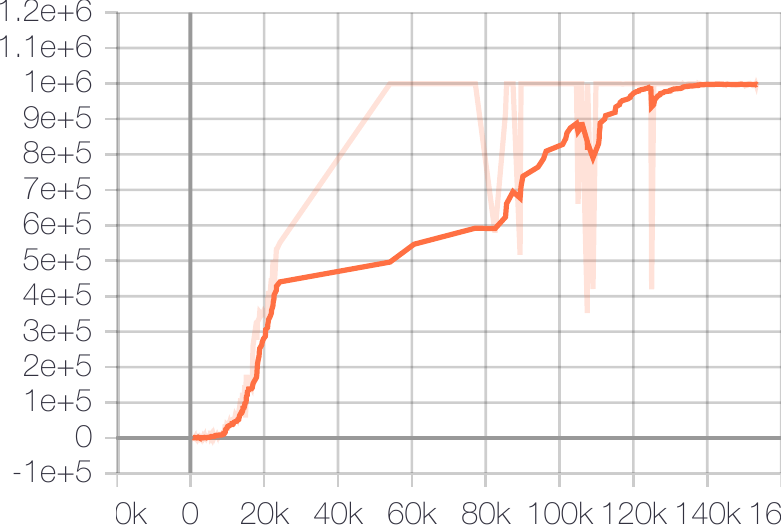}
    }
    \subfigure[asteroids]{
    \includegraphics[width=0.3\textwidth]{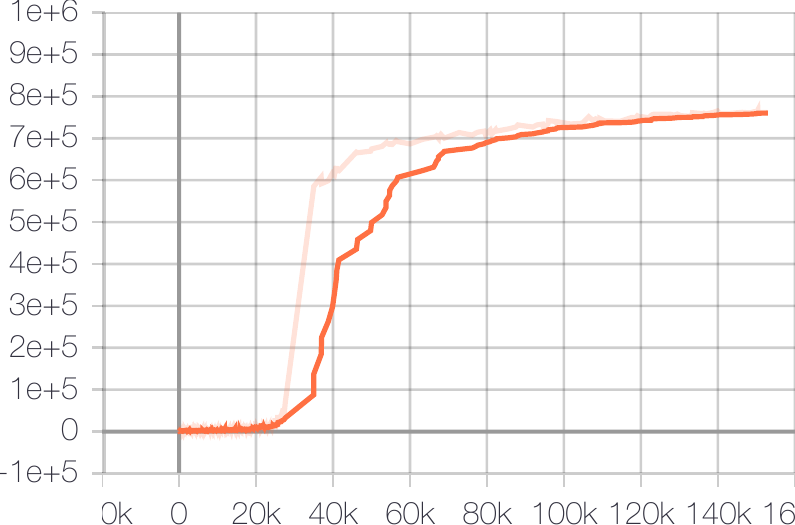}
    }
    \subfigure[atlantis]{
    \includegraphics[width=0.3\textwidth]{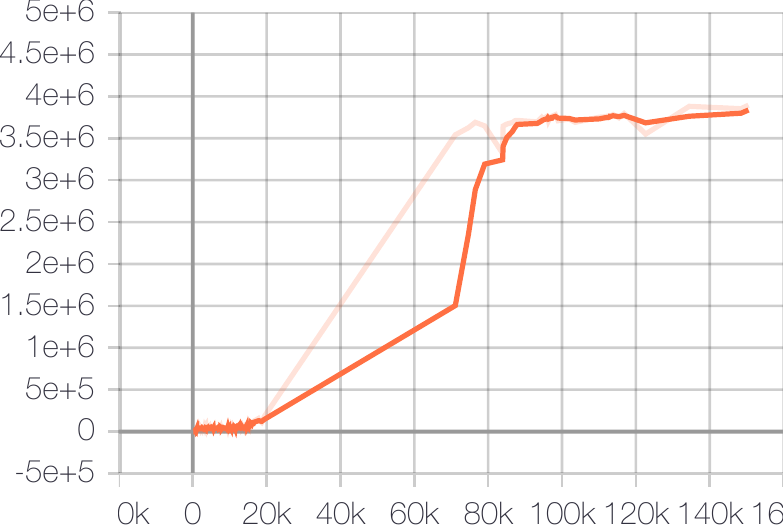}
    }
\end{figure}

\begin{figure}[!ht]
    \subfigure[bank\_heist]{
    \includegraphics[width=0.3\textwidth]{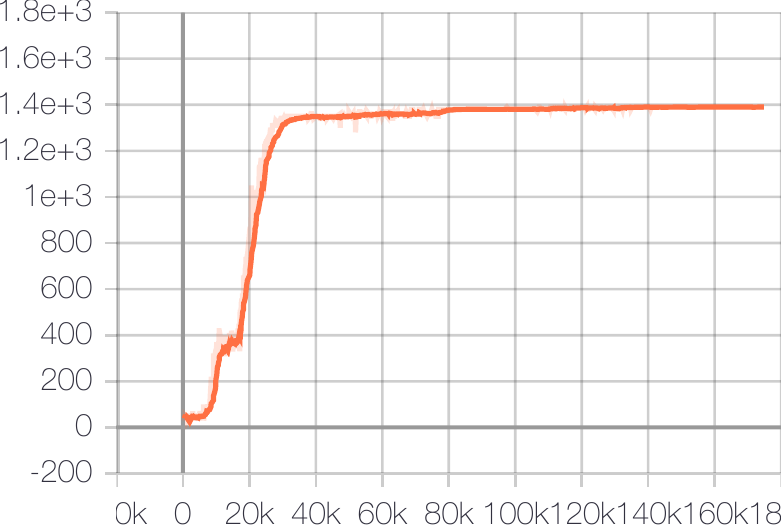}
    }
    \subfigure[battle\_zone]{
    \includegraphics[width=0.3\textwidth]{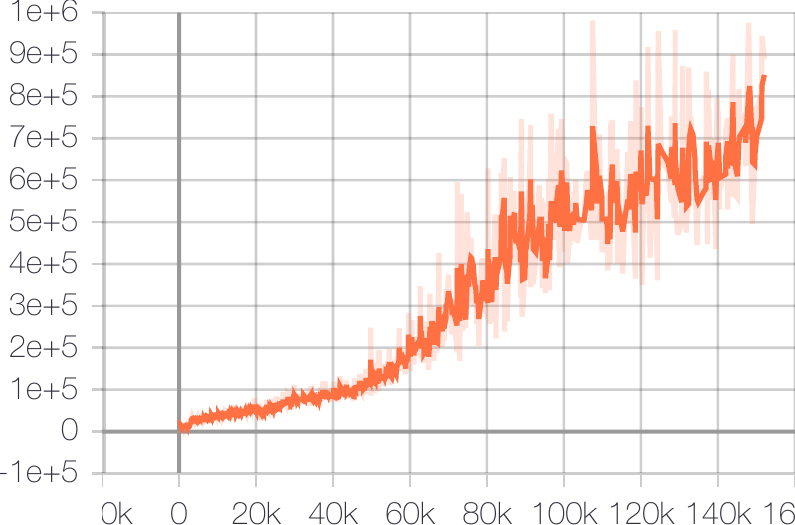}
    }
    \subfigure[beam\_rider]{
    \includegraphics[width=0.3\textwidth]{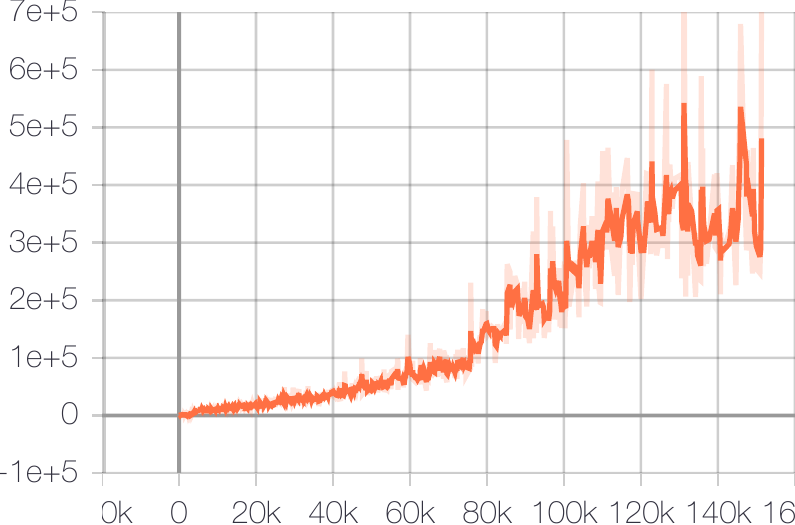}
    }
\end{figure}

\begin{figure}[!ht]
    \subfigure[berzerk]{
    \includegraphics[width=0.3\textwidth]{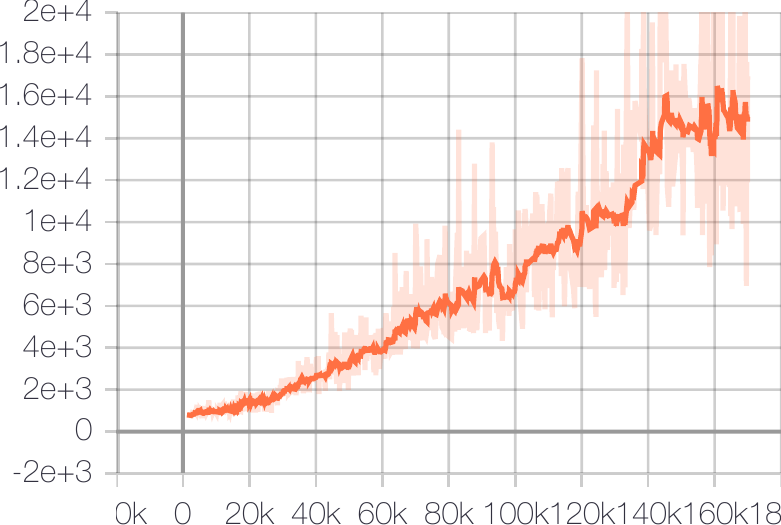}
    }
    \subfigure[bowling]{
    \includegraphics[width=0.3\textwidth]{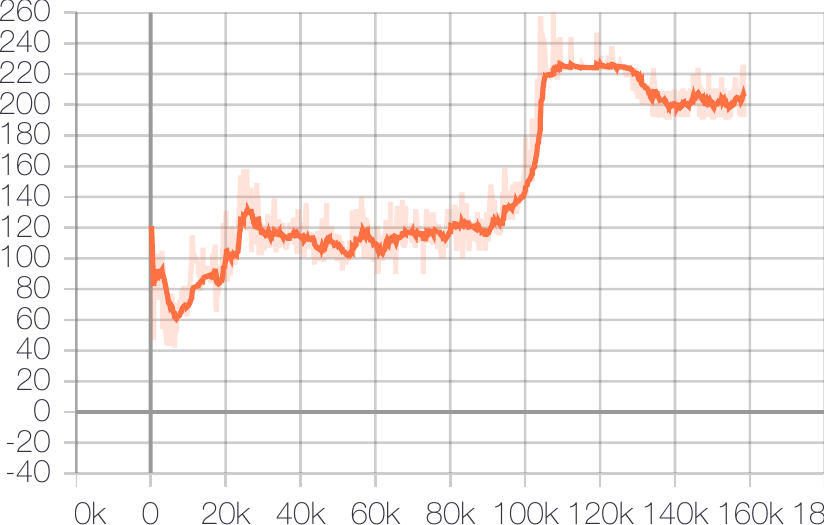}
    }
    \subfigure[boxing]{
    \includegraphics[width=0.3\textwidth]{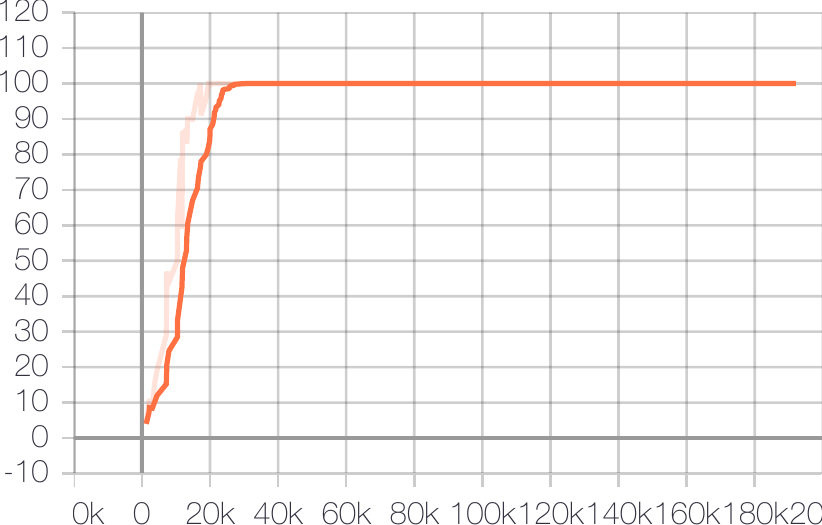}
    }
\end{figure}

\begin{figure}[!ht]
    \subfigure[breakout]{
    \includegraphics[width=0.3\textwidth]{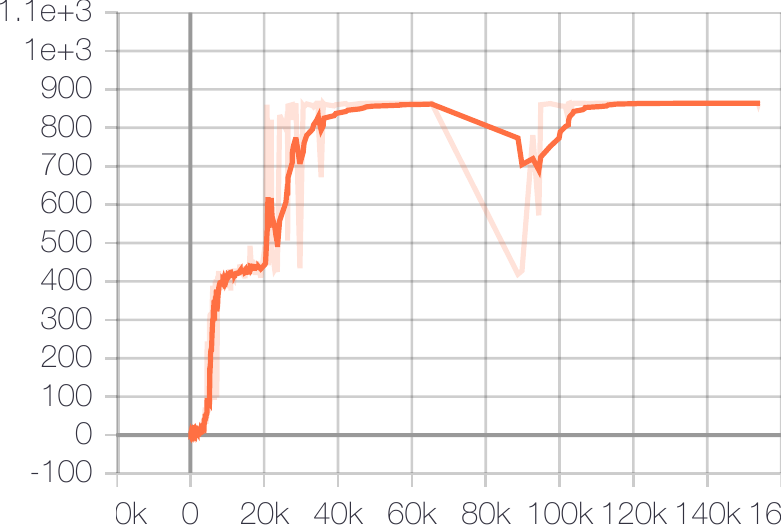}
    }
    \subfigure[centipede]{
    \includegraphics[width=0.3\textwidth]{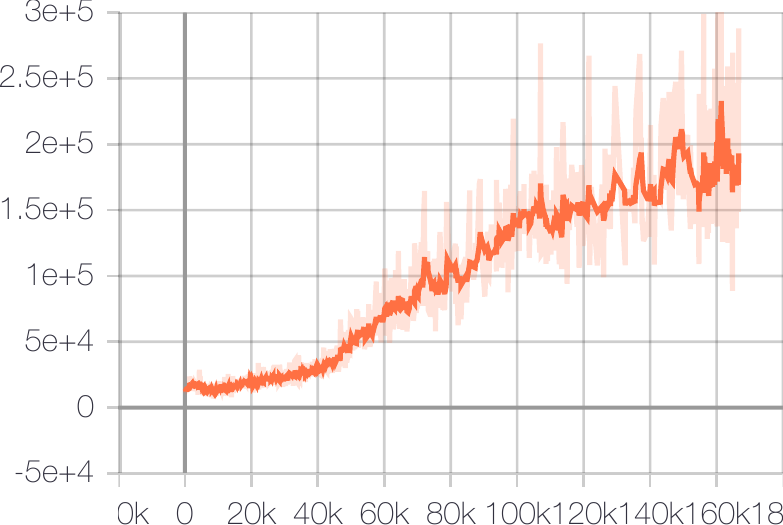}
    }
    \subfigure[chopper\_command]{
    \includegraphics[width=0.3\textwidth]{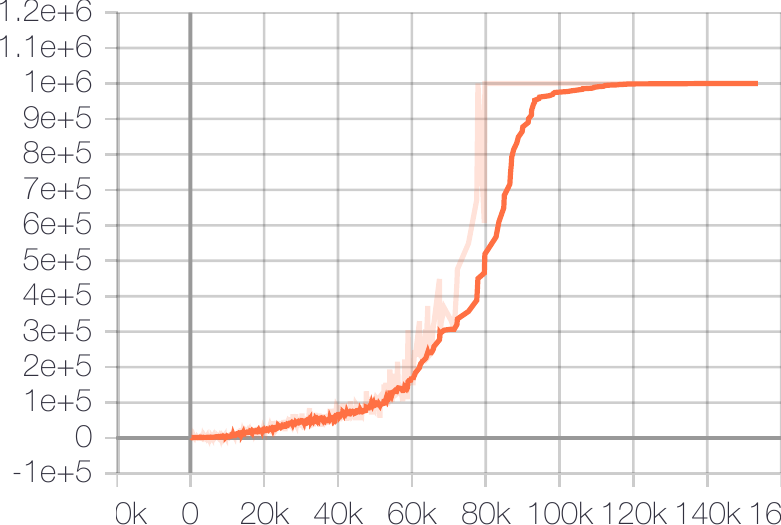}
    }
\end{figure}

\begin{figure}[!ht]
    \subfigure[crazy\_climber]{
    \includegraphics[width=0.3\textwidth]{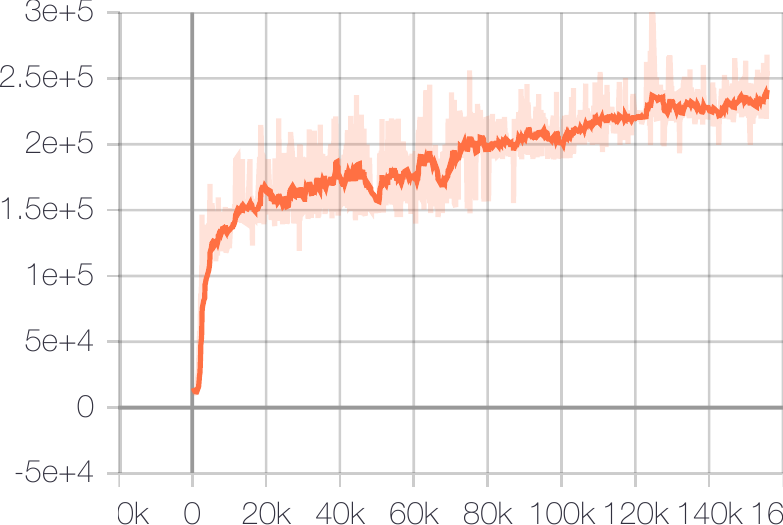}
    }
    \subfigure[defender]{
    \includegraphics[width=0.3\textwidth]{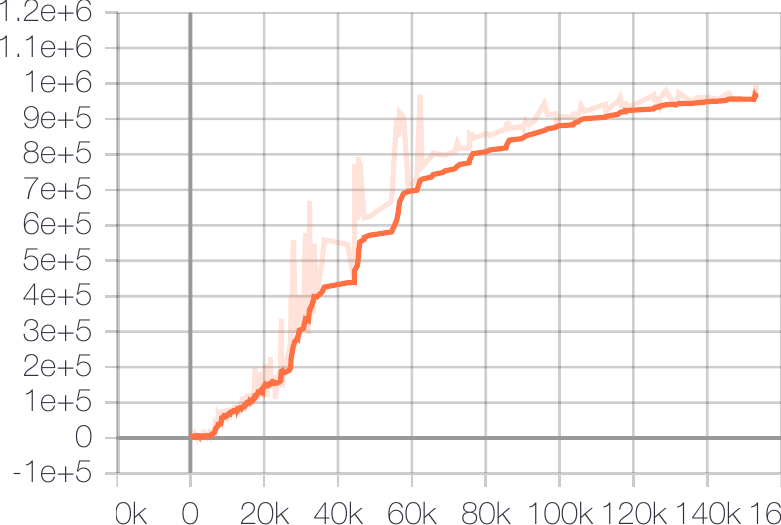}
    }
    \subfigure[demon\_attack]{
    \includegraphics[width=0.3\textwidth]{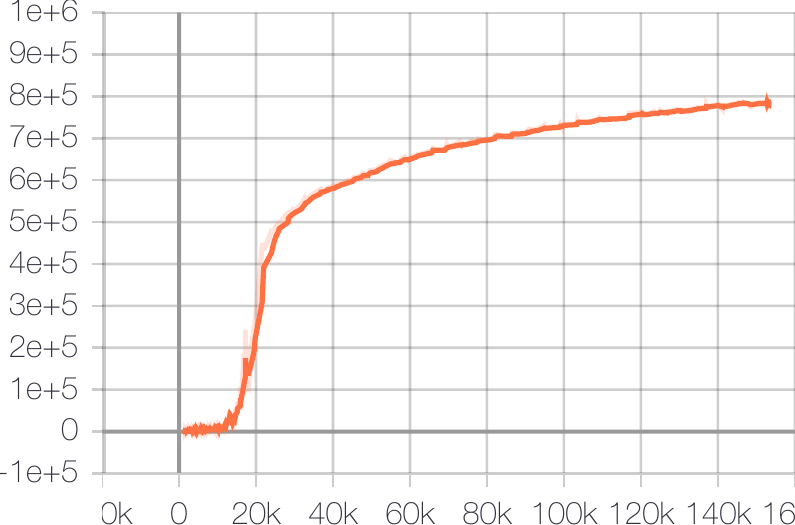}
    }
\end{figure}

\begin{figure}[!ht]
    \subfigure[double\_dunk]{
    \includegraphics[width=0.3\textwidth]{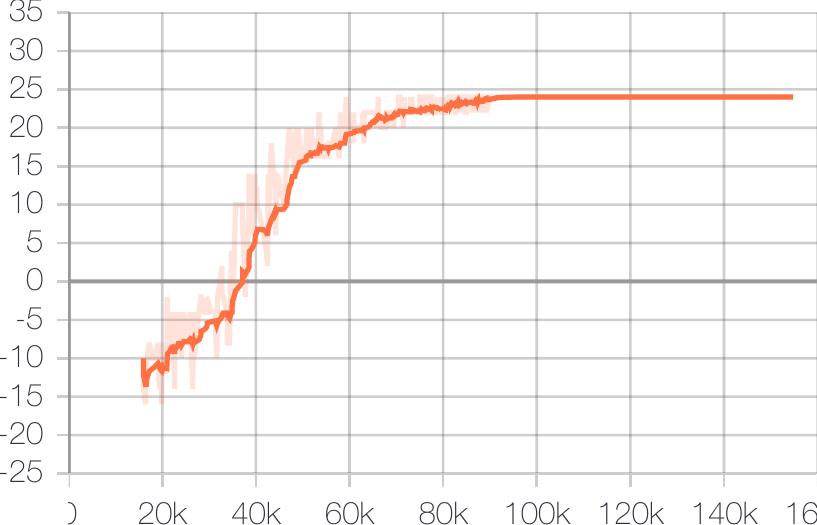}
    }
    \subfigure[enduro]{
    \includegraphics[width=0.3\textwidth]{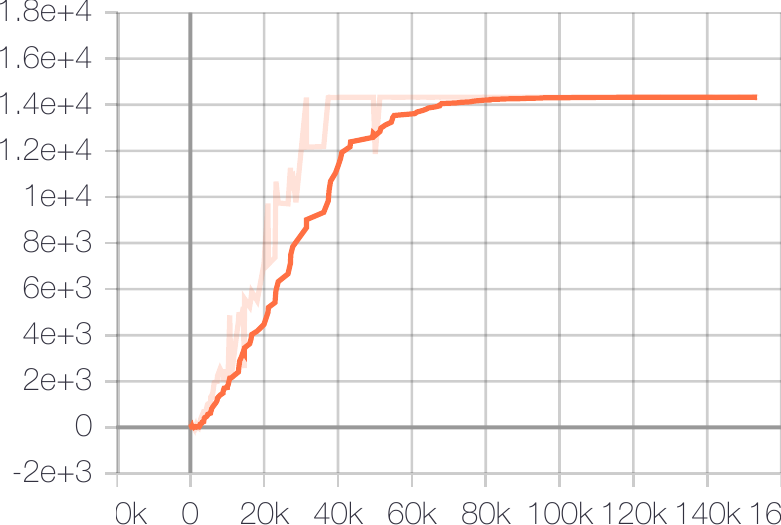}
    }
    \subfigure[fishing\_derby]{
    \includegraphics[width=0.3\textwidth]{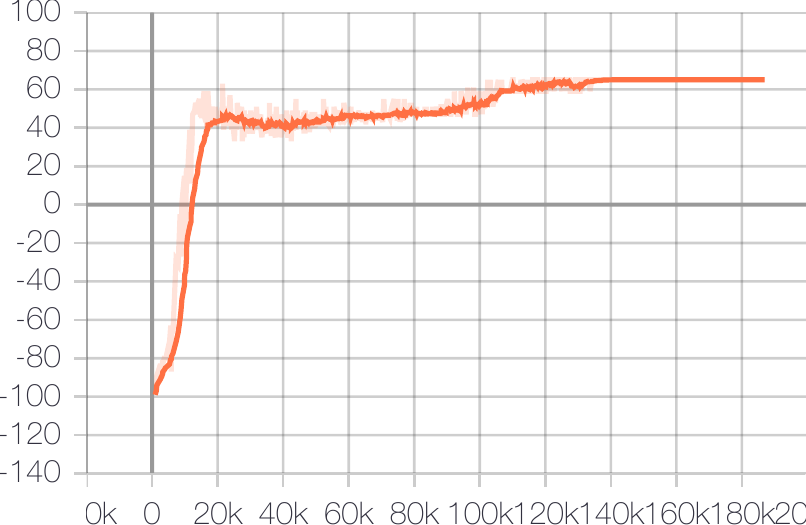}
    }
\end{figure}

\begin{figure}[!ht]
    \subfigure[freeway]{
    \includegraphics[width=0.3\textwidth]{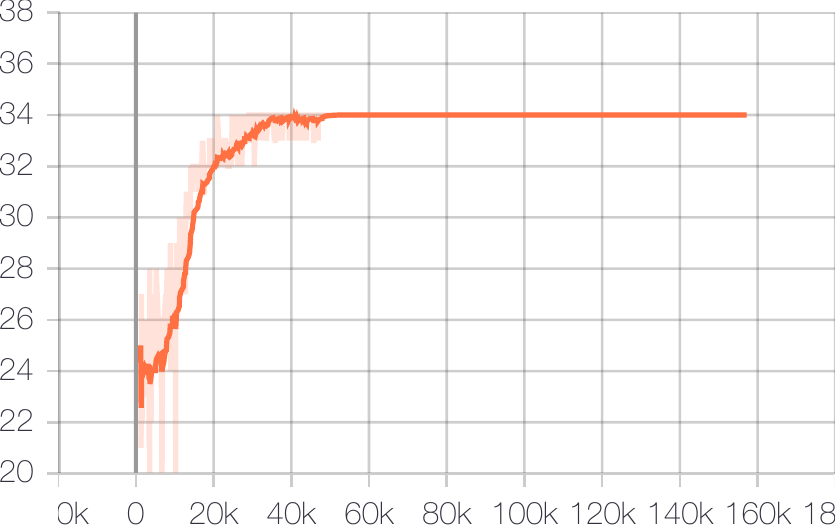}
    }
    \subfigure[frostbite]{
    \includegraphics[width=0.3\textwidth]{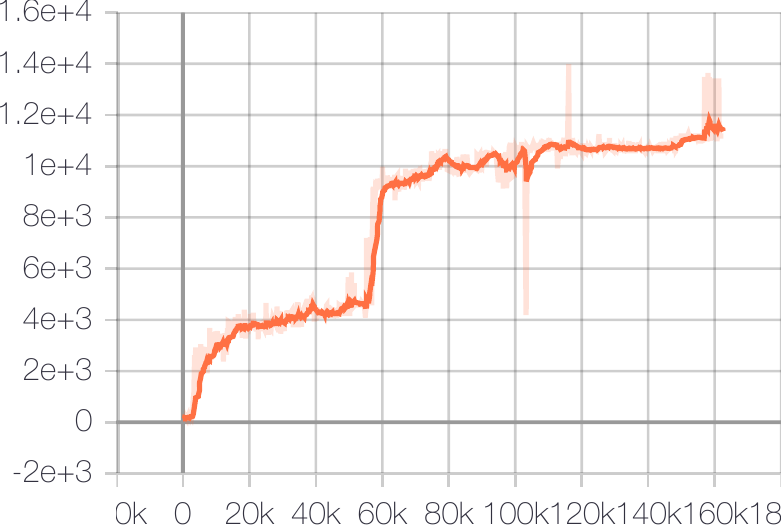}
    }
    \subfigure[gopher]{
    \includegraphics[width=0.3\textwidth]{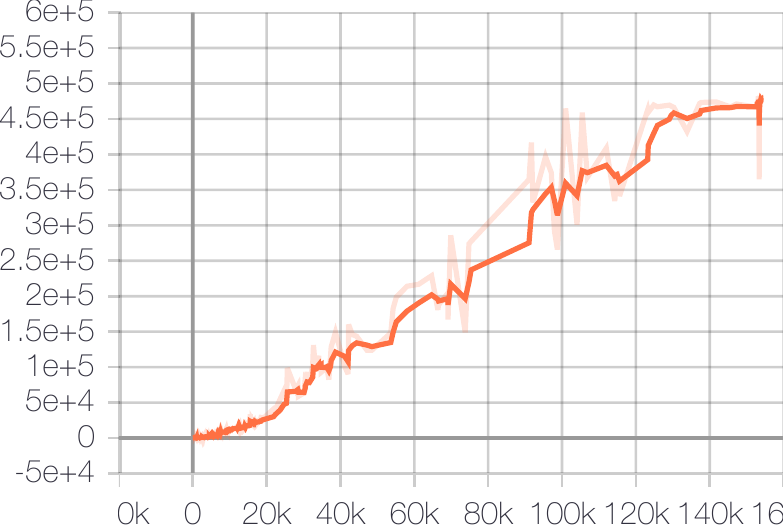}
    }
\end{figure}

\begin{figure}[!ht]
    \subfigure[gravitar]{
    \includegraphics[width=0.3\textwidth]{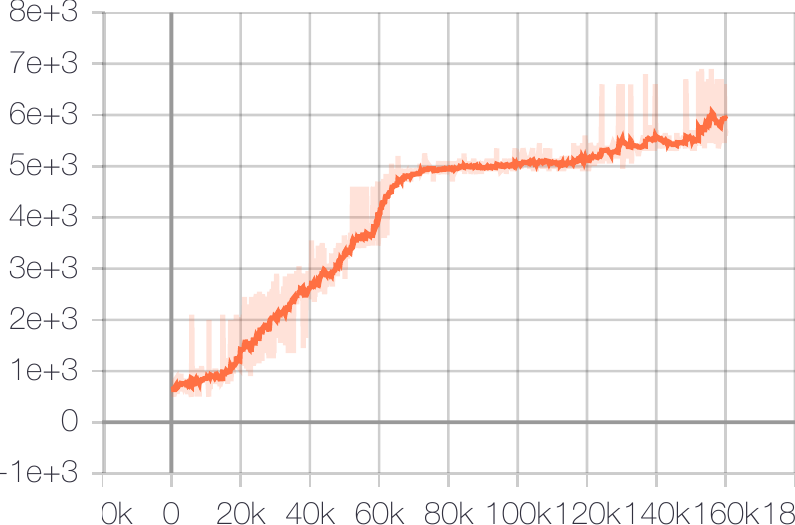}
    }
    \subfigure[hero]{
    \includegraphics[width=0.3\textwidth]{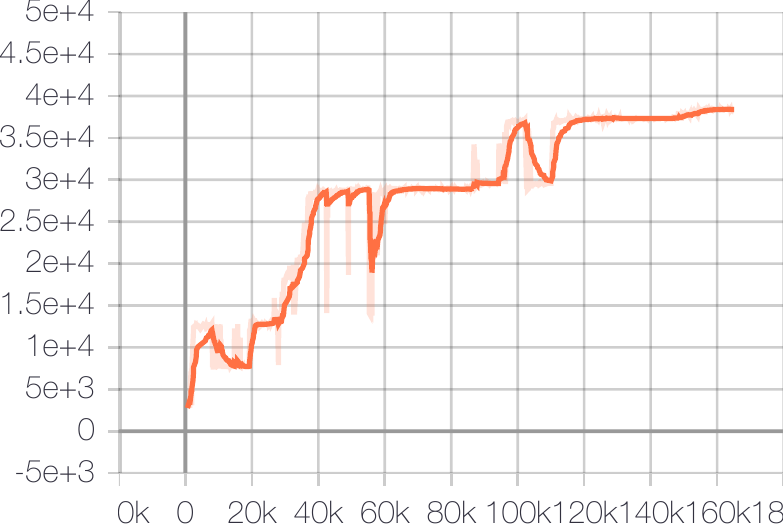}
    }
    \subfigure[ice\_hockey]{
    \includegraphics[width=0.3\textwidth]{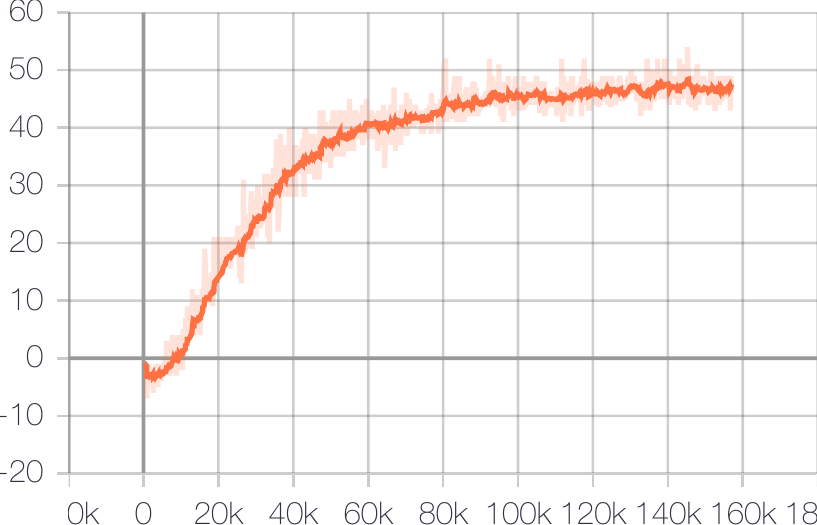}
    }
\end{figure}

\begin{figure}[!ht]
    \subfigure[jamesbond]{
    \includegraphics[width=0.3\textwidth]{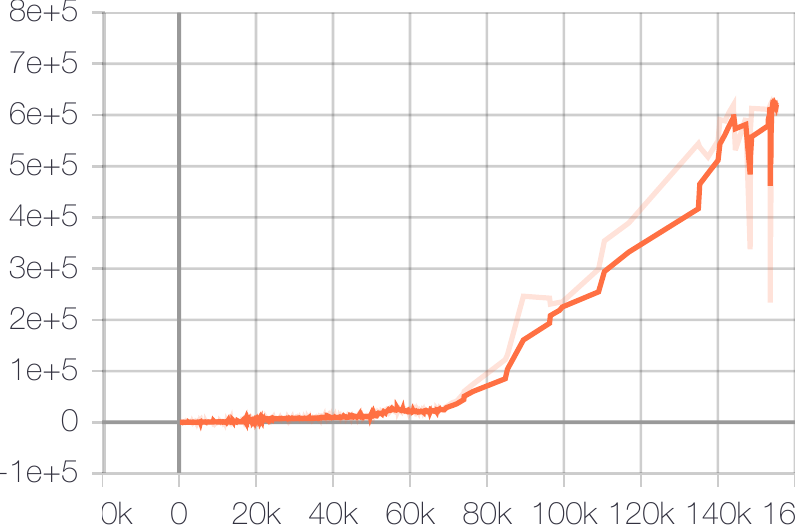}
    }
    \subfigure[kangaroo]{
    \includegraphics[width=0.3\textwidth]{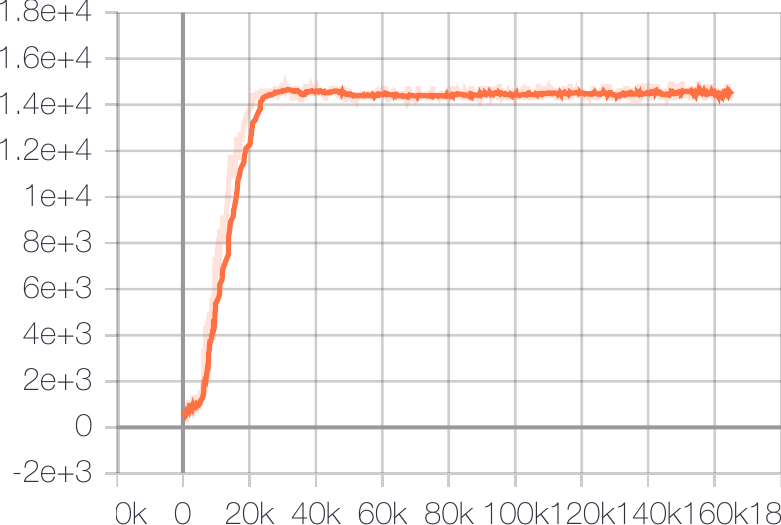}
    }
    \subfigure[krull]{
    \includegraphics[width=0.3\textwidth]{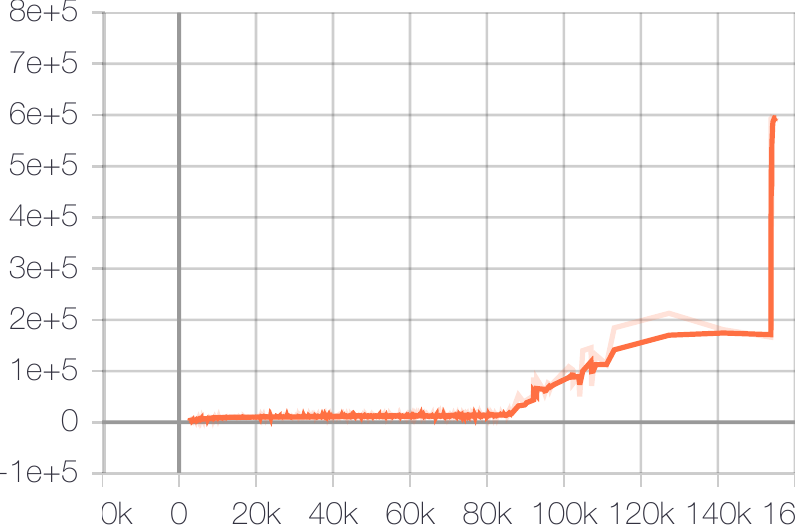}
    }
\end{figure}

\begin{figure}[!ht]
    \subfigure[kung\_fu\_master]{
    \includegraphics[width=0.3\textwidth]{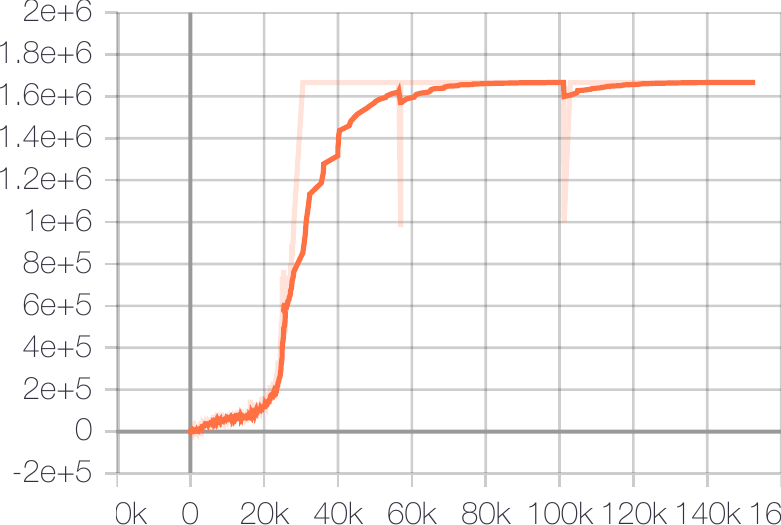}
    }
    \subfigure[montezuma\_revenge]{
     \includegraphics[width=0.3\textwidth]{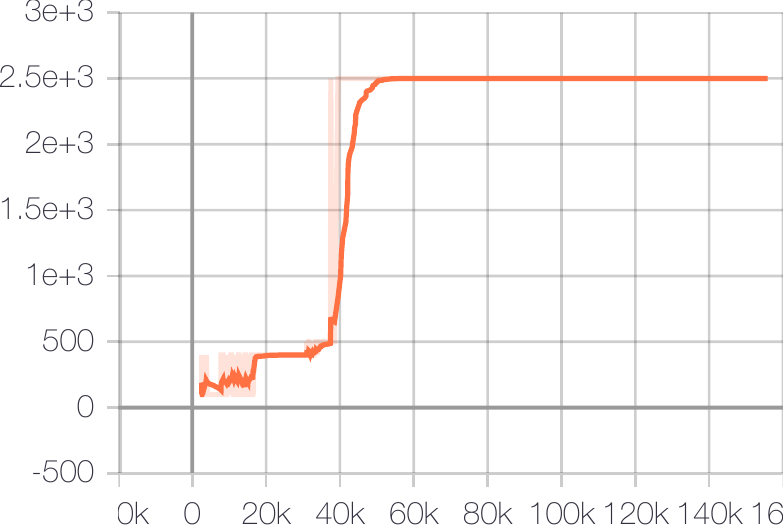}
    }
    \subfigure[ms\_pacman]{
    \includegraphics[width=0.3\textwidth]{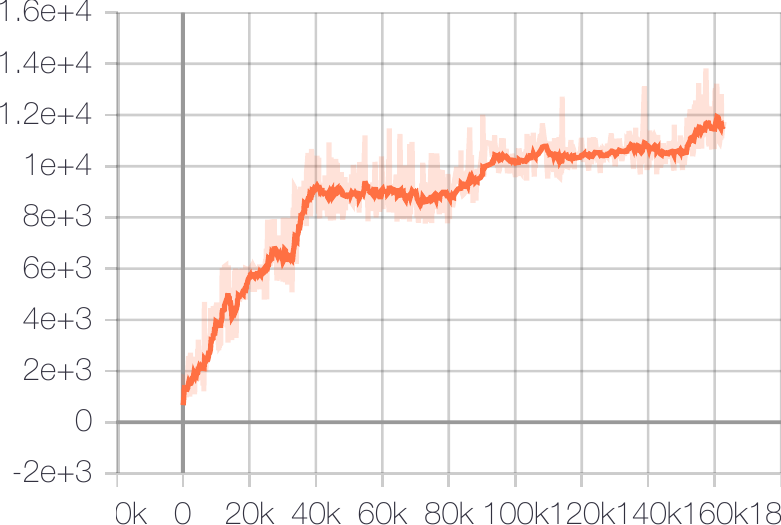}
    }
\end{figure}

\begin{figure}[!ht]
    \subfigure[name\_this\_game]{
    \includegraphics[width=0.3\textwidth]{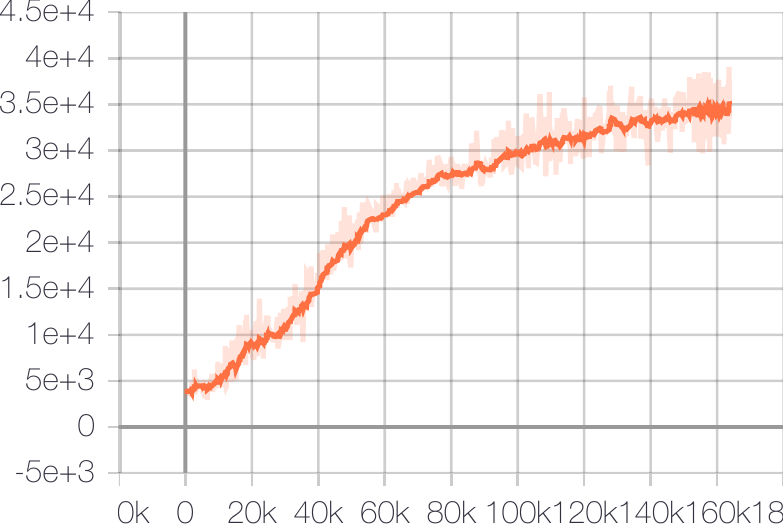}
    }
    \subfigure[phoenix]{
     \includegraphics[width=0.3\textwidth]{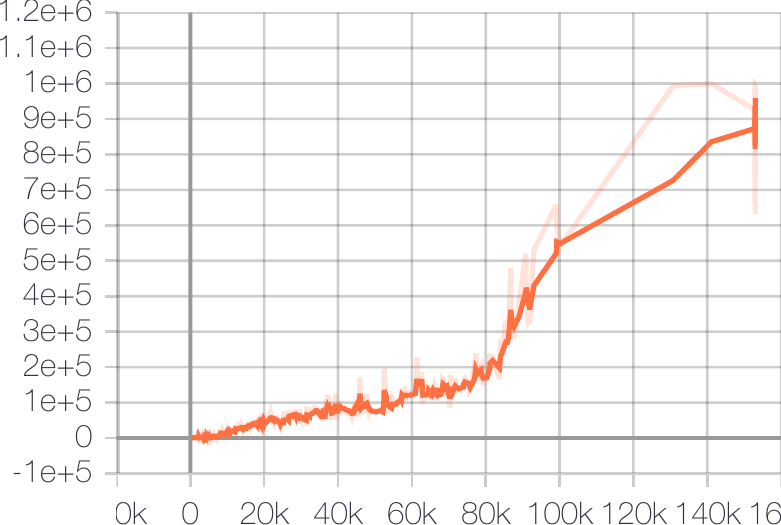}
    }
    \subfigure[pitfall]{
    \includegraphics[width=0.3\textwidth]{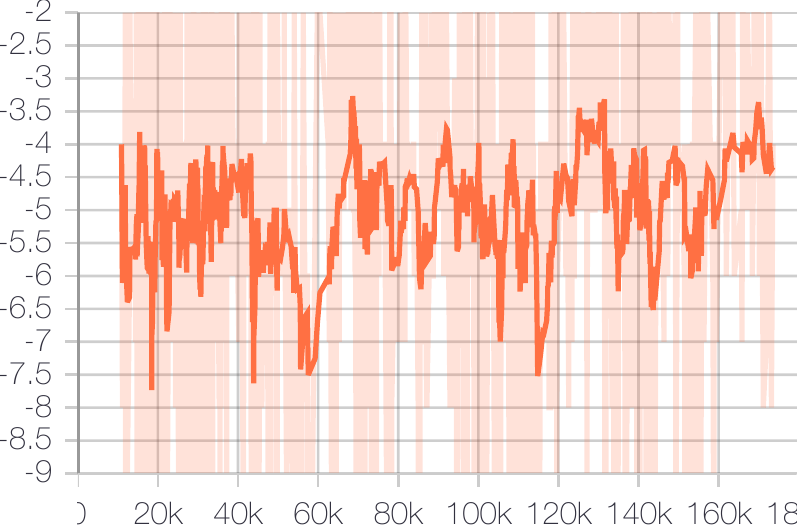}
    }
\end{figure}

\begin{figure}[!ht]
    \subfigure[pong]{
    \includegraphics[width=0.3\textwidth]{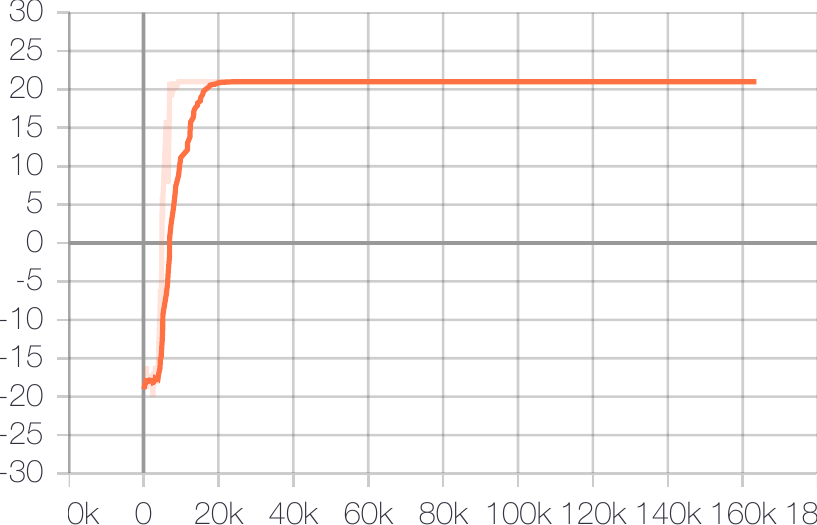}
    }
    \subfigure[private\_eye]{
    \includegraphics[width=0.3\textwidth]{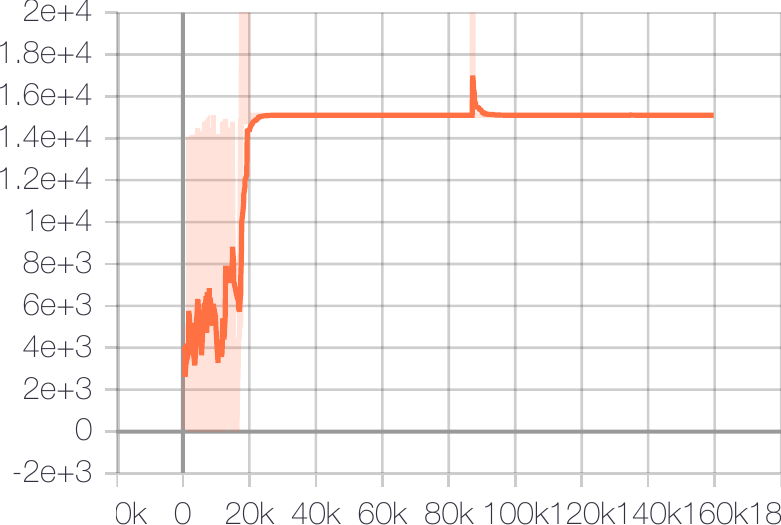}
    }
    \subfigure[qbert]{
     \includegraphics[width=0.3\textwidth]{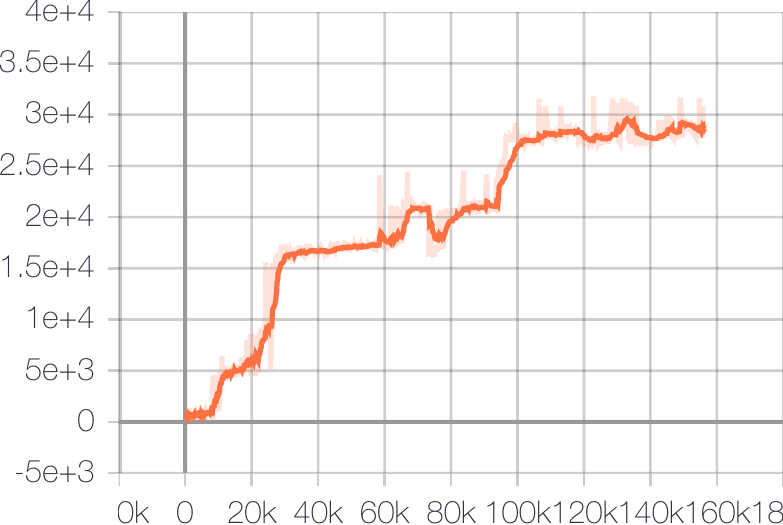}
    }
\end{figure}

\begin{figure}[!ht]
    \subfigure[riverraid]{
    \includegraphics[width=0.3\textwidth]{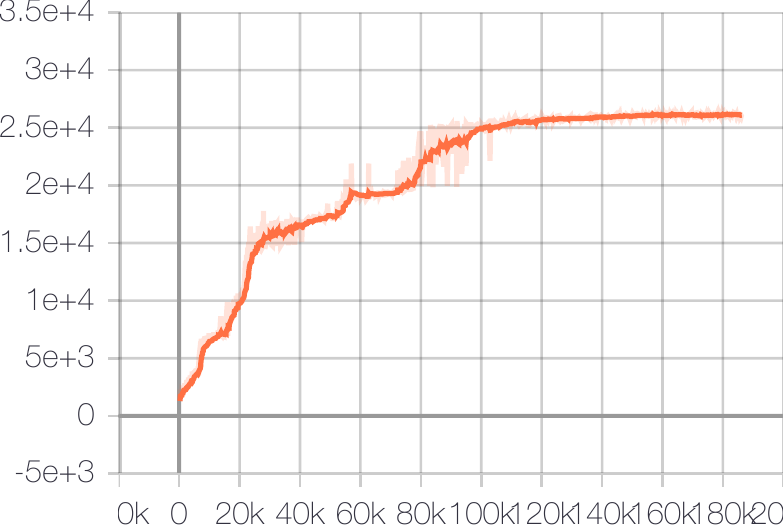}
    }
    \subfigure[road\_runner]{
    \includegraphics[width=0.3\textwidth]{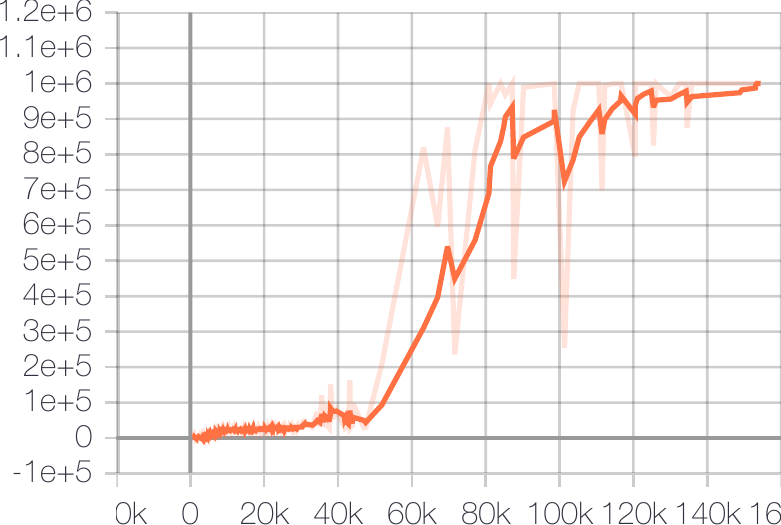}
    }
    \subfigure[robotank]{
    \includegraphics[width=0.3\textwidth]{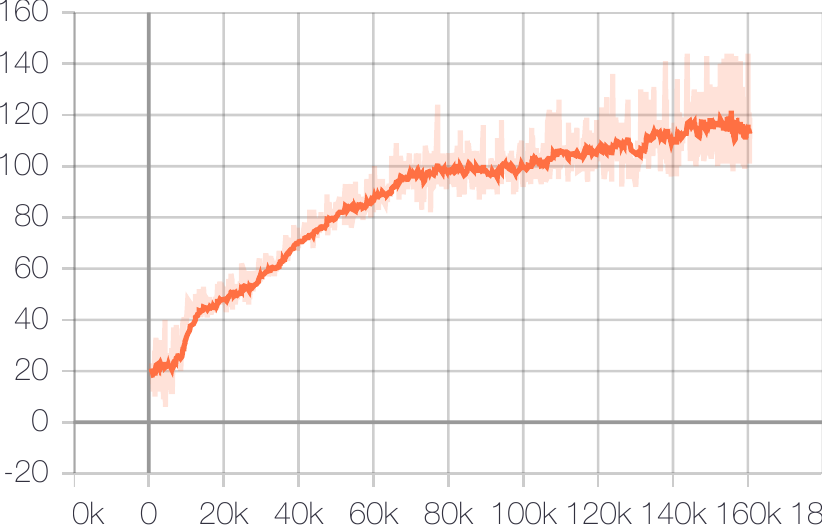}
    }
\end{figure}

\begin{figure}[!ht]
    \subfigure[seaquest]{
    \includegraphics[width=0.3\textwidth]{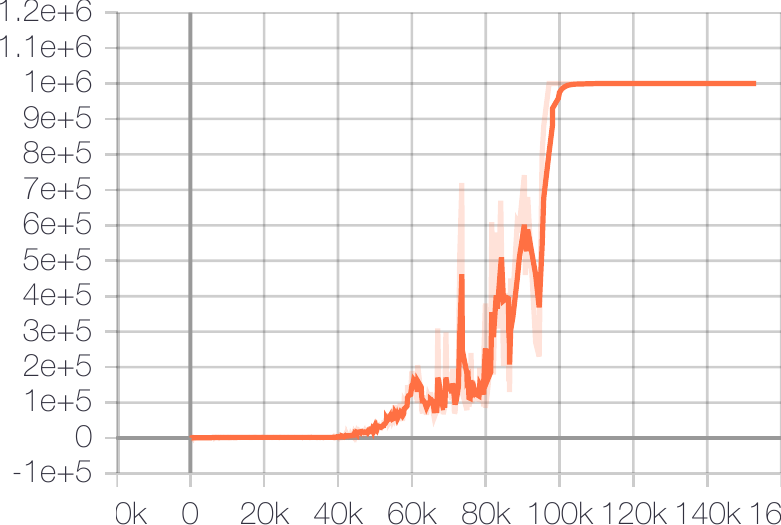}
    }
    \subfigure[skiing]{
    \includegraphics[width=0.3\textwidth]{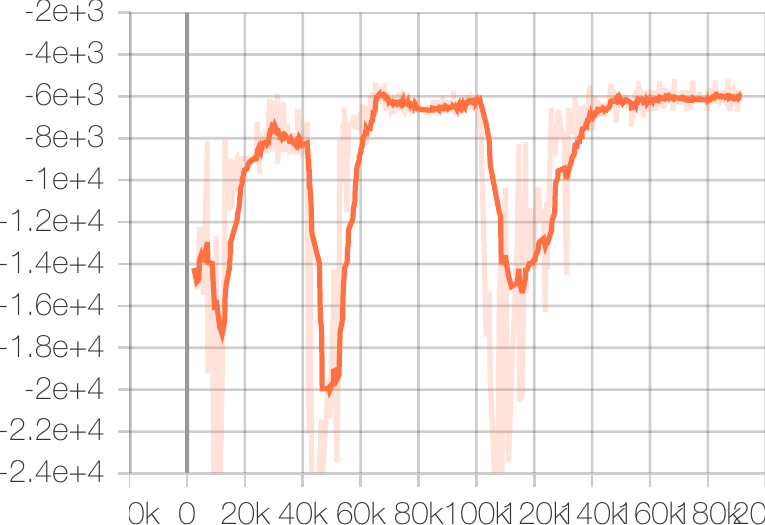}
    }
    \subfigure[solaris]{
    \includegraphics[width=0.3\textwidth]{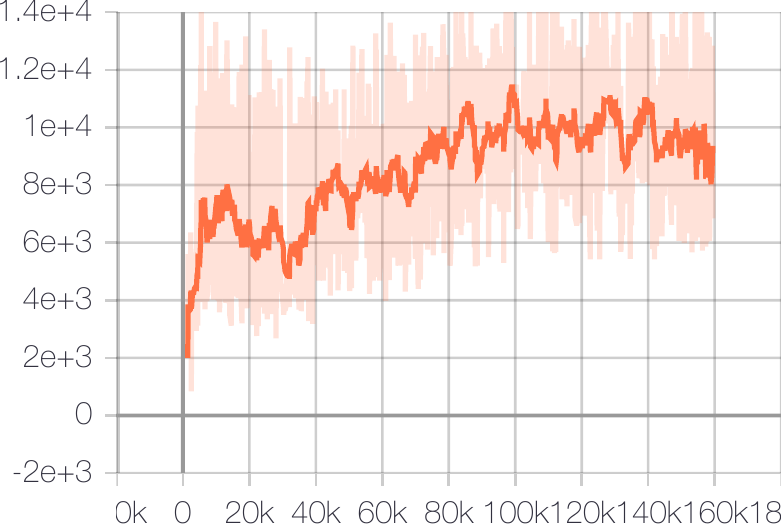}
    }
\end{figure}

\begin{figure}[!ht]
    \subfigure[space\_invaders]{
     \includegraphics[width=0.3\textwidth]{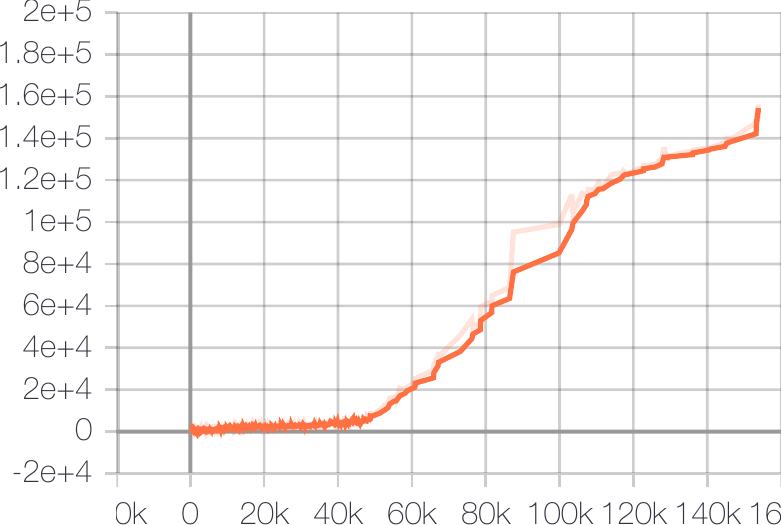}
    }
    \subfigure[star\_gunner]{
    \includegraphics[width=0.3\textwidth]{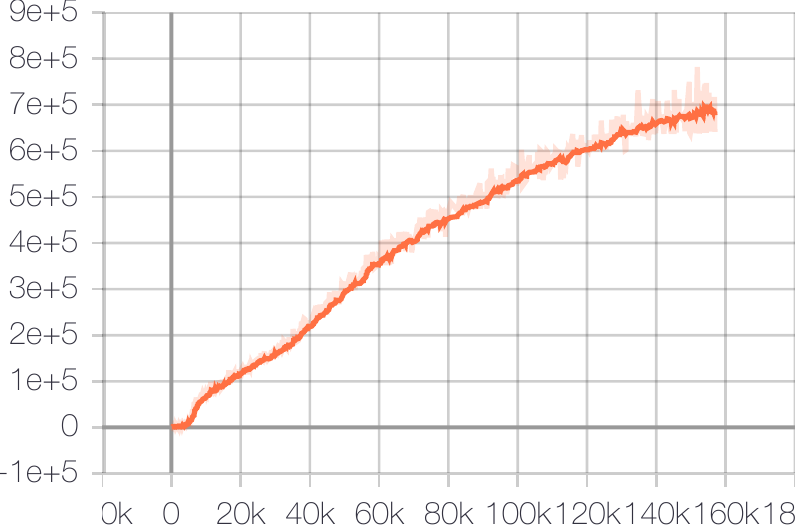}
    }
    \subfigure[surround]{
    \includegraphics[width=0.3\textwidth]{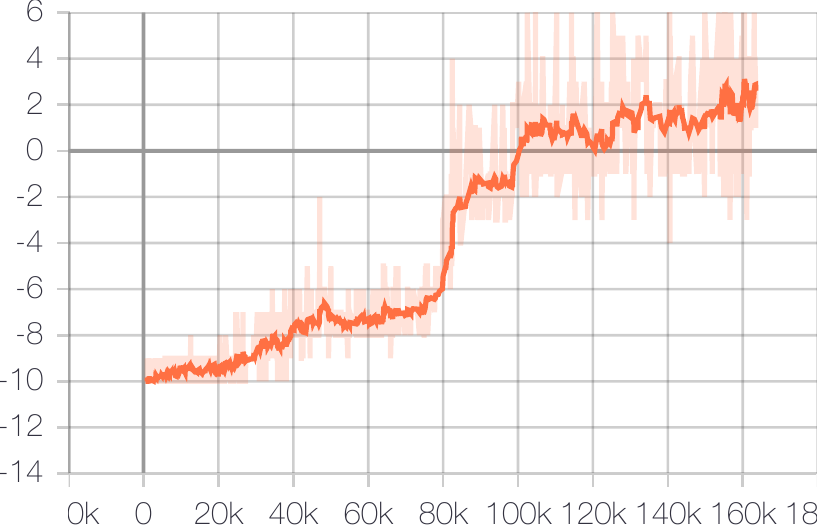}
    }
\end{figure}

\begin{figure}[!ht]
    \subfigure[tennis]{
    \includegraphics[width=0.3\textwidth]{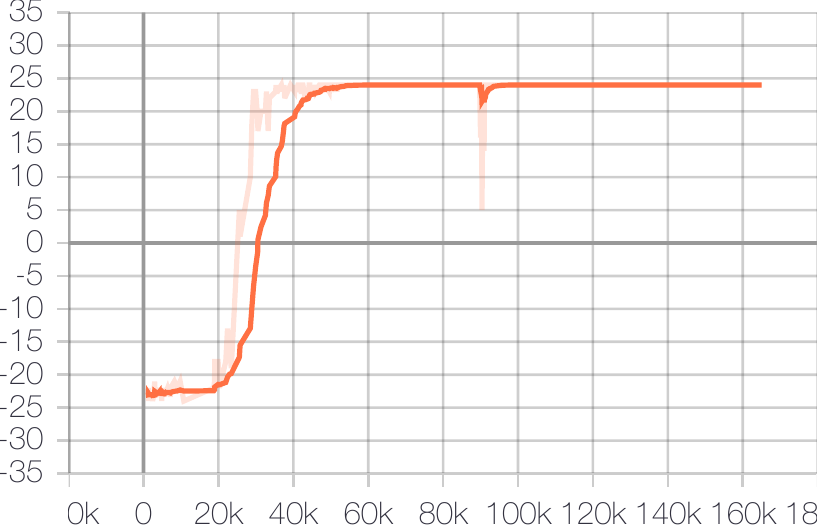}
    }
    \subfigure[time\_pilot]{
    \includegraphics[width=0.3\textwidth]{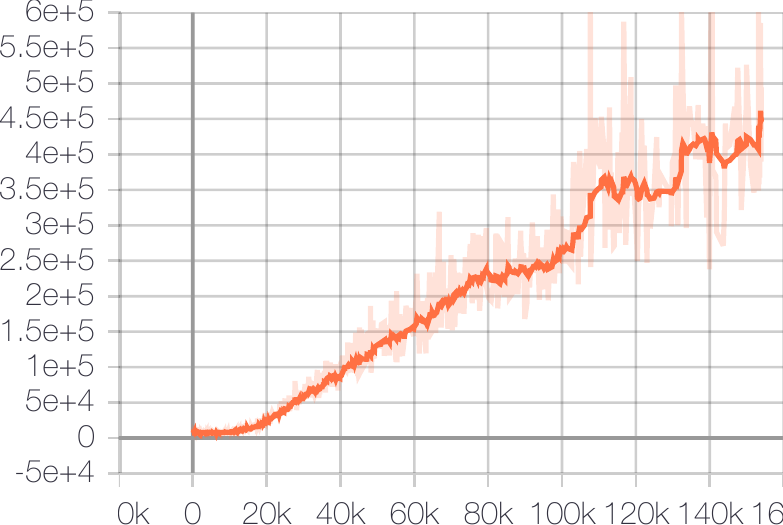}
    }
    \subfigure[tutankham]{
    \includegraphics[width=0.3\textwidth]{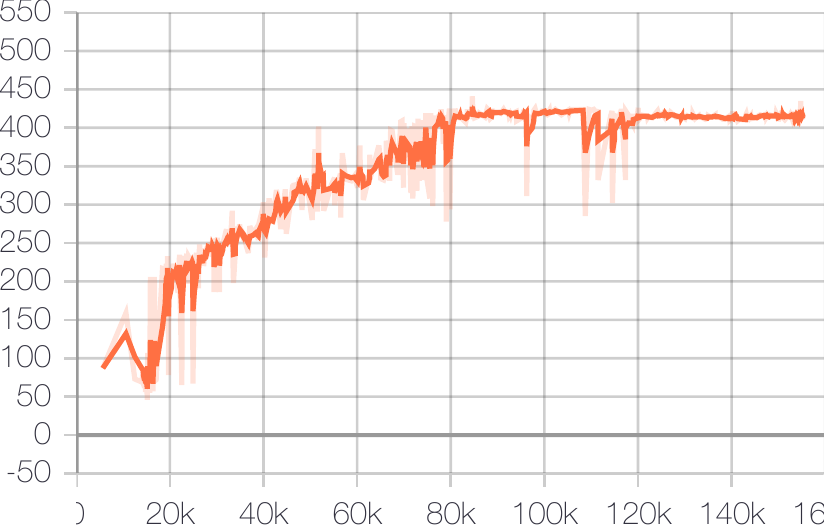}
    }
\end{figure}

\clearpage

\begin{figure}[!ht]
    \subfigure[up\_n\_down]{
    \includegraphics[width=0.3\textwidth]{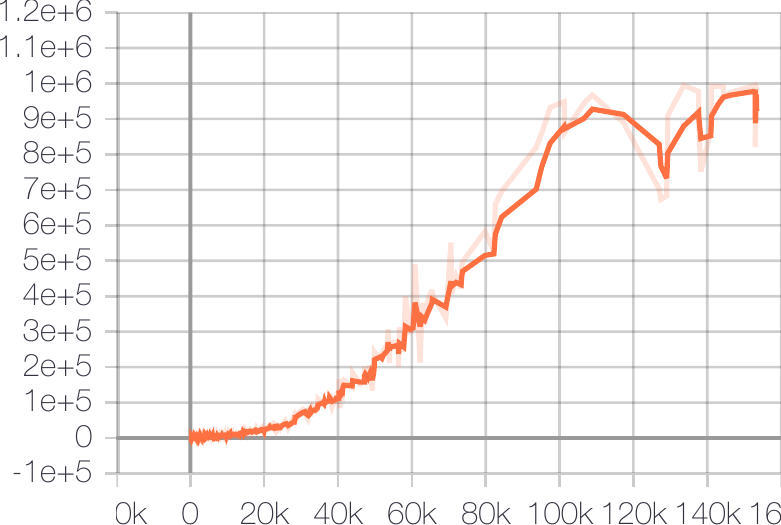}
    }
    \subfigure[venture]{
    \includegraphics[width=0.3\textwidth]{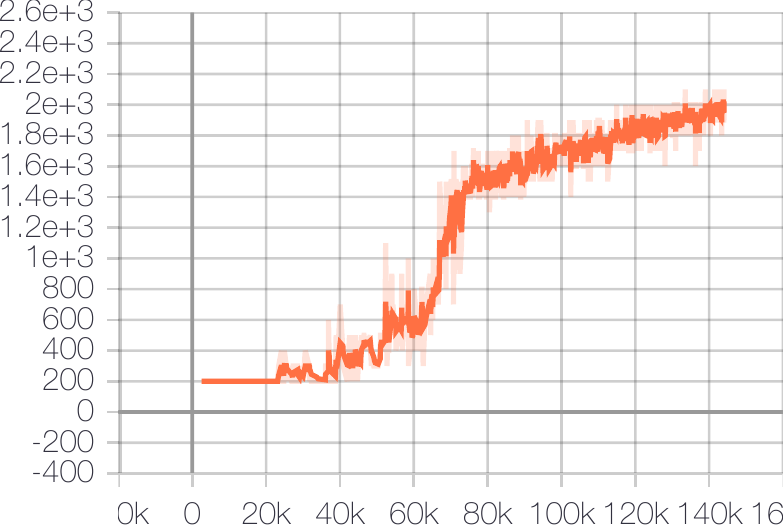}
    }
    \subfigure[video\_pinball]{
    \includegraphics[width=0.3\textwidth]{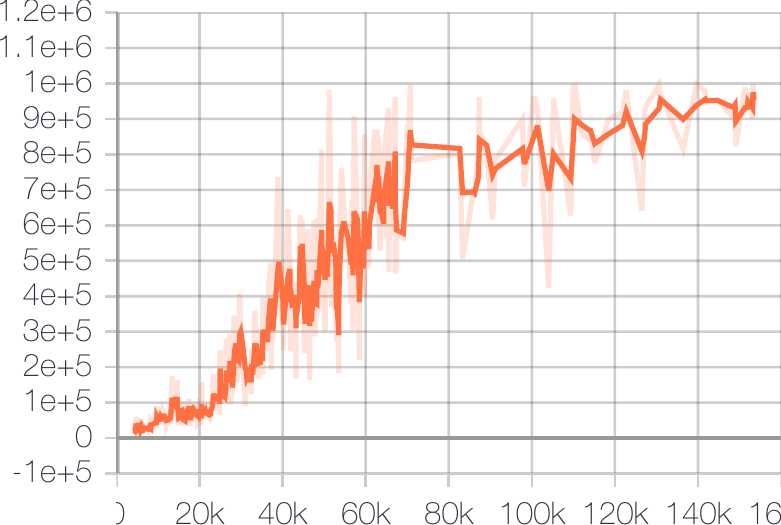}
    }
\end{figure}

\begin{figure}[!ht]
    \subfigure[wizard\_of\_wor]{
    \includegraphics[width=0.3\textwidth]{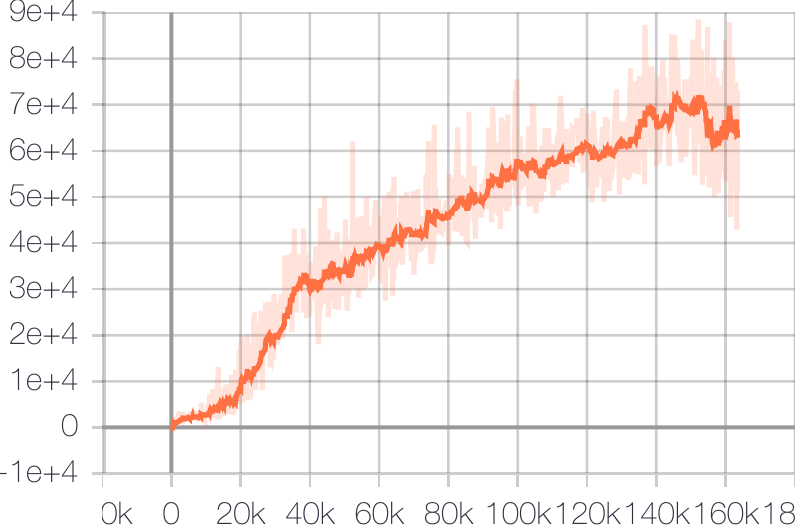}
    }
    \subfigure[yars\_revenge]{
    \includegraphics[width=0.3\textwidth]{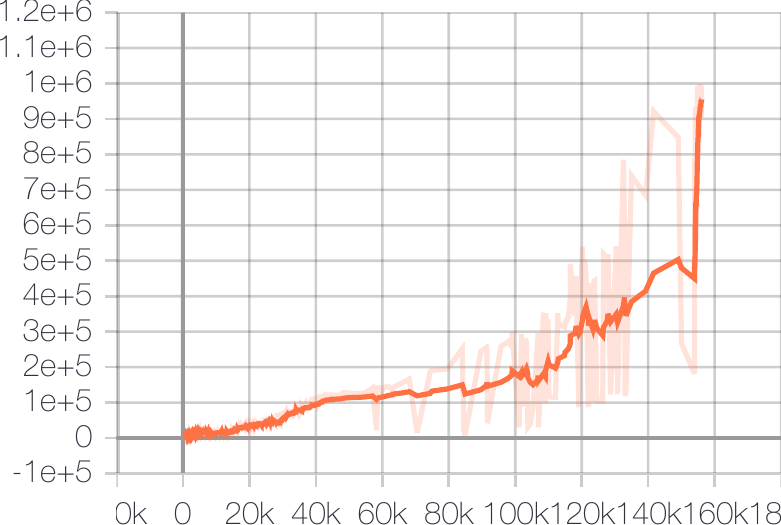}
    }
    \subfigure[zaxxon]{
    \includegraphics[width=0.3\textwidth]{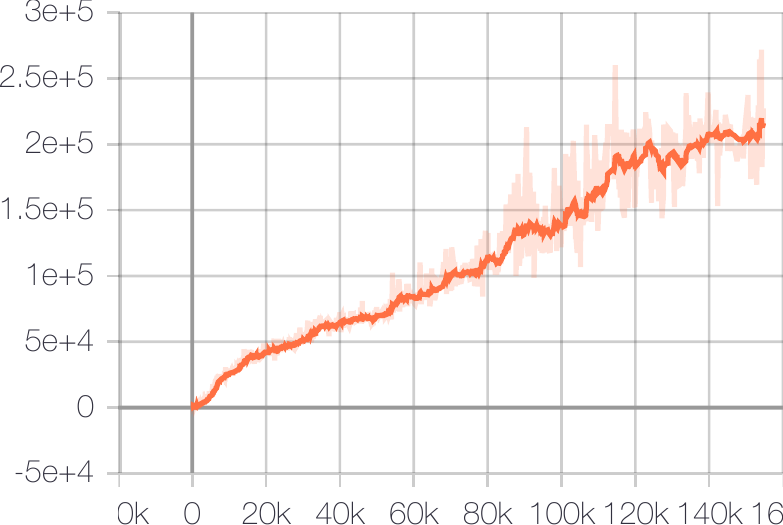}
    }
\end{figure}

\clearpage

\section{Ablation Study}
\label{Sec: appendix Ablation Study}

In this section, we firstly demonstrate the settings of our ablation studies. Then, we illustrate the performance graphs among all ablation cases in three representative Atari games. Lastly, we offer the t-SNE of three Atari games to further prove expansion of policy space in GDI brings more diverse data and more chance to obtain elite data.

\subsection{Ablation Study Details}
\label{app: appendix Ablation Study Details}
The details of ablation changes are listed in Tab. \ref{tab:ablation_setting}. All the ablation studies are carried out using the evaluation system and 200M training frames. The operator $\mathcal{T}$ is achieved with Vtrace, Retrace and policy gradient. Except for the differences listed in the table, other settings and the shared hyperparameters remain the same in all ablation cases. The hyperparameters can see App. \ref{app: Hyperparameters Used}.

\begin{table}[H]
\begin{center}
\setlength{\tabcolsep}{1.0pt}
\caption{Settings of ablation study.}
\begin{tabular}{c c c c c}
\toprule
\textbf{Name}               & \textbf{Category}                 & $\Lambda$                                                  & $P_{\Lambda}^{(0)}$ & $\mathcal{E}$ \\
\midrule
GDI-I$^{3}$                 &   GDI-I$^3$                       &$\Lambda = \{\lambda| \lambda =(\epsilon,\tau_1,\tau_2)\}$  & Uniform              & MAB\\
GDI-H$^{3}$                 &   GDI-H$^3$                       &$\Lambda = \{\lambda| \lambda =(\epsilon,\tau_1,\tau_2)\}$  & Uniform              & MAB\\
Fix Selection               &   GDI-I$^0$ w/o $\mathcal{E}$     &$\Lambda = \{\lambda| \lambda =(\epsilon,\tau_1,\tau_2)\}$  & Delta           & Identical Mapping\\
Random Selection            &   GDI-I$^3$ w/o $\mathcal{E}$     &$\Lambda = \{\lambda| \lambda =(\epsilon,\tau_1,\tau_2)\}$  & Uniform              & Identical Mapping\\
Boltzmann Selection         &   GDI-I$^1$                       &$\Lambda = \{\lambda| \lambda =(\tau)\}$                    & Uniform              & MAB\\
$\epsilon$-greedy Selection &   GDI-I$^1$                       &$\Lambda = \{\lambda| \lambda =(\epsilon)\}$                & Uniform              & MAB\\
\bottomrule
\end{tabular}
\end{center}
\label{tab:ablation_setting}
\end{table}

\subsection{Ablation Results}
\label{Sec: Appendix Ablation Study Results}

In this part, we show three representative experiments of different Atari games among all the ablation cases in Figs. \ref{fig:ablation_evl}.
\setcounter{subfigure}{0}
\begin{figure*}[ht]
\subfigure[Seaquest]{
\includegraphics[width=0.5\textwidth]{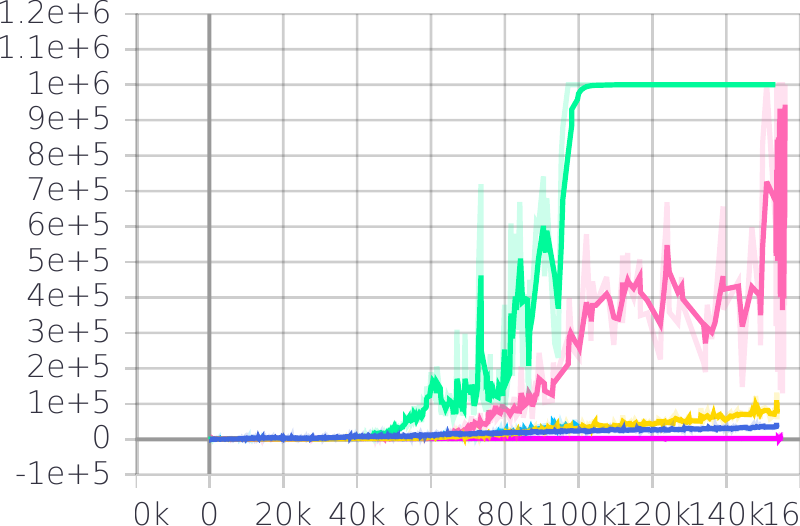}
}
\subfigure[ChopperCommand]{
\includegraphics[width=0.5\textwidth]{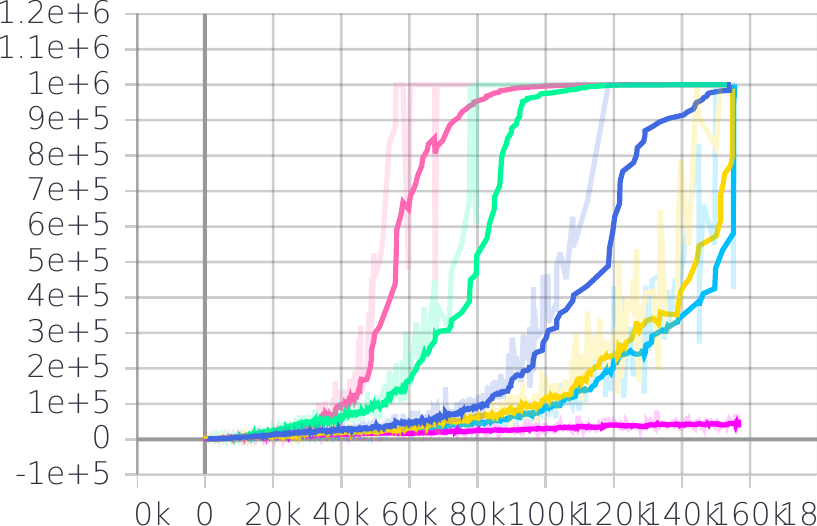}
}

\subfigure[Krull]{
\includegraphics[width=0.5\textwidth]{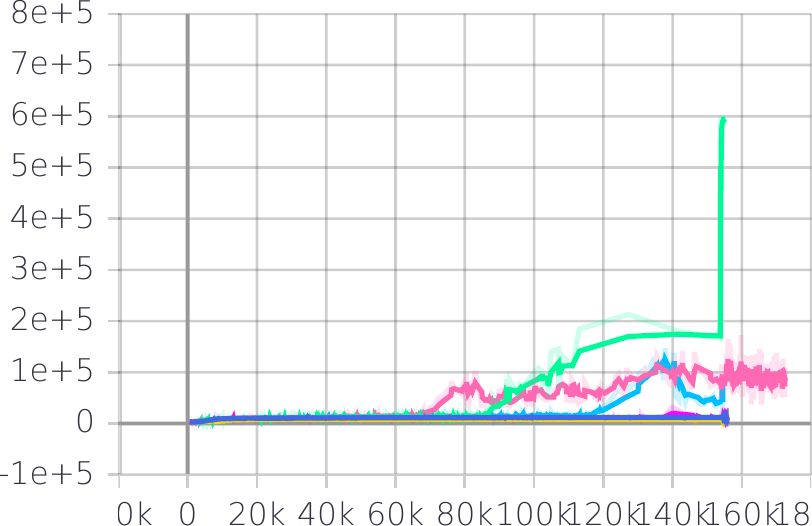}
}\hspace{1cm} 
\subfigure{
		\includegraphics[width=0.3\textwidth,height=0.2\textheight]{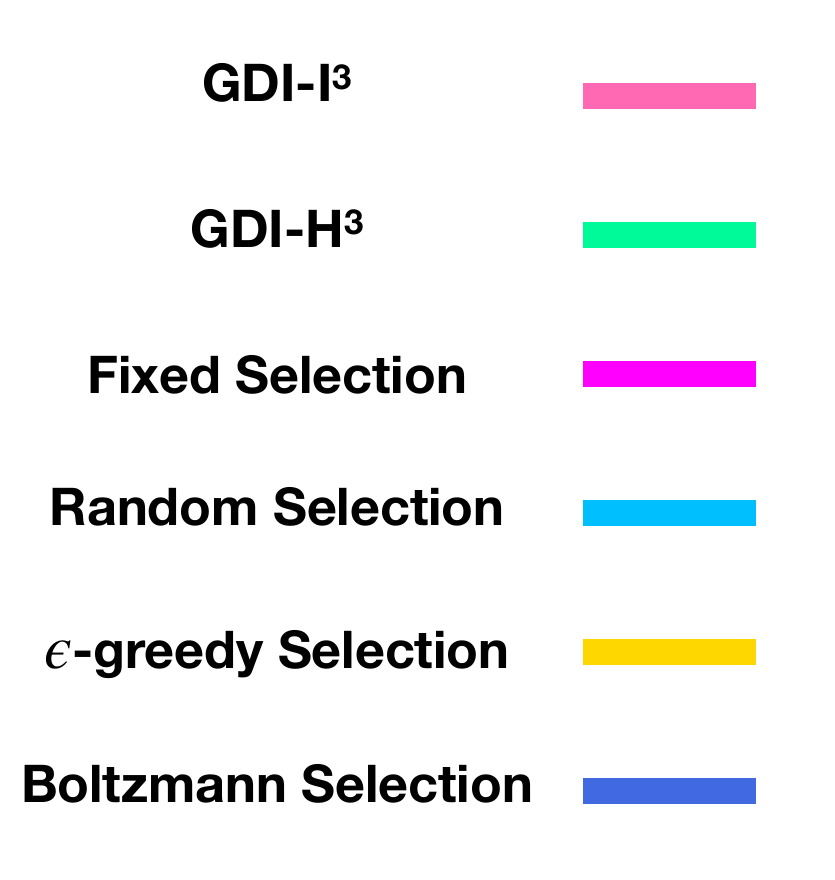}
	}
\caption{Evaluation curves of ablation study.
For ease of comparison, all curves are smoothed with rate 0.9.}
\label{fig:ablation_evl}
\end{figure*}

\clearpage
\subsection{t-SNE}
\label{app: tsne}
In all the t-SNE, we mark the state generated by GDI-I$^3$ as A$_i$ and mark the state generated by GDI-I$^1$ as B$_i$, where i = 1, 2, 3 represents three stages of the training process. WLOG, we choose the Boltzmann Selection as the representative of GDI-I$^1$ algorithms.

\begin{figure}[!ht]
	\subfigure[Early stage of GDI-I$^3$]{
		\includegraphics[width=0.5\textwidth,height=0.25\textheight]{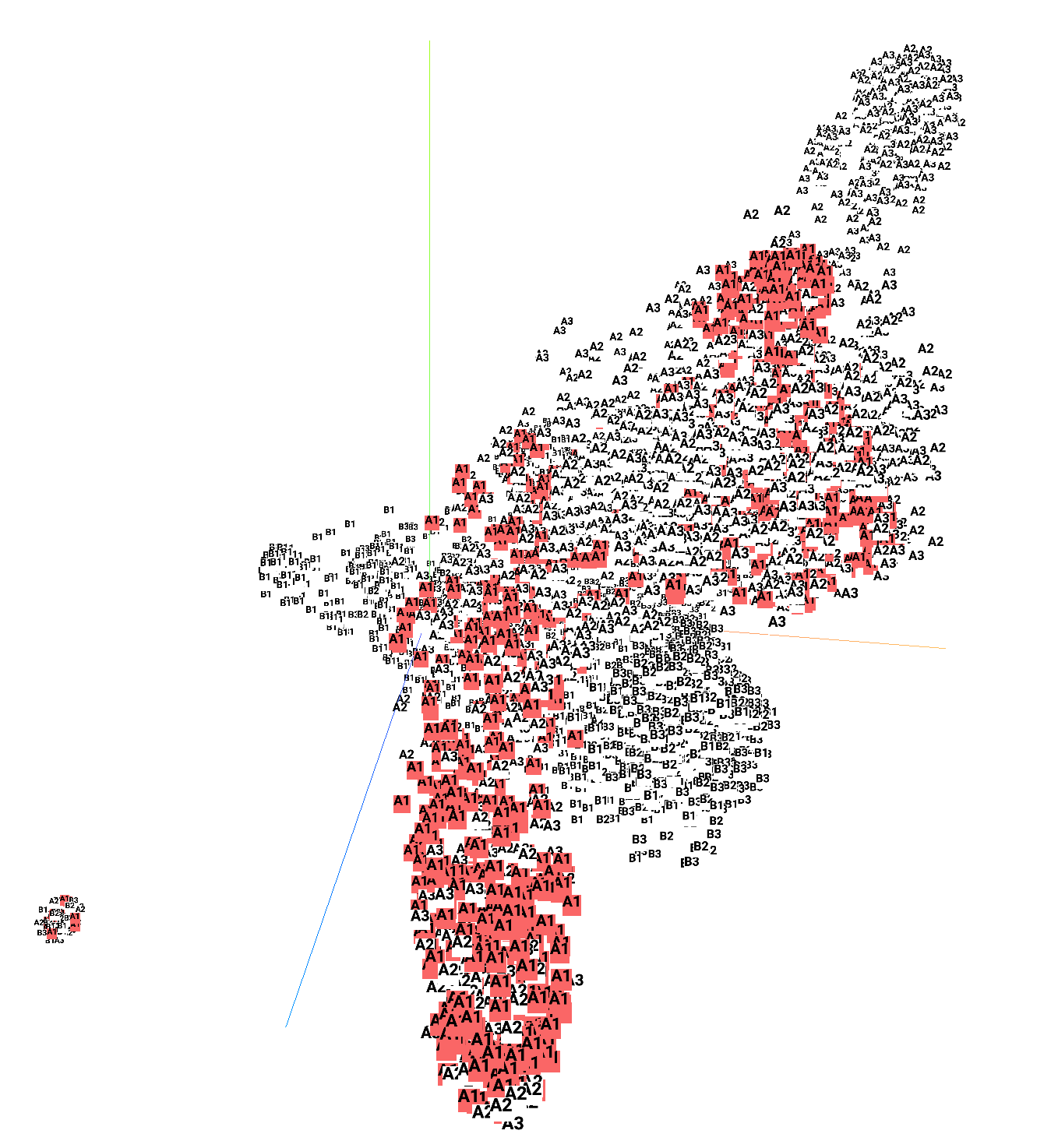}
	    
	   }
	\subfigure[Early stage of GDI-I$^1$]{
		\includegraphics[width=0.5\textwidth,height=0.25\textheight]{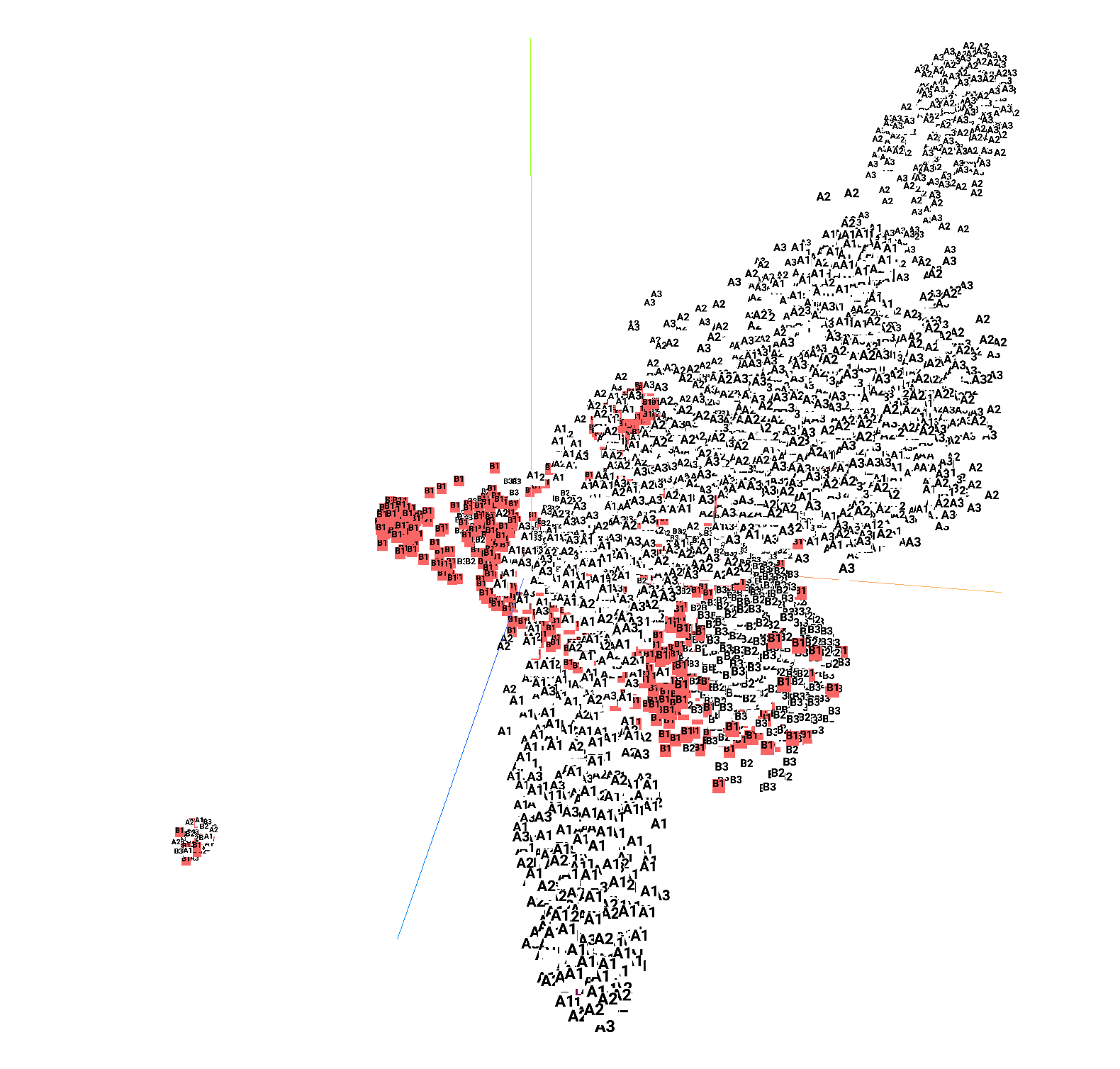}
	}
	
	\subfigure[Middle stage of GDI-I$^3$]{
		\includegraphics[width=0.5\textwidth,height=0.25\textheight]{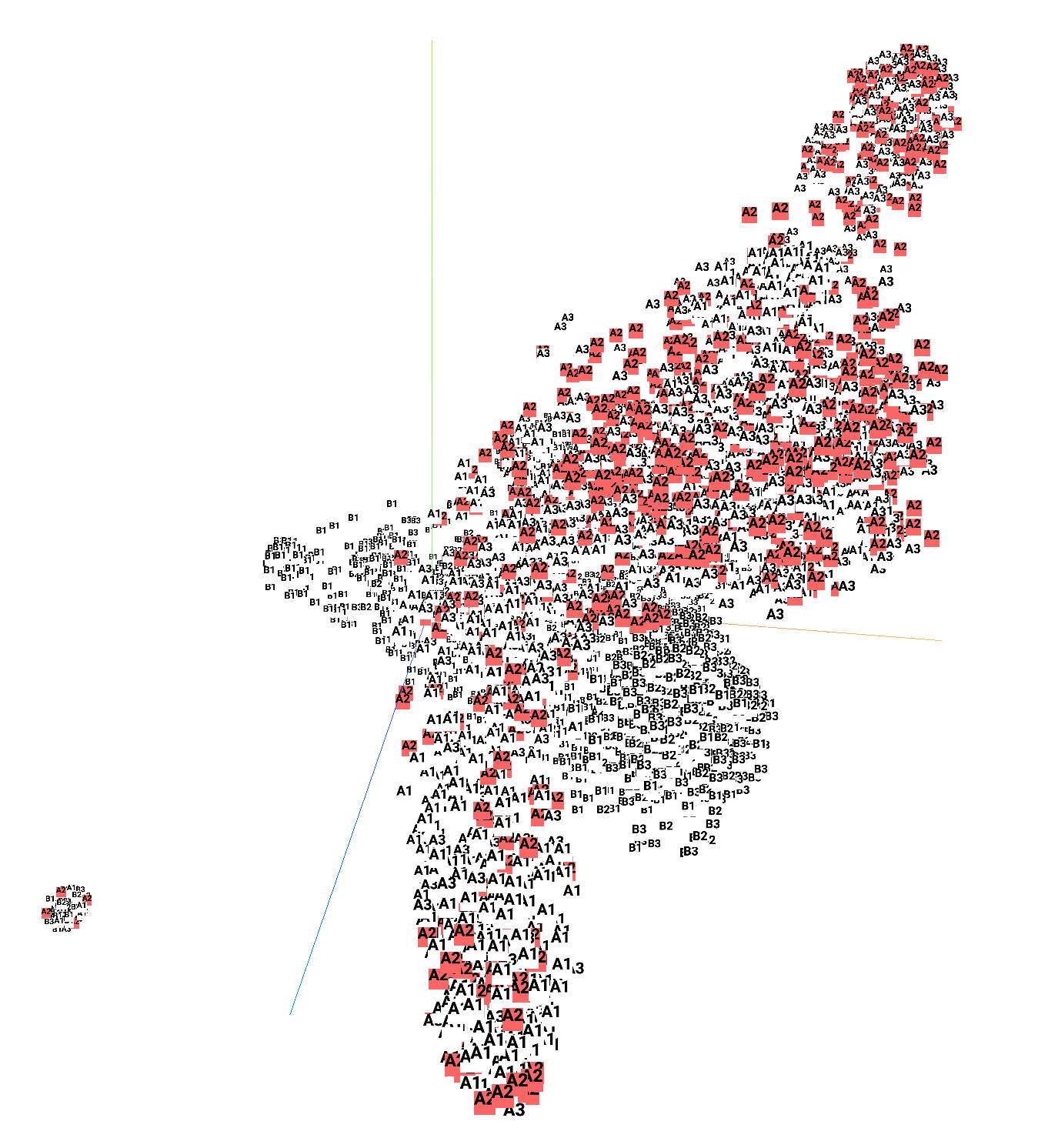}
	}
	\subfigure[Middle stage of GDI-I$^1$]{
		\includegraphics[width=0.5\textwidth,height=0.25\textheight]{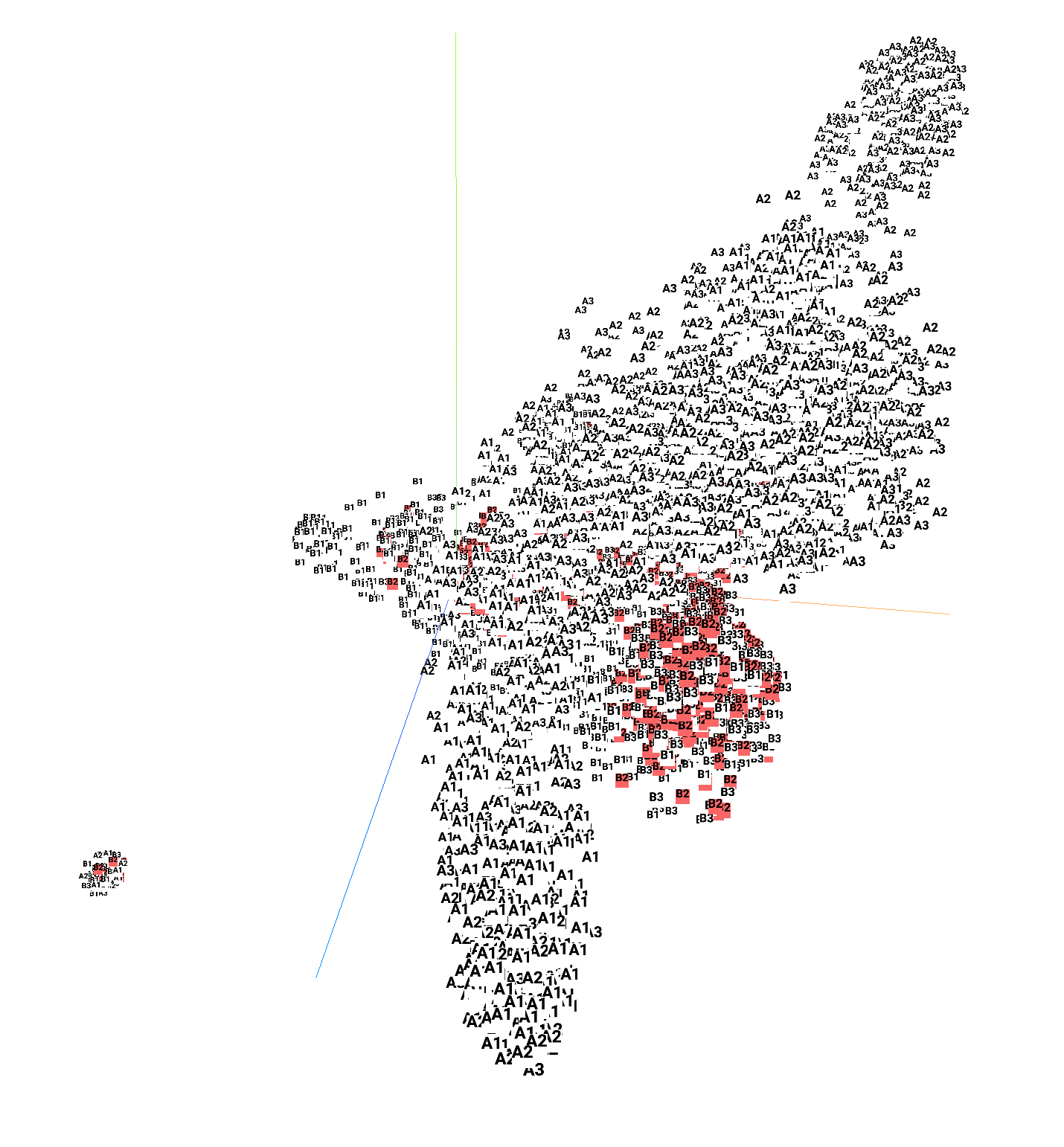}
	}
	
	\subfigure[Later stage of GDI-I$^3$]{
		\includegraphics[width=0.5\textwidth,height=0.25\textheight]{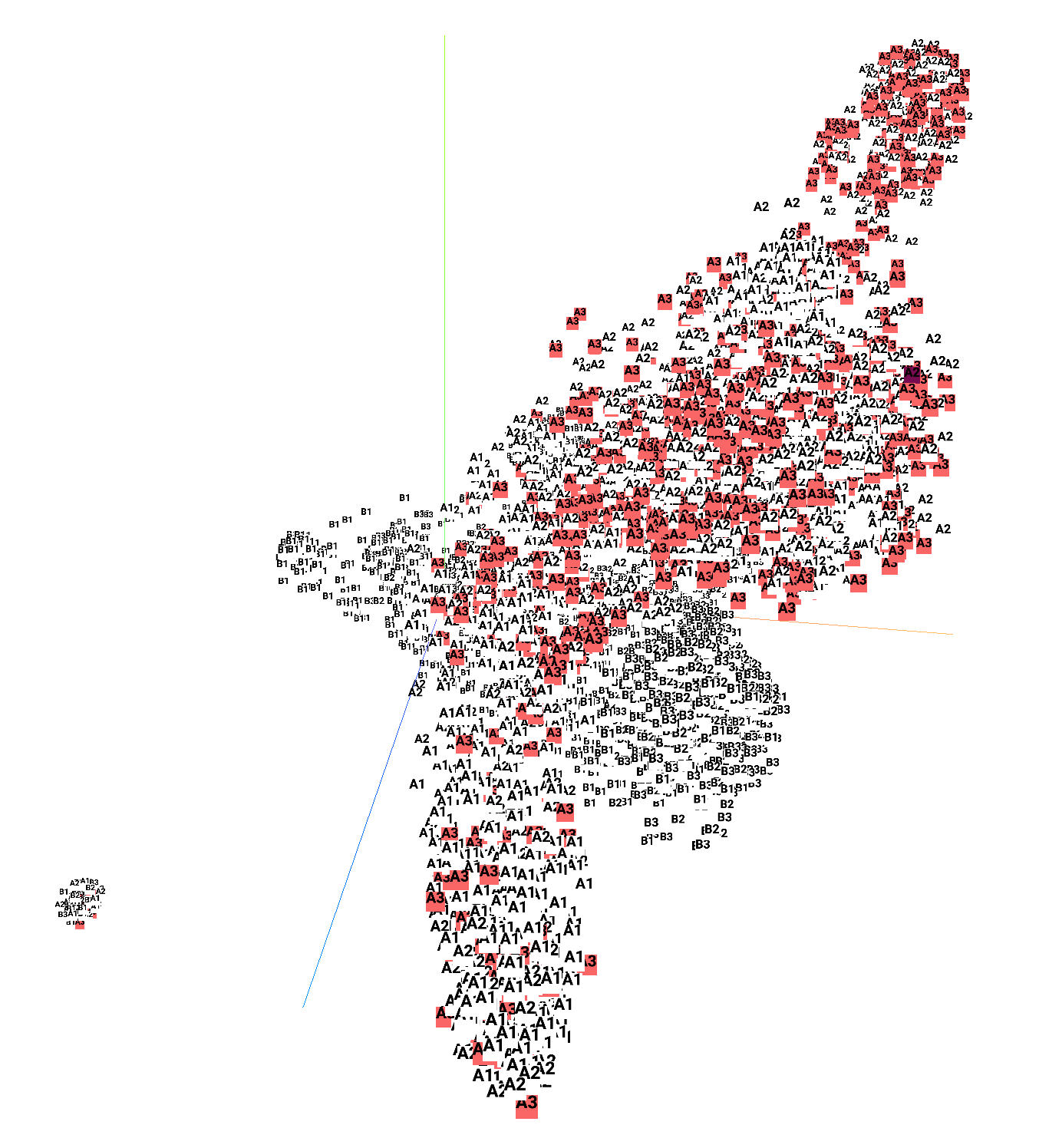}
	}
	\subfigure[Later stage of GDI-I$^1$]{
		\includegraphics[width=0.5\textwidth,height=0.25\textheight]{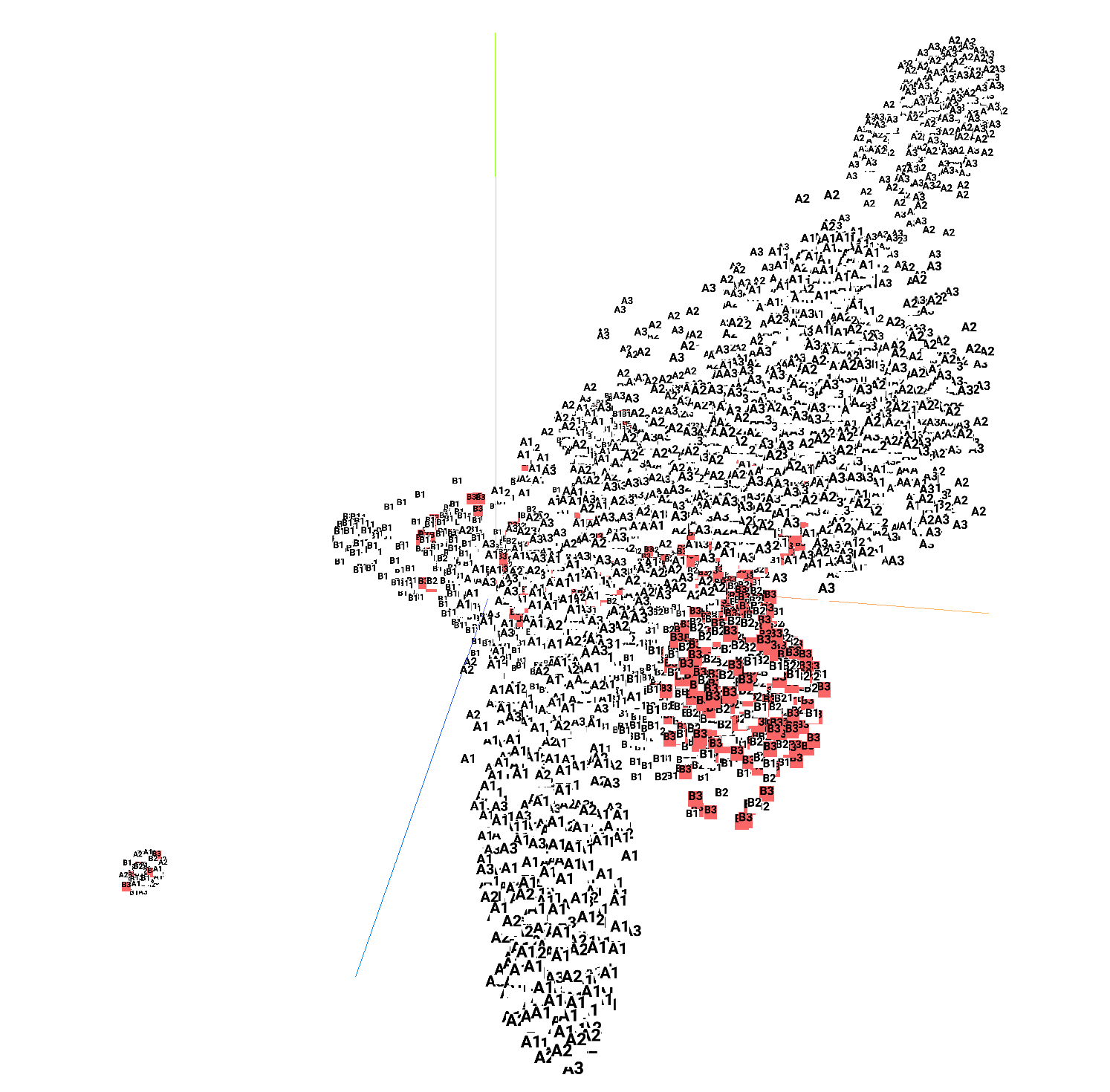}
	}
	\caption{t-SNE of Seaquest. 
t-SNE is drawn from 6k states.
	We sample 1k states from each stage of GDI-I$^3$ and GDI-I$^1$.
	We highlight 1k states of each stage of GDI-I$^3$ and GDI-I$^1$.}
\end{figure}

\begin{figure}[!ht]
	\subfigure[Early stage of GDI-I$^3$]{
		\includegraphics[width=0.45\textwidth]{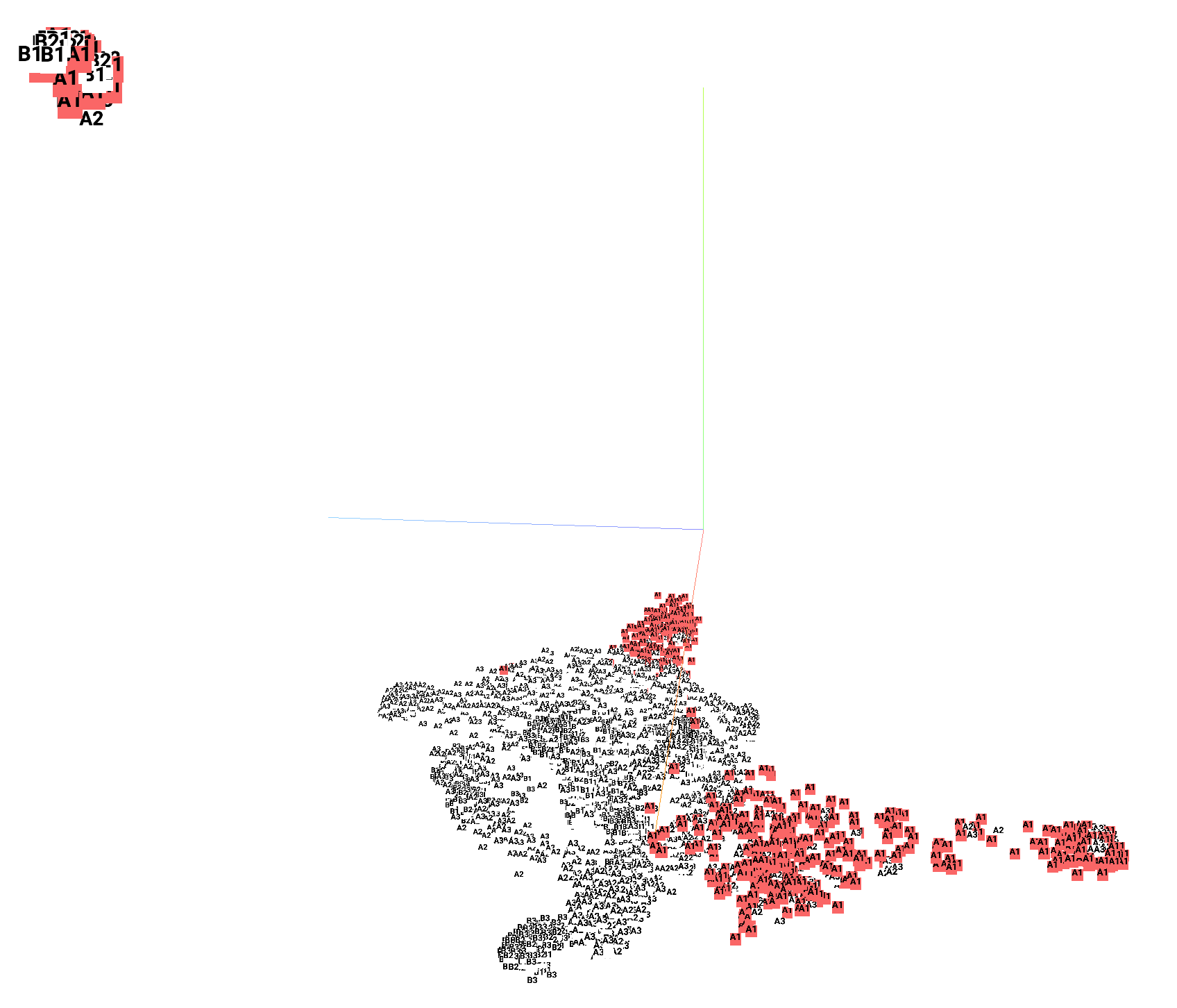}
	   }
	\subfigure[Early stage of GDI-I$^1$]{
		\includegraphics[width=0.45\textwidth]{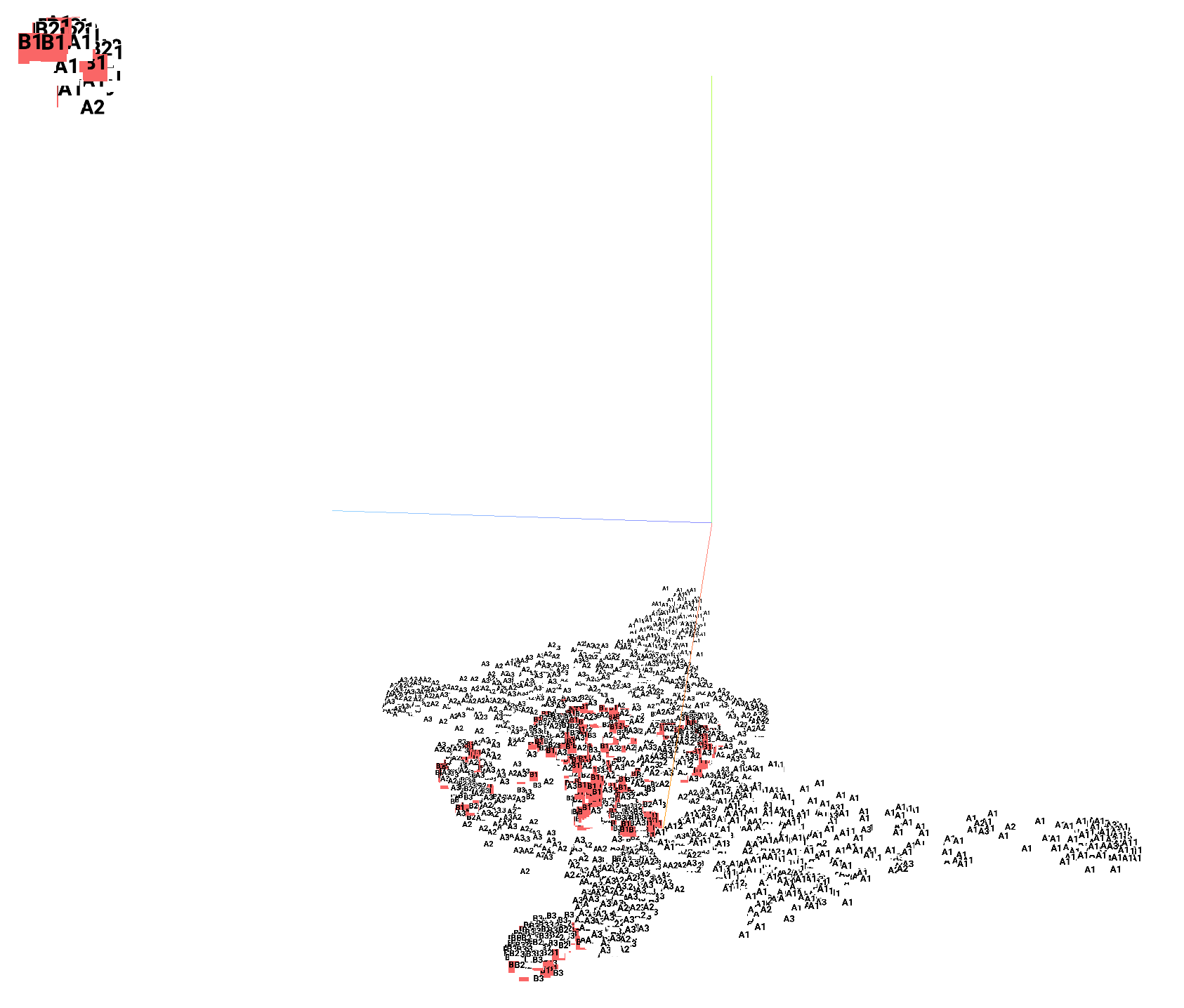}
	}
	
	\subfigure[Middle stage of GDI-I$^3$]{
		\includegraphics[width=0.45\textwidth]{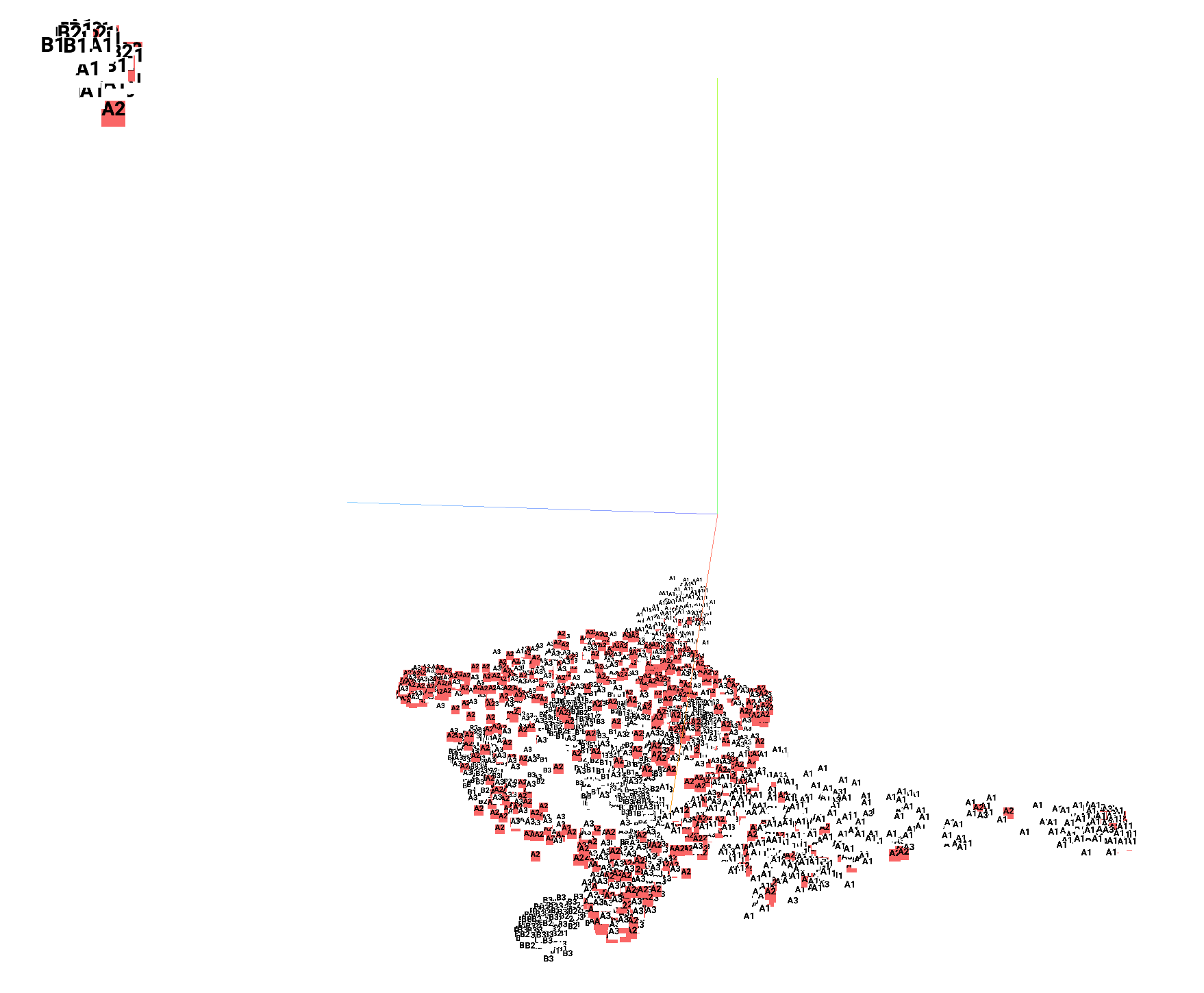}
	}
	\subfigure[Middle stage of GDI-I$^1$]{
		\includegraphics[width=0.45\textwidth]{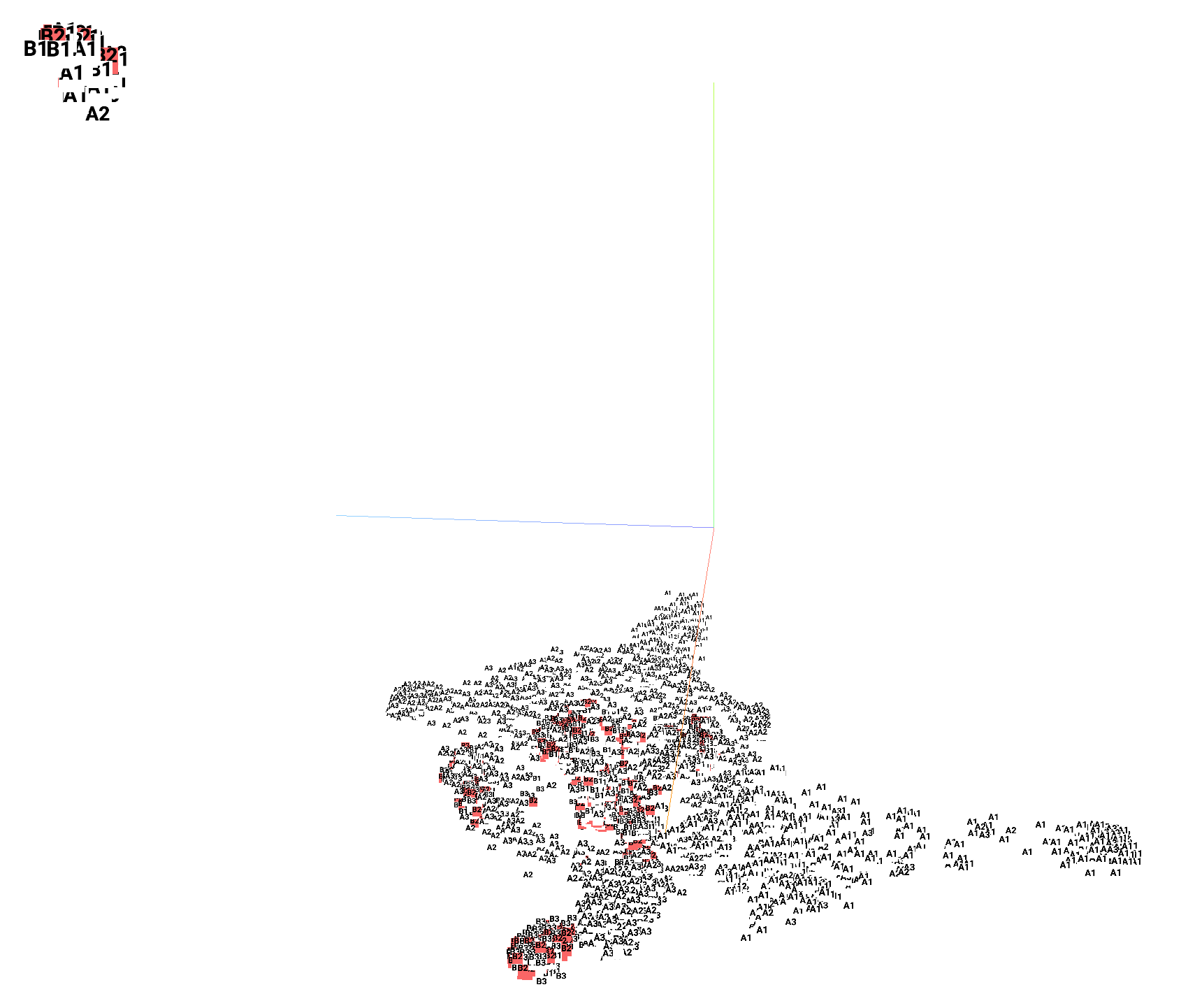}
	}
	
	\subfigure[Later stage of GDI-I$^3$]{
		\includegraphics[width=0.45\textwidth]{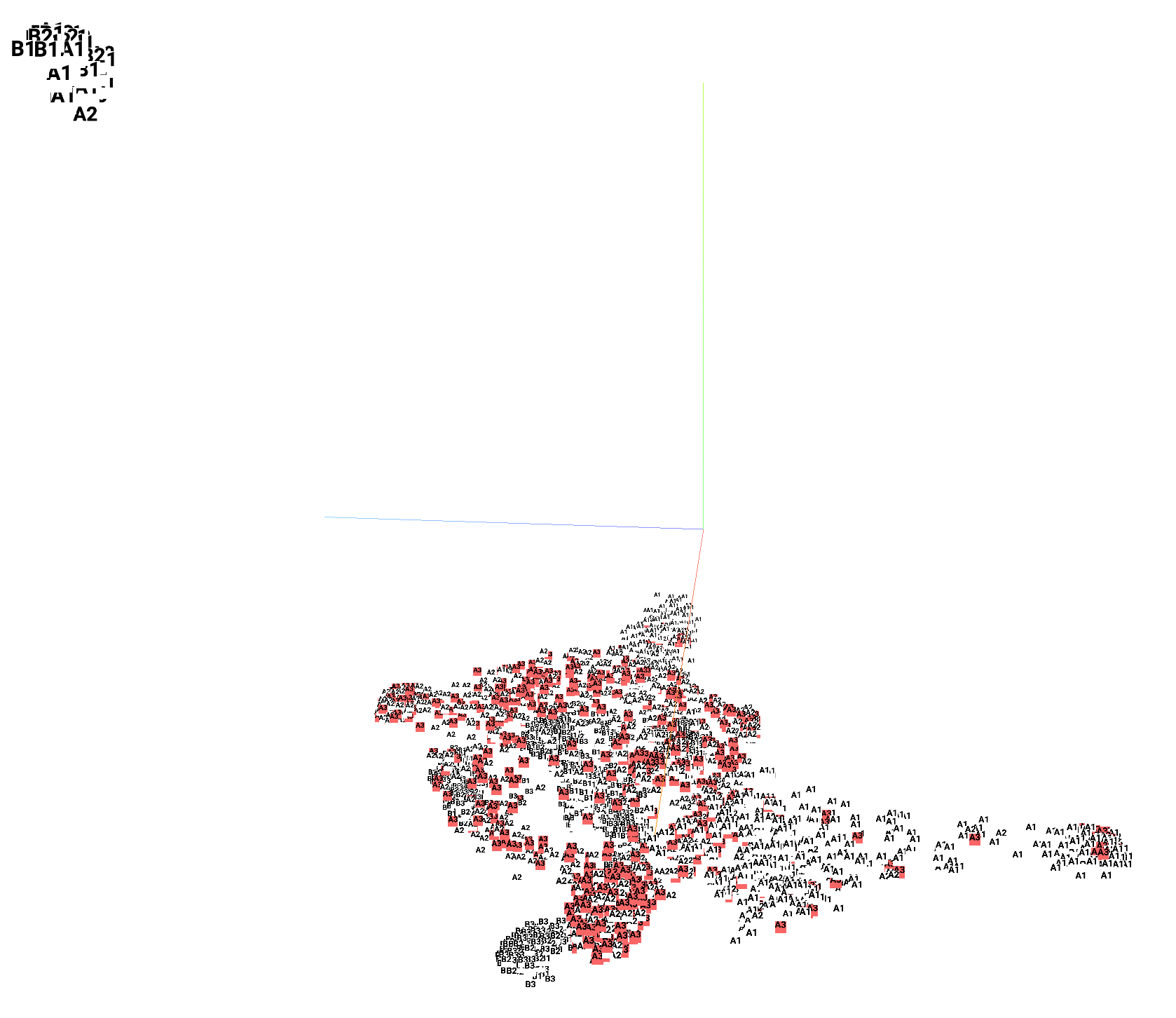}
	}
	\subfigure[Later stage of GDI-I$^1$]{
		\includegraphics[width=0.45\textwidth]{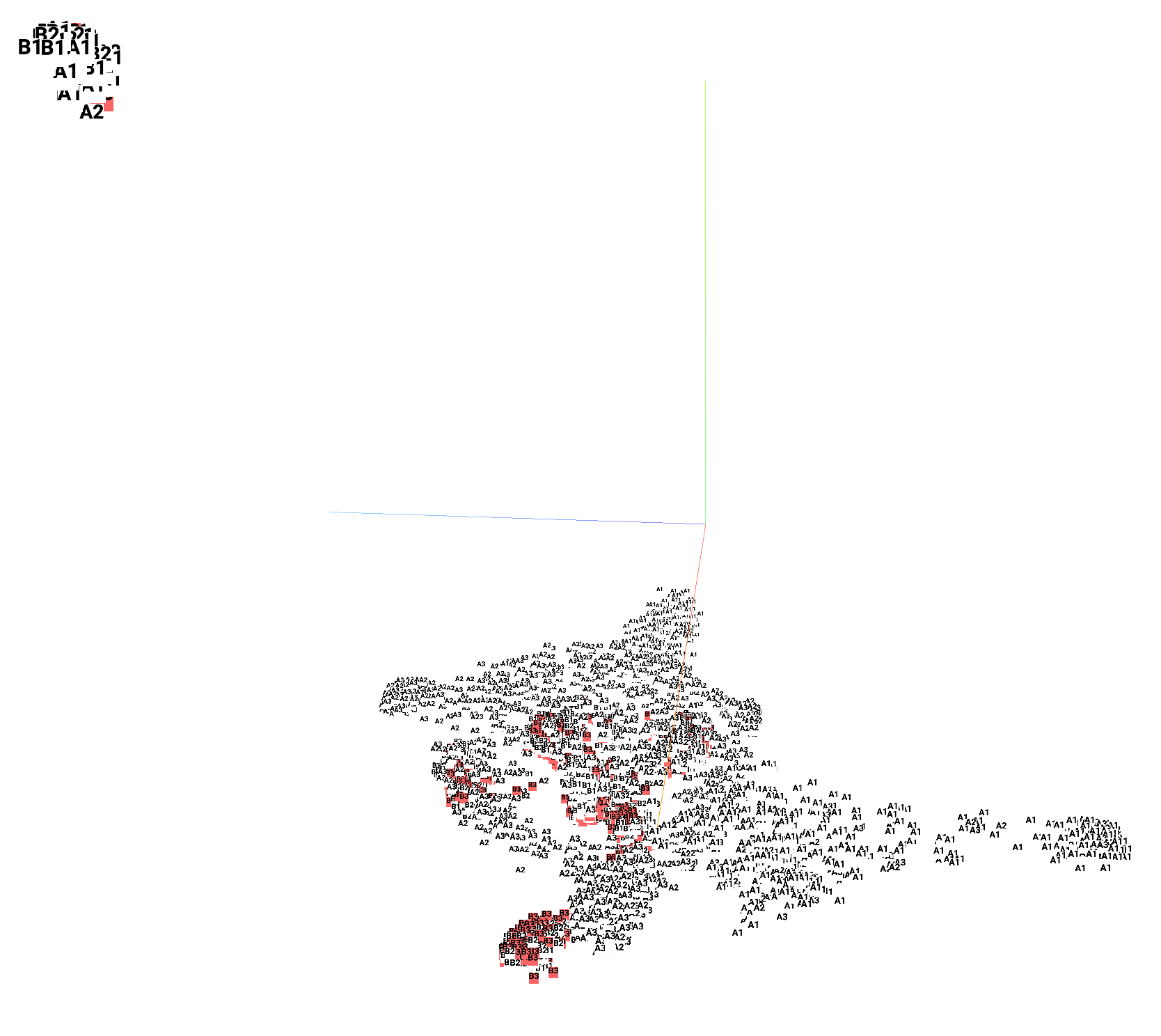}
	}
	\caption{t-SNE of ChopperCommand. 
	t-SNE is drawn from 6k states.
	We sample 1k states from each stage of GDI-I$^3$ and GDI-I$^1$.
	We highlight 1k states of each stage of GDI-I$^3$ and GDI-I$^1$.}
\end{figure}

\begin{figure}[!ht]
	\subfigure[Early stage of GDI-H$^3$]{
		\includegraphics[width=0.5\textwidth]{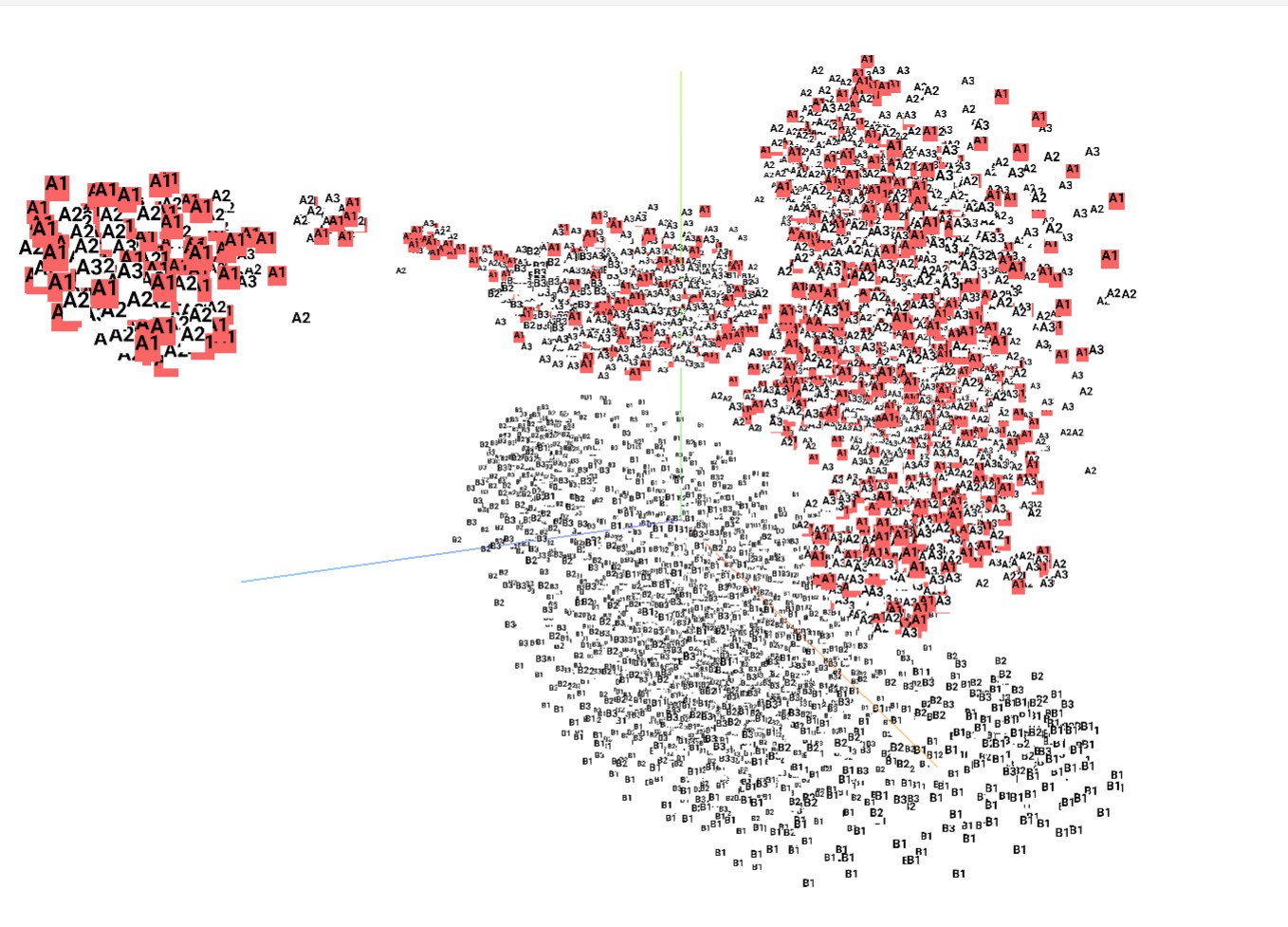}
	   }
	\subfigure[Early stage of GDI-I$^1$]{
		\includegraphics[width=0.5\textwidth]{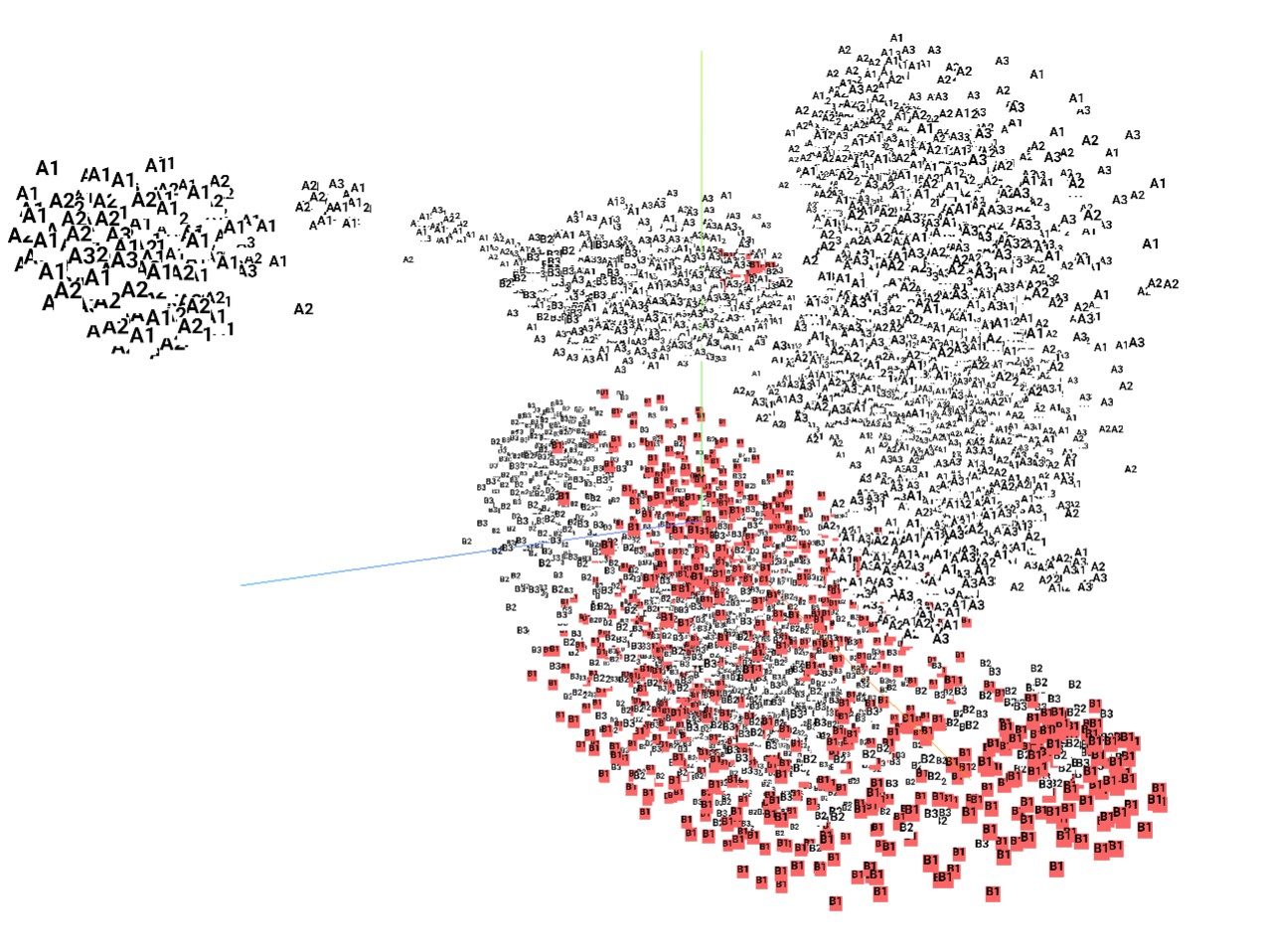}
	}
	
	\subfigure[Middle stage of GDI-H$^3$]{
		\includegraphics[width=0.5\textwidth]{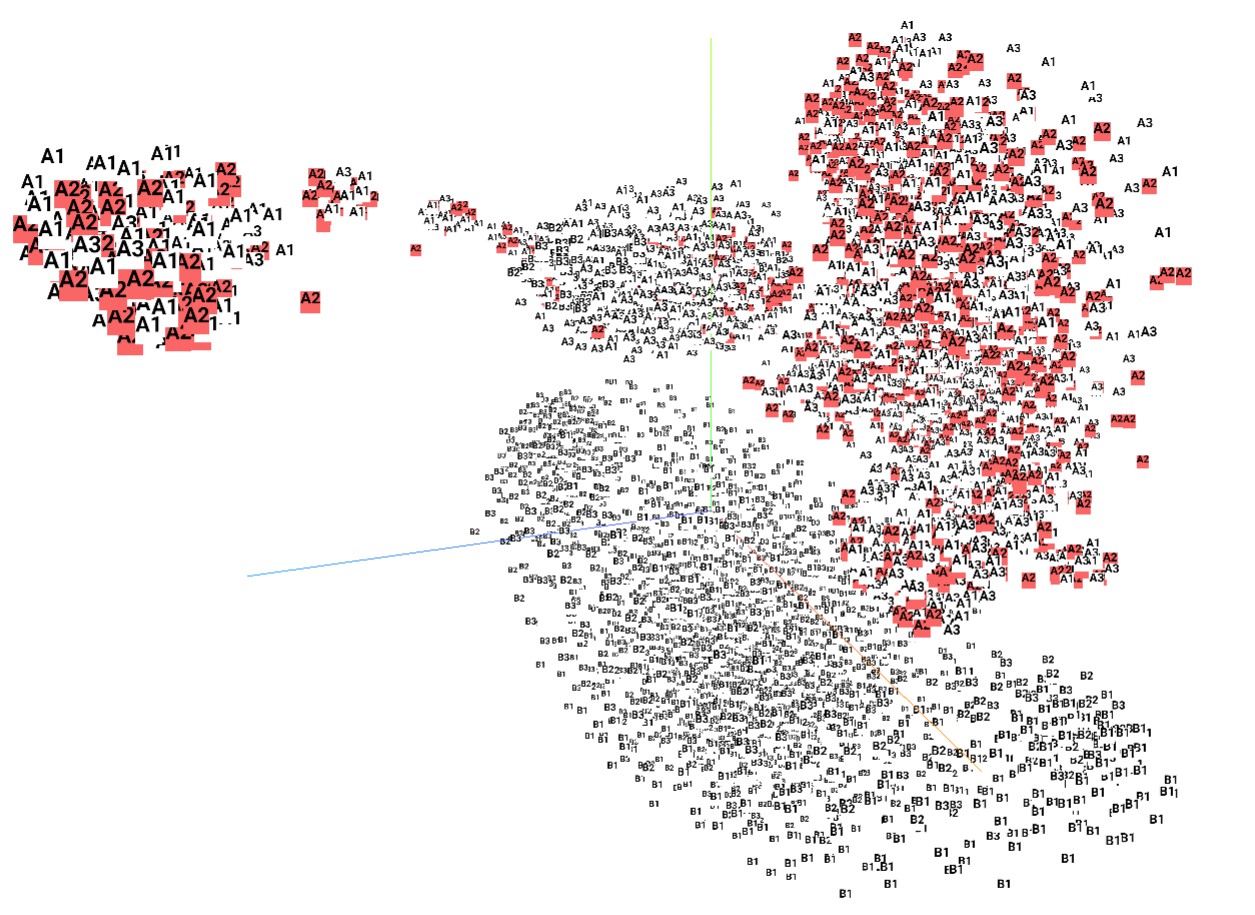}
	}
	\subfigure[Middle stage of GDI-I$^1$]{
		\includegraphics[width=0.5\textwidth]{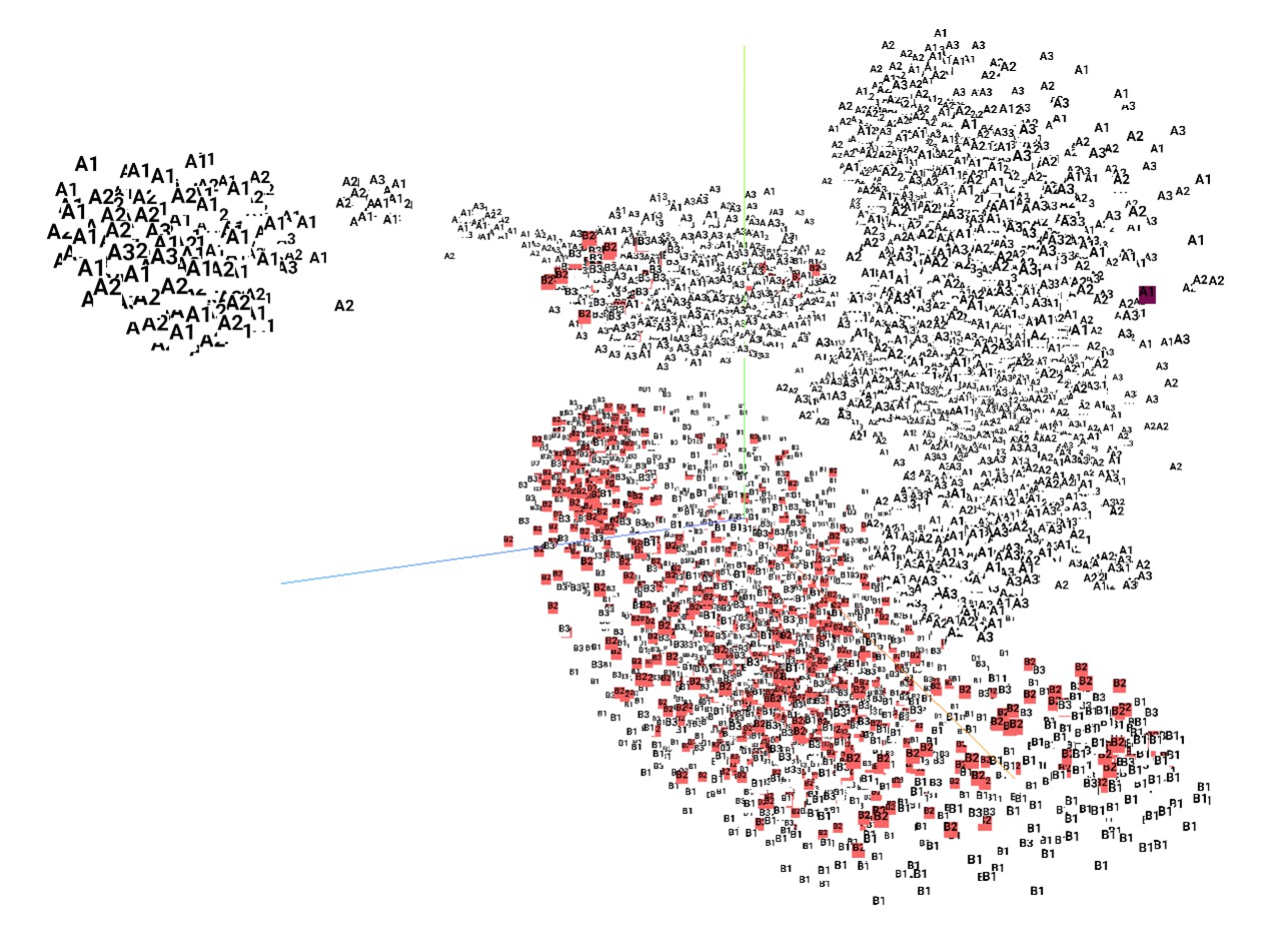}
	}
	
	\subfigure[Later stage of GDI-H$^3$]{
		\includegraphics[width=0.5\textwidth]{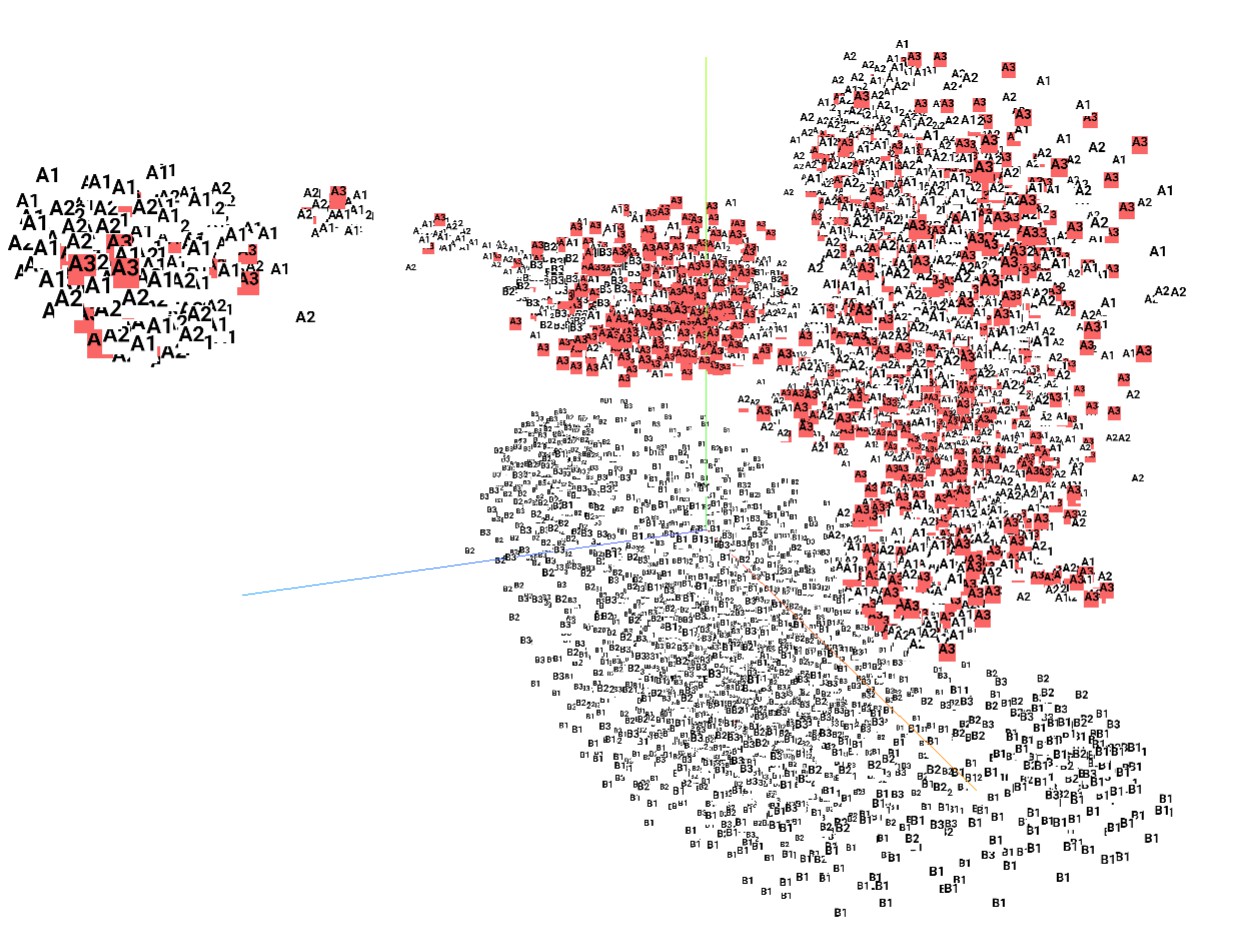}
	}
	\subfigure[Later stage of GDI-I$^1$]{
		\includegraphics[width=0.5\textwidth]{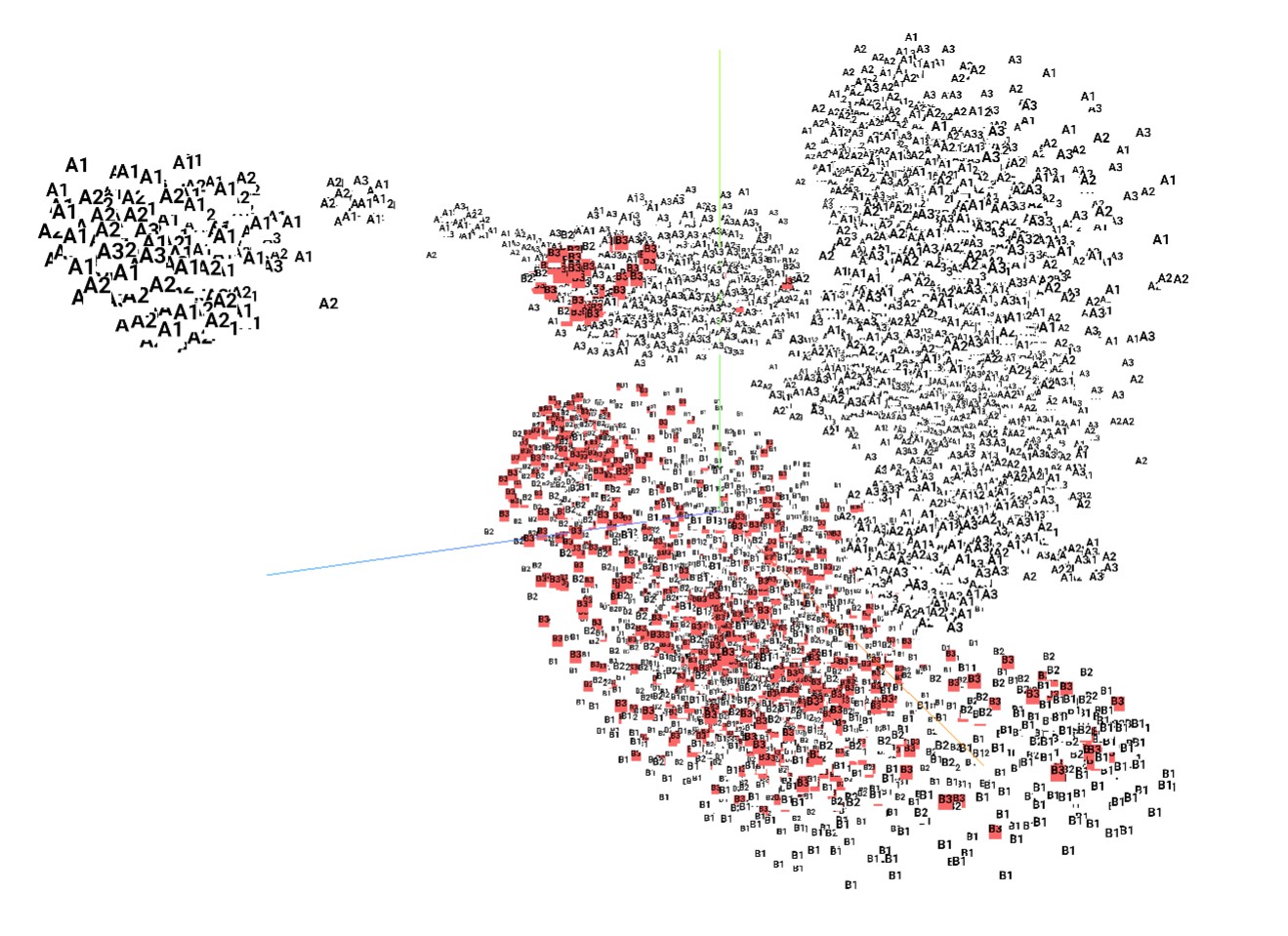}
	}
	\caption{t-SNE of  Krull. 
	t-SNE is drawn from 6k states.
	We sample 1k states from each stage of GDI-H$^3$ and GDI-I$^1$.
	We highlight 1k states of each stage of GDI-H$^3$ and GDI-I$^1$.}
\end{figure}

\begin{figure}[!ht]
    \subfigure[GDI-I$^3$ on Seaquest]{
		\includegraphics[width=0.5\textwidth]{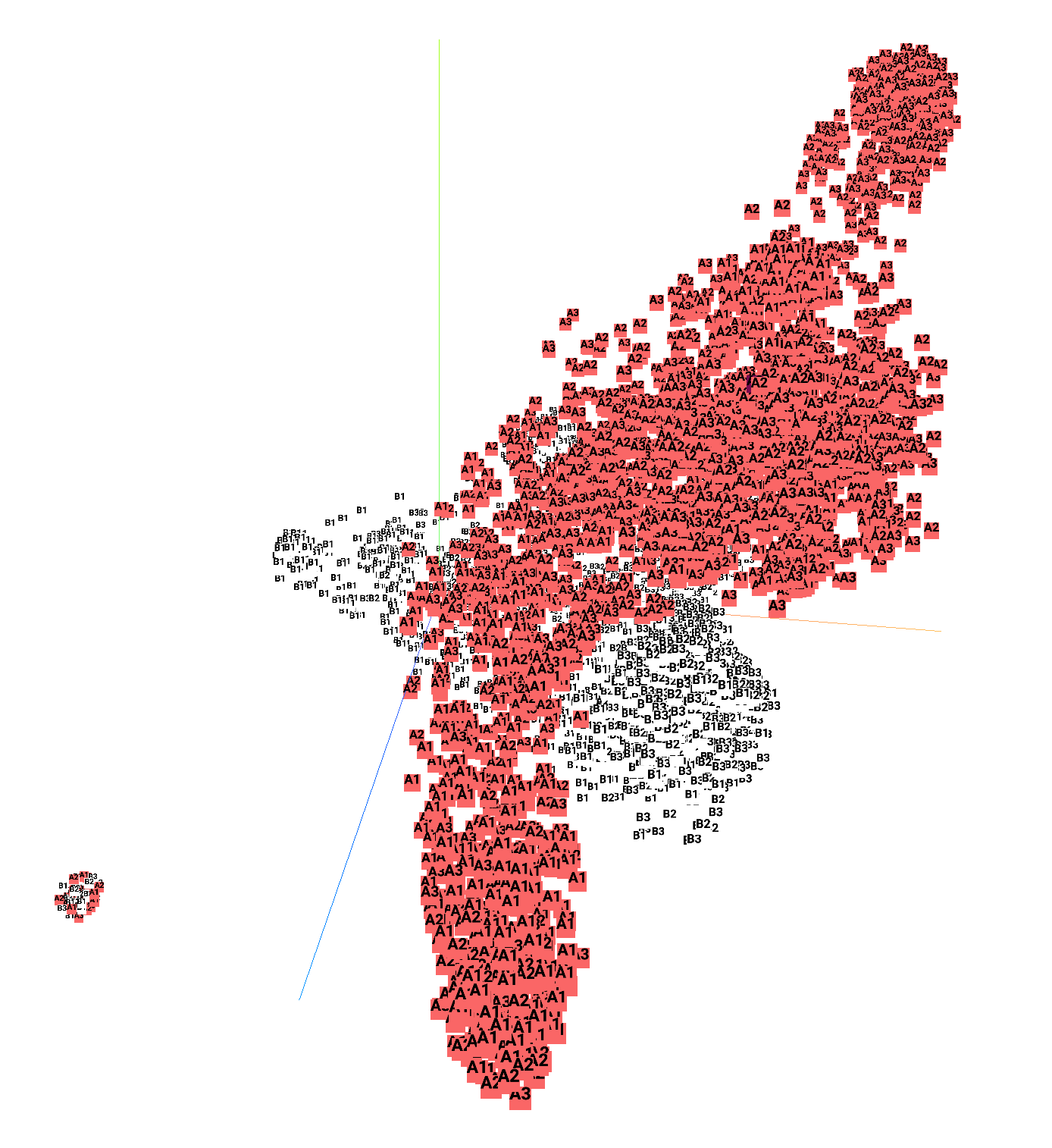}
	}
	\subfigure[GDI-I$^1$  on Seaquest]{
		\includegraphics[width=0.5\textwidth]{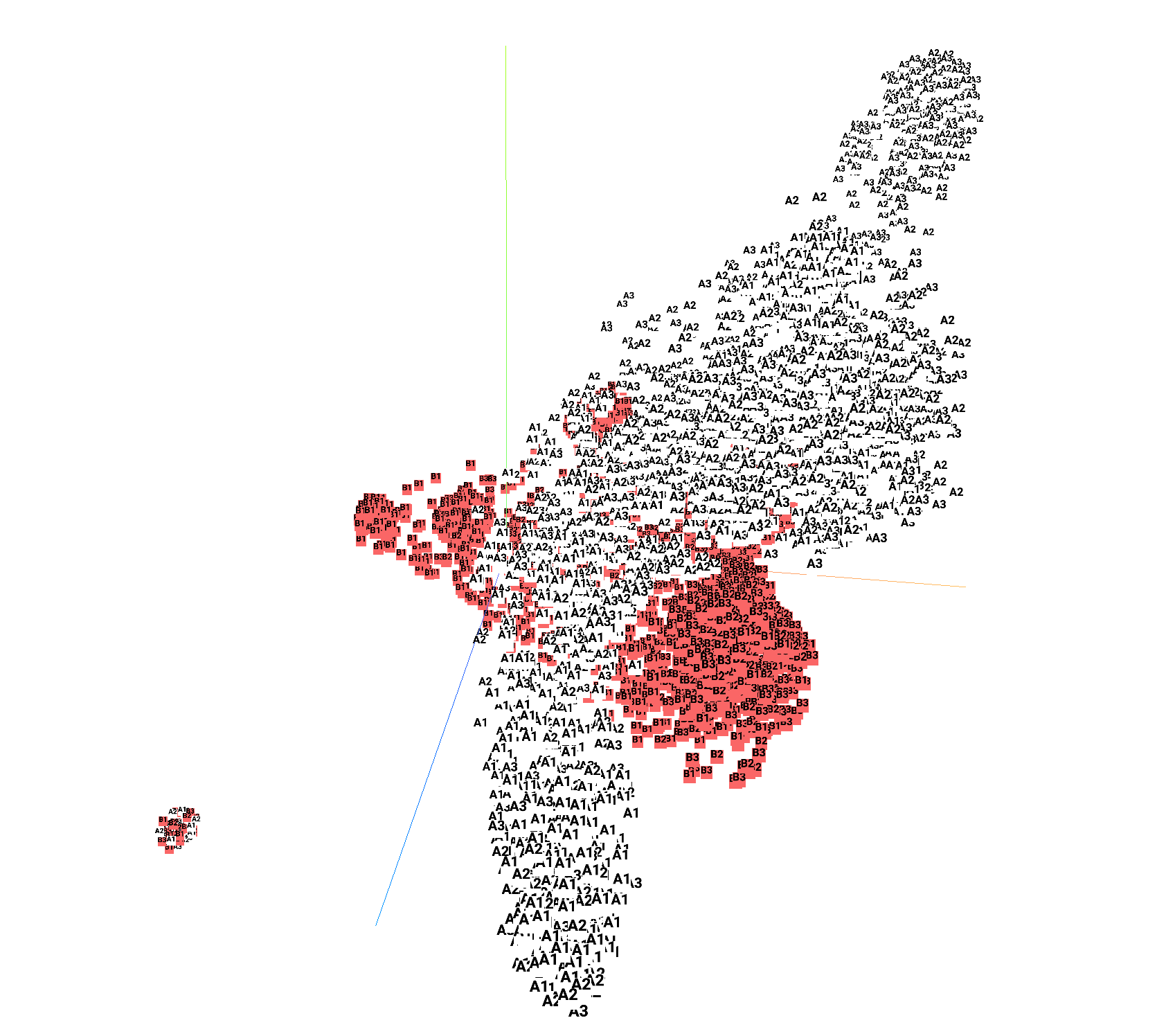}
	}
	\subfigure[GDI-I$^3$ on ChopperCommand]{
		\includegraphics[width=0.5\textwidth]{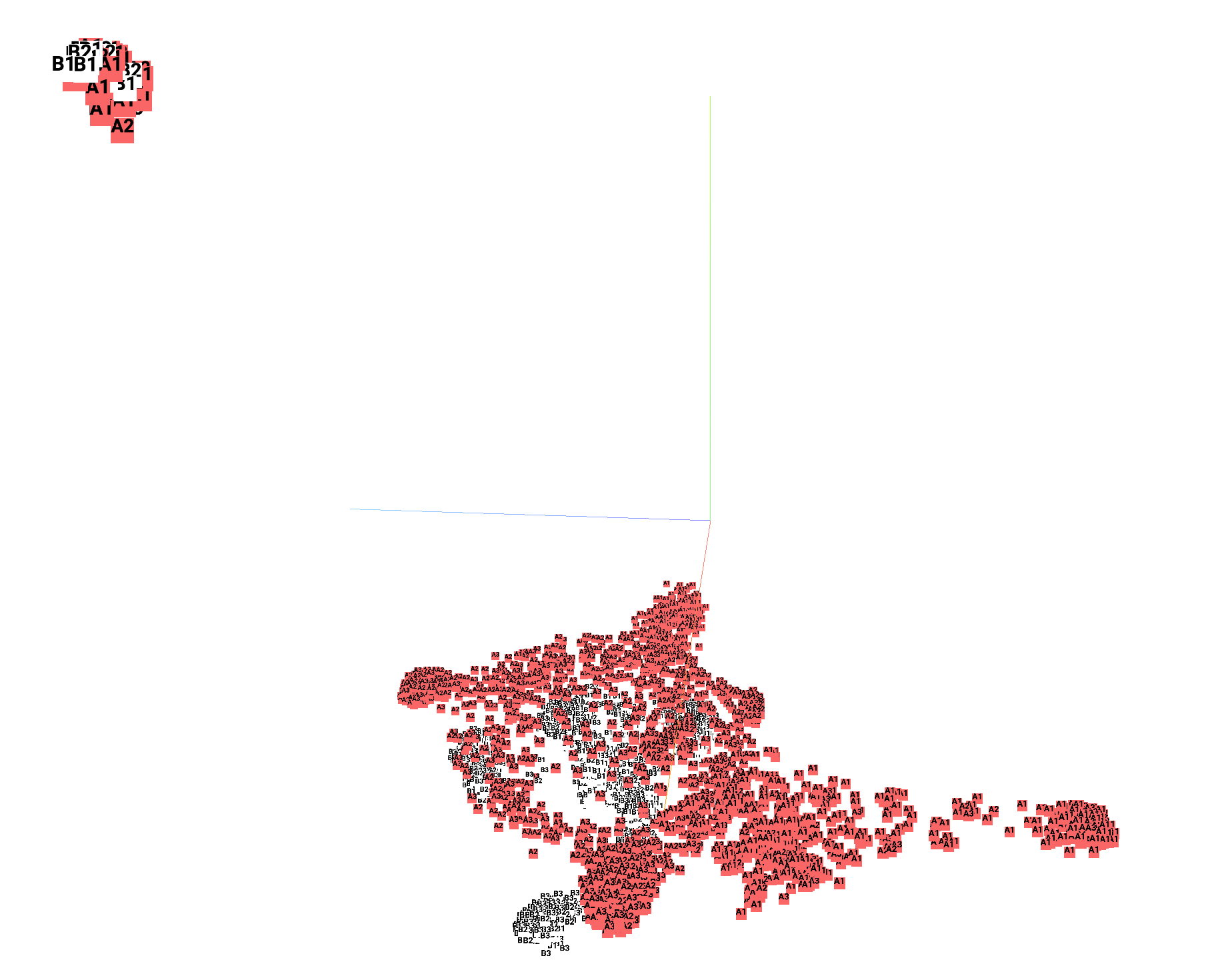}
	   }
	\subfigure[GDI-I$^1$ on ChopperCommand]{
		\includegraphics[width=0.5\textwidth]{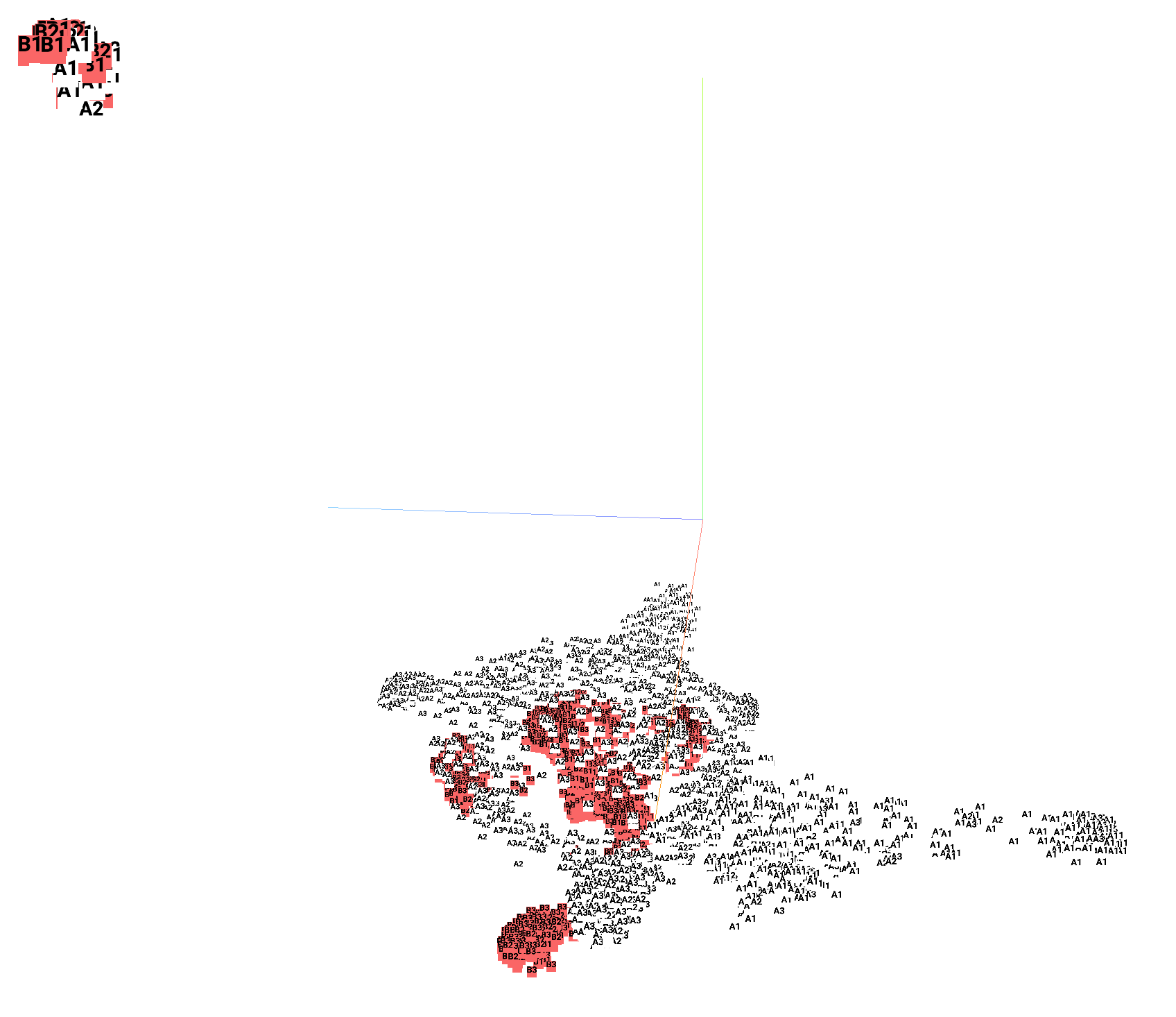}
	}
	
	\subfigure[ GDI-H$^3$ on  Krull]{
		\includegraphics[width=0.5\textwidth]{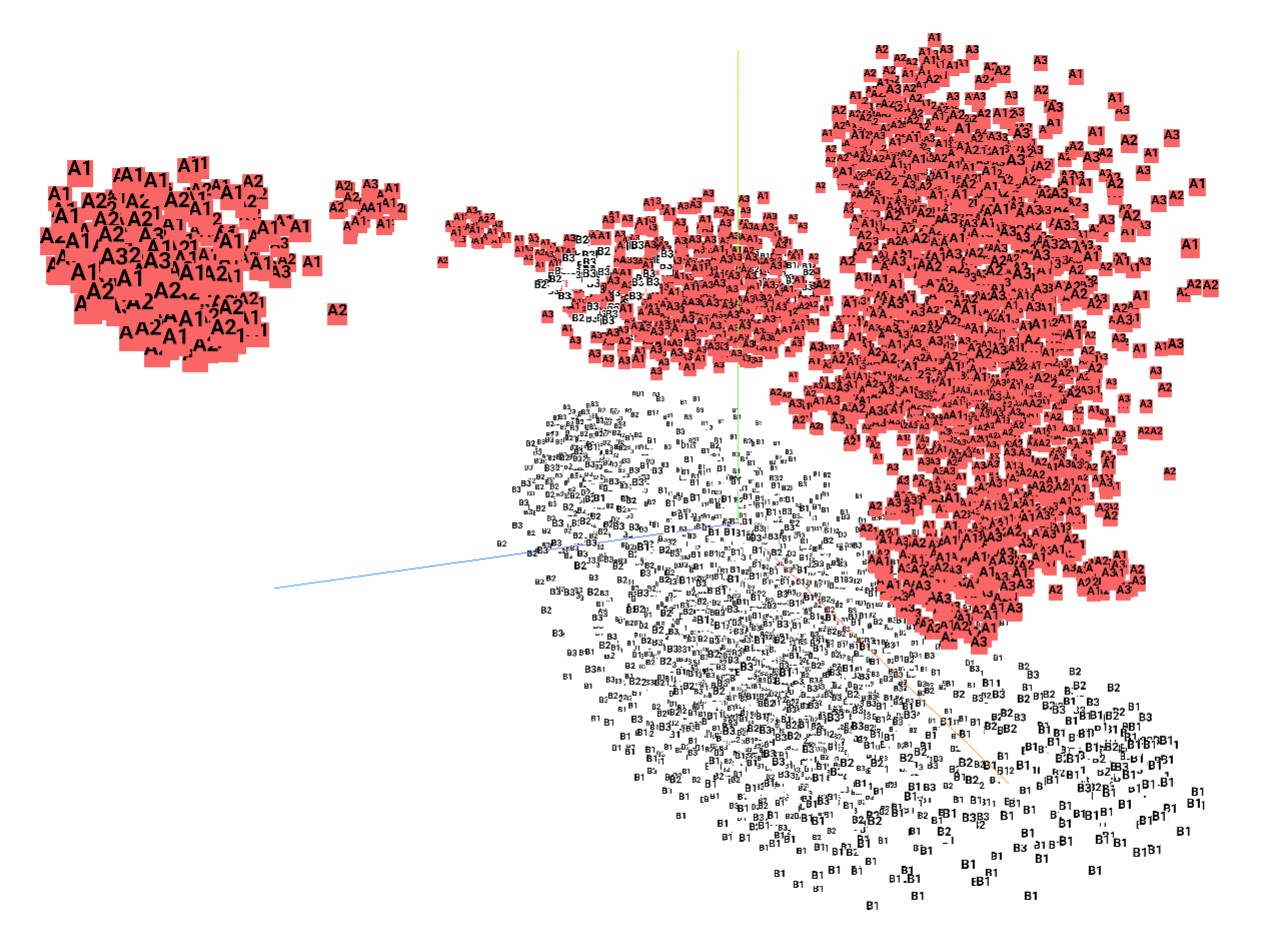}
	}
	\subfigure[GDI-I$^1$ on Krull]{
		\includegraphics[width=0.5\textwidth]{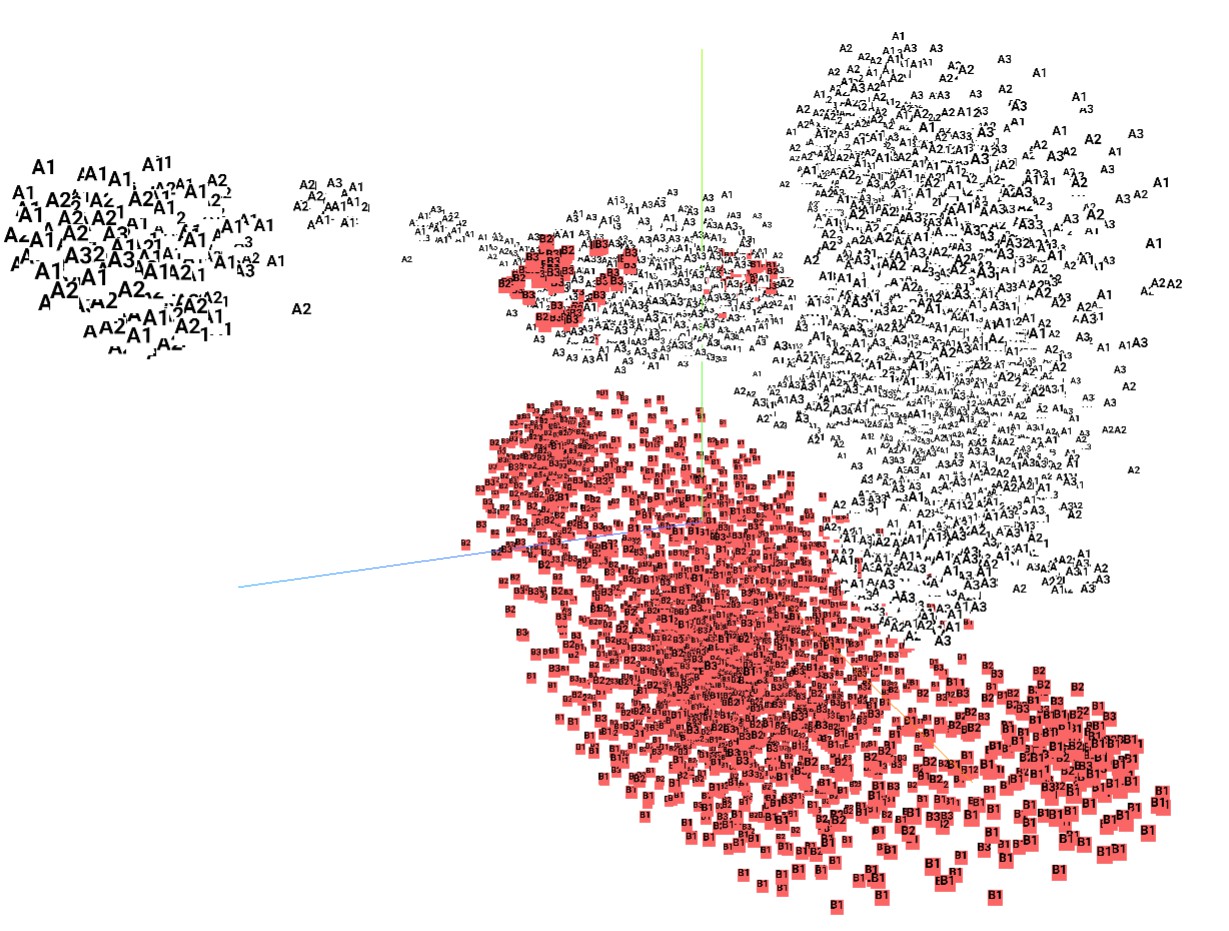}
	}
	
	\caption{Overview of t-SNE in Atari games. 
	Each t-SNE figure is drawn from 6k states.
	We highlight 3k states of GDI-I$^3$, GDI-H$^3$ and GDI-I$^1$, respectively.}
\end{figure}

\clearpage
\subsection{Table of Score}
\label{app: ablation score}

\begin{table}[!hb]
\scriptsize
\begin{center}
\setlength{\tabcolsep}{1.0pt}
\begin{tabular}{c c c c c c c c c}            
\toprule
 Games        & Fix Selection & HNS(\%)      & Boltzmann Selection              & HNS(\%)                     & GDI-I$^3$ & HNS(\%)               & GDI-H$^3$ & HNS(\%) \\
\midrule
Category      &  GDI-I$^0$ w/o $\mathcal{E}$   &            & GDI-I$^1$                     &                             & GDI-I$^3$              &                &  GDI-H$^3$   &\\
Scale         &  200M   &            & 200M                     &                             & 200M              &                &  200M   &\\
\midrule    
 alien              &13720        &195.54                         &10641        &150.92                & 43384             &625.45            &\textbf{48735}	             &\textbf{703.00}              \\
 amidar             &560        &32.34         &653.9        &37.82               & \textbf{1442}              &\textbf{83.81}             &1065              &61.81           \\
 assault            &16228        &3080.37                       &36251        &6933.91             & 63876      &12250.50   &\textbf{97155	}    &\textbf{18655.23}       \\
 asterix            &213580        &2572.80                        &851210       &10261.30                  & 759910            &9160.41           &\textbf{999999}    &\textbf{12055.38}\\
 asteroids          &18621        &38.36                  &759170       &1625.15                      & 751970     &1609.72    &\textbf{760005}            &\textbf{1626.94}       \\
 atlantis           &3211600        &19772.10                      &3670700      &22609.89                            & 3803000    &23427.66   &\textbf{3837300}           &\textbf{23639.67}       \\
 bank heist         &895.3        &119.24                          &1381         &184.98              & \textbf{1401}              &\textbf{187.68}            &1380              &184.84\\
 battle zone        &70137        &189.17                         &130410       &352.28                 & 478830            &1295.20           &\textbf{824360}            &\textbf{2230.29}\\
 beam rider         &34920        &208.64                         &104030       &625.90                 & 162100            &976.51            &\textbf{422390}   &\textbf{2548.07}\\
 berzerk            &1648        &60.81                  &1222         &43.81                  & 7607              &298.53            &\textbf{14649}             &\textbf{579.46}\\
 bowling            &162.4        &101.24                          &176.4        &111.41                  & 202               &129.94           &\textbf{205.2}             &\textbf{132.34}\\
 boxing             &98.3        &818.33                         &99.9         &831.67                       & \best{100}        &\best{832.50}    &\textbf{100}      &\textbf{832.50}        \\
 breakout           &624.3        &2161.81                       &696          &2410.76                         & \best{864}        &\best{2994.10}   &\textbf{864}      &\textbf{2994.10}        \\
 centipede          &102600        &1012.57               &38938        &371.21             & 155830            &1548.84                    &\textbf{195630}            &\textbf{1949.80}\\
 chopper command    &66690        &1001.69                       &\textbf{999999}        &\textbf{15192.62}                         & \best{999999}     &\best{15192.62}          &\textbf{999999}   &\textbf{15192.62}\\
 crazy climber      &161250        &600.70                        &157250       &584.73          & 201000            &759.39                   &\textbf{241170}	            &\textbf{919.76}\\
 defender           &421600        &2647.75                       &837750       &5279.21                 & 893110     &5629.27                        &\textbf{970540}   &\textbf{6118.89}\\
 demon attack       &291590        &16022.76                      &549450       &30199.46                                   & 675530     &37131.12         &\textbf{787985}                     &\textbf{43313.70}\\
 double dunk        &20.25        &1765.91                         &23           &1890.91                         & \best{24}         &\best{1936.36}          &\textbf{24}       &\textbf{1936.36}\\
 enduro             &10019        &1164.32                         &14317        &1663.80                          & \best{14330}      &\best{1665.31}          &14300             &1661.82\\
 fishing derby      &53.24        &273.99                          &48.8         &265.60                  & 59                &285.71                  &\textbf{65}               &\textbf{296.22}\\
 freeway            &3.46        &11.69                   &33.7         &113.85                  & \best{34}         &\best{114.86}           &\textbf{34}        &\textbf{114.86}\\
 frostbite          &1583        &35.55                  &8102         &188.24                & 10485             &244.05                   &\textbf{11330}	            &\textbf{263.84}\\
 gopher             &188680        &8743.90                        &454150       &21063.27                        & \best{488830}     &\best{22672.63}          &473560           &21964.01\\
 gravitar           &4311        &130.19                         &\textbf{6150}         &\textbf{188.05}                  & 5905              &180.34                   &5915             &180.66\\
 hero               &24236        &77.88                  &17655        &55.80                        & \textbf{38330}      &\textbf{125.18}            &38225	   &124.83\\
 ice hockey         &1.56        &105.45                         &-8.1         &25.62                    & 44.94         &463.97        &\textbf{47.11}           &\textbf{481.90}    \\
 jamesbond          &12468        &4543.10                        &567020       &207082.18                          & 594500     &217118.70         &\textbf{620780	}          &\textbf{226716.95}\\
 kangaroo           &5399        &179.25                         &14286        &477.17                  & 14500             &484.34                   &\textbf{14636}           &\textbf{488.90}\\
 krull              &12104.7        &984.23                        &11104        &890.49                          & 97575      &8990.82           &\textbf{594540}          &\textbf{55544.92}\\
 kung fu master     &124630.1        &553.31                        &1270800      &5652.43                          & 140440     &623.64            &\textbf{1666665	}          &\textbf{7413.57}\\
 montezuma revenge  &2488.4        &52.35                 &2528         &53.18                   & \textbf{3000}              &\textbf{63.11}                   &2500            &52.60\\
 ms pacman          &7579        &109.44                        &4296         &60.03                  & 11536             &169.00                  &\textbf{11573}           &\textbf{169.55}\\
 name this game     &32098        &517.76                          &30037        &481.95                 & 34434             &558.34                 &\textbf{36296}           &\textbf{590.68}\\
 phoenix            &498590        &7681.23                       &597580       &9208.60                   & 894460     &13789.30 &\textbf{959580	}          &\textbf{14794.07}   \\
 pitfall            &-17.8        &3.16                  &-21.8        &3.10                    & \textbf{0}                 &\textbf{3.43}               &-4.3            &3.36\\
 pong               &20.39        &116.40                        &21           &118.13              & \best{21}         &\best{118.13}      &\textbf{21}              &\textbf{118.13}      \\
 private eye        &134.1        &0.16                   &15095        &21.67               & \textbf{15100}             &\textbf{21.68}               &\textbf{15100}           &\textbf{21.68}\\
 qbert              &21043        &157.09                         &19091        &142.40                 & 27800             &207.93              &\textbf{28657}           &\textbf{214.38}\\
 riverraid          &11182        &62.38          &17081        &99.77                   & 28075             &169.44                       &\textbf{28349}           &\textbf{171.17}\\
 road runner        &251360        &3208.64                        &57102        &728.80                      &878600            &11215.78           &\textbf{999999}	          &\textbf{12765.53}\\
 robotank           &10.44        &84.95                      &69.7         &695.88                 & 108               &1092.78            &\textbf{113.4}           &\textbf{1146.39}\\
 seaquest           &2728        &6.33                  &11862         &28.09                   & \best{943910}    &\best{2247.98}               &\textbf{1000000}          &\textbf{2381.57}\\
 skiing             &-12730        &34.23                      &-9327        &60.90                   & -6774             &80.90               &\textbf{-6025}	         &\textbf{86.77}\\
 solaris            &2319        &9.76                       &3653         &21.79                   & \textbf{11074}             &\textbf{88.70}               &9105            &70.95\\
 space invaders     &3031        &189.58                         &105810       &6948.25                         & 140460     &9226.80       &\textbf{154380}          &\textbf{10142.17}\\
 star gunner        &337150        &3510.18                        &358650       &3734.47                & 465750     &4851.72                     &\textbf{677590}	          &\textbf{7061.61}\\
 surround           &-9.9998        &0.00          &-9.8         &1.21      & -8        &13.33                        &\textbf{2.606}           &\textbf{76.40}\\
 tennis             &-21.05        &17.74                  &23.7         &306.45                    & \best{24}         &\best{308.39}       &\textbf{24}    &\textbf{308.39}     \\
 time pilot         &84341        &4862.62             &150930       &8871.35        & 216770     &12834.99                     &\textbf{450810}	          &\textbf{26924.45}\\
 tutankham          &381        &236.62                         &380.3        &236.17                  & \textbf{424}               &\textbf{264.08}              &418.2           &260.44\\
 up n down          &416020        &3723.06                        &907170       &8124.13                         & \best{986440}     &\best{8834.45}       &966590          &8656.58    \\
 venture            &0        &0.00                    &1969         &165.81                  & \textbf{2035}              &\textbf{171.37}               &2000	            &168.42\\
 video pinball      &297920        &1686.22                        &673840       &3813.92                & 925830            &5240.18              &\textbf{978190}          &\textbf{5536.54}\\
 wizard of wor      &26008        &606.83                         &21325        &495.15                 & \textbf{64293}             &\textbf{1519.90}              &63735           &1506.59\\
 yars revenge       &76903.5        &143.37                         &84684        &158.48                 & \textbf{972000}            &\textbf{1881.96}              &968090          &1874.36\\
 zaxxon             &46070.8        &503.66                         &62133        &679.38                       & 109140     &1193.63       &\textbf{216020	}          &\textbf{2362.89}    \\
\hline    
MEAN HNS(\%)        & & 1783.24 &                       & 6712.31                       &            &  \GDIImeanhns &      & \textbf{\GDIHmeanhns} \\
\hline
MEDIAN HNS(\%)      & & 195.54           &                       & 477.17              &            & \GDIImedianhns   &      & \textbf{\GDIHmedianhns} \\
\bottomrule
\end{tabular}
\caption{Score table of the ablation study  on HNS.}
\end{center}
\end{table}

\clearpage

\section{Videos}

In the future, we will put all the videos of the performance of our algorithm on sites, which will contain the following:

\begin{itemize}
    \item Our performance on all Atari games: We provide an example video for each game in the Atari games sweep, wherein our algorithm surpass all the existing SOTA algorithms and achieves SOTA. 
    \item Adaptive Entropy Control: We show example videos for  algorithms in the game James Bond, wherein the agents automatically learns how to trade off the  exploitation policy and exploration.
    \item Human World Records Breakthroughs: We also provide an example video for \GDIHHWRB \ Atari games, wherein our algorithms achieved \GDIHHWRB \ human world records breakthroughs.
\end{itemize}

\end{appendices}
\end{document}